%% file: minimax.tex
\documentclass[twoside,11pt]{article}

\usepackage{journal}
\usepackage{natbib}
\usepackage{bbm,bm}
\usepackage{times}
\usepackage{epsfig,ifthen}
\usepackage{latexsym}
\usepackage{amsfonts}
\usepackage{eepic}
\usepackage{amsopn}
\usepackage{amsmath,amssymb,rotating,graphicx,rotating,float}
\usepackage{accents}
\usepackage[dvips,backref,pageanchor=true,plainpages=false,
pdfpagelabels,bookmarks,bookmarksnumbered,breaklinks,
pdfborder={0 0 0},
]{hyperref}
\usepackage[mathscr]{eucal}

\usepackage{cite}
\usepackage{url}

\usepackage{ulem}
\usepackage{color}

\input{macros}

\renewenvironment{proof}[1][]{\par\noindent{\bf Proof #1\ }}{\hfill\BlackBox\\[2mm]}

\jmlrheading{}{}{}{}{}{Steve Hanneke and Liu Yang}

\ShortHeadings{Minimax Analysis of Active Learning}{Hanneke and Yang}
\firstpageno{1}

\begin{document}

\title{Minimax Analysis of Active Learning}

\author{%
\name Steve Hanneke \email steve.hanneke@gmail.com \\
\addr Princeton, NJ 08542\\
\name Liu Yang \email liuy@cs.cmu.edu \\
\addr IBM T. J. Watson Research Center, Yorktown Heights, NY 10598%
}

\editor{}

\maketitle

\begin{abstract}
This work establishes distribution-free upper and lower bounds on the 
minimax label complexity of active learning with general hypothesis classes, 
under various noise models.  
The results reveal a number of surprising facts.
In particular, under the noise model of Tsybakov (2004), 
the minimax label complexity of active learning with a VC class is always
asymptotically smaller than that of passive learning, and is 
typically significantly smaller than the best
previously-published upper bounds in the active learning literature.
In high-noise regimes, it turns out that all active learning
problems of a given VC dimension have roughly the same minimax label complexity, 
which contrasts with well-known results for bounded noise.
In low-noise regimes, we find that the label complexity 
is well-characterized by a simple combinatorial complexity measure we call
the \emph{star number}.
Interestingly, we find that almost all of the complexity measures previously explored 
in the active learning literature have worst-case values exactly equal to the star number.
We also propose new active learning strategies that nearly achieve these minimax label complexities.
\end{abstract}

\begin{keywords}
Active Learning, Selective Sampling, Sequential Design, Adaptive Sampling, Statistical Learning Theory, Margin Condition, Tsybakov Noise, Sample Complexity, Minimax Analysis
\end{keywords}

\section{Introduction}
\label{sec:intro}

In many machine learning applications, in the process of training a high-accuracy classifier,
the primary bottleneck in time and effort is often the annotation of the large quantities of data required
for supervised learning.  
Active learning is a protocol designed to reduce this cost by allowing the 
learning algorithm to sequentially identify highly-informative data points to be annotated.  
In the specific model we study below, called \emph{pool-based} active learning, 
the learning algorithm is initially given access to a large pool of unlabeled data points,
which are considered inexpensive and abundant.
It is then able to select any unlabeled data point from the pool and request to observe its label.
Given the label of this data point, the algorithm can then select another unlabeled data point
to be labeled, and so on.  This interactive process continues until at most some prespecified
number of rounds is reached, at which time the algorithm must halt and produce a classifier.
This contrasts with \emph{passive learning}, where the learning algorithm would be given 
access to a number of \emph{random} labeled data points.  The hope is that, by sequentially
selecting the data points to be labeled, the active learning algorithm can direct the annotation
effort toward only the highly-informative data points, given the information already gathered
by the previously-labeled data points, and thereby reduce the total number of data points 
required to produce a classifier capable of predicting the labels of new instances with at most 
a desired error rate.
This model of active learning has been successfully applied to a variety of learning problems,
often with significant reductions in the number of label observations required to obtain a 
given error rate for the resulting classifier \citep*[see][for a survey of several such applications]{settles:12}.

This article studies the theoretical capabilities of active learning, regarding the number 
of label requests sufficient to learn a classifier to a desired error rate, known as the \emph{label complexity}.  
There is now a substantial literature on this subject \citep*[see][for a survey of known results]{hanneke:survey},
but on the important question of \emph{optimal} performance in the general setting, the gaps present
in the literature are quite substantial in some cases.  In this work, we address this question by carefully
studying the \emph{minimax} performance.  Specifically, we are interested in the \emph{minimax label complexity},
defined as the smallest (over the choice of active learning algorithm) worst-case
number of label requests sufficient for the active learning algorithm to produce a classifier of a specified error rate,
in the context of various noise models (e.g., Tsybakov noise, bounded noise, agnostic noise, etc.).
We derive upper and lower bounds on the minimax label complexity for several noise models,
which reveal a variety of interesting and (in some cases) surprising observations.
Furthermore, in establishing the upper bounds, we propose a novel active learning strategy, 
which often achieves significantly smaller label complexities than the active learning methods 
studied in the prior literature.

\subsection{The Prior Literature on the Theory of Active Learning}

Before getting into the technical details, we first review some background information 
about the prior literature on the theory of active learning.  This will also allow us to 
introduce the key contributions of the present work.

The literature on the theory of active learning began with studies of the \emph{realizable case}, 
a setting in which the labels are assumed to be consistent with some classifier in a known hypothesis class,
and have no noise \citep*{cohn:94,freund:97,dasgupta:04,dasgupta:05}.  In this simple setting, 
\citet*{dasgupta:05} supplied the first general analysis of the label complexity of active learning,
applicable to arbitrary hypothesis classes.  However, \citet*{dasgupta:05} found that there are a range
of minimax label complexities, depending on the structure of the hypothesis class, so that even 
among hypothesis classes of roughly the same minimax sample complexities for passive learning, 
there can be widely varying minimax label complexities for active learning.  In particular, he found
that some hypothesis classes (e.g., interval classifiers) have minimax label complexity essentially no better than that of passive learning,
while others have a minimax label complexity exponentially smaller than that of passive learning (e.g., threshold classifiers).
Furthermore, most nontrivial hypothesis classes of interest in learning theory seem to fall into the former category, 
with minimax label complexities essentially no better than passive learning.  Fortunately, \citet*{dasgupta:05}
also found that in some of these hard cases, it is still possible to show improvements over passive learning
under restrictions on the data distribution.

Stemming from these observations, much of the literature on active learning in the
realizable case has focused on describing various special conditions under which the label complexity of 
active learning is significantly better than that of passive learning: for instance, by placing restrictions 
on the marginal distribution of the unlabeled data \citep*[e.g.,][]{dasgupta:05b,balcan:07,el-yaniv:12,balcan:13,hanneke:survey},
or abandoning the minimax approach by expressing the label complexity with an explicit dependence on the optimal classifier \citep*[e.g.,][]{dasgupta:05,hanneke:10a,hanneke:thesis,hanneke:12a}.
In the general case, such results have been abstracted into various distribution-dependent (or sometimes data-dependent) 
\emph{complexity measures}, such as the \emph{splitting index} \citep*{dasgupta:05}, 
the \emph{disagreement coefficient} \citep*{hanneke:07b,hanneke:thesis},
the \emph{extended teaching dimension growth function} \citep*{hanneke:07a}, 
and the related \emph{version space compression set size} \citep*{el-yaniv:10,el-yaniv:12,hanneke:14a}.
For each of these, there are general upper bounds (and in some cases, minimax lower bounds) on the label complexities
achievable by active learning methods in the realizable case, expressed in terms of the complexity measure.
By expressing bounds on the label complexity in terms of these quantities, the analysis of the label complexities
achievable by active learning methods in the realizable case has been effectively reduced to the problem of 
bounding one of these complexity measures.
In particular, these complexity measures are capable of exhibiting a range of behaviors, corresponding to the 
range of label complexities achievable by active learning.  For certain values of the complexity measures, 
the resulting bounds reveal significant improvements over the minimax sample complexity of passive learning,
while for other values, the resulting bounds are essentially no better than the 
minimax sample complexity of passive learning.

Moving beyond these initial studies of the realizable case, the more-recent literature has developed active
learning algorithms that are provably robust to label noise.  This advance was initiated by the seminal
work of \citet*{balcan:06,balcan:09} on the $A^{2}$ (Agnostic Active) algorithm, and continued by a 
number of subsequent works \citep*[e.g.,][]{dasgupta:07,balcan:07,castro:06,castro:08,hanneke:07a,hanneke:09a,hanneke:thesis,hanneke:11a,hanneke:12a,minsker:12,koltchinskii:10,beygelzimer:09,beygelzimer:10,hsu:thesis,ailon:12,hanneke:12b}.
When moving into the analysis of label complexity in noisy settings, the literature continues to follow the same intuition
from the realizable case: that is, that there should be some active learning problems that are inherently hard, sometimes no better than passive learning, 
while others are significantly easier, with significant savings compared to passive learning.   As such, the general label complexity bounds
proven in noisy settings have tended to follow similar patterns to those found in the realizable case.  In some scenarios, the 
bounds reflect interesting savings compared to passive learning, while in other scenarios the bounds do not reflect any 
improvements at all.  However, unlike the realizable case, these upper bounds on the label complexities of the various proposed methods
for noisy settings lacked complementary minimax lower bounds showing that they were accurately describing the fundamental 
capabilities of active learning in these settings.
For instance, in the setting of Tsybakov noise, there are essentially only two types of general lower bounds on the minimax label complexity in the prior literature: 
(1) lower bounds that hold for all nontrivial hypothesis classes of a given VC dimension,  which therefore reflect a kind of best-case scenario \citep*{hanneke:11a,hanneke:survey}, 
and (2) lower bounds inherited from the realizable case (which is a special case of Tsybakov noise).
In particular, both of these lower bounds are always smaller than the minimax sample complexity of passive learning under Tsybakov noise.
Thus, although the upper bounds on the label complexity of active learning in the literature are sometimes no better than the minimax sample complexity of passive learning,
the existing lower bounds are unable to confirm that active learning truly cannot outperform passive learning in these scenarios.
This gap in our understanding of active learning with noise has persisted for a number of years now, without really receiving a good 
explanation for why the gap exists and how it might be closed.

In the present work, we show that there is a very good reason for why better lower bounds have not been 
discovered in general for the noisy case.  For certain ranges of the noise parameters (corresponding to the high-noise regime), 
these simple lower bounds are actually \emph{tight} (up to certain constant and logarithmic factors):
that is, the upper bounds can actually be reduced to nearly match these basic lower bounds.
Proving this surprising fact requires the introduction of a new type of active learning strategy, which selects its
queries based on both the structure of the hypothesis class and the estimated variances of the labels.
In particular, in these high-noise regimes, we find that \emph{all} hypothesis classes of the same VC dimension
have essentially the same minimax label complexities (up to logarithmic factors), in stark contrast to the 
well-known differentiation of hypothesis classes observed in the realizable case by \citet*{dasgupta:05}.

For the remaining range of the noise parameters (the low-noise regime), we argue that the label complexity 
takes a value sometimes larger than this basic lower bound, yet still typically smaller than the known upper bounds.
In this case, we further argue that the minimax label complexity is well-characterized by a simple combinatorial
complexity measure, which we call the \emph{star number}.  
In particular, these results reveal that for nonextremal parameter values, the minimax label complexity of active
learning under Tsybakov noise with \emph{any} VC class is \emph{always} smaller than that of passive learning,
a fact not implied by any results in the prior literature.

We further find that the star number can be 
used to characterize the minimax label complexities for a variety of other noise models.  Interestingly, 
we also show that almost all of the distribution-dependent or data-dependent complexity measures from the prior literature on the label complexity 
of active learning are exactly \emph{equal} to the star number when maximized over the choice of distribution or data set (including all of those mentioned above).
Thus, the star number represents a unifying core concept within these disparate styles of analysis.

\subsection{Our Contributions} 

Below, we summarize a few of the main contributions and interesting implications of this work.

\begin{itemize}
\item[$\bullet$] We develop a general noise-robust active learning strategy, which unlike previously-proposed general methods, 
selects its queries based on both the structure of the hypothesis class \emph{and} the estimated variances of the labels.
\item[$\bullet$] We obtain the first near-matching general distribution-free upper and lower bounds on the minimax label complexity of active learning, under a variety of noise models.
\item[$\bullet$] In many cases, the upper bounds significantly improve over the best upper bounds implied by the prior literature.
\item[$\bullet$] The upper bounds for Tsybakov noise \emph{always} reflect improvements over the minimax sample complexity of passive learning (for non-extremal noise parameter values), a feat not previously known to be possible.
\item[$\bullet$] In high-noise regimes of Tsybakov noise, our results imply that
all hypothesis classes of a given VC dimension have roughly the same minimax label complexity (up to logarithmic factors),
in contrast to well-known results for bounded noise.  This fact is not implied by any results in the prior literature.  
\item[$\bullet$] We express our upper and lower bounds on the label complexity in terms of a simple combinatorial complexity measure, which we refer to as the \emph{star number}.
\item[$\bullet$] We show that for any hypothesis class, almost every complexity measure proposed to date in the active learning literature has worst-case value exactly \emph{equal} to the star number,
thus unifying the disparate styles of analysis in the active learning literature.  
We also prove that the doubling dimension is bounded if and only if the star number is finite.
\item[$\bullet$] For most of the noise models studied here, we exhibit examples of hypothesis classes spanning the gaps between 
the upper and lower bounds, thus demonstrating that the gaps cannot generally be reduced (aside from logarithmic factors)
without introducing additional complexity measures.
\item[$\bullet$] We prove a separation result for Tsybakov noise vs the Bernstein class condition, establishing
that the respective minimax label complexities can be significantly different.  This contrasts with passive learning, where
they are known to be equivalent up to a logarithmic factor.
\end{itemize}

The algorithmic techniques underlying the proofs of the most-interesting of our upper bounds involve a combination
of the disagreement-based strategy of \citet*{cohn:94} (and the analysis thereof by \citealp*{hanneke:11a}, and \citealp*{hanneke:14a}),
along with a repeated-querying technique of \citet*{kaariainen:06}, modified to account for variations in label variances
so that the algorithm does not waste too many queries determining the optimal classification of highly-noisy points;
this modification represents the main algorithmic innovation in this work.  In a supporting role, we also rely on 
auxiliary lemmas on the construction of $\eps$-nets and $\eps$-covers based on random samples, and the 
use of these to effectively discretize the instance space.
The mathematical techniques underlying the proofs of the lower bounds are largely taken directly from the 
literature.  Most of the lower bounds are established by a combination of a technique originating with \citet*{kaariainen:06}
and refined by \citet*{beygelzimer:09} and \citet*{hanneke:11a,hanneke:survey}, and a technique of \citet*{raginsky:11} for 
incorporating a complexity measure into the lower bounds.

We note that, while the present work focuses on the distribution-free setting, in which the marginal distribution over the 
instance space is unrestricted, our results reveal that low-noise settings can still benefit from distribution-dependent
analysis, as expected given the aforementioned observations by \citet*{dasgupta:05} for the realizable case.  For instance,
under Tsybakov noise, it is often possible to obtain stronger upper bounds in low-noise regimes under
assumptions restricting the distribution of the unlabeled data \citep*[see e.g.,][]{balcan:07}.  We leave for future work the important problem of 
characterizing the minimax label complexity of active learning in the general case for an arbitrary fixed marginal distribution
over the instance space.

\subsection{Outline}

The rest of this article is organized as follows.
Section~\ref{sec:definitions} introduces the formal setting and basic notation used throughout,
followed in Section~\ref{sec:noise-models} with the introduction of the noise models studied in this work.
Section~\ref{sec:star} defines a combinatorial complexity measure -- the star number -- in terms of which we will express the label complexity bounds below.
Section~\ref{sec:main} provides statements of the main results of this work: upper and lower bounds on the minimax label complexities of active learning under each of the noise models defined in Section~\ref{sec:noise-models}.
That section also includes a discussion of the results, and a brief sketch of the arguments underlying the most-interesting among them.
Section~\ref{sec:passive} compares the results from Section~\ref{sec:main} to the known results on the minimax sample complexity of passive learning, revealing which scenarios yield improvements of active over passive.
Next, in Section~\ref{sec:complexities}, we go through the various results on the label complexity of active learning from the literature, along with their
corresponding complexity measures (most of which are distribution-dependent or data-dependent).  
We argue that all of these complexity measures are exactly equal to the star number when maximized over the
choice of distribution or data set.  This section also relates the star number to the well-known concept of \emph{doubling dimension},
in particular showing that the doubling dimension is bounded if and only if the star number is finite.

We note that the article is written with the intention that it be read in-order; for instance,
while Appendix~\ref{app:main-proofs} contains proofs of the results in Section~\ref{sec:main},
those proofs refer to quantities and results introduced in Sections \ref{sec:passive} and \ref{sec:complexities}
(which follow Section~\ref{sec:main}, but precede Appendix~\ref{app:main-proofs}).  

\section{Definitions}
\label{sec:definitions}

The rest of this paper makes use of the following formal definitions.
There is a space $\X$, called the \emph{instance space}.  
We suppose $\X$ is equipped with a $\sigma$-algebra $\BorelX$, and for simplicity we will assume $\{\{x\} : x \in \X\} \subseteq \BorelX$.
There is also a set $\Y = \{-1,+1\}$, known as the \emph{label space}.
Any measurable function $h : \X \to \Y$ is called a \emph{classifier}.
There is an arbitrary set $\C$ of classifiers, known as the \emph{hypothesis class}.
To focus on nontrivial cases, we suppose $|\C| \geq 3$ throughout.

For any probability measure $P$ over $\X \times \Y$ and any $x \in \X$, define $\eta(x; P) = \P( Y = +1 | X=x )$ for $(X,Y) \sim P$, 
and let $\target_{P}(x) = \sign( 2\eta(x;P) - 1 )$ denote the \emph{Bayes optimal classifier},\footnote{Since 
conditional probabilities are only defined up to probability zero differences, there can be multiple valid functions $\eta(\cdot;P)$ and $\target_{P}$,
with any two such functions being equal with probability one.  As such, we will interpret statements such as ``$\target_{P} \in \C$'' to mean that 
\emph{there exists} a version of $\target_{P}$ contained in $\C$, and similarly for other claims and conditions for $\target_{P}$ and $\eta(\cdot;P)$.}
where $\sign(t) = +1$ if $t \geq 0$, and $\sign(t) = -1$ if $t < 0$.
Define the \emph{error rate} of a classifier $h$ with respect to $P$ as
$\er_{P}(h) = P((x,y) : h(x) \neq y)$.

In the learning problem, there is a \emph{target distribution} $\PXY$ over $\X \times \Y$, 
and a \emph{data sequence} $(X_{1},Y_{1}),(X_{2},Y_{2}),\ldots$, which are independent $\PXY$-distributed random variables.
However, in the active learning protocol, the $Y_{i}$ values are initially ``hidden'' until individually requested by the algorithm (see below).
We refer to the sequence $X_{1},X_{2},\ldots$ as the \emph{unlabeled data sequence}.\footnote{Although, in practice, we would expect to have 
access to only a finite number of unlabeled samples, we expect this number would often be quite large (as unlabeled samples are considered 
inexpensive and abundant in many applications).  For simplicity, and to focus the analysis purely on the number of \emph{labels} required
for learning, we approximate this scenario by supposing an \emph{inexhaustible} source of unlabeled samples.  We leave open the question 
of the number of unlabeled samples sufficient to obtain the minimax label complexity; in particular, we expect the number of such samples
used by the methods obtaining our upper bounds to be quite large indeed.}
We will sometimes denote by $\Px$ the marginal distribution of $\PXY$ over $\X$: that is, $\Px(\cdot) = \PXY(\cdot \times \Y)$.

In the \emph{pool-based active learning} protocol,\footnote{Although technically we study the pool-based active learning protocol,
all of our results apply equally well to the stream-based (selective sampling) model of active learning
(in which the algorithm must decide whether or not to request the label $Y_{i}$ before 
observing any $X_{j}$ with $j > i$ or requesting any $Y_{j}$ with $j > i$).}
we define an \emph{active learning algorithm} $\alg$ as an algorithm taking as input a budget $n \in \nats \cup \{0\}$,
and proceeding as follows.
The algorithm initially has access to the unlabeled data sequence $X_{1},X_{2},\ldots$.
If $n > 0$, the algorithm may then select an index $i_{1} \in \nats$ and request to observe the label $Y_{i_{1}}$.
The algorithm may then observe the value of $Y_{i_{1}}$, and if $n \geq 2$, then based on both the unlabeled sequence and this new observation $Y_{i_{1}}$,
it may select another index $i_{2} \in \nats$ and request to observe $Y_{i_{2}}$.
This continues for a number of rounds at most $n$ (i.e., it may request at most $n$ labels),
after which the algorithm must halt and produce a classifier $\hat{h}_{n}$.
More formally, an active learning algorithm is defined by a random sequence $\{i_{t}\}_{t=1}^{\infty}$ in $\nats$, a random variable $N$ in $\nats$, 
and a random classifier $\hat{h}_{n}$, satisfying the following properties.
Each $i_{t}$ is conditionally independent from $\{(X_{i},Y_{i})\}_{i=1}^{\infty}$ given $\{i_{j}\}_{j=1}^{t-1}$, $\{Y_{i_{j}}\}_{j=1}^{t-1}$, and $\{X_{i}\}_{i=1}^{\infty}$.
The random variable $N$ always has $N \leq n$, and for any $k \in \{0,\ldots,n\}$, 
$\ind[ N = k ]$ is independent from $\{(X_{i},Y_{i})\}_{i=1}^{\infty}$ given $\{i_{j}\}_{j=1}^{k}$, $\{Y_{i_{j}}\}_{j=1}^{k}$, and $\{X_{i}\}_{i=1}^{\infty}$.
Finally, $\hat{h}_{n}$ is independent from $\{(X_{i},Y_{i})\}_{i=1}^{\infty}$ given $N$, $\{i_{j}\}_{j=1}^{N}$, $\{Y_{i_{j}}\}_{j=1}^{N}$, and $\{X_{i}\}_{i=1}^{\infty}$.

We are now ready for the definition of our primary quantity of study: the minimax label complexity.
In the next section, we define several well-known noise models as specifications of the set $\Dset$ referenced in this definition.

\begin{definition}
\label{def:LC}
For a given set $\Dset$ of probability measures on $\X \times \Y$, 
$\forall \eps \geq 0$, $\forall \conf \in [0,1]$, 
the \emph{minimax label complexity} (of active learning) under $\Dset$ with respect to $\C$, 
denoted $\LC_{\Dset}(\eps,\conf)$,
is the smallest $n \in \nats \cup \{0\}$ such that there exists an active learning algorithm $\alg$
with the property that, for every $\PXY \in \Dset$,
the classifier $\hat{h}_{n}$ produced by $\alg(n)$ based on the (independent $\PXY$-distributed)
data sequence $(X_{1},Y_{1}),(X_{2},Y_{2}),\ldots$ satisfies
\begin{equation*}
\P\left( \er_{\PXY}\left( \hat{h}_{n} \right) - \inf_{h \in \C} \er_{\PXY}(h) > \eps \right) \leq \conf.
\end{equation*}
If no such $n$ exists, we define $\LC_{\Dset}(\eps,\conf) = \infty$.
\end{definition}

Following \citet*{vapnik:71,anthony:99}, we say a collection of sets $\T \subseteq 2^{\X}$
\emph{shatters} a sequence $S \in \X^{k}$ (for $k \in \nats$) if $\{ A \cap S : A \in \T \} = 2^{S}$.
The \emph{VC dimension} of $\T$ is then defined as the largest $k \in \nats \cup \{0\}$ such that
there exists $S \in \X^{k}$ shattered by $\T$; if no such largest $k$ exists, the VC dimension is 
defined to be $\infty$.
Overloading this terminology, the VC dimension of a set $\H$ of classifiers is defined as
the VC dimension of the collection of sets $\{ \{x : h(x) = +1\} : h \in \H \}$.
Throughout this article,
we denote by $\vc$ the VC dimension of $\C$.
We are particularly interested in the case $\vc < \infty$, in which case $\C$ is called a \emph{VC class}.

For any set $\H$ of classifiers, 
define $\DIS(\H) = \{x \in \X : \exists h,g \in \H \text{ s.t. } h(x) \neq g(x)\}$, the \emph{region of disagreement} of $\H$.
Also, for any classifier $h$, any $r \geq 0$, and any probability measure $P$ on $\X$,
define $\Ball_{P}(h,r) = \{g \in \C : P(x : g(x) \neq h(x)) \leq r\}$, the \emph{$r$-ball centered at $h$}.

Before proceeding, we introduce a few additional notational conventions that help to simplify the theorem statements and proofs.
For any $\reals$-valued functions $f$ and $g$, we write
$f(x) \lesssim g(x)$ (or equivalently $g(x) \gtrsim f(x)$) to express the fact
that there is a \emph{universal} finite numerical constant $c > 0$ such that $f(x) \leq c g(x)$.
For any $x \in [0,\infty]$, we define $\Log(x) = \max\{\ln(x),1\}$, where $\ln(0) = -\infty$ and $\ln(\infty) = \infty$.
For simplicity, we define $\frac{\infty}{\Log(\infty)} = \infty$,
but in any other context, we always define $0 \cdot \infty = 0$, and also define $\frac{a}{0} = \infty$ for any $a > 0$.
For any function $\phi : \reals \to \reals$, we use the notation
``$\lim_{\gamma \to 0} \phi(\gamma)$'' to indicating taking the limit as $\gamma$ approaches $0$ 
\emph{from above}: i.e., $\gamma \downarrow 0$.
For $a,b \in \reals$, we denote $a \land b = \min\{a,b\}$ and $a \lor b = \max\{a,b\}$.
Finally, we remark that some of the claims below technically require additional qualifications
to guarantee measurability of certain quantities (as is typically the case in empirical process theory); 
see \citet*{blumer:89,van-der-Vaart:96,van-der-Vaart:11} for some discussion of this issue.
For simplicity, we do not mention these issues in the analysis below;
rather, we implicitly qualify all of these results with the 
condition that $\C$ is such that all of the random variables and events 
arising in the proofs are measurable.

\section{Noise Models}
\label{sec:noise-models}

We now introduce the noise models under which we will study the minimax label complexity of active learning.
These are defined as sets of probability measures on $\X \times \Y$, corresponding to specifications of the
set $\Dset$ in Definition~\ref{def:LC}.

\begin{itemize}
\item[$\bullet$] (Realizable Case) Define $\RE$ as the collection of $\PXY$ for which $\target_{\PXY} \in \C$ and $2\eta(\cdot; \PXY)-1 = \target_{\PXY}(\cdot)$ (almost everywhere w.r.t. $\Px$).
\item[$\bullet$] (Bounded Noise) For $\bound \in [0,1/2)$, define $\BN(\bound)$ as the collection of joint distributions $\PXY$ over $\X \times \Y$ such that
$\target_{\PXY} \in \C$ and
\begin{equation*}
\Px\left( x : \left| \eta(x;\PXY) - 1/2 \right| \geq 1/2 - \bound \right) = 1.
\end{equation*}
\item[$\bullet$] (Tsybakov Noise)
For $\tsybca \in [1,\infty)$ and $\tsyba \in (0,1)$, define 
$\TN(\tsybca, \tsyba)$ as the collection of joint distributions $\PXY$ over $\X \times \Y$ such that
$\target_{\PXY} \in \C$ and  
$\forall \gamma > 0$,
\begin{equation*}
\Px\left( x : \left| \eta(x;\PXY) - 1/2 \right| \leq \gamma \right) \leq \tsybca^{\prime} \gamma^{\tsyba/(1-\tsyba)},
\end{equation*}
where $\tsybca^{\prime} = (1-\tsyba)(2\tsyba)^{\tsyba/(1-\tsyba)}\tsybca^{1/(1-\tsyba)}$.
\item[$\bullet$]
(Bernstein Class Condition)
For $\tsybca \in [1,\infty)$ and $\tsyba \in [0,1]$, define
$\DL(\tsybca, \tsyba)$ as the collection of joint distributions $\PXY$ over $\X \times \Y$ such that,
$\exists h_{\PXY} \in \C$ for which $\forall h \in \C$, 
\begin{equation*}
\Px( x : h(x) \neq h_{\PXY}(x) ) \leq \tsybca ( \er_{\PXY}(h) - \er_{\PXY}(h_{\PXY}) )^{\tsyba}.
\end{equation*}
\item[$\bullet$] (Benign Noise) For $\nu \in [0,1/2]$, define $\BE(\nu)$ as the collection of all joint distributions $\PXY$ over $\X \times \Y$ such that $\target_{\PXY} \in \C$ and $\er_{\PXY}(\target_{\PXY}) \leq \nu$.
\item[$\bullet$] (Agnostic Noise) For $\nu \in [0,1]$, define $\AG(\nu)$ as the collection of all joint distributions $\PXY$ over $\X \times \Y$ such that $\inf_{h \in \C} \er_{\PXY}(h) \leq \nu$.
\end{itemize}

It is known that $\RE \subseteq \BN(\bound) \subseteq \DL(1/(1-2\bound), 1)$,
and also 
that $\RE \subseteq \TN(\tsybca, \tsyba) \subseteq \DL(\tsybca, \tsyba)$.
Furthermore, $\TN(\tsybca,\tsyba)$ is equivalent to the conditions in $\DL(\tsybca,\tsyba)$ being 
satisfied for \emph{all} classifiers $h$, rather than merely those in $\C$ \citep*{mammen:99,tsybakov:04,boucheron:05}.
All of $\RE$, $\BN(\bound)$, and $\TN(\tsybca,\tsyba)$ are contained in $\bigcup_{\nu < 1/2} \BE(\nu)$,
and in particular, $\BN(\bound) \subseteq \BE(\bound)$.

The realizable case is the simplest setting studied here, corresponding to the ``optimistic case'' of \citet*{vapnik:98}
or the PAC model of \citet*{valiant:84}.  
The bounded noise model has been studied under various names \citep*[e.g.,][]{massart:06,gine:06,kaariainen:06,koltchinskii:10,raginsky:11};
it is sometimes referred to as \emph{Massart's noise condition}.
The Tsybakov noise condition was introduced by \citet*{mammen:99} in a slightly stronger form (in the related context of discrimination analysis)
and was distilled into the form stated above by \citet*{tsybakov:04}.  There is now a substantial 
literature on the label complexity under this condition, both for passive learning and active learning 
\citep*[e.g.,][]{mammen:99,tsybakov:04,bartlett:06,koltchinskii:06,balcan:07,hanneke:11a,hanneke:12a,hanneke:survey,hanneke:12b}.
However, in much of this literature, the results are in fact established under the weaker assumption given 
by the Bernstein class condition \citep*{bartlett:04}, which is known to be implied by the Tsybakov noise condition \citep*{mammen:99,tsybakov:04}.
For passive learning, it is known that the minimax sample complexities under Tsybakov noise and under the Bernstein class condition
are equivalent up to a logarithmic factor.  Interestingly, our results below imply that this is not the case for active learning.
The benign noise condition \citep*[studied by][]{hanneke:thesis} requires only that the Bayes optimal classifier be contained 
within the hypothesis class, and that the Bayes error rate be at most the value of the parameter $\nu$.
The agnostic noise condition (sometimes called \emph{adversarial noise} in related contexts) 
is the weakest of the noise assumptions studied here, and admits any distribution for which 
the best error rate among classifiers in the hypothesis class is at most the value of the parameter $\nu$.
This model has been widely studied in the literature, for both passive and active learning 
\citep*[e.g.,][]{vapnik:71,vapnik:82,vapnik:98,kearns:94a,kalai:05,balcan:06,hanneke:07b,hanneke:07a,awasthi:14}.

\section{A Combinatorial Complexity Measure}
\label{sec:star}

There is presently a substantial literature on distribution-dependent bounds on the 
label complexities of various active learning algorithms.  These bounds are expressed in terms of a variety
of interesting complexity measures, designed to capture the behavior of each of these particular algorithms.
These measures of complexity include the disagreement coefficient \citep*{hanneke:07b}, the reciprocal of 
the splitting index \citep*{dasgupta:05}, the extended teaching dimension growth function \citep*{hanneke:07a},
and the version space compression set size \citep*{el-yaniv:10,el-yaniv:12}.  These quantities have been studied
and bounded for a variety of learning problems (see \citealp*{hanneke:survey}, for a summary).
They each have many interesting properties, and in general can exhibit a wide variety of behaviors,
as functions of the distribution over $\X$ (and in some cases, the distribution over $\X \times \Y$) and $\eps$,
or in some cases, the data itself.
However, something remarkable happens when we maximize each of these complexity measures
over the choice of distribution (or data set): they all become equal to a simple and easy-to-calculate
combinatorial quantity (see Section~\ref{sec:complexities} for proofs of these equivalences).
Specifically, consider the following definition.\footnote{A similar notion previously appeared in a lower-bound
argument of \citet*{dasgupta:05}, including a kind of distribution-dependent version of the ``star set'' idea.  
Indeed, we explore these connections formally in Section~\ref{sec:complexities},
where we additionally prove this definition is exactly equivalent to a quantity studied by \citet*{hanneke:07a}
(namely, the distribution-free version of the extended teaching dimension growth function), and has connections
to several other complexity measures in the literature.}

\begin{definition}
\label{def:star}
Define the \emph{star number} $\s$ as the largest integer $s$ 
such that there exist $x_{1},\ldots,x_{s} \in \X$ and $h_{0},h_{1},\ldots,h_{s} \in \C$
with the property that $\forall i \in \{1,\ldots,s\}$, $\DIS(\{h_{0},h_{i}\}) \cap \{x_{1},\ldots,x_{s}\} = \{x_{i}\}$;
if no such largest integer exists, define $\s = \infty$.
\end{definition}

For any set $\H$ of functions $\X \to \Y$, 
any $t \in \nats$, $x_{1},\ldots,x_{t} \in \X$, and $h_{0},h_{1},\ldots,h_{t} \in \H$,
we will say $\{x_{1},\ldots,x_{t}\}$ is a \emph{star set} for $\H$, \emph{witnessed by} $\{h_{0},h_{1},\ldots,h_{t}\}$,
if $\forall i \in \{1,\ldots,t\}$, $\DIS(\{h_{0},h_{i}\}) \cap \{x_{1},\ldots,x_{t}\} = \{x_{i}\}$.
For brevity, in some instances below, we may simply say that $\{x_{1},\ldots,x_{t}\}$ \emph{is a star set for} $\H$, 
indicating that $\exists h_{0},h_{1},\ldots,h_{t} \in \H$ such that
$\{x_{1},\ldots,x_{t}\}$ is a star set for $\H$, witnessed by $\{h_{0},h_{1},\ldots,h_{t}\}$.
We may also say that $\{x_{1},\ldots,x_{t}\}$ \emph{is a star set for} $\H$ \emph{centered at} $h_{0} \in \H$
if $\exists h_{1},\ldots,h_{t} \in \H$ such that $\{x_{1},\ldots,x_{t}\}$ is a star set for $\H$, witnessed by $\{h_{0},h_{1},\ldots,h_{t}\}$.
For completeness, we also say that $\{\}$ (the empty sequence) is a star set for $\H$ (witnessed by $\{h_{0}\}$ for any $h_{0} \in \H$), for any nonempty $\H$.
In these terms, the star number of $\C$ is the maximum possible cardinality of a star set for $\C$, 
or $\infty$ if no such maximum exists.

\paragraph{Remark:} The star number can equivalently be described as the maximum possible degree in the data-induced one-inclusion graph for $\C$ \citep*[see][]{haussler:94},
where the maximum is over all possible data sets and nodes in the graph.%
\footnote{The maximum degree in the one-inclusion graph was recently studied in the context of teaching complexity by \citet*{fan:12}.
However, using the data-induced one-inclusion graph of \citet*{haussler:94} (rather than the graph based on the full space $\X$) can substantially increase the maximum degree by omitting certain highly-informative points.}
To relate this to the VC dimension, 
one can show
that the VC dimension is the maximum possible degree of a \emph{hypercube} in the data-induced one-inclusion graph
for $\C$ (maximized over all possible data sets).
From this, it is clear that $\s \geq \vc$.  Indeed,
any set $\{x_{1},\ldots,x_{k}\}$ shatterable by $\C$ is also a star set for $\C$,
since some $h_{0} \in \C$ classifies all $k$ points $-1$, and for each $x_{i}$,
some $h_{i} \in \C$ has $h_{i}(x_{i}) = +1$ while $h_{i}(x_{j}) = -1$ for every $j \neq i$
(where $h_{i}$ is guaranteed to exist by shatterability of the set).
On the other hand, there is no general upper bound on $\s$ in terms of $\vc$,
and the gap between $\s$ and $\vc$ can generally be infinite.

\paragraph{Examples:} 
Before continuing, we briefly go through a few simple example calculations of the star number.
For the class of \emph{threshold} classifiers on $\reals$ (i.e., $\C = \{ x \mapsto 2 \ind_{[t,\infty)}(x) - 1 : t \in \reals \}$), 
we have $\s = 2$, as $\{x_{1},x_{2}\}$ is a star set for $\C$ centered at $2 \ind_{[t,\infty)}-1$ if and only if $x_{1} < t \leq x_{2}$,
and any set $\{x_{1},x_{2},x_{3}\}$ cannot be a star set for $\C$ centered at any given $2 \ind_{[t,\infty)}-1$ since, of the (at least) two of these points
on the same side of $t$, any threshold classifier disagreeing with $2 \ind_{[t,\infty)}-1$ on the one further from $t$ must also 
disagree with $2 \ind_{[t,\infty)}-1$ on the one closer to $t$.
In contrast, for the class of \emph{interval} classifiers on $\reals$ (i.e., $\C = \{ x \mapsto 2 \ind_{[a,b]}(x) - 1 : -\infty < a \leq b < \infty \}$),
we have $\s = \infty$, since for \emph{any} distinct points $x_{0},x_{1},\ldots,x_{s} \in \reals$, 
$\{x_{1},\ldots,x_{s}\}$ is a star set for $\C$ witnessed by $\{ 2 \ind_{[x_{0},x_{0}]}-1, 2\ind_{[x_{1},x_{1}]}-1,\ldots,2\ind_{[x_{s},x_{s}]}-1 \}$.
It is an easy exercise to verify that we also have $\s=\infty$
for the classes of \emph{linear separators} on $\reals^{k}$ ($k \geq 2$) and axis-aligned rectangles on $\reals^{k}$ ($k \geq 1$),
since the above construction for interval classifiers can be embedded into
these spaces, with the star set lying within a lower-dimensional manifold in $\reals^{k}$ \citep*[see][]{dasgupta:04,dasgupta:05,hanneke:survey}.

As an intermediate case, where $\s$ has a range of values, 
consider the class of \emph{intervals of width at least $w \in (0,1)$} (i.e., $\C = \{ x \mapsto 2 \ind_{[a,b]}(x) - 1 : -\infty < a \leq b < \infty, b-a \geq w \}$),
for the space $\X = [0,1]$.  In this case, we can show that $\lfloor 2/w \rfloor \leq \s \leq \lfloor 2/w \rfloor + 2$, as follows.
We may note that letting $k = \lfloor 2/(w+\eps) \rfloor + 1$ (for $\eps > 0$), 
and taking $x_{i} = (w+\eps) (i-1) / 2$ for $1 \leq i \leq k$,
we have that $\{x_{1},\ldots,x_{k}\}$ is a star set for $\C$, witnessed by $\{ 2 \ind_{[-2w,-w]}-1, 2\ind_{[x_{1}-w/2,x_{1}+w/2]}-1,\ldots,2\ind_{[x_{k}-w/2,x_{k}+w/2]}-1 \}$.
Thus, taking $\eps \to 0$ reveals that $\s \geq \lfloor 2/w \rfloor$.
On the other hand, for any $k^{\prime} \in \nats$ with $k^{\prime} > 2$, and points $x_{1},\ldots,x_{k^{\prime}} \in [0,1]$,
suppose $\{x_{1},\ldots,x_{k^{\prime}}\}$ is a star set for $\C$ witnessed by $\{h_{0},h_{1},\ldots,h_{k^{\prime}}\}$.
Without loss of generality, suppose $x_{1} \leq x_{2} \leq \cdots \leq x_{k^{\prime}}$.
First suppose $h_{0}$ classifies all of these points $-1$.
Note that, for any $i \in \{3,\ldots,k^{\prime}\}$, since the interval corresponding to $h_{i-1}$ has width at least $w$ and contains $x_{i-1}$ but not $x_{i-2}$ or $x_{i}$, 
we have $x_{i}-x_{i-1} > \max\{ 0, w - (x_{i-1}-x_{i-2}) \}$.
Thus, $1 \geq \sum_{i=2}^{k^{\prime}} x_{i} - x_{i-1} > x_{2} - x_{1} + \sum_{i=3}^{k^{\prime}} \max\{ 0, w - (x_{i-1}-x_{i-2}) \} \geq (k^{\prime}-2) w - \sum_{i=3}^{k^{\prime}-1} x_{i} - x_{i-1} = (k^{\prime}-2) w - ( x_{k^{\prime}-1} - x_{2} )$,
so that $x_{k^{\prime}-1} - x_{2} > (k^{\prime}-2)w - 1$.
But $x_{k^{\prime}-1} - x_{2} \leq 1$, so that 
$k^{\prime} < 2/w + 2$.
Since $k^{\prime}$ is an integer, this implies $k^{\prime} \leq \lfloor 2/w \rfloor + 2$.
For the remaining case, if $h_{0}$ classifies some $x_{i}$ as $+1$,
then let $x_{i_{0}} = \min\{ x_{i} : h_{0}(x_{i}) = +1 \}$ and $x_{i_{1}} = \max\{ x_{i} : h_{0}(x_{i}) = +1 \}$.
Note that, if $i_{0} > 1$, then for any $x < x_{i_{0}-1}$, any $h \in \C$ with $h(x_{i_{0}}) = h(x) = +1 \neq h_{0}(x)$ must have $h(x_{i_{0}-1}) = +1 \neq h_{0}(x_{i_{0}-1})$,
so that $\{x,x_{i_{0}-1}\} \subseteq \DIS(\{h,h_{0}\})$.  Therefore, $\nexists x_{i} < x_{i_{0}-1}$ 
(since otherwise $\DIS(\{h_{i},h_{0}\}) \cap \{x_{1},\ldots,x_{k^{\prime}}\} = \{x_{i}\}$ would be violated),
so that $i_{0} \leq 2$. Symmetric reasoning implies $i_{1} \geq k^{\prime}-1$.
Similarly, if $\exists x \in (x_{i_{0}},x_{i_{1}})$, then any $h \in \C$ with $h(x) = -1 \neq h_{0}(x)$ must have either $h(x_{i_{0}}) = -1 \neq h_{0}(x_{i_{0}})$ or $h(x_{i_{1}}) = -1 \neq h_{0}(x_{i_{1}})$,
so that either $\{x,x_{i_{0}}\} \subseteq \DIS(\{h,h_{0}\})$ or $\{x,x_{i_{1}}\} \subseteq \DIS(\{h,h_{0}\})$.  Therefore, $\nexists x_{i} \in (x_{i_{0}},x_{i_{1}})$
(since again, $\DIS(\{h_{i},h_{0}\}) \cap \{x_{1},\ldots,x_{k^{\prime}}\} = \{x_{i}\}$ would be violated),
so that $i_{1} \in \{i_{0},i_{0}+1\}$.
Combined, these facts imply $k^{\prime} \leq i_{1} + 1 \leq i_{0} + 2\leq 4 \leq \lfloor 2/w \rfloor + 2$.
Altogether, we have $\s \leq \lfloor 2/w \rfloor + 2$.

\section{Main Results}
\label{sec:main}

We are now ready to state the main results of this article: upper and lower bounds on the minimax label complexities under the above noise models.
For the sake of making the theorem statements more concise, we abstract the dependence on logarithmic
factors in several of the upper bounds into a simple ``$\polylog(x)$'' factor, meaning a value $\lesssim \Log^{k}(x)$, for some $k \in [1,\infty)$ (in fact, all of these results hold with values of $k \leq 4$);
the reader is referred to the proofs for a description of the actual logarithmic factors this $\polylog$ function represents,
along with tighter expressions of the upper bounds.
The formal proofs of all of these results are included in Appendix~\ref{app:main-proofs}.

\begin{theorem}
\label{thm:realizable}
For any $\eps \in (0,1/9)$, $\conf \in (0,1/3)$, 
\begin{equation*}
\max\left\{ \min\left\{\s, \frac{1}{\eps}\right\}, \vc, \Log\left(\min\left\{\frac{1}{\eps}, |\C|\right\}\right)\right\}
\lesssim \LC_{\RE}(\eps,\conf) \lesssim 
\min\left\{ \s, \frac{\vc}{\eps}, \frac{\s \vc}{\Log(\s)} \right\} \Log\left(\frac{1}{\eps}\right).
\end{equation*}
\end{theorem}

\begin{theorem}
\label{thm:bounded}
For any $\bound \in [0,1/2)$, $\eps \in (0,(1-2\bound)/24)$, $\conf \in (0,1/24]$,
\begin{multline*}
\frac{1}{(1-2\bound)^{2}} \max\left\{ \min\left\{\s, \frac{1-2\bound}{\eps}\right\} \bound \Log\left(\frac{1}{\conf}\right), \vc\right\}
\\ \lesssim \LC_{\BN(\bound)}(\eps,\conf) \lesssim 
\frac{1}{(1-2\bound)^{2}} \min\left\{ \s, \frac{(1-2\bound) \vc}{\eps} \right\} \polylog\left(\frac{\vc}{\eps\conf}\right).
\end{multline*}
\end{theorem}

\begin{theorem}
\label{thm:tsybakov}
For any $\tsybca \in [4,\infty)$, $\tsyba \in (0,1)$, $\eps \in (0,1/(24 \tsybca^{1/\tsyba}))$, and $\conf \in (0,1/24]$,
\\if $0 < \tsyba \leq 1/2$, 
\begin{equation*}
\tsybca^{2} \left(\frac{1}{\eps}\right)^{2-2\tsyba} \left(\vc + \Log\left(\frac{1}{\conf}\right)\right)
\lesssim \LC_{\TN(\tsybca,\tsyba)}(\eps,\conf) 
\lesssim \tsybca^{2} \left(\frac{1}{\eps}\right)^{2-2\tsyba} \vc \cdot \polylog\left(\frac{\vc}{\eps\conf}\right)
\end{equation*}
and if $1/2 < \tsyba < 1$,
\begin{multline*}
\tsybca^{2} \left(\frac{1}{\eps}\right)^{2-2\tsyba} \max\left\{ \min\left\{ \s, \frac{1}{\tsybca^{1/\tsyba} \eps}\right\}^{2\tsyba-1} \Log\left(\frac{1}{\conf}\right), \vc\right\}
\\ \lesssim \LC_{\TN(\tsybca,\tsyba)}(\eps,\conf) 
\lesssim \tsybca^{2} \left(\frac{1}{\eps}\right)^{2-2\tsyba} \min\left\{ \frac{\s}{\vc}, \frac{1}{\tsybca^{1/\tsyba}\eps} \right\}^{2\tsyba-1} \vc \cdot \polylog\left(\frac{\vc}{\eps\conf}\right).
\end{multline*}
\end{theorem}

\begin{theorem}
\label{thm:diameter-localization}
For any $\tsybca \in [4,\infty)$, $\tsyba \in (0,1)$, $\eps \in (0,1/(24 \tsybca^{1/\tsyba}))$, and $\conf \in (0,1/24]$,
\\if $0 \leq \tsyba \leq 1/2$, 
\begin{equation*}
\tsybca^{2} \left(\frac{1}{\eps}\right)^{2-2\tsyba} \left(\vc + \Log\left(\frac{1}{\conf}\right)\right)
\lesssim \LC_{\DL(\tsybca,\tsyba)}(\eps,\conf)
\lesssim \tsybca^{2} \left(\frac{1}{\eps}\right)^{2-2\tsyba} \!\!\min\left\{ \s, \frac{1}{\tsybca \eps^{\tsyba}} \right\} \vc \cdot \polylog\left(\frac{1}{\eps\conf}\right),
\end{equation*}
and if $1/2 < \tsyba \leq 1$, 
\begin{multline*}
\tsybca^{2} \left(\frac{1}{\eps}\right)^{2-2\tsyba} \max\left\{\min\left\{ \s, \frac{1}{\tsybca^{1/\tsyba} \eps}\right\}^{2\tsyba-1} \Log\left(\frac{1}{\conf}\right), \vc\right\}
\\ \lesssim \LC_{\DL(\tsybca,\tsyba)}(\eps,\conf)
\lesssim \tsybca^{2} \left(\frac{1}{\eps}\right)^{2-2\tsyba} \min\left\{ \s, \frac{1}{\tsybca \eps^{\tsyba}} \right\} \vc \cdot \polylog\left(\frac{1}{\eps\conf}\right).
\end{multline*}
\end{theorem}

\begin{theorem}
\label{thm:benign}
For any $\nu \in [0,1/2)$, $\eps \in (0,(1-2\nu)/24)$, and $\conf \in (0,1/24]$,
\begin{equation*}
\frac{\nu^{2}}{\eps^{2}} \left(\vc + \Log\left(\frac{1}{\conf}\right)\right) + \min\left\{\s,\frac{1}{\eps}\right\}
\lesssim \LC_{\BE(\nu)}(\eps,\conf) 
\lesssim \left(  \frac{\nu^{2}}{\eps^{2}} \vc + \min\left\{\s,\frac{\vc}{\eps}\right\}\right) \polylog\left(\frac{\vc}{\eps\conf}\right).
\end{equation*}
\end{theorem}

\begin{theorem}
\label{thm:agnostic}
For any $\nu \in [0,1/2)$, $\eps \in (0,(1-2\nu)/24)$, and $\conf \in (0,1/24]$,
\begin{multline*}
\frac{\nu^{2}}{\eps^{2}} \left(\vc + \Log\left(\frac{1}{\conf}\right)\right) + \min\left\{\s,\frac{1}{\eps}\right\}
\\ \lesssim \LC_{\AG(\nu)}(\eps,\conf)
\lesssim \min\left\{ \s, \frac{1}{\nu+\eps} \right\} \left( \frac{\nu^{2}}{\eps^{2}} + 1 \right) \vc \cdot \polylog\left(\frac{1}{\eps\conf}\right).
\end{multline*}
\end{theorem}

Here, we mention a few noteworthy 
observations and comments regarding the above theorems.
We sketch the main innovations underlying the active learning algorithm achieving these upper bounds in Section~\ref{sec:algorithm}.
Sections \ref{sec:passive} and \ref{sec:complexities} include further detailed and thorough comparisons 
of each of these results to those in the prior literature on passive and active learning.

\paragraph{Comparison to the previous best known results:}
Aside from Theorems \ref{thm:diameter-localization} and \ref{thm:agnostic}, each of the above 
results offers some kind of refinement over the previous best known results on the label complexity of active learning.  
Some of these refinements are relatively mild, such as those for the realizable case and bounded noise.
However, our refinements under Tsybakov noise and benign noise are far more significant.
In particular, 
perhaps the most surprising and interesting of the above results are the upper bounds in Theorem~\ref{thm:tsybakov},
which can be considered the primary contribution of this work.

As discussed above, the prior literature on noise-robust active learning is largely rooted in the intuitions and techniques developed for the realizable case.
As indicated by Theorem~\ref{thm:realizable}, there is a wide spread of label complexities
for active learning problems in the realizable case, depending on the structure of the hypothesis class.  In particular,
when $\s < \infty$, we have $O(\Log(1/\eps))$ label complexity in the realizable case, representing a nearly-exponential 
improvement over passive learning, which has $\tilde{\Theta}(1/\eps)$ dependence on $\eps$.  On the other hand,
when $\s = \infty$, we have $\Omega(1/\eps)$ minimax label complexity for active learning, which is the same
dependence on $\eps$ as known for passive learning (see Section~\ref{sec:passive}).  Thus, for active learning in the realizable case,
some hypothesis classes are ``easy'' (such as threshold classifiers), offering strong improvements over passive learning,
while others are ``hard'' (such as interval classifiers), offering almost no improvements over passive.

With the realizable case as inspiration, the results in the prior literature on general noise-robust active learning have all continued to reflect
these distinctions, and the label complexity bounds in those works continue to exhibit this wide spread.  In the case of Tsybakov
noise, the best general results in the prior literature \citep*[from][]{hanneke:12b,hanneke:survey} correspond to an upper bound of roughly 
$\tsybca^{2} \left(\frac{1}{\eps}\right)^{2-2\tsyba} \min\left\{ \s, \frac{1}{\tsybca \eps^{\tsyba}} \right\} \vc \cdot \polylog\left(\frac{1}{\eps\conf}\right)$
(after converting those complexity measures into the star number via the results in Section~\ref{sec:complexities} below).
When $\s < \infty$, this has dependence $\tilde{\Theta}(\eps^{2\tsyba-2})$ on $\eps$, which reflects a strong improvement
over the $\tilde{\Theta}(\eps^{\tsyba-2})$ minimax sample complexity of passive learning for this problem (see Section~\ref{sec:passive}).
On the other hand, when $\s = \infty$, this bound is $\tilde{\Theta}(\eps^{\tsyba-2})$, so that as in the realizable case, 
the bound is no better than that of passive learning for these hypothesis classes.  Thus, the prior results in the literature 
continue the trend observed in the realizable case, in which the ``easy'' hypothesis classes admit strong improvements
over passive learning, while the ``hard'' hypothesis classes have a bound that is no better than the sample complexity of passive learning.

With this as background, it comes as quite a surprise that the upper bounds in Theorem~\ref{thm:tsybakov} are 
\emph{always} smaller than the corresponding minimax sample complexities of passive learning,
in terms of their asymptotic dependence on $\eps$ for $0 < \tsyba < 1$.  Specifically, these upper bounds
reveal a label complexity $\tilde{O}( \eps^{2\tsyba-2} )$ when $\s < \infty$, and $\tilde{O}( \eps^{2\tsyba-2} \lor (1/\eps) )$ when $\s = \infty$.
Comparing to the $\tilde{\Theta}(\eps^{\tsyba-2})$ minimax sample complexity of passive learning, the improvement for active learning
is by a factor of $\tilde{\Theta}( \eps^{-\tsyba} )$ when $\s < \infty$, and by a factor of $\tilde{\Theta}( \eps^{-\min\{\tsyba,1-\tsyba\}} )$ when $\s = \infty$.
As a further surprise, when $0 < \tsyba \leq 1/2$ (the high-noise regime),
we see that the distinctions between active learning problems of a given VC dimension essentially \emph{vanish} (up to logarithmic factors), so that the familiar
spread of label complexities from the realizable case is no longer present.  Indeed, in this latter case, \emph{all} hypothesis classes
with finite VC dimension exhibit the strong improvements over passive learning, previously only known to hold for the ``easy'' 
hypothesis classes (such as threshold classifiers): that is, $\tilde{O}(\eps^{2\tsyba-2})$ label complexity.

Further examining these upper bounds, we see that the spread of label complexities between ``easy'' and ``hard'' hypothesis classes
increasingly re-emerges as $\tsyba$ approaches $1$, beginning with $\tsyba = 1/2$.  This transition point is quite sensible, since
this is precisely the point at which the label complexity has dependence on $\eps$ of $\tilde{\Theta}(1/\eps)$,
which is roughly the same as the minimax label complexity of the ``hard'' hypothesis classes in the realizable case,
which is, after all, included in $\TN(\tsybca,\tsyba)$.  Thus, as $\tsyba$ 
increases above $1/2$, the ``easy'' hypothesis classes (with $\s < \infty$) exhibit stronger improvements over passive learning,
while the ``hard'' hypothesis classes (with $\s = \infty$) continue to exhibit precisely this $\tilde{\Theta}\left(\frac{1}{\eps}\right)$
behavior.  In either case, the label complexity exhibits an improvement in dependence on $\eps$ compared to passive learning for the same $\tsyba$ value.
But since the label complexity of passive learning decreases to $\tilde{\Theta}\left(\frac{1}{\eps}\right)$ as $\tsyba \to 1$, we naturally have that 
for the ``hard'' hypothesis classes, the gap between the passive and active label complexities shrinks as $\tsyba$ approaches $1$.
In contrast, the ``easy'' hypothesis classes exhibit a gap between passive and active label complexities that 
becomes more pronounced as $\tsyba$ approaches $1$ (with a near-exponential improvement over passive learning exhibited in the limiting case, corresponding to bounded noise).

This same pattern is present, though to a lesser extent, in the benign noise case.
In this case, the best general results in the prior literature \citep*[from][]{dasgupta:07,hanneke:07a,hanneke:survey}
correspond to an upper bound of roughly $\min\left\{\s, \frac{1}{\nu+\eps}\right\} \left(\frac{\nu^{2}}{\eps^{2}}+1\right) \vc \cdot \polylog\left(\frac{1}{\eps\conf}\right)$
(again, after converting those complexity measures into the star number via the results in Section~\ref{sec:complexities} below).
When $\s < \infty$, the dependence on $\nu$ and $\eps$ is roughly $\tilde{\Theta}\left(\frac{\nu^{2}}{\eps^{2}}\right)$ (aside from logarithmic factors and constants, and for $\nu > \eps$).
However, when $\s = \infty$, this dependence becomes roughly $\tilde{\Theta}\left(\frac{\nu}{\eps^{2}}\right)$, 
which is the same as in the minimax sample complexity of passive learning (see Section~\ref{sec:passive}).
Thus, for these results in the prior literature, we again see that the ``easy'' hypothesis classes have a bound reflecting improvements over passive learning,
while the bound for the ``hard'' hypothesis classes fail to reflect any improvements over passive learning at all.

In contrast, consider the upper bound in Theorem~\ref{thm:benign}.
In this case, when $\nu \geq \sqrt{\eps}$ (again, the high-noise regime), for \emph{all} hypothesis classes with finite VC dimension,
the dependence on $\nu$ and $\eps$ is roughly $\tilde{\Theta}\left(\frac{\nu^{2}}{\eps^{2}}\right)$ (aside from logarithmic factors and constants).
Again, this makes almost no distinction between ``easy'' hypothesis classes (with $\s < \infty$) and ``hard'' hypothesis classes (with $\s=\infty$),
and instead always exhibits the strongest possible improvements (up to logarithmic factors), 
previously only known to hold for the ``easy'' classes (such as threshold classifiers):
namely, reduction in label complexity by roughly a factor of $1/\nu$ compared to passive learning.
The improvements in this case are typically milder than we found in Theorem~\ref{thm:tsybakov},
but noteworthy nonetheless.  Again, as $\nu$ decreases below $\sqrt{\eps}$, the distinction
between ``easy'' and ``hard'' hypothesis classes begins to re-emerge, with the harder classes maintaining a $\tilde{\Theta}\left(\frac{1}{\eps}\right)$
dependence 
(which is equivalent to the realizable-case label complexity for these classes, up to logarithmic factors),
while the easier classes continue to exhibit the $\tilde{\Theta}\left(\frac{\nu^{2}}{\eps^{2}}\right)$ behavior,
approaching $O\left( \polylog\left(\frac{1}{\eps}\right) \right)$ as $\nu$ shrinks.

\paragraph{The dependence on $\boldsymbol{\conf}$:}
One remarkable fact about $\LC_{\RE}(\eps,\conf)$ is that there is \emph{no} significant dependence on 
$\conf$ in the optimal label complexity for the given range of $\conf$.\footnote{We should expect 
a more significant dependence on $\conf$ near $1$, since one case easily prove that $\LC_{\RE}(\eps,\conf) \to 0$ as $\conf \to 1$.} 
Note that this is not the case in noisy settings, where the lower bounds have an explicit dependence on $\conf$.
In the proofs, this dependence on $\conf$ is introduced via randomness of the labels.  
However, as argued by \citet*{kaariainen:06},
a dependence on $\conf$ is sometimes still required in $\LC_{\Dset}(\eps,\conf)$, 
even if we restrict $\Dset$ to those $\PXY \in \AG(\nu)$ inducing \emph{deterministic} labels: 
that is, $\eta(x;\PXY) \in \{0,1\}$ for all $x$.

\paragraph{Spanning the gaps:}
All of these results have gaps between the lower and upper bounds.
It is interesting to note that one can construct examples of hypothesis classes spanning these gaps,
for Theorems \ref{thm:realizable}, \ref{thm:bounded}, \ref{thm:tsybakov}, and \ref{thm:benign} (up to logarithmic factors).
For instance, for sufficiently large $\vc$ and $\s$ and sufficiently small $\eps$ and $\conf$, these upper bounds are tight (up to logarithmic factors) in the case where
$\C = \{ x \mapsto 2\ind_{S}(x)-1 : S \subseteq \{1,\ldots,\s\}, |S|\leq\vc\}$, for $\X = \nats$ 
(taking inspiration from a suggested modification by \citealp*{hanneke:survey}, of the proof of a related result of \citealp*{raginsky:11}).
Likewise, these lower bounds are tight (up to logarithmic factors) in the case that $\X = \nats$ and
$\C = \{ x \mapsto 2\ind_{S}(x)-1 : S \in 2^{\{1,\ldots,\vc\}} \cup \{ \{i\} : \vc+1 \leq i \leq \s \} \}$.\footnote{Technically, 
for Theorems~\ref{thm:bounded} and \ref{thm:benign}, we require slightly stronger versions of the lower bound to establish 
tightness for $\bound$ or $\nu$ near $0$: namely, adding the lower bound from Theorem~\ref{thm:realizable}
to these lower bounds.  The validity of this stronger lower bound follows immediately from the facts that $\RE \subseteq \BN(\bound)$ and $\RE \subseteq \BE(\nu)$.}
Thus, these upper and lower bounds cannot be significantly refined (without loss of generality)
without introducing additional complexity measures to distinguish these cases.
For completeness, we include proofs of these claims in Appendix~\ref{app:gaps}.
It immediately follows from this (and monotonicity of the respective noise models in $\C$) 
that the upper and lower bounds in Theorems \ref{thm:realizable}, \ref{thm:bounded}, \ref{thm:tsybakov}, and \ref{thm:benign}
are each sometimes tight in the case $\s = \infty$, as limiting cases of the above constructions: that is, the upper bounds are tight (up to logarithmic factors) for
$\C = \{ x \mapsto 2\ind_{S}(x)-1 : S \subseteq \nats, |S| \leq \vc \}$,
and the lower bounds are tight (up to logarithmic factors) for 
$\C = \{ x \mapsto 2\ind_{S}(x)-1 : S \in 2^{\{1,\ldots,\vc\}} \cup \{ \{i\} : \vc+1 \leq i < \infty \} \}$.
It is interesting to note that the above space $\C$ for which the upper bounds are tight can be embedded in a variety of 
hypothesis classes in common use in machine learning (while maintaining VC dimension $\lesssim \vc$ and star number $\lesssim \s$):
for instance, in the case of $\s=\infty$, this is true of linear separators in $\reals^{3\vc}$ and axis-aligned rectangles in $\reals^{2\vc}$.
It follows that the upper bounds in these theorems are tight (up to logarithmic factors) for each of these hypothesis classes.

\paragraph{Separation of $\boldsymbol{\TN(\tsybca,\tsyba)}$ and $\boldsymbol{\DL(\tsybca,\tsyba)}$:}
Another interesting implication of these results is a separation between the noise models $\TN(\tsybca,\tsyba)$ and $\DL(\tsybca,\tsyba)$
not previously noted in the literature.  Specifically, if we consider any class $\C$ comprised of only the $\s+1$ classifiers in Definition~\ref{def:star}, 
then one can show\footnote{Specifically, this follows by taking $\zeta = \frac{\tsybca}{2} (4\eps)^{\tsyba}$, $\bound = \frac{1}{2} - \frac{2}{\tsybca 4^{\tsyba}} \eps^{1-\tsyba}$,
and $k = \min\left\{ \s-1, \lfloor 1/\zeta \rfloor \right\}$ in Lemma~\ref{lem:rr11-star} of Appendix~\ref{sec:rr-lemma},
and noting that the resulting set of distributions $\RR(k,\zeta,\bound)$ is contained in $\DL(\tsybca,\tsyba)$ for this $\C$.}
that (for $\s \geq 3$), for any $\tsyba \in (0,1]$, $\tsybca \in [4,\infty)$, $\eps \in (0,1/(4\tsybca^{1/\tsyba}))$, and $\conf \in (0,1/16]$,
\begin{equation*}
\LC_{\DL(\tsybca,\tsyba)}(\eps,\conf) \gtrsim \tsybca^{2} \left(\frac{1}{\eps}\right)^{2-2\tsyba}\min\left\{\s,\frac{1}{\tsybca\eps^{\tsyba}}\right\} \Log\left(\frac{1}{\conf}\right).
\end{equation*}
In particular, when $\s > \frac{1}{\tsybca \eps^{\tsyba}}$, we have $\LC_{\DL(\tsybca,\tsyba)}(\eps,\conf) \gtrsim \tsybca \eps^{\tsyba-2} \Log(1/\conf)$,
which is larger than the upper bound on $\LC_{\TN(\tsybca,\tsyba)}(\eps,\conf)$.
Furthermore, when $\s = \infty$, this lower bound has asymptotic dependence on $\eps$ that is $\Omega( \eps^{\tsyba-2} )$,
which is the same dependence
found in the sample complexity of passive learning, up to a logarithmic factor (see Section~\ref{sec:passive} below).
Comparing this to the upper bounds in Theorem~\ref{thm:tsybakov}, 
which exhibit asymptotic dependence on $\eps$ as $\LC_{\TN(\tsybca,\tsyba)}(\eps,\conf) = \tilde{O}( \eps^{\min\{2\tsyba-1,0\}-1} )$
when $\s = \infty$, we see that for this class, any $\tsyba \in (0,1)$ has 
$\LC_{\TN(\tsybca,\tsyba)}(\eps,\conf) \ll \LC_{\DL(\tsybca,\tsyba)}(\eps,\conf)$.
One reason this separation is interesting is that most of the existing
literature on active learning under $\TN(\tsybca,\tsyba)$ makes use
of the noise condition via the fact that it implies $\Px(x : h(x) \neq \target_{\PXY}(x)) \leq \tsybca (\er_{\PXY}(h) - \er_{\PXY}(\target_{\PXY}))^{\tsyba}$ for all $h \in \C$:
that is, $\TN(\tsybca,\tsyba) \subseteq \DL(\tsybca,\tsyba)$.  This 
separation indicates that, to achieve the optimal performance under $\TN(\tsybca,\tsyba)$,
one needs to consider more-specific properties of this noise model, beyond
those satisfied by $\DL(\tsybca,\tsyba)$.
Another reason this separation is quite interesting is that it 
contrasts with the known results for \emph{passive} learning, where 
(as we discuss in Section~\ref{sec:passive} below)
the sample complexities under these two noise models are \emph{equivalent}
(up to an unresolved logarithmic factor).

\paragraph{Gaps in Theorems~\ref{thm:diameter-localization} and \ref{thm:agnostic}, and related open problems:}
We conjecture that the dependence on $\vc$ and $\s$ in the upper bounds of 
Theorem~\ref{thm:diameter-localization} can be refined in general
(where presently it is linear in $\s \vc$).  
More specifically, we conjecture that the upper bound can be improved to 
\begin{equation*}
\LC_{\DL(\tsybca,\tsyba)}(\eps,\conf) \lesssim \tsybca^{2} \left(\frac{1}{\eps}\right)^{2-2\tsyba} \min\left\{ \s, \frac{\vc}{\tsybca \eps^{\tsyba}} \right\} \polylog\left(\frac{1}{\eps\conf}\right),
\end{equation*}
though it is unclear at this time as to how this might be achieved.
The above example (separating $\DL(\tsybca,\tsyba)$ from $\TN(\tsybca,\tsyba)$)
indicates that we generally cannot hope to reduce the upper bound on the label complexity for $\DL(\tsybca,\tsyba)$
much beyond this.

As for whether the form of the upper bound on $\LC_{\AG(\nu)}(\eps,\conf)$ in Theorem~\ref{thm:agnostic} can 
generally be improved to match the form of the upper bound for $\LC_{\BE(\nu)}(\eps,\conf)$, this remains a fascinating open question.
We conjecture that at least the dependence on $\vc$ and $\s$ can be improved to some extent (where presently it is linear in $\vc \s$).

\paragraph{Minutiae:}
We note that the restrictions to the ranges of $\eps$ and $\conf$ in the above results are required only for the lower bounds (aside from $\conf \in (0,1]$, $\eps > 0$),
as are the restrictions to the ranges of the parameters $\tsybca$, $\tsyba$, and $\nu$, aside from the constraints in the definitions in Section~\ref{sec:noise-models};
the upper bounds are proven without any such restrictions in Appendix~\ref{app:main-proofs}.
Also, several of the upper bounds above (e.g., Theorems~\ref{thm:tsybakov} and \ref{thm:benign}) are slightly looser (by logarithmic factors)
than those actually proven in Appendix~\ref{app:main-proofs}, which are typically stated in a different form (e.g., with factors of 
$\vc\Log\left(\frac{1}{\eps}\right)+\Log\left(\frac{1}{\conf}\right)$, rather than simply $\vc \cdot \polylog\left(\frac{1}{\eps\conf}\right)$).  
We state the weaker results here purely to simplify the theorem statements, referring the interested reader to the proofs for the refined versions.  
However, aside from Theorem~\ref{thm:realizable}, we believe it is possible to further optimize the logarithmic factors in all of these upper bounds.

We additionally note that we can also obtain results by the subset relations between the noise models.
For instance, since $\RE \subseteq \BN(\bound) \subseteq \BE(\bound) \subseteq \AG(\bound)$, 
in the case $\bound$ is close to $0$ we can increase the lower bounds in Theorems \ref{thm:bounded}, \ref{thm:benign}, and \ref{thm:agnostic}
based on the lower bound in Theorem~\ref{thm:realizable}: that is, for $\nu \geq \bound \geq 0$, 
\begin{equation*}
\LC_{\AG(\nu)}(\eps,\conf) \geq \LC_{\BE(\nu)}(\eps,\conf) \geq \LC_{\BN(\bound)}(\eps,\conf) \geq \LC_{\RE}(\eps,\conf) \gtrsim \max\left\{ \min\left\{ \s, \frac{1}{\eps}\right\},\vc\right\}.
\end{equation*}
Similarly, since $\RE$ is contained in all of the noise models studied here,
$\Log\left(\min\left\{\frac{1}{\eps}, |\C|\right\}\right)$ can also be included as a lower bound in each of these results.
Likewise, in the cases that $\tsybca$ is very large or $\tsyba$ is very close to $0$, we can get a more informative upper bound in Theorem~\ref{thm:tsybakov}
via Theorem~\ref{thm:benign}, since $\TN(\tsybca,\tsyba) \subseteq \BE(1/2)$.
For simplicity, in many cases we have not explicitly included the various compositions of the above results that can be obtained in this way
(with only a few exceptions).

\subsection{The Strategy behind Theorems~\ref{thm:tsybakov} and \ref{thm:benign}}
\label{sec:algorithm}

The upper bounds in Theorems \ref{thm:tsybakov} and \ref{thm:benign} represent the main results of this work,
and along with the upper bound in Theorem~\ref{thm:bounded}, are based on a general argument with essentially
three main components.  The first component is a more-sophisticated variant of a basic approach introduced
to the active learning literature by \citet*{kaariainen:06}: namely, reduction to the realizable case via repeatedly 
querying for the label at a point in $\X$ until its Bayes optimal classification can be determined (based
on a sequential probability ratio test, as studied by \citealp*{wald:45,wald:47}).  
Of course, in the present model of active learning,
repeatedly requesting a label $Y_{i}$ yields no new information beyond requesting $Y_{i}$ once, since we are 
not able to resample from the distribution of $Y_{i}$ given $X_{i}$ (as \citealp*{kaariainen:06}, does).  To resolve this,
we argue that it is possible to partition the space $\X$ into cells, in a way such that $\target_{\PXY}$ is nearly constant
in the vast majority of cells (without direct knowledge of $\target_{\PXY}$ or $\Px$); this is essentially a data-dependent
approximation to the recently-discovered finite approximability property of VC classes \citep*{adams:12}.
Given this partition, for a given point $X_{i}$, we can find many other points $X_{j}$ in the same cell of the partition with $X_{i}$,
and request labels for these points until we can determine what the majority label for the cell is.  We show that,
with high probability, this value will equal $\target_{\PXY}(X_{i})$, so that we can effectively use these majority
labels in an active learning algorithm for the realizable case.

However, we note that in the case of $\TN(\tsybca,\tsyba)$, if we simply apply 
this repeated querying strategy to random $\Px$-distributed samples, the resulting label complexity would be too large,
and we would sometimes expect to exhaust most of the queries determining the optimal labels in very \emph{noisy} regions
(i.e., in cells of the partition where $\eta(\cdot;\PXY)$ is close to $1/2$ on average).  This is because Tsybakov's condition
allows that such regions can have non-negligible probability, and the number of samples required to determine the
majority value of a $\pm 1$ random variable becomes unbounded as its mean approaches zero.  However, we can note
that it is also less important for the final classifier $\hat{h}$ to agree with $\target_{\PXY}$ on these high-noise points than it is for low-noise
points, since classifying them opposite from $\target_{\PXY}$ has less impact on the excess error rate $\er_{\PXY}(\hat{h}) - \er_{\PXY}(\target_{\PXY})$.
Therefore, as the second main component of our active learning strategy, we take
a tiered approach to learning, effectively shifting the distribution $\Px$ to favor points in cells with average $\eta(\cdot;\PXY)$
value further from $1/2$.
We achieve this by discarding a point $X_{i}$ if the number of queries exhausted toward determining the majority label in
its cell of the partition becomes excessively large, and we gradually decrease this threshold as the data set grows,
so that the points making it through this filter have progressively less and less noisy labels.  
By choosing $\hat{h}$ to agree with the inferred $\target_{\PXY}$ classification of every point passing this filter,
and combining this with the standard analysis of learning in the realizable case \citep*{vapnik:82,vapnik:98,blumer:89},
this allows us to provide a bound
on the fraction of points in $\X$ at a given level of noisiness (i.e., $|\eta(\cdot;\PXY)-1/2|$) on which the produced classifier $\hat{h}$ disagrees with $\target_{\PXY}$,
such that this bound decreases as the noisiness decreases (i.e., as $|\eta(\cdot;\PXY)-1/2|$ increases).  Furthermore, by discarding many of the 
points in high-noise regions without exhausting too many label requests trying to determine their $\target_{\PXY}$ classifications, we are able to reduce the total
number of label requests needed to obtain $\eps$ excess error rate.

Already these two components comprise the essential strategy that achieves these upper bounds in the case 
of $\s=\infty$.
However, to obtain the stated dependence on $\s$ in these bounds when $\s < \infty$, we need to 
introduce a third component: namely, using the inferred values of $\target_{\PXY}(X_{i})$ in the context 
of an active learning algorithm for the realizable case.  For this, we specifically use the disagreement-based 
strategy of \citet*{cohn:94} (known as CAL), which processes the unlabeled data in sequence, and requests to observe
the classification $\target_{\PXY}(X_{i})$ if and only if $X_{i}$ is in the region of disagreement of the set of 
classifiers in $\C$ consistent with all previously-observed $\target_{\PXY}(X_{j})$ values.
Using a modification of a recent analysis of this algorithm by \citet*{hanneke:14a}
(applied to each tier of label-noise separately), 
combined with the results below (in Section~\ref{sec:xtd}) relating the complexity measure 
used in that analysis to the star number, we obtain the dependence on $\s$ stated in the above results.

\section{Comparison to Passive Learning}
\label{sec:passive}

The natural baseline for comparison in active learning is the \emph{passive learning} protocol, 
in which the labeled data are i.i.d. samples with common distribution $\PXY$:
that is, the input to the passive learning algorithm is $(X_1,Y_1),\ldots,(X_n,Y_n)$.
In this context, the minimax sample complexity of passive learning, denoted $\SC_{\Dset}(\eps,\conf)$,
is defined as the smallest $n \in \nats \cup \{0\}$ for which there exists a passive learning rule
mapping $(X_1,Y_1),\ldots,(X_n,Y_n)$ to a classifier $\hat{h} : \X \to \Y$ such that,
for any $\PXY \in \Dset$, with probability at least $1-\conf$, $\er_{\PXY}(\hat{h}) - \inf_{h \in \C} \er_{\PXY}(h) \leq \eps$.

Clearly $\LC_{\Dset}(\eps,\conf) \leq \SC_{\Dset}(\eps,\conf)$ for any $\Dset$, since for every passive learning algorithm $\alg$,
there is an active learning algorithm that requests $Y_{1},\ldots,Y_{n}$ and then runs $\alg$ with $(X_{1},Y_{1}),\ldots,(X_{n},Y_{n})$
to determine the returned classifier.
One of the main interests in the theory of 
active learning is determining the size of the gap between these two complexities, for various sets $\Dset$.
For the purpose of this comparison, we now review several results known to hold for $\SC_{\Dset}(\eps,\conf)$, for 
various sets $\Dset$.
Specifically, the following bounds are known to hold for any choice of hypothesis class $\C$, 
and for $\bound$, $\tsybca$, $\tsyba$, $\nu$, $\eps$, and $\conf$ as in the respective theorems from Section~\ref{sec:main}
\citep*{vapnik:71,vapnik:82,vapnik:98,blumer:89,ehrenfeucht:89,haussler:94,massart:06,hanneke:survey}.

\begin{itemize}
\item $\frac{1}{\eps}\left( \vc + \Log\left(\frac{1}{\conf}\right) \right) \lesssim \SC_{\RE}(\eps,\conf) \lesssim \frac{1}{\eps}\left( \vc \Log\left( \frac{1}{\max\{\eps,\conf\}} \right) + \Log\left(\frac{1}{\conf}\right) \right)$.
\item $\frac{1}{(1-2\bound)\eps}\left( \vc + \Log\left(\frac{1}{\conf}\right) \right) \lesssim \SC_{\BN(\bound)}(\eps,\conf) \lesssim \frac{1}{(1-2\bound)\eps}\left( \vc \Log\left( \frac{1-2\bound}{\eps} \right) + \Log\left(\frac{1}{\conf}\right) \right)$.
\item $\frac{\tsybca}{\eps^{2-\tsyba}} \left( \vc + \Log\left(\frac{1}{\conf}\right) \right) \lesssim \SC_{\TN(\tsybca,\tsyba)}(\eps,\conf) \leq \SC_{\DL(\tsybca,\tsyba)} \lesssim \frac{\tsybca}{\eps^{2-\tsyba}} \left( \vc \Log\left( \frac{1}{\tsybca \eps^{\tsyba}} \right) + \Log\left(\frac{1}{\conf}\right) \right)$.
\item $\frac{\nu+\eps}{\eps^{2}}\left( \vc + \Log\left(\frac{1}{\conf}\right) \right) \lesssim \SC_{\BE(\nu)}(\eps,\conf) \leq \SC_{\AG(\nu)}(\eps,\conf) \lesssim \frac{\nu+\eps}{\eps^{2}} \left( \vc \Log\left(\frac{1}{\nu+\eps}\right) + \Log\left(\frac{1}{\conf}\right) \right)$.
\end{itemize}

Let us compare these to the results for active learning in Section~\ref{sec:main} on a case-by-case basis.
In the realizable case, we observe clear improvements of active learning over passive learning in the case $\s \ll \frac{\vc}{\eps}$ (aside from logarithmic factors).
In particular, based on the upper and lower bounds for both passive and active learning, we may conclude that 
$\s < \infty$ is necessary and sufficient for the asymptotic dependence on $\eps$ to satisfy
$\LC_{\RE}(\eps,\cdot) = o(\SC_{\RE}(\eps,\cdot))$; specifically, when $\s < \infty$, $\LC_{\RE}(\eps,\cdot) = O(\Log(\SC_{\RE}(\eps,\cdot)))$,
and when $\s = \infty$, $\LC_{\RE}(\eps,\cdot) = \Theta(\SC_{\RE}(\eps,\cdot))$.
For bounded noise, we have a similar asymptotic behavior.  When $\s < \infty$, again $\LC_{\BN(\bound)}(\eps,\cdot) = O(\polylog(\SC_{\BN(\bound)}(\eps,\cdot)))$,
and when $\s = \infty$, $\LC_{\BN(\bound)}(\eps,\cdot) = \tilde{\Theta}(\SC_{\BN(\bound)}(\eps,\cdot))$.
In terms of the constants, to obtain improvements over passive learning (aside from the effects of logarithmic factors),
it suffices to have $\s \ll \frac{(1-2\bound) \vc}{\eps}$, which is somewhat smaller (depending on $\bound$)
than was sufficient in the realizable case.

Under Tsybakov's noise condition, every $\tsyba \in (0,1/2]$ shows an improvement in the upper bounds for active learning 
over the lower bound for passive learning by a factor of roughly $\frac{1}{\tsybca \eps^{\tsyba}}$ (aside from logarithmic factors). 
On the other hand, when $\tsyba \in (1/2,1)$, 
if $\s < \frac{\vc}{\tsybca^{1/\tsyba} \eps}$, the improvement of active upper bounds over the passive lower bound is by a factor of roughly 
$\frac{1}{\tsybca \eps^{\tsyba}} \left(\frac{\vc}{\s}\right)^{2\tsyba-1}$, while for $\s \geq \frac{\vc}{\tsybca^{1/\tsyba} \eps}$,
the improvement is by a factor of roughly $\frac{1}{\tsybca^{\frac{1-\tsyba}{\tsyba}} \eps^{1-\tsyba}}$ (again, ignoring logarithmic factors in both cases).
In particular, for \emph{any} $\tsyba \in (0,1)$, when $\s < \infty$, the asymptotic dependence
on $\eps$ satisfies $\LC_{\TN(\tsybca,\tsyba)}(\eps,\cdot) = \tilde{\Theta}\left( \eps^{\tsyba} \SC_{\TN(\tsybca,\tsyba)}(\eps,\cdot) \right)$,
and when $\s = \infty$, the asymptotic dependence on $\eps$ satisfies
$\LC_{\TN(\tsybca,\tsyba)}(\eps,\cdot) = \tilde{\Theta}\left( \eps^{\min\{\tsyba,1-\tsyba\}} \SC_{\TN(\tsybca,\tsyba)}(\eps,\cdot) \right)$.
In either case, we have that for any $\tsyba \in (0,1)$, 
$\LC_{\TN(\tsybca,\tsyba)}(\eps,\cdot) = o( \SC_{\TN(\tsybca,\tsyba)}(\eps,\cdot) )$.

For the Bernstein class condition, the gaps in the upper and lower bounds of Theorem~\ref{thm:diameter-localization}
render unclear the necessary and sufficient conditions for $\LC_{\DL(\tsybca,\tsyba)}(\eps,\cdot) = o(\SC_{\DL(\tsybca,\tsyba)}(\eps,\cdot))$.
Certainly $\s < \infty$ is a sufficient condition for this, in which case the improvements are by a factor of roughly $\frac{1}{\tsybca \eps^{\tsyba}}$.
However, in the case of $\s = \infty$, the upper bounds do not reveal any improvements over those given above for $\SC_{\DL(\tsybca,\tsyba)}(\eps,\conf)$.
Indeed, the example given above in Section~\ref{sec:main} reveals that, in some nontrivial cases, $\LC_{\DL(\tsybca,\tsyba)}(\eps,\conf) \gtrsim \SC_{\DL(\tsybca,\tsyba)}(\eps,\conf) / \Log(1/\eps)$,
in which case any improvements would be, at best, in the constant and logarithmic factors.  Note that this example also presents an interesting contrast between active and passive learning,
since it indicates that in some cases $\LC_{\DL(\tsybca,\tsyba)}(\eps,\conf)$ and $\LC_{\TN(\tsybca,\tsyba)}(\eps,\conf)$ are quite different, 
while the above bounds for passive learning reveal that $\SC_{\DL(\tsybca,\tsyba)}(\eps,\conf)$ is equivalent to $\SC_{\TN(\tsybca,\tsyba)}(\eps,\conf)$ up to constant and logarithmic factors.

In the case of benign noise, comparing the above bounds for passive learning to Theorem~\ref{thm:benign}, we see
that (aside from logarithmic factors) the upper bound for active learning improves over the lower bound for passive learning
by a factor of roughly $\frac{1}{\nu}$ when $\nu \geq \sqrt{\eps}$.  When $\nu < \sqrt{\eps}$, if $\s > \frac{\vc}{\eps}$, 
the improvements are by a factor of roughly $\frac{\nu+\eps}{\eps}$, and if $\s \leq \frac{\vc}{\eps}$, the improvements
are by roughly a factor of $\min\left\{\frac{1}{\nu}, \frac{(\nu+\eps) \vc}{\eps^{2} \s}\right\}$ (again, ignoring logarithmic factors).
However, as has been known for this noise model for some time \citep*{kaariainen:06}, there are no gains in terms of 
the asymptotic dependence on $\eps$ for fixed $\nu$.  However, if we consider $\nu_{\eps}$ such that $\eps \leq \nu_{\eps} = o(1)$, 
then for $\s < \infty$ we have $\LC_{\BE(\nu_{\eps})}(\eps,\cdot) = \tilde{\Theta}( \nu_{\eps} \SC_{\BE(\nu_{\eps})}(\eps,\cdot) )$,
and for $\s = \infty$ we have $\LC_{\BE(\nu_{\eps})}(\eps,\cdot) = \tilde{O}\left( \max\left\{\nu_{\eps}, \frac{\eps}{\nu_{\eps}} \right\} \SC_{\BE(\nu_{\eps})}(\eps,\cdot) \right)$.

Finally, in the case of agnostic noise, similarly to the Bernstein class condition, the gaps between the upper and lower
bounds in Theorem~\ref{thm:agnostic} render unclear precisely what types of improvements we can expect when $\s > \frac{1}{\nu+\eps}$,
ranging from the lower bound, which has the behavior described above for $\LC_{\BE(\nu)}$, to the upper bound, which reflects
no improvements over passive learning in this case.  When $\s < \frac{1}{\nu+\eps}$, the upper bound for active learning
reflects an improvement over the lower bound for passive learning by roughly a factor of $\frac{1}{(\nu+\eps)\s}$ (aside from logarithmic factors).
It remains an interesting open problem to determine whether the stronger improvements observed for benign noise
generally also hold for agnostic noise.

\paragraph{A remark on logarithmic factors:}
It is known that the terms of the form ``$\vc \Log( x )$'' in each of the above upper bounds for passive learning can be refined
to replace $x$ with the maximum of the disagreement coefficient (see Section~\ref{sec:dc} below) over the distributions in $\Dset$ 
\citep*{gine:06,hanneke:12b,hanneke:survey}.
Therefore, based on the results in Section~\ref{sec:dc} relating the disagreement coefficient to the star number, we can 
replace these ``$\vc \Log(x)$'' terms with ``$\vc \Log(\s \land x)$''.
In the case of $\BN(\bound)$, \citet*{massart:06} and \citet*{raginsky:11} have argued that, at least in some cases,
this logarithmic factor can also be included in the lower bounds.  It is presently not known whether this is the case
for the other noise models studied here.

\section{Connections to the Prior Literature on Active Learning}
\label{sec:complexities}

\begin{table}[t]
\centering
\begin{tabular}{|l|c|c|}
\hline
technique & source & relation to $\s$ \\ \hline\hline
disagreement coefficient & \citep*{hanneke:07b} & $\sup\limits_{P} \dc_{P}(\eps) = \s \land \frac{1}{\eps}$\\ \hline
splitting index & \citep*{dasgupta:05} & $\sup\limits_{h,P} \lim\limits_{\tau\to0} \left\lfloor \frac{1}{\rho_{h,P}(\eps;\tau)} \right\rfloor = \s \land \left\lfloor \frac{1}{\eps} \right\rfloor$ \\ \hline
teaching dimension & \citep*{hanneke:07a} & $\XTD(\C,m) = \s \land m$\\ \hline
version space compression & \citep*{el-yaniv:10} & $\max\limits_{h \in \C} \max\limits_{\U \in \X^{m}} \hat{n}_{h}(\U) = \s \land m$\\ \hline
doubling dimension & \citep*{long:07} & $\sup\limits_{h,P} \dd_{h,P}(\eps) \!\in\! [1,O(\vc)] \log\!\left(\s \land \frac{1}{\eps}\right)$\\ \hline
\end{tabular}
\caption{Many of the complexity measures from the literature are related to the star number.}
\label{tab:complexities}
\end{table}

As mentioned, there is already a substantial literature bounding the label complexities of various active
learning algorithms under various noise models.  
It is natural to ask how the results in the prior literature compare to those stated above.
However, as most of the prior results are $\PXY$-dependent, the appropriate comparison 
is to the worst-case values of those results: that is, maximizing the bounds over $\PXY$ in 
the respective noise model.
This section makes this comparison.  In particular, we will see that the label complexity upper bounds above
for $\RE$, $\BN(\bound)$, $\TN(\tsybca,\tsyba)$, and $\BE(\nu)$ all show some improvements over the 
known results, with the last two of these showing the strongest improvements.

The general results in the prior literature each express 
their label complexity bounds in terms of some kind of complexity measure.  There are now
several such complexity measures in use, each appropriate for studying some family of active learning
algorithms under certain noise models.  Most of these quantities are dependent on the distribution $\PXY$
or the data, and their definitions are quite diverse.  For some pairs of them, there are known inequalities
loosely relating them, while other pairs have defied attempts to formally relate the quantities.
The dependence on $\PXY$ in the general results in the prior literature is
typically isolated to the various complexity measures they are expressed in terms of.  Thus, the natural first
step is to characterize the worst-case values of these complexity measures, for any given
hypothesis class $\C$.  Plugging these worst-case values into the original bounds then allows
us to compare to the results stated above.

In the process of studying the worst-case behaviors of these complexity measures,
we also identify a \emph{very} interesting fact that has heretofore gone unnoticed: namely, that 
almost all of the complexity measures in the relevant prior literature on the label complexity of active learning 
are in fact \emph{equal} to the star number when maximized over the choice of distribution or data set.
In some sense, this fact is quite surprising, as this seemingly-eclectic collection of complexity measures includes disparate
definitions and interpretations, corresponding to entirely distinct approaches to the analysis of the respective 
algorithms these quantities are used to bound the label complexities of.
Thus, this equivalence is interesting in its own right; additionally, it plays an important role in our proofs 
of the main results above, since it allows us to build on these diverse techniques from the prior literature
when establishing these results.

Each subsection below is devoted to a particular complexity measure from the prior literature on active learning,
each representing an established technique for obtaining label complexity bounds.  Together, they represent a 
summary of the best-known general results from the prior literature relevant to our present discussion.  
In each case, we show the equivalence of the worst-case value of the complexity
measure to the star number, and then combine this fact with the known results to obtain the corresponding
bounds on the minimax label complexities implicit in the prior literature.  In each case, we then compare this result to those
obtained above.

We additionally study the \emph{doubling dimension}, a quantity which has been used to bound the sample
complexity of passive learning, and can be used to provide a loose bound on the label complexity of certain 
active learning algorithms.  Below we argue that, when maximized over the choice of distribution, the doubling
dimension can be upper and lower bounded in terms of the star number.  One immediate implication of these
bounds is that the doubling dimension is bounded if and only if the star number is finite.

Our findings on the relations of these various complexity measures to the star number are summarized in Table~\ref{tab:complexities}.

\subsection{The Disagreement Coefficient}
\label{sec:dc}

We begin with, what is perhaps the most well-studied complexity measure in the active learning literature: the \emph{disagreement coefficient} \citep*{hanneke:07b,hanneke:thesis}.

\begin{definition}
\label{def:dc}
For any $r_{0} \geq 0$, any classifier $h$, and any probability measure $\Px$ over $\X$,
the disagreement coefficient of $h$ with respect to $\C$ under $\Px$ is defined as
\begin{equation*}
\dc_{h,\Px}(r_{0}) = \sup_{r > r_{0}} \frac{\Px\left( \DIS\left( \Ball_{\Px}\left(h, r \right) \right) \right)}{r} \lor 1.
\end{equation*}
Also, for any probability measure $\PXY$ over $\X \times \Y$, letting $\Px$ denote the marginal distribution of $\PXY$ over $\X$,
and letting $h^{*}_{\PXY}$ denote a classifier with $\er_{\PXY}(h^{*}_{\PXY}) = \inf_{h \in \C} \er_{\PXY}(h)$ and $\inf_{h \in \C} \Px(x : h(x) \neq h^{*}_{\PXY}(x)) = 0$,\footnote{See \citet*{hanneke:12a} for a proof that such a classifier always exists (though not necessarily in $\C$).}
define the disagreement coefficient of the class $\C$ with respect to $\PXY$ as $\dc_{\PXY}(r_{0}) = \dc_{h^{*}_{\PXY},\Px}(r_{0})$.
\end{definition}

The disagreement coefficient is used to bound the label complexities of a family of active learning algorithms, described as \emph{disagreement-based}.
This line of work was initiated by \citet*{cohn:94}, who propose an algorithm effective in the realizable case.
That method was extended to be robust to label noise by \citet*{balcan:06,balcan:09}, 
which then inspired a slew of papers studying variants of this idea; the interested reader is referred to 
\citet*{hanneke:survey} for a thorough survey of this literature.
The general-case label complexity analysis of disagreement-based active learning (in terms of the disagreement coefficient) was initiated in the work of \citet*{hanneke:07b,hanneke:thesis},
and followed up by many papers since then \citep*[e.g.,][]{dasgupta:07,hanneke:09a,hanneke:11a,hanneke:12a,koltchinskii:10,hanneke:12b}, as well as many works
characterizing the value of the disagreement coefficient under various conditions \citep*[e.g.,][]{hanneke:07b,friedman:09,hanneke:10a,wang:11,balcan:13,hanneke:survey};
again, see \citet*{hanneke:survey} for a thorough survey of the known results on the disagreement coefficient.

To study the worst-case values of the label complexity bounds expressed in terms of the disagreement coefficient, let us define
\begin{equation*}
\supdc(\eps) = \sup_{\PXY} \dc_{\PXY}(\eps).
\end{equation*}
In fact, a result of \citet*[][Theorem 7.4]{hanneke:survey} implies that $\supdc(\eps) = \sup_{\Px} \sup_{h \in \C} \dc_{h,\Px}(\eps)$,
so that this would be an equivalent way to define $\supdc(\eps)$, which can sometimes be simpler to work with.
We can now express the bounds on the minimax label complexity implied by the best general results to date in the prior literature on
disagreement-based active learning \citep*[namely, the results of][]{hanneke:11a,dasgupta:07,koltchinskii:10,hanneke:12b,hanneke:survey},
summarized as follows (see the survey of \citealp*{hanneke:survey}, for detailed descriptions of the best-known logarithmic factors in these results).
\begin{itemize}
\item $\LC_{\RE}(\eps,\conf) \lesssim \supdc(\eps) \vc \cdot \polylog\left(\frac{1}{\eps\conf}\right)$.
\item $\LC_{\BN(\bound)}(\eps,\conf) \lesssim \frac{1}{(1-2\bound)^{2}} \supdc(\eps/(1-2\bound)) \vc \cdot \polylog\left(\frac{1}{\eps\conf}\right)$.
\item $\LC_{\TN(\tsybca,\tsyba)}(\eps,\conf) \lesssim \tsybca^{2} \left(\frac{1}{\eps}\right)^{2-2\tsyba} \supdc( \tsybca \eps^{\tsyba} ) \vc \cdot \polylog\left(\frac{1}{\eps\conf}\right)$.
\item $\LC_{\DL(\tsybca,\tsyba)}(\eps,\conf) \lesssim \tsybca^{2} \left(\frac{1}{\eps}\right)^{2-2\tsyba} \supdc( \tsybca \eps^{\tsyba} ) \vc \cdot \polylog\left(\frac{1}{\eps\conf}\right)$.
\item $\LC_{\BE(\nu)}(\eps,\conf) \lesssim \left(\frac{\nu^{2}}{\eps^{2}}+1\right) \supdc(\nu+\eps) \vc \cdot \polylog\left(\frac{1}{\eps\conf}\right)$.
\item $\LC_{\AG(\nu)}(\eps,\conf) \lesssim \left(\frac{\nu^{2}}{\eps^{2}}+1\right) \supdc(\nu+\eps) \vc \cdot \polylog\left(\frac{1}{\eps\conf}\right)$.
\end{itemize}

In particular, these bounds on $\LC_{\TN(\tsybca,\tsyba)}(\eps,\conf)$, $\LC_{\DL(\tsybca,\tsyba)}(\eps,\conf)$, $\LC_{\BE(\nu)}(\eps,\conf)$, and $\LC_{\AG(\nu)}(\eps,\conf)$
are the best general-case bounds on the label complexity of active learning in the prior literature (up to logarithmic factors), so that any improvements over these should be considered
an interesting advance in our understanding of the capabilities of active learning methods.
To compare these results to those stated in Section~\ref{sec:main}, we need to relate $\supdc(\eps)$ to the star number.
Interestingly, we find that these quantities are \emph{equal} (for $\eps = 0$).
Specifically, the following result describes the relation between these two quantities;
its proof is included in Appendix~\ref{app:dc}.
This connection also plays a role in the proofs of some of our results from Section~\ref{sec:main}.

\begin{theorem}
\label{thm:dc-star}
$\forall \eps \in (0,1]$, 
$\supdc(\eps) = \s \land \frac{1}{\eps}$ and $\supdc(0) = \s$.
\end{theorem}

With this result in hand, we immediately observe that several of the upper bounds from Section~\ref{sec:main} offer 
refinements over those stated in terms of $\supdc(\cdot)$ above.  For simplicity, we do not discuss differences in the logarithmic factors here.
Specifically, the upper bound on $\LC_{\RE}(\eps,\conf)$ in Theorem~\ref{thm:realizable} refines that stated here by replacing the factor 
$\supdc(\eps) \vc = \min\left\{ \s \vc, \frac{\vc}{\eps} \right\}$ with the sometimes-smaller factor $\min\left\{ \s, \frac{\vc}{\eps} \right\}$.  
Likewise, the upper bound on $\LC_{\BN(\bound)}(\eps,\conf)$ in Theorem~\ref{thm:bounded} refines the result stated here,
again by replacing the factor $\supdc(\eps/(1-2\bound)) \vc = \min\left\{ \s \vc, \frac{(1-2\bound) \vc}{\eps} \right\}$
with the sometimes-smaller factor $\min\left\{ \s, \frac{(1-2\bound)\vc}{\eps} \right\}$.
On the other hand, Theorem~\ref{thm:tsybakov} offers a much stronger refinement over the result stated above.
Specifically, in the case $\tsyba \leq 1/2$, the upper bound in Theorem~\ref{thm:tsybakov} completely \emph{eliminates}
the factor of $\supdc(\tsybca \eps^{\tsyba})$ from the upper bound on $\LC_{\TN(\tsybca,\tsyba)}(\eps,\conf)$ stated here
(i.e., replacing it with a universal constant).  For the case $\tsyba > 1/2$, the upper bound on $\LC_{\TN(\tsybca,\tsyba)}(\eps,\conf)$
in Theorem~\ref{thm:tsybakov} replaces this factor of $\supdc(\tsybca \eps^{\tsyba}) = \min\left\{ \s, \frac{1}{\tsybca \eps^{\tsyba}} \right\}$
with the factor $\min\left\{ \frac{\s}{\vc}, \frac{1}{\tsybca^{1/\tsyba} \eps} \right\}^{2\tsyba-1}$,
which is always smaller (for small $\eps$ and large $\vc$).
The upper bounds on $\LC_{\DL(\tsybca,\tsyba)}(\eps,\conf)$ and $\LC_{\AG(\nu)}(\eps,\conf)$ in Theorems \ref{thm:diameter-localization} and \ref{thm:agnostic}
are equivalent to those stated here; indeed, this is precisely how these results are obtained in Appendix~\ref{app:main-proofs}.
We have conjectured above that at least the dependence on $\vc$ and $\s$ can be refined, analogous to the refinements for the 
realizable case and bounded noise noted above.
However, we \emph{do} obtain refinements for the bound on $\LC_{\BE(\nu)}(\eps,\conf)$ in Theorem~\ref{thm:benign}, 
replacing the factor of $\left(\frac{\nu^{2}}{\eps^{2}} + 1\right) \supdc(\nu+\eps) \vc = \left(\frac{\nu^{2}}{\eps^{2}}+1\right)\min\left\{\s \vc, \frac{\vc}{\nu+\eps} \right\}$ in the upper bound here
with a factor $\frac{\nu^{2}}{\eps^{2}} \vc + \min\left\{\s,\frac{\vc}{\eps}\right\}$, which is sometimes significantly smaller (for $\eps \ll \nu \ll 1$ and large $\vc$).

\subsection{The Splitting Index}
\label{sec:splitting}

Another, very different, approach to the design and analysis of active learning algorithms
was proposed by \citet*{dasgupta:05}: namely, the \emph{splitting} approach.
In particular, this technique has the desirable property that it yields distribution-dependent
label complexity bounds for the realizable case which, even when the marginal distribution $\Px$ 
is held fixed, (almost) imply near-minimax performance.
The intuition behind this technique is that the objective in the realizable case (achieving error rate
at most $\eps$) is typically well-approximated by the related objective of reducing the \emph{diameter}
of the version space (set of classifiers consistent with the observed labels) to size at most $\eps$.
From this perspective, at any given time, the impediments to achieving this objective are clearly 
identifiable: pairs of classifiers $\{h,g\}$ in $\C$ consistent with all labels observed thus far,
yet with $\Px(x : h(x) \neq g(x)) > \eps$.  Supposing we have only a finite number of such classifiers
(which can be obtained if we first replace $\C$ by a fine-grained finite \emph{cover} of $\C$),
we can then estimate the \emph{usefulness} of a given point $X_{i}$ by the number of these
pairs it would be guaranteed to eliminate if we were to request its label (supposing the worse
of the two possible labels); by ``eliminate,'' we mean that at least one of the two classifiers 
will be inconsistent with the observed label.  If we always request labels of points guaranteed
to eliminate a large fraction of the surviving $\eps$-separated pairs, we will quickly arrive
at a version space of diameter $\eps$, and can then return any surviving classifier.
\citet*{dasgupta:05} further applies this strategy in tiers, first eliminating at least one classifier 
from every $\frac{1}{2}$-separated pair, then repeating this for the remaining $\frac{1}{4}$-separated 
pairs, and so on.  This allows the label complexity to be \emph{localized}, in the sense that
the surviving $\Delta$-separated pairs we need to eliminate will be composed of classifiers 
within distance $2\Delta$ of $\target_{\PXY}$ (or the representative thereof in the initial 
finite cover of $\C$).
The analysis of this method naturally leads to the following definition from \citet*{dasgupta:05}.

For any finite set $Q \subseteq \{\{h,g\} : h,g \in \C\}$ of unordered pairs of classifiers in $\C$,
for any $x \in \X$ and $y \in \Y$, let $Q_{x}^{y} = \{\{h,g\} \in Q : h(x) = g(x) = y\}$, and define
\begin{equation*}
\Split(Q,x) = |Q| - \max_{y \in \Y} |Q_{x}^{y}|.
\end{equation*}
This represents the number of pairs guaranteed to be eliminated (as described above) by 
requesting the label at a point $x$.
The splitting index is then defined as follows.

\begin{definition}
\label{def:splitting}
For any $\rho,\Delta,\tau \in [0,1]$, a set $\H \subseteq \C$ is said to be $(\rho,\Delta,\tau)$-splittable under a probability measure $\Px$ over $\X$ if,
for all finite $Q \subseteq \{\{h,g\} \subseteq \H : \Px(x : h(x) \neq g(x)) \geq \Delta \}$, 
\begin{equation*}
\Px(x : \Split(Q,x) \geq \rho |Q|) \geq \tau.
\end{equation*}
For any classifier $h : \X \to \Y$, any probability measure $\Px$ over $\X$, and any
$\eps,\tau \in [0,1]$, the \emph{splitting index} is defined as
\begin{equation*}
\rho_{h,\Px}(\eps;\tau) = \sup\left\{ \rho \in [0,1] : \forall \Delta \geq \eps, \Ball_{\Px}(h,4\Delta) \text{ is } (\rho,\Delta,\tau)\text{-splittable under } \Px \right\}.
\end{equation*}
\end{definition}

\citet*{dasgupta:05} proves a bound on the label complexity of a general active learning algorithm based on the above strategy, in the realizable case,
expressed in terms of the splitting index.  Specifically, for any $\tau > 0$, letting $\rho = \rho_{\target_{\PXY},\Px}(\eps/4;\tau)$,
\citet*{dasgupta:05} finds that for that algorithm to achieve error rate at most $\eps$ with probability at least $1-\conf$,
it suffices to use a number of label requests
\begin{equation}
\label{eqn:dasgupta-splitting-bound}
\frac{\vc}{\rho} \polylog\left(\frac{\vc}{\eps\conf\tau\rho}\right).
\end{equation}

The $\tau$ argument to $\rho_{h,\Px}(\eps;\tau)$ captures the trade-off between the number of label requests and the number of unlabeled samples available,
with smaller $\tau$ corresponding to the scenario where more unlabeled data are available, and a larger value of $\rho_{h,\Px}(\eps;\tau)$.  
Specifically, \citet*{dasgupta:05} argues that $\tilde{O}\left(\frac{\vc}{\tau \rho}\right)$ unlabeled samples suffice to achieve the above result.
In our present model, we suppose an abundance of unlabeled data, and as such, we are interested in the behavior for very small $\tau$.
However, note that the logarithmic factors in the above bound have an inverse dependence on $\tau$, so that taking $\tau$ too small can
potentially increase the value of the bound.  It is not presently known whether or not this is necessary (though intuitively it seems not to be).
However, for the purpose of comparison to our results in Section~\ref{sec:main}, we will ignore this logarithmic dependence on $1/\tau$,
and focus on the leading factor.  In this case, we are interested in the value $\lim\limits_{\tau \to 0} \rho_{h,\Px}(\eps;\tau)$.
Additionally, to convert \eqref{eqn:dasgupta-splitting-bound} into a distribution-free bound for the purpose of comparison to the results in Section~\ref{sec:main},
we should minimize this value over the choice of $\Px$ and $h \in \C$.
Formally, we are interested in the following quantity, defined for any $\eps \in [0,1]$.
\begin{equation*}
\infsplit(\eps) = \inf_{P} \inf_{h \in \C} \lim_{\tau \to 0} \rho_{h,P}(\eps;\tau).
\end{equation*}
In particular, in terms of this quantity, the maximum possible value of the bound \eqref{eqn:dasgupta-splitting-bound} for a given hypothesis class $\C$
is at least
\begin{equation*}
\frac{\vc}{\infsplit(\eps/4)} \polylog\left(\frac{\vc}{\eps\conf}\right).
\end{equation*}
To compare this to the upper bound in Theorem~\ref{thm:realizable}, we need to 
relate $\frac{1}{\infsplit(\eps)}$ to the star number.  Again, we find that these quantities are 
essentially \emph{equal} (as $\eps \to 0$), as stated in the following theorem.

\begin{theorem}
\label{thm:splitting-star}
$\forall \eps \in (0,1]$, 
$\left\lfloor \frac{1}{\infsplit(\eps)} \right\rfloor = \s \land \left\lfloor \frac{1}{\eps} \right\rfloor$.
\end{theorem}

The proof of this result is included in Appendix~\ref{app:splitting}.
We note that the inequalities $\s \land \left\lfloor\frac{1}{\eps}\right\rfloor \leq \left\lfloor \frac{1}{\infsplit(\eps)} \right\rfloor \leq \left\lfloor \frac{1}{\eps} \right\rfloor$ 
were already implicit in the original work of \citet*[][Corollary 3 and Lemma 1]{dasgupta:05}.
For completeness (and to make the connection explicit), we include these arguments in the proof given in Appendix~\ref{app:splitting},
along with our proof that $\left\lfloor \frac{1}{\infsplit(\eps)} \right\rfloor \leq \s$ (which was heretofore unknown).

Plugging this into the above bound, we see that the maximum possible value of the bound \eqref{eqn:dasgupta-splitting-bound} for a given hypothesis class $\C$
is at least
\begin{equation*}
\min\left\{ \s \vc, \frac{\vc}{\eps} \right\} \polylog\left(\frac{\vc}{\eps\conf}\right).
\end{equation*}
Note that the upper bound in Theorem~\ref{thm:realizable} refines this by reducing the first term in the ``$\min$'' from $\s \vc$ to simply $\s$.

\citet*{dasgupta:05} also argues for a kind of lower bound in terms of the splitting index, 
which was reformulated as a lower bound on the minimax label complexity (for a fixed $\Px$) 
in the realizable case by \citet*{hanneke:12c,hanneke:survey}.
In our present distribution-free style of analysis, the implication of that result is the following lower bound.
\begin{equation*}
\LC_{\RE}(\eps,\conf) \gtrsim \frac{1}{\infsplit(4\eps)}.
\end{equation*}
Based on Theorem~\ref{thm:splitting-star}, we see that the $\min\left\{ \s, \frac{1}{\eps} \right\}$ term in the lower bound
of Theorem~\ref{thm:realizable} follows immediately from this lower bound.  For completeness, in Appendix~\ref{app:main-proofs}, 
we directly prove this term in the lower bound, based on a more-direct argument than that used to establish the above lower bound.
We note, however, that \citet*[][Corollary 3]{dasgupta:05} also describes a technique for obtaining lower bounds, which is 
essentially equivalent to that used in Appendix~\ref{app:main-proofs} to obtain this term (and furthermore, makes use of 
a distribution-dependent version of the ``star'' idea).

The upper bounds of \citet*{dasgupta:05} have also been extended to the bounded noise setting.
In particular, \citet*{hanneke:12c} and \citet*{hanneke:survey} have proposed variants of the 
splitting approach, which are robust to bounded noise.  They have additionally bounded the label 
complexities of these methods in terms of the splitting index.  Similarly to the above discussion 
of the realizable case, the worst-case values of these bounds for any given hypothesis class $\C$ are larger 
than those stated in Theorem~\ref{thm:bounded} by factors related to the VC dimension (logarithmic
factors aside).  We refer the interested readers to these sources for the details of those bounds.

\subsection{The Teaching Dimension}
\label{sec:xtd}

Another quantity that has been used to bound the label complexity of certain active learning 
methods is the \emph{extended teaching dimension growth function}.  This quantity was introduced
by \citet*{hanneke:07a}, inspired by analogous notions used to tightly-characterize the query complexity 
of \emph{Exact} learning with membership queries \citep*{hegedus:95,hellerstein:96}. The term
\emph{teaching dimension} takes its name from the literature on Exact teaching \citep*{goldman:95},
where the teaching dimension characterizes the minimum number of well-chosen labeled data points
sufficient to guarantee that the only classifier in $\C$ consistent with these labels is the target function.
\citet*{hegedus:95} extends this to target functions not contained in $\C$, in which case the objective
is simply to leave at most one consistent classifier in $\C$; he refers to the minimum number of 
points sufficient to achieve this as the \emph{extended teaching dimension}, and argues that this
quantity can be used to characterize the minimum number of \emph{membership queries} by a 
learning algorithm sufficient to guarantee that the only classifier in $\C$ consistent with the 
returned labels is the target function (which is the objective in the \emph{Exact} learning model).

\citet*{hanneke:07a} transfers this strategy to the statistical setting studied here (where the objective
is only to obtain excess error rate $\eps$ with probability $1-\conf$, rather than exactly identifying
a target function).  That work introduces empirical versions of the teaching dimension and 
extended teaching dimension, and defines distribution-dependent bounds on these quantities.
It then proves upper and lower bounds on the label complexity in terms of these quantities.  
For our present purposes, we will be most-interested in a particular distribution-free
upper bound on these quantities, called the \emph{extended teaching dimension growth function},
also introduced by \citet*{hanneke:06,hanneke:07a}.  Since both this quantity and the star number are distribution-free, 
they can be directly compared.

We introduce these quantities formally as follows.
For any $m \in \nats \cup \{0\}$ and $S \in \X^{m}$, and for any $h : \X \to \Y$,
define the \emph{version space} $V_{S,h} = \{ g \in \C : \forall x \in S, g(x) = h(x) \}$ \citep*{mitchell:77}.
For any $m \in \nats$ and $\U \in \X^{m}$, 
let $\C[\U]$ denote an arbitrary subset of classifiers in $\C$ such that, $\forall h \in \C$, $|\C[\U] \cap V_{\U,h}| = 1$:
that is, $\C[\U]$ contains exactly one classifier from each equivalence class in $\C$ induced by the classifications of $\U$.
For any classifier $h : \X \to \Y$, define 
\begin{equation*}
\TD(h,\C[\U],\U) = \min\{ t \in \nats \cup \{0\} : \exists S \in \U^{t} \text{ s.t. } |V_{S,h} \cap \C[\U]| \leq 1 \},
\end{equation*}
the \emph{empirical teaching dimension} of $h$ on $\U$ with respect to $\C[\U]$.
Any $S \in \bigcup_{t} \U^{t}$ with 
$|V_{S,h} \cap \C[\U]| \leq 1$ is called a \emph{specifying set} for $h$ on $\U$ with respect to $\C[\U]$;
thus, $\TD(h,\C[\U],\U)$ is the size of a \emph{minimal specifying set} for $h$ on $\U$ with respect to $\C[\U]$.
Equivalently, $S \in \bigcup_{t} \U^{t}$ is a specifying set for $h$ on $\U$ with respect to $\C[\U]$ if and only if 
$\DIS(V_{S,h}) \cap \U = \emptyset$.
Also define $\TD(h,\C,m) = \max\limits_{\U \in \X^{m}} \TD(h,\C[\U],\U)$,
$\TD(\C,m) = \max\limits_{h \in \C} \TD(h,\C,m)$ (the \emph{teaching dimension growth function}),
and $\XTD(\C,m) = \max\limits_{h : \X \to \Y} \TD(h,\C,m)$ (the \emph{extended teaching dimension growth function}).

\citet*{hanneke:07a} proves two upper bounds on the label complexity of active learning relevant to our present discussion.
They are summarized as follows (see the original source for the precise logarithmic factors).\footnote{Here we have simplified
the arguments $m$ to the $\XTD(\C,m)$ instances compared to those of \citet*{hanneke:07a}, 
using monotonicity of $m \mapsto \XTD(\C,m)$, combined with the basic observation that $\XTD(\C,m k) \leq \XTD(\C,m) k$ for any integer $k \geq 1$.}
\begin{itemize}
\item[$\bullet$] $\LC_{\RE}(\eps,\conf) \lesssim \XTD\left( \C, \left\lceil \frac{1}{\eps} \right\rceil \right) \vc \cdot \polylog\left(\frac{\vc}{\eps\conf}\right)$.
\item[$\bullet$] $\LC_{\AG(\nu)}(\eps,\conf) \lesssim \left(\frac{\nu^{2}}{\eps^{2}}+1\right) \XTD\left( \C, \left\lceil \frac{1}{\nu+\eps} \right\rceil \right) \vc \cdot \polylog\left(\frac{\vc}{\eps\conf}\right)$.
\end{itemize}
Since $\BE(\nu) \subseteq \AG(\nu)$, we have the further implication that 
\begin{equation*}
\LC_{\BE(\nu)}(\eps,\conf) \lesssim \left(\frac{\nu^{2}}{\eps^{2}}+1\right) \XTD\left( \C, \left\lceil \frac{1}{\nu+\eps} \right\rceil \right) \vc \cdot \polylog\left(\frac{\vc}{\eps\conf}\right).
\end{equation*}
Additionally, by a refined argument of \citet*{hegedus:95}, the ideas of \citet*{hanneke:07a} can be applied (see \citealp*{hanneke:06,hanneke:thesis}) to show that
\begin{equation*}
\LC_{\RE}(\eps,\conf) \lesssim \frac{\XTD( \C, \lceil \vc/\eps \rceil )}{\log_{2}(\XTD(\C, \lceil \vc/\eps \rceil))} \vc \cdot \polylog\left(\frac{\vc}{\eps\conf}\right).
\end{equation*}

To compare these bounds to the results stated in Section~\ref{sec:main}, we will need to relate the quantity $\XTD(\C,m)$ to the star number.
Although it may not be obvious from a superficial reading of the definitions,
we find that these quantities are \emph{exactly equal} (as $m\to\infty$).  Thus, the
extended teaching dimension growth function is simply an alternative way of referring to the star number (and vice versa),
as they define the same quantity.\footnote{In this sense, the star number is not really a \emph{new} quantity
to the active learning literature, but rather a simplified definition for the already-familiar
extended teaching dimension growth function.}
This equivalence is stated formally in the following theorem, the proof of which is included in Appendix~\ref{app:xtd}.

\begin{theorem}
\label{thm:xtd}
$\forall m \in \nats$, $\XTD(\C, m) = \TD(\C,m) = \min\{\s, m\}$.
\end{theorem}

We note that the inequalities $\min\{\s,m\} \leq \TD(\C,m) \leq \XTD(\C,m) \leq m$ follow readily from previously-established facts about the teaching dimension.
For instance, \citet*{fan:12} notes that the teaching dimension of any class is at least the maximum degree of its one-inclusion graph;
applying this fact to $\C[\U]$ and maximizing over the choice of $\U \in \X^{m}$, this maximum degree becomes $\min\{\s,m\}$ (by definition of $\s$).
However, the inequality $\XTD(\C,m) \leq \s$ and the resulting fact that $\XTD(\C,m) = \TD(\C,m)$ are apparently new.

In fact, in the process of proving this theorem, we establish another remarkable fact:
that \emph{every} minimal specifying set is a star set.  This is stated formally
in the following lemma, the proof of which is also included in Appendix~\ref{app:xtd}.

\begin{lemma}
\label{lem:spec-star-set}
For any $h : \X \to \Y$, $m \in \nats$, and $\U \in \X^{m}$,
every minimal specifying set for $h$ on $\U$ with respect to $\C[\U]$
is a star set for $\C \cup \{h\}$ centered at $h$.
\end{lemma}

Using Theorem~\ref{thm:xtd}, we can now compare the results above to those in Section~\ref{sec:main}.
For simplicity, we will not discuss the differences in logarithmic factors here.
Specifically, Theorem~\ref{thm:realizable} refines the results here on $\LC_{\RE}(\eps,\conf)$,
replacing a factor of $\min\left\{ \XTD(\C,\lceil 1/\eps \rceil) \vc, \frac{\XTD(\C,\lceil \vc/\eps \rceil) \vc}{\log(\XTD(\C,\lceil \vc/\eps \rceil))} \right\} \approx \min\left\{ \s \vc, \frac{\vc}{\eps}, \frac{\s \vc}{\log(\s)}, \frac{\vc^{2}}{\eps \log(\vc/\eps)}  \right\}$ implied by the above results
with a factor of $\min\left\{ \s, \frac{\vc}{\eps}, \frac{\s \vc}{\log(\s)} \right\}$, thus reducing the first term in the ``$\min$'' by a factor of $\vc$
(though see below, as \citealp*{hanneke:14a}, have already shown this to be possible, directly in terms of $\XTD(\C,m)$).
Theorem~\ref{thm:xtd} further reveals that the above bound on $\LC_{\AG(\nu)}(\eps,\conf)$ is equivalent (up to logarithmic factors)
to that stated in Theorem~\ref{thm:agnostic}.  However, the bound on $\LC_{\BE(\nu)}(\eps,\conf)$ in Theorem~\ref{thm:benign} 
refines that implied above, replacing a factor $\left(\frac{\nu^{2}}{\eps^{2}}+1\right) \XTD\left(\C, \left\lceil \frac{1}{\nu+\eps} \right\rceil \right) \vc \approx \left(\frac{\nu^{2}}{\eps^{2}}+1\right) \min\left\{ \s \vc, \frac{\vc}{\nu+\eps}\right\}$
with a factor $\frac{\nu^{2}}{\eps^{2}}\vc + \min\left\{ \s, \frac{\vc}{\eps} \right\}$, which can be significantly smaller for $\eps \ll \nu \ll 1$ and large $\vc$.

\citet*{hanneke:06,hanneke:07a} also proves a \emph{lower bound} on the label complexity
of active learning in the realizable case, based on the following modification of 
the extended teaching dimension.
For any set $\H \subseteq \C$, classifier $h : \X \to \Y$, $m \in \nats$, $\U \in \X^{m}$, and $\conf \in [0,1]$,
define the \emph{partial teaching dimension} as
\begin{equation*}
\XPTD(h,\H[\U],\U,\conf) = \min\{ t \in \nats \cup \{0\} : \exists S \in \U^{t} \text{ s.t. } |V_{S,h} \cap \H[\U]| \leq \conf |\H[\U]|+1 \},
\end{equation*}
and let $\XPTD(\H,m,\conf) = \max\limits_{h : \X \to \Y} \max\limits_{\U \in \X^{m}} \XPTD(h,\H[\U],\U,\conf)$.
\citet*{hanneke:06,hanneke:07a} proves that
\begin{equation*}
\LC_{\RE}(\eps,\conf) \geq \max_{\H \subseteq \C} \XPTD\left(\H, \left\lceil \frac{1-\eps}{\eps} \right\rceil, \conf\right).
\end{equation*}
The following result relates this quantity to the star number.

\begin{theorem}
\label{thm:xptd}
$\forall m \in \nats$, $\forall \conf \in [0,1/2]$, 
\begin{equation*}
\left\lceil (1-2\conf) \min\{ \s, m \} \right\rceil
\leq \max_{\H \subseteq \C} \XPTD(\H,m,\conf) 
\leq \left\lceil \left(1-\frac{\conf}{1+\conf}\right) \min\{\s,m\} \right\rceil.
\end{equation*}
\end{theorem}

The proof is in Appendix~\ref{app:xtd}.
Note that, combined with the lower bound of \citet*{hanneke:06,hanneke:07a}, 
this immediately implies the part of the lower bound in Theorem~\ref{thm:realizable}
involving $\s$.  In Appendix~\ref{app:main-proofs}, we provide a direct proof for
this term in the lower bound, based on an argument similar to that of \citet*{hanneke:07a}.

\subsubsection{The Version Space Compression Set Size}
\label{sec:hatn}

More-recently, \citet*{el-yaniv:10,el-yaniv:12,hanneke:14a} have studied a quantity $\hat{n}_{h}(\U)$
(for a sequence $\U \in \bigcup_{m} \X^{m}$ and classifier $h$),
termed the minimal \emph{version space compression set size}, defined as the size of the 
smallest subsequence $S \subseteq \U$ for which $V_{S,h} = V_{\U,h}$.\footnote{The quantity 
studied there is defined slightly differently, but is easily seen to be equivalent to this definition.}

It is easy to see that, when $h \in \C$, the version space compression set size is equivalent to the empirical teaching dimension:
that is, $\forall h \in \C$, 
\begin{equation*}
\hat{n}_{h}(\U) = \TD(h,\C[\U],\U).
\end{equation*}
To see this, note that since $|V_{\U,h} \cap \C[\U]| = 1$, 
any $S \subseteq \U$ with $V_{S,h} = V_{\U,h}$ has $|V_{S,h} \cap \C[\U]| = 1$,
and hence is a specifying set for $h$ on $\U$ with respect to $\C[\U]$.
On the other hand, for any $S \subseteq \U$, we (always) have $V_{S,h} \supseteq V_{\U,h}$,
so that if $|V_{S,h} \cap \C[\U]| \leq 1$, then $V_{S,h} \cap \C[\U] \supseteq V_{\U,h} \cap \C[\U]$
and $|V_{S,h} \cap \C[\U]| \geq |V_{\U,h} \cap \C[\U]| = 1 \geq |V_{S,h} \cap \C[\U]|$, which together imply 
$V_{S,h} \cap \C[\U] = V_{\U,h} \cap \C[\U]$; thus, $V_{S,h} \subseteq \{ g \in \C : \forall x \in \U, g(x) = h(x) \} = V_{\U,h} \subseteq V_{S,h}$,
so that $V_{S,h} = V_{\U,h}$: that is, $S$ is a version space compression set.
Thus, in the case $h \in \C$, any version space compression set $S$ is a specifying set
for $h$ on $\U$ with respect to $\C[\U]$ and vice versa.
That $\hat{n}_{h}(\U) = \TD(h,\C[\U],\U)$ $\forall h \in \C$ follows immediately from this equivalence.

In particular, combined with Theorem~\ref{thm:xtd}, this implies that $\forall m \in \nats$,
\begin{equation}
\label{eqn:hatn-star}
\max_{\U \in \X^{m}} \max_{h \in \C} \hat{n}_{h}(\U) = \TD(\C,m) = \min\{\s,m\}.
\end{equation}

Letting $\hat{n}_{m} = \hat{n}_{\target_{\PXY}}(\{X_{1},\ldots,X_{m}\})$, 
\citet*{hanneke:14a} have shown that, in the realizable case, for the CAL active learning algorithm \citep*[proposed by][]{cohn:94} 
to achieve error rate at most $\eps$ with probability at least $1-\conf$, it suffices to use a budget $n$ of any size at least
\begin{equation*}
\max_{1 \leq m \leq M_{\eps,\conf}}
\hat{n}_{m} \cdot \polylog\left(\frac{1}{\eps\conf}\right),
\end{equation*}
where $M_{\eps,\conf} \lesssim \frac{1}{\eps}\left(\vc\Log\left(\frac{1}{\eps}\right)+\Log\left(\frac{1}{\conf}\right)\right)$ is a bound
on the sample complexity of passive learning by returning an arbitrary classifier in the version space
\citep*{vapnik:82,vapnik:98,blumer:89}.
They further provide a distribution-dependent bound (to remove the dependence on the data here)
based on confidence bounds on $\hat{n}_{m}$ (analogous to the aforementioned distribution-dependent
bounds on the empirical teaching dimension studied by \citealp*{hanneke:07a}).
For our purposes (distribution-free, data-independent bounds), we can simply take the maximum over possible data sets and possible $\target_{\PXY}$ functions,
so that the above bound becomes
\begin{multline*}
\max_{x_{1},x_{2},\ldots \in \X} \max_{h \in \C} \max_{1 \leq m \leq M_{\eps,\conf}} \hat{n}_{h}(\{x_{1},\ldots,x_{m}\}) \polylog\left(\frac{1}{\eps\conf}\right)
\\ = \TD\left(\C, M_{\eps,\conf}\right) \polylog\left(\frac{1}{\eps\conf}\right)
\lesssim \TD\left(\C, \left\lfloor \frac{\vc}{\eps} \right\rfloor \right) \polylog\left(\frac{1}{\eps\conf}\right).
\end{multline*}
Combining this with \eqref{eqn:hatn-star}, we find that the label complexity of CAL in the realizable case is at most
\begin{equation*}
\min\left\{\s, \frac{\vc}{\eps}\right\} \polylog\left(\frac{1}{\eps\conf}\right),
\end{equation*}
which matches the upper bound on the minimax label complexity from Theorem~\ref{thm:realizable} up to logarithmic factors.

\subsection{The Doubling Dimension}
\label{sec:doubling}

Another quantity of interest in the learning theory literature is the \emph{doubling dimension},
also known as the \emph{local metric entropy} \citep*{lecam:73,yang:99b,gupta:03,long:09}.
Specifically, for any set $\H$ of classifiers, a set of classifiers $\G$ is an $\eps$-cover of $\H$ 
(with respect to the $\Px(\DIS(\{\cdot,\cdot\}))$ pseudometric) if 
\begin{equation*}
\sup_{h \in \H} \inf_{g \in \G} \Px(x : g(x) \neq h(x)) \leq \eps.
\end{equation*}
Let $\covering(\eps,\H,\Px)$ denote the minimum cardinality $|\G|$ over all $\eps$-covers $\G$ of $\H$,
or else $\covering(\eps,\H,\Px) = \infty$ if no finite $\eps$-cover of $\H$ exists.
The doubling dimension (at $h$) is defined as follows.

\begin{definition}
\label{def:doubling}
For any $\eps \in (0,1]$, any probability measure $P$ over $\X$, and any classifier $h$, 
define 
\begin{equation*}
\dd_{h,P}(\eps) = \max_{r \geq \eps} \log_{2}\left( \covering\left(r/2, \Ball_{P}(h,r), P \right) \right).
\end{equation*}
\end{definition}

The quantity $\dd_{\eps} = \dd_{\target_{\PXY},\Px}(\eps)$ is known to be useful in bounding the sample complexity of passive learning.
Specifically, \citet*{long:07,long:09} have shown that there is a passive learning algorithm achieving
sample complexity $\lesssim \frac{\dd_{\eps/4}}{\eps} + \frac{1}{\eps} \log\left(\frac{1}{\conf}\right)$ 
for $\PXY \in \RE$.
Furthermore, though we do not go into the details here, by a combination of the ideas from \citet*{dasgupta:05}, \citet*{balcan:09}, and \citet*{hanneke:07b}, 
it is possible to show that a certain active learning algorithm achieves a label complexity $\lesssim 4^{\dd_{\eps}} \dd_{\eps} \cdot \polylog(\frac{1}{\eps\conf})$
for $\PXY \in \RE$, though this is typically a very loose upper bound.

To our knowledge, the question of the worst-case value of the doubling dimension for a given hypothesis class $\C$
has not previously been explored in the literature (though there is an obvious $O(\vc \log(1/\eps))$ upper bound).
Here we obtain upper and lower bounds on this worst-case value, expressed in terms of the star number.
While this relation generally has a wide range (roughly a factor of $\vc$), it does have the interesting implication 
that the doubling dimension is \emph{bounded} if and only if $\s < \infty$.  
Specifically, we have the following theorem, the proof of which is included in Appendix~\ref{app:doubling}.

\begin{theorem}
\label{thm:dd-star}
$\forall \eps \in (0,1/4]$,
$\max\left\{ \vc, \Log\left( \s \land \frac{1}{\eps} \right) \right\} \lesssim \sup\limits_{P} \sup\limits_{h \in \C} \dd_{h,P}(\eps) 
\lesssim \vc \Log\left( \s \land \frac{1}{\eps} \right)$.
\end{theorem}

One can show that the gap between the upper and lower bounds on $\sup_{P} \sup_{h \in \C} \dd_{h,P}(\eps)$ in this result 
cannot generally be improved by much without sacrificing generality or introducing additional quantities.
Specifically, for the class $\C$ discussed in Appendix~\ref{app:lb-tight}, we have
$\sup_{P} \sup_{h \in \C} \dd_{h,P}(\eps)$ $\leq \sup_{P} \log_{2}( \covering(\eps/2,\C,P) ) \lesssim \max\left\{ \vc, \Log\left(\s \land \frac{1}{\eps}\right) \right\}$,
so that the lower bound above is sometimes tight to within a universal constant factor.
For the class $\C$ discussed in Appendix~\ref{app:ub-tight}, based on a result of \citet*[][Lemma 4]{raginsky:11},
one can show $\sup_{P} \sup_{h \in \C} \dd_{h,P}(\eps) \gtrsim \vc \Log\left( \frac{\s}{\vc} \land \frac{1}{\eps} \right)$, 
so that the above upper bound is sometimes tight, aside from a small difference in the logarithmic factor (dividing $\s$ by $\vc$).

Interestingly, in the process of proving the upper bound in Theorem~\ref{thm:dd-star}, we also establish the following 
inequality relating the doubling dimension and the disagreement coefficient, holding for any classifier $h$, any probability measure $\Px$ over $\X$, and any $\eps \in (0,1]$.
\begin{equation*}
\dd_{h,\Px}(\eps) \leq 2 \vc \log_{2}\left( 22 e^{2} \dc_{h,\Px}(\eps) \right).
\end{equation*}
This inequality may be of independent interest, as it enables comparisons between
results in the literature expressed in terms of these quantities.  For instance, 
it implies that in the realizable case, the passive learning sample complexity 
bound of \citet*{long:09} is no larger than that of \citet*{gine:06} (aside 
from constant factors).

\section{Conclusions}
\label{sec:conclusions}

In this work, we derived upper and lower bounds on the minimax label complexity of active 
learning under several noise models.  In most cases, these new bounds offer refinements
over the best results in the prior literature.  Furthermore, in the case of Tsybakov
noise, we discovered the heretofore-unknown fact that the minimax label complexity of 
active learning with VC classes is \emph{always} smaller than that of passive learning.
We expressed each of these bounds in terms of a simple combinatorial complexity measure,
termed the \emph{star number}.  We further found that almost all of the distribution-dependent
and sample-dependent complexity measures in the prior active learning literature are exactly
equal to the star number when maximized over the choice of distribution or data set.

The bounds derived here are all distribution-free, in the sense that they are expressed without
dependence or restrictions on the marginal distribution $\Px$ over $\X$.  They are also worst-case 
bounds, in the sense that they express the maximum of the label complexity over the distributions
in the noise model $\Dset$, rather than expressing a bound on the label complexity achieved by
a given algorithm as a function of $\PXY$.  As observed by \citet*{dasgupta:05}, there are some 
cases in which smaller label complexities can be achieved under restrictions on the marginal 
distribution $\Px$, and some cases in which there are achievable label complexities which 
exhibit a range of values depending on $\PXY$ \citep*[see also][for further exploration of this]{hanneke:10a,hanneke:12a}.
Our results reveal that in some cases, such as Tsybakov noise with $\tsyba \leq 1/2$,
these issues might typically not be of much significance (aside from logarithmic factors).
However, in other cases, particularly when $\s = \infty$, the issue of expressing 
distribution-dependent bounds on the label complexity is clearly an important one.  
In particular, the question of the minimax label complexity of active learning under the
restrictions of the above noise models that explicitly fix the marginal distribution 
$\Px$ remains an important and challenging open problem.  In deriving such bounds, the 
present work should be considered a kind of guide, in that we should restrict our focus to 
deriving distribution-dependent label complexity bounds with worst-case values that are 
never worse than the distribution-free bounds proven here.

\appendix

\section{Preliminary Lemmas}
\label{app:proofs}

Before presenting the proofs of the main results above, 
we begin by introducing some basic lemmas, which will
be useful in the main proofs below.

\subsection{$\boldsymbol{\eps}$-nets and $\boldsymbol{\eps}$-covers}
For a collection $\T$ of measurable subsets of $\X$, a value $\eps \geq 0$, and a probability measure $\Px$ on $\X$,
we say a set $N \subseteq \X$ is an $\eps$-net of $\Px$ for $\T$ if
$N \cap A \neq \emptyset$ for every $A \in \T$ with $\Px(A) > \eps$ \citep*{haussler:87}.
Also, a finite set $\H$ of classifiers is called an $\eps$-cover of $\C$ (under the $\Px(\DIS(\{\cdot,\cdot\}))$ pseudometric)
if $\sup_{g \in \C} \min_{h \in \H} \Px( x : h(x) \neq g(x) ) \leq \eps$.

The following lemma
bounds the probabilities and empirical probabilities of sets in a collection 
in terms of each other.
This result is based on the work of \citet*{vapnik:74}
(see also \citealp*[][Theorem A.3]{vapnik:82});
this version is taken from \citet*[][Theorem 7]{bousquet:04}, in combination with the VC-Sauer Lemma \citep*{vapnik:71,sauer:72} and a union bound.

\begin{lemma}
\label{lem:vc-relative}
For any collection $\T$ of measurable subsets of $\X$, letting $k$ denote the VC dimension of $\T$,
for any $\conf \in (0,1)$,
for any integer $m > k$,
for any probability measure $\Px$ over $\X$, if $X_{1}^{\prime},\ldots,X_{m}^{\prime}$ are independent $\Px$-distributed random variables,
then with probability at least $1-\conf$, it holds that $\forall A \in \T$, 
letting $\hat{\Px}(A) = \frac{1}{m} \sum_{i=1}^{m} \ind_{A}(X_{i}^{\prime})$,
\begin{align*}
& \Px(A) \leq \hat{\Px}(A) + 2 \sqrt{ \Px(A) \frac{k \Log\left(\frac{2em}{k}\right) + \Log\left(\frac{8}{\conf}\right)}{m} }
\\ \text{and } & \hat{\Px}(A) \leq \Px(A) + 2 \sqrt{ \hat{\Px}(A) \frac{k \Log\left(\frac{2em}{k}\right) + \Log\left(\frac{8}{\conf}\right)}{m} }.
\end{align*}
\end{lemma}

In particular, with a bit of algebra, this implies the following corollary.

\begin{corollary}
\label{cor:vc-relative}
There exists a finite universal constant $c_{0} \geq 1$ such that,
for any collection $\T$ of measurable subsets of $\X$, letting $k$ denote the VC dimension of $\T$,
for any $\eps,\conf \in (0,1)$,
for any integer $m \geq \frac{c_{0}}{\eps} \left( k \Log\left(\frac{1}{\eps}\right) + \Log\left(\frac{1}{\conf}\right) \right)$,
for any probability measure $\Px$ over $\X$, if $X_{1}^{\prime},\ldots,X_{m}^{\prime}$ are independent $\Px$-distributed random variables,
then with probability at least $1-\conf$, it holds that $\forall A \in \T$, 
letting $\hat{\Px}(A) = \frac{1}{m} \sum_{i=1}^{m} \ind_{A}(X_{i}^{\prime})$,
\begin{itemize}
\item[$\bullet$] $\hat{\Px}(A) \leq \frac{3}{4} \eps \implies \Px(A) < \eps$.
\item[$\bullet$] $\Px(A) \leq \frac{1}{2} \eps \implies \hat{\Px}(A) < \frac{3}{4} \eps$.
\end{itemize}
\end{corollary}
\begin{proof}
Let $\Epsilon(m) = 4\frac{k \Log\left(\frac{2em}{k}\right) + \Log\left(\frac{8}{\conf}\right)}{m}$,
and note that for $m \geq \frac{c_{0}}{\eps} \left( k \Log\left(\frac{1}{\eps}\right)+\Log\left(\frac{1}{\conf}\right)\right)$,
\begin{equation}
\label{eqn:cor-vc-relative-Epsilon-bound}
\Epsilon(m) 
\leq \frac{4 \eps}{c_{0}} \frac{k \Log\left(\frac{2e c_{0}}{k \eps} \left( k \Log\left( \frac{1}{\eps}\right) + \Log\left(\frac{1}{\conf}\right)\right)\right) + \Log\left(\frac{8}{\conf}\right)}{k \Log\left(\frac{1}{\eps}\right) + \Log\left(\frac{1}{\conf}\right)}
\end{equation}
If $k \Log\left(\frac{1}{\eps}\right) \geq \Log\left(\frac{1}{\conf}\right)$, then 
\begin{align*}
& k \Log\left( \frac{2ec_{0}}{k\eps} \left( k \Log\left(\frac{1}{\eps}\right) + \Log\left(\frac{1}{\conf}\right)\right)\right) + \Log\left(\frac{8}{\conf}\right)
\\ & \leq k \Log\left( \frac{4ec_{0}}{\eps} \Log\left(\frac{1}{\eps}\right)\right) + \Log\left(\frac{8}{\conf}\right)
\leq k \Log\left( \frac{4ec_{0}}{\eps^{2}}\right) + \Log\left(\frac{8}{\conf}\right)
\\ & \leq 2 k \Log\!\left(\frac{1}{\eps}\right) + k \Log\!\left( 4 e c_{0} \right) + \Log(8) + \Log\!\left(\frac{1}{\conf}\right)
\leq \Log\!\left( 32 e^{3} c_{0} \right) \left( k \Log\!\left(\frac{1}{\eps}\right) + \Log\!\left(\frac{1}{\conf}\right) \right)\!.
\end{align*}
Otherwise, if $k \Log\left( \frac{1}{\eps} \right) < \Log\left( \frac{1}{\conf} \right)$, then 
\begin{align*}
& k \Log\left( \frac{2ec_{0}}{k\eps} \left( k \Log\left(\frac{1}{\eps}\right) + \Log\left(\frac{1}{\conf}\right)\right)\right) + \Log\left(\frac{8}{\conf}\right)
\\ & \leq k \Log\left( \frac{4ec_{0}}{k\eps} \Log\left(\frac{1}{\conf}\right) \right) + \Log\left(\frac{8}{\conf}\right)
\leq k \Log\left( \frac{4ec_{0}}{\eps} \right) + k \Log\left(\frac{1}{k}\Log\left(\frac{1}{\conf}\right) \right) + \Log\left(\frac{8}{\conf}\right),
\end{align*}
and since $x \mapsto x \Log\left( \frac{1}{x}  \Log\left( \frac{1}{\conf} \right) \right)$ is nondecreasing for $x > 0$,
and $k \leq k \Log\left(\frac{1}{\eps}\right) \leq \Log\left(\frac{1}{\conf}\right)$, the above is at most
\begin{align*}
& k \Log\left( \frac{4ec_{0}}{\eps} \right) + \Log\left(\frac{1}{\conf}\right) + \Log\left(\frac{8}{\conf}\right)
\\ & \leq k \Log\!\left( \frac{1}{\eps} \right) + k\Log(4ec_{0}) + \Log(8) + 2\Log\!\left(\frac{1}{\conf}\right)
\leq \Log\!\left(32e^{2}c_{0}\right) \!\left( k \Log\!\left( \frac{1}{\eps} \right) + \Log\!\left(\frac{1}{\conf}\right) \right)\!.
\end{align*}
In either case, we have that the right hand side of \eqref{eqn:cor-vc-relative-Epsilon-bound} is at most
$\frac{4 \eps}{c_{0}} \Log\left(32 e^{3} c_{0} \right)$.
In particular, taking $c_{0} = 2^{14}$ suffices to make $\frac{4}{c_{0}} \Log\left(32 e^{3} c_{0} \right) \leq \frac{1}{64}$,
so that \eqref{eqn:cor-vc-relative-Epsilon-bound} implies
$\Epsilon(m) \leq \frac{\eps}{64}$.

Lemma~\ref{lem:vc-relative} implies that with probability at least $1-\conf$, every $A \in \T$ has
\begin{equation*}
\Px(A) \leq \hat{\Px}(A) + \sqrt{ \Px(A) \Epsilon(m) }
\end{equation*}
and 
\begin{equation*}
\hat{\Px}(A) \leq \Px(A) + \sqrt{ \hat{\Px}(A) \Epsilon(m) }.
\end{equation*}
Solving these quadratic expressions in $\sqrt{\Px(A)}$ and $\sqrt{\hat{\Px}(A)}$, respectively, we have
\begin{equation}
\label{eqn:cor-vc-relative-1}
\Px(A) \leq \hat{\Px}(A) + \frac{1}{2} \Epsilon(m) + \frac{1}{2} \sqrt{ \Epsilon(m)^{2} + 4 \Epsilon(m) \hat{\Px}(A) }
\end{equation}
and 
\begin{equation}
\label{eqn:cor-vc-relative-2}
\hat{\Px}(A) \leq \Px(A) + \frac{1}{2} \Epsilon(m) + \frac{1}{2} \sqrt{ \Epsilon(m)^{2} + 4 \Epsilon(m) \Px(A) }.
\end{equation}
Therefore, if $\hat{\Px}(A) \leq \frac{3}{4} \eps$, then \eqref{eqn:cor-vc-relative-1} implies
\begin{align*}
\Px(A) 
& \leq \frac{3}{4} \eps + \frac{1}{2} \Epsilon(m) + \frac{1}{2} \sqrt{ \Epsilon(m)^{2} + 3 \Epsilon(m) \eps }
\\ & \leq \left(\frac{3}{4} + \frac{1}{128} + \frac{1}{2} \sqrt{ \frac{1}{64^{2}} + \frac{3}{64}} \right) \eps
< \left(\frac{3}{4} + \frac{1}{128} + \frac{1}{8}\right) \eps
< \eps,
\end{align*}
and likewise, if $\Px(A) \leq \frac{1}{2} \eps$, then \eqref{eqn:cor-vc-relative-2} implies
\begin{align*}
\hat{\Px}(A) 
& \leq \frac{1}{2} \eps + \frac{1}{2} \Epsilon(m) + \frac{1}{2} \sqrt{ \Epsilon(m)^{2} + 2 \Epsilon(m) \eps }
\\ & \leq \left( \frac{1}{2} + \frac{1}{128} + \frac{1}{2} \sqrt{ \frac{1}{64^{2}} + \frac{1}{32} } \right) \eps
< \left( \frac{1}{2} + \frac{1}{128} + \frac{1}{8} \right) \eps
< \frac{3}{4} \eps.
\end{align*}
\end{proof}

We will be interested in applying these results to the collection of sets 
$\{\DIS(\{h,g\}) : h,g \in \C\}$.
For this, the following lemma of \citet*[][Theorem 4.5]{vidyasagar:03} will be useful.

\begin{lemma}
\label{lem:vc-of-pairs}
The VC dimension of the collection $\{\DIS(\{h,g\}) : h,g \in \C\}$ is at most $10\vc$.
\end{lemma}

Together, these results imply the following lemma (see also \citealp*{vapnik:74,vapnik:82,blumer:89,haussler:87}).

\begin{lemma}
\label{lem:vc-cover}
There exists a finite universal constant $c \geq 1$ such that,
for any $\eps,\conf \in (0,1)$,
for any integer $m \geq \frac{c}{\eps} \left( \vc \Log\left(\frac{1}{\eps}\right) + \Log\left(\frac{1}{\conf}\right) \right)$,
for any probability measure $\Px$ over $\X$, if $X_{1}^{\prime},\ldots,X_{m}^{\prime}$ are independent $\Px$-distributed random variables,
then with probability at least $1-\conf$, 
it holds that $\forall h,g \in \C$, 
if $(g(X_{1}^{\prime}),\ldots,g(X_{m}^{\prime})) = (h(X_{1}^{\prime}),\ldots,h(X_{m}^{\prime}))$,
then $\Px(x : g(x) \neq h(x)) \leq \eps$.
\\In particular, this implies that with probability at least $1-\conf$,
letting $\C[(X_{1}^{\prime},\ldots,X_{m}^{\prime})]$ be as in Section~\ref{sec:xtd}, 
$\C[(X_{1}^{\prime},\ldots,X_{m}^{\prime})]$ is an $\eps$-cover of $\C$ (under the $\Px(\DIS(\{\cdot,\cdot\}))$ pseudometric),
and $\{X_{1}^{\prime},\ldots,X_{m}^{\prime}\}$ is an $\eps$-net of $\Px$ for $\{\DIS(\{h,g\}) : h,g \in \C\}$.
\end{lemma}
\begin{proof}
Let $c_{0}$ be as in Corollary~\ref{cor:vc-relative}, and let $k$ denote the VC dimension of
$\{ \DIS(\{h,g\}) : h,g \in \C \}$.
Corollary~\ref{cor:vc-relative} implies that, if $m \geq \frac{c_{0}}{\eps} \left( k \Log\left(\frac{1}{\eps}\right) + \Log\left(\frac{1}{\conf}\right) \right)$, 
then there is an event $E$ of probability at least $1-\conf$, on which every $h,g \in \C$ with $\sum_{t=1}^{m} \ind_{\DIS(\{h,g\})}(X_{t}^{\prime}) = 0$ satisfy
$\Px(\DIS(\{h,g\})) < \eps$; in particular, this proves that on the event $E$,
$\{X_{1}^{\prime},\ldots,X_{m}^{\prime}\}$ is an $\eps$-net of $\Px$ for $\{ \DIS(\{h,g\}) : h,g \in \C \}$.
Furthermore, by definition of $\C[(X_{1}^{\prime},\ldots,X_{m}^{\prime})]$,
for every $h \in \C$, $\exists g \in \C[(X_{1}^{\prime},\ldots,X_{m}^{\prime})]$ with $\sum_{t=1}^{m} \ind_{\DIS(\{h,g\})}(X_{t}^{\prime}) = 0$,
which (on the event $E$) therefore also satisfies $\Px(\DIS(\{h,g\})) < \eps$.  Thus, on the event $E$, 
$\C[(X_{1}^{\prime},\ldots,X_{m}^{\prime})]$ is an $\eps$-cover of $\C$ (under the $\Px(\DIS(\{\cdot,\cdot\}))$ pseudometric).
To complete the proof, we note that Lemma~\ref{lem:vc-of-pairs} implies $k \leq 10 \vc$, so that by choosing $c = 10 c_{0}$, 
the condition $m \geq \frac{c_{0}}{\eps} \left( k \Log\left(\frac{1}{\eps}\right) + \Log\left(\frac{1}{\conf}\right) \right)$ 
will be satisfied for any $m \geq \frac{c}{\eps}\left( \vc \Log\left(\frac{1}{\eps}\right) + \Log\left(\frac{1}{\conf}\right) \right)$.
\end{proof}

Based on this result, it is straightforward to construct an $\eps$-net of $\Px$ for $\{\DIS(\{h,g\}) : h,g \in \C\}$
of size $\lesssim \frac{\vc}{\eps} \Log\left(\frac{1}{\eps}\right)$, based on a relatively small number of random samples.
Specifically, we have the following lemma.

\begin{lemma}
\label{lem:eps-net}
There exists a finite universal constant $c^{\prime} \geq 1$ such that, 
for any probability measure $\Px$ on $\X$, 
if $X_{1}^{\prime},X_{2}^{\prime},\ldots$ are independent $\Px$-distributed random variables,
then $\forall \eps,\conf \in (0,1)$, 
for any integers $m \geq \frac{c^{\prime} \vc}{\eps} \Log\left( \frac{1}{\eps} \right)$
and $\ell \geq \frac{c^{\prime}}{\eps}\left( \vc \Log\left( \frac{1}{\eps} \right) + \Log\left( \frac{1}{\conf} \right) \right)$,
defining $N_{i} = \{X_{m(i-1)+1}^{\prime},\ldots,X_{mi}^{\prime}\}$ for each $i \in \{1,\ldots,\lceil \log_{2}(2/\conf) \rceil \}$,
letting
\begin{multline*}
\hat{i} = \argmin_{i \in \{1,\ldots,\lceil \log_{2}(2/\conf) \rceil\}} 
\max\vast\{  \sum_{j = m \lceil \log_{2}(2/\conf) \rceil + 1}^{m \lceil \log_{2}(2/\conf) \rceil + \ell} \ind_{\DIS(\{h,g\})}(X_{j}^{\prime}) : 
\\ h,g \in \C, \sum_{j = m (i-1) + 1}^{m i} \ind_{\DIS(\{h,g\})}(X_{j}^{\prime}) = 0 \vast\},
\end{multline*}
and $\hat{N} = N_{\hat{i}}$,
with probability at least $1-\conf$,
$\hat{N}$ is an $\eps$-net of $\Px$ for $\{\DIS(\{h,g\}) : h,g \in \C\}$.
\end{lemma}
\begin{proof}
Let $k$ denote the VC dimension of the collection of sets $\{\DIS(\{h,g\}) : h,g \in \C\}$.
Letting $c_{0}$ be as in Corollary~\ref{cor:vc-relative}, taking $c^{\prime} \geq 10 c_{0}$, we have
$\ell \geq \frac{c_{0}}{\eps} \left( 10 \vc \Log\left(\frac{1}{\eps}\right) + \Log\left(\frac{2}{\conf}\right) \right)$,
which is at least $\frac{c_{0}}{\eps} \left( k \Log\left(\frac{1}{\eps}\right) + \Log\left(\frac{2}{\conf}\right) \right)$
by Lemma~\ref{lem:vc-of-pairs}.  Therefore, Corollary~\ref{cor:vc-relative} implies there exists an event $E^{\prime}$
of probability at least $1-\conf/2$ such that, on $E^{\prime}$, $\forall h,g \in \C$,
\begin{align}
\sum_{m \lceil \log_{2}(2/\conf) \rceil + 1}^{m \lceil \log_{2}(2/\conf) \rceil + \ell} \ind_{\DIS(\{h,g\})}(X_{j}^{\prime}) \leq \frac{3}{4} \eps \ell
& \implies \Px(\DIS(\{h,g\})) \leq \eps, \label{eqn:eps-net-Eprime-1}
\\ \Px(\DIS(\{h,g\})) \leq \frac{\eps}{2} & \implies \sum_{m \lceil \log_{2}(2/\conf) \rceil + 1}^{m \lceil \log_{2}(2/\conf) \rceil + \ell} \ind_{\DIS(\{h,g\})}(X_{j}^{\prime}) \leq \frac{3}{4} \eps \ell. \label{eqn:eps-net-Eprime-2}
\end{align}

Let $c$ be as in Lemma~\ref{lem:vc-cover}.
Taking $c^{\prime} \geq 6c$, we have $m \geq \frac{2c}{\eps} \left( \vc \Log\left(\frac{2}{\eps}\right) + \Log\left(2\right)\right)$,
so that Lemma~\ref{lem:vc-cover} implies that, for each $i \in \left\{ 1,\ldots, \lceil \log_{2}(2/\conf) \rceil \right\}$, 
$N_{i}$ is an $\frac{\eps}{2}$-net of $\Px$ for $\{ \DIS(\{h,g\}) : h,g \in \C \}$ with probability at least $1/2$.
Since the $N_{i}$ sets are independent, there is an event $E$ of probability at least $1 - (1-1/2)^{\lceil \log_{2}(2/\conf) \rceil} \geq 1-\conf/2$,
on which $\exists i^{*} \in \left\{1,\ldots,\lceil \log_{2}(2/\conf) \rceil \right\}$ such that $N_{i^{*}}$ is an $\frac{\eps}{2}$-net of $\Px$ for $\{ \DIS(\{h,g\}) : h,g \in \C \}$.
In particular, this implies that on $E$, 
\begin{equation}
\label{eqn:eps-net-istar-probs}
\sup\left\{ \Px(\DIS(\{h,g\})) : h,g \in \C, \sum_{j=m (i^{*}-1)+1}^{m i^{*}} \ind_{\DIS(\{h,g\})}(X_{j}^{\prime}) = 0 \right\}
\leq \frac{\eps}{2}.
\end{equation}

Therefore, on the event $E^{\prime} \cap E$, we have
\begin{align*}
& \max\left\{ \sum_{j= m\lceil \log_{2}(2/\conf) \rceil + 1}^{m \lceil \log_{2}(2/\conf) \rceil + \ell} \ind_{\DIS(\{h,g\})}(X_{j}^{\prime}) : h,g \in \C, \sum_{j=m (\hat{i}-1)+1}^{m\hat{i}} \ind_{\DIS(\{h,g\})}(X_{j}^{\prime}) = 0 \right\}
\\ & \leq \max\left\{ \sum_{j= m\lceil \log_{2}(2/\conf) \rceil + 1}^{m \lceil \log_{2}(2/\conf) \rceil + \ell} \ind_{\DIS(\{h,g\})}(X_{j}^{\prime}) : h,g \in \C, \sum_{j=m (i^{*}-1)+1}^{m i^{*}} \ind_{\DIS(\{h,g\})}(X_{j}^{\prime}) = 0 \right\}
\leq \frac{3}{4}\eps\ell,
\end{align*}
where the first inequality is by definition of $\hat{i}$, and the second inequality is by a combination of \eqref{eqn:eps-net-istar-probs} with \eqref{eqn:eps-net-Eprime-2}.
Therefore, by \eqref{eqn:eps-net-Eprime-1}, on the event $E^{\prime} \cap E$, we have
\begin{equation*}
\max\left\{ \Px(\DIS(\{h,g\})) : h,g \in \C, \sum_{j=m (\hat{i}-1)+1}^{m\hat{i}} \ind_{\DIS(\{h,g\})}(X_{j}^{\prime}) = 0 \right\} \leq \eps,
\end{equation*}
or equivalently, $N_{\hat{i}}$ is an $\eps$-net of $\Px$ for $\{\DIS(\{h,g\}) : h,g \in \C\}$.
To complete the proof, we take $c^{\prime} = \max\{ 10 c_{0}, 6c \}$, 
and note that the event $E^{\prime} \cap E$ has probability at least $1-\conf$ by a union bound.
\end{proof}

There are also variants of the above two lemmas applicable to sample compression schemes.
Specifically, the next lemma is due to \citet*{littlestone:86,floyd:95}.

\begin{lemma}
\label{lem:compression-bound}
There exists a finite universal constant $\tilde{c} \geq 1$ such that,
for any collection $\T$ of measurable subsets of $\X$,
any $n \in \nats \cup \{0\}$, and any
function $\phi_{n} : \X^{n} \to \T$, for any $\eps,\conf \in (0,1)$, for any integer 
$m \geq \frac{\tilde{c}}{\eps}\left( n \Log\left(\frac{1}{\eps}\right) + \Log\left(\frac{1}{\conf}\right)\right)$,
for any probability measure $\Px$ over $\X$, if $X_{1}^{\prime},\ldots,X_{m}^{\prime}$ are independent $\Px$-distributed random variables,
then with probability at least $1-\conf$, it holds that every $i_{1},\ldots,i_{n} \in \{1,\ldots,m\}$ with $i_{1} \leq \cdots \leq i_{n}$ and
$\{X_{1}^{\prime},\ldots,X_{m}^{\prime}\} \cap \phi_{n}(X_{i_{1}}^{\prime},\ldots,X_{i_{n}}^{\prime}) = \emptyset$ has 
$\Px\left( \phi_{n}(X_{i_{1}}^{\prime},\ldots,X_{i_{n}}^{\prime}) \right) \leq \eps$: that is, 
$\{X_{1}^{\prime},\ldots,X_{m}^{\prime}\}$ is an $\eps$-net of $\Px$ for $\{ \phi_{n}(X_{i_{1}}^{\prime},\ldots,X_{i_{n}}^{\prime}) : i_{1},\ldots,i_{n} \in \{1,\ldots,m\}, i_{1} \leq \cdots \leq i_{n}\}$.
\end{lemma}

This implies the following result.

\begin{lemma}
\label{lem:compression-eps-net}
There exists a finite universal constant $\tilde{c}^{\prime} \geq 1$ such that, 
for any collection $\T$ of measurable subsets of $\X$,
any $n \in \nats$, and any function $\phi_{n} : \X^{n} \times \Y^{n} \to \T$,
for any probability measure $\Px$ on $\X$, 
if $X_{1}^{\prime},X_{2}^{\prime},\ldots$ are independent $\Px$-distributed random variables,
then for any $\eps,\conf \in (0,1)$, 
for any integers $m \geq \frac{\tilde{c}^{\prime} n }{\eps} \Log\left( \frac{1}{\eps} \right)$
and $\ell \geq \frac{\tilde{c}^{\prime}}{\eps}\left( n \Log\left( \frac{m}{n} \right) + \Log\left( \frac{1}{\conf} \right) \right)$,
defining $N_{i} = \{X_{m(i-1)+1}^{\prime},\ldots,X_{mi}^{\prime}\}$ for each $i \in \{1,\ldots,\lceil \log_{2}(2/\conf) \rceil \}$,
letting
\begin{multline*}
\hat{i} = \argmin_{i \in \{1,\ldots,\lceil \log_{2}(2/\conf) \rceil\}} 
\max\vast\{
\sum_{j = m \lceil \log_{2}(2/\conf) \rceil + 1}^{m \lceil \log_{2}(2/\conf) \rceil + \ell} \ind_{\phi_{n}(X_{i_{1}}^{\prime},\ldots,X_{i_{n}}^{\prime},y_{1},\ldots,y_{n})}(X_{j}^{\prime}) :
y_{1},\ldots,y_{n} \in \Y, 
\\ m (i-1) < i_{1} \leq \cdots \leq i_{n} \leq m i, \sum_{j = m (i-1) + 1}^{m i} \ind_{\phi_{n}(X_{i_{1}}^{\prime},\ldots,X_{i_{n}}^{\prime},y_{1},\ldots,y_{n})}(X_{j}^{\prime}) = 0
\vast\} \cup \{0\},
\end{multline*}
and $\hat{N} = N_{\hat{i}}$,
with probability at least $1-\conf$,
$\hat{N}$ is an $\eps$-net of $\Px$ for
$\{\phi_{n}(X_{i_{1}}^{\prime},\ldots,X_{i_{n}}^{\prime},y_{1},\ldots,y_{n}) : m(\hat{i}-1) < i_{1} \leq \cdots \leq i_{n} \leq m \hat{i}, y_{1},\ldots,y_{n} \in \Y\}$.
\end{lemma}
\begin{proof} 
Let $\tilde{c}$ be as in Lemma~\ref{lem:compression-bound}, define $\tilde{c}^{\prime} = \max\left\{8\tilde{c},128\right\}$,
and let $m$ and $\ell$ be as described in the lemma statement.
Noting that 
$\frac{2 \tilde{c}}{\eps}\left( n \Log\left(\frac{2}{\eps}\right) + \Log\left(2^{n+1}\right) \right) 
\leq \frac{8 \tilde{c} n}{\eps} \Log\left(\frac{1}{\eps}\right)$,
we have that $m \geq \frac{2\tilde{c}}{\eps}\left( n \Log\left(\frac{2}{\eps}\right) + \Log\left(2^{n+1}\right) \right)$.
Thus, by Lemma~\ref{lem:compression-bound}, 
for each $i \in \left\{1,\ldots,\lceil \log_{2}(2/\conf) \rceil\right\}$ and $y_{1},\ldots,y_{n} \in \Y$, 
with probability at least $1 - 2^{-n-1}$, 
$\left\{ X_{m (i-1) + 1}^{\prime}, \ldots, X_{m i}^{\prime} \right\}$ is an $\frac{\eps}{2}$-net of $\Px$ for 
$\left\{ \phi_{n}(X_{i_{1}}^{\prime},\ldots,X_{i_{n}}^{\prime},y_{1},\ldots,y_{n}) : m (i-1) < i_{1} \leq \cdots \leq i_{n} \leq m i \right\}$.
By a union bound, this holds simultaneously for all $y_{1},\ldots,y_{n} \in \Y$ with probability at least $\frac{1}{2}$.
In particular, since the $\left\{X_{m (i-1)+1}^{\prime}, \ldots, X_{m i}^{\prime}\right\}$ subsequences are independent over values of $i$, we have that
there is an event $E$ of probability at least $1 - \left(\frac{1}{2}\right)^{\lceil \log_{2}(2/\conf) \rceil} \geq 1 - \frac{\conf}{2}$,
on which $\exists i^{*} \in \left\{ 1, \ldots, \lceil \log_{2}(2/\conf) \rceil \right\}$ such that
$\left\{ X_{m (i^{*}-1)+1}^{\prime}, \ldots, X_{m i^{*}}^{\prime} \right\}$ is an $\frac{\eps}{2}$-net of $\Px$ for 
$\big\{ \phi_{n}(X_{i_{1}}^{\prime},\ldots,X_{i_{n}}^{\prime},y_{1},\ldots,y_{n}) : m (i^{*}-1) < i_{1}$ $\leq \cdots \leq i_{n} \leq m i^{*}, y_{1},\ldots,y_{n} \in \Y \big\}$.

For any $i \in \left\{ 1,\ldots, \lceil \log_{2}(2/\conf) \rceil \right\}$, any $i_{1},\ldots,i_{n} \in \{ m (i-1)+1,\ldots,m i\}$ with $i_{1} \leq \cdots \leq i_{n}$, and any $y_{1},\ldots,y_{n} \in \Y$,
Chernoff bounds (applied under the conditional distribution given $X_{i_{1}}^{\prime},\ldots,X_{i_{n}}^{\prime}$) and the law of total probability imply that,
with probability at least 
$1 - \exp\left\{ - \eps \ell / 32 \right\}$, 
if $\Px\left( \phi_{n}(X_{i_{1}}^{\prime},\ldots,X_{i_{n}}^{\prime},y_{1},\ldots,y_{n}) \right) \leq \frac{\eps}{2}$,
then 
\begin{equation*}
\sum_{j = m \lceil \log_{2}(2/\conf) \rceil + 1}^{m \lceil \log_{2}(2/\conf) \rceil + \ell} \ind_{\phi_{n}(X_{i_{1}}^{\prime},\ldots,X_{i_{n}}^{\prime},y_{1},\ldots,y_{n})}(X_{j}^{\prime}) \leq \frac{3}{4} \eps \ell,
\end{equation*}
while if $\Px\left( \phi_{n}(X_{i_{1}}^{\prime},\ldots,X_{i_{n}}^{\prime},y_{1},\ldots,y_{n}) \right) > \eps$,
then
\begin{equation*}
\sum_{j = m \lceil \log_{2}(2/\conf) \rceil + 1}^{m \lceil \log_{2}(2/\conf) \rceil + \ell} \ind_{\phi_{n}(X_{i_{1}}^{\prime},\ldots,X_{i_{n}}^{\prime},y_{1},\ldots,y_{n})}(X_{j}^{\prime}) > \frac{3}{4} \eps \ell.
\end{equation*}
The number of distinct nondecreasing sequences $(i_{1},\ldots,i_{n}) \in \{ m (i-1) + 1, \ldots, m i \}^{n}$ is $\binom{ n + m-1 }{n}$ $\leq \left(\frac{2em}{n}\right)^{n}$.
Therefore, by a union bound, there exists an event $E^{\prime}$ of probability at least 
\begin{equation*}
1 - 2^{n} \left( \frac{2 e m}{n} \right)^{n} \lceil \log_{2}(2/\conf) \rceil \exp\left\{ - \eps \ell / 32 \right\},
\end{equation*}
on which this holds for every such $y_{1},\ldots,y_{n}, i, i_{1},\ldots,i_{n}$.
Noting that 
\begin{equation*}
\frac{32}{\eps} \Log\left( 2^{n} \lceil \log_{2}(2/\conf) \rceil \left(\frac{2 e m}{n}\right)^{n} \frac{2}{\conf} \right)
\leq \frac{128}{\eps}\left( n \Log\left( \frac{m}{n} \right) + \Log\left( \frac{1}{\conf} \right) \right)
\leq \ell,
\end{equation*}
we have that $E^{\prime}$ has probability at least $1 - \frac{\conf}{2}$.

In particular, defining for each $i \in \left\{ 1, \ldots, \lceil \log_{2}(2/\conf) \rceil \right\}$, 
\begin{multline*}
\hat{p}_{i} = 
\max\vast\{
\sum_{j = m \lceil \log_{2}(2/\conf) \rceil + 1}^{m \lceil \log_{2}(2/\conf) \rceil + \ell} \ind_{\phi_{n}(X_{i_{1}}^{\prime},\ldots,X_{i_{n}}^{\prime},y_{1},\ldots,y_{n})}(X_{j}^{\prime})
: y_{1},\ldots,y_{n} \in \Y, 
\\ m (i-1) < i_{1} \leq \cdots \leq i_{n} \leq m i, \sum_{j = m (i-1) + 1}^{m i} \ind_{\phi_{n}(X_{i_{1}}^{\prime},\ldots,X_{i_{n}}^{\prime},y_{1},\ldots,y_{n})}(X_{j}^{\prime}) = 0
\vast\} \cup \{0\},
\end{multline*}
we have that, on $E \cap E^{\prime}$,
$\hat{p}_{i^{*}} \leq \frac{3}{4} \eps \ell$.
Furthermore, for every $i \in \left\{1,\ldots,\lceil \log_{2}(2/\conf) \rceil \right\}$ for which $\left\{ X_{m (i-1) +1}^{\prime}, \ldots, X_{m i}^{\prime} \right\}$
is \emph{not} an $\eps$-net of $\Px$ for $\big\{ \phi_{n}(X_{i_{1}}^{\prime},\ldots,X_{i_{n}}^{\prime},y_{1},\ldots,y_{n}) : m (i-1) < i_{1} \leq \cdots \leq i_{n} \leq m i, y_{1},\ldots,y_{n} \in \Y \big\}$,
by definition $\exists i_{1},\ldots,i_{n} \in \{ m (i-1)+1,\ldots, m i \}$ with $i_{1} \leq \cdots \leq i_{n}$, and $y_{1},\ldots,y_{n} \in \Y$, such that 
$\Px\left( \phi_{n}(X_{i_{1}}^{\prime},\ldots,X_{i_{n}}^{\prime},y_{1},\ldots,y_{n}) \right) > \eps$
while $\sum_{j= m(i-1)+1}^{m i} \ind_{\phi_{n}(X_{i_{1}}^{\prime},\ldots,X_{i_{n}}^{\prime},y_{1},\ldots,y_{n})}(X_{j}^{\prime})=0$;
thus, on the event $E^{\prime}$, 
\begin{equation*}
\sum_{j = m \lceil \log_{2}(2/\conf) \rceil + 1}^{m \lceil \log_{2}(2/\conf) \rceil + \ell} \ind_{\phi_{n}(X_{i_{1}}^{\prime},\ldots,X_{i_{n}}^{\prime},y_{1},\ldots,y_{n})}(X_{j}^{\prime}) > \frac{3}{4} \eps \ell
\end{equation*}
for this choice of $i_{1},\ldots,i_{n},y_{1},\ldots,y_{n}$.
In particular, this implies that $\hat{p}_{i} > \frac{3}{4} \eps \ell$.
Altogether, we have that on the event $E \cap E^{\prime}$, any such $i$ has $\hat{p}_{\hat{i}} \leq \hat{p}_{i^{*}} \leq \frac{3}{4} \eps \ell < \hat{p}_{i}$, 
so that $\hat{i} \neq i$.  Therefore, on the event $E \cap E^{\prime}$, 
$\left\{ X_{m (\hat{i}-1) +1}^{\prime}, \ldots, X_{m \hat{i}}^{\prime} \right\}$
is an $\eps$-net of $\Px$ for $\big\{ \phi_{n}(X_{i_{1}}^{\prime},\ldots,X_{i_{n}}^{\prime},y_{1},\ldots,y_{n}) : m (\hat{i}-1) < i_{1} \leq \cdots \leq i_{n} \leq m \hat{i}, y_{1},\ldots,y_{n} \in \Y \big\}$.

To complete the proof, we note that the event $E \cap E^{\prime}$ has probability at least $1-\conf$ by a union bound.
\end{proof}

\subsection{Lower Bound Constructions for Noisy Settings}
\label{sec:rr-lemma}

Fix any $\zeta \in (0,1]$, $\bound \in [0,1/2)$,
and $k \in \nats$ with $k \leq 1 / \zeta$.
Let $\X_{k} = \{x_{1},\ldots,x_{k+1}\}$ be any $k+1$ distinct elements of $\X$ (assuming $|\X| \geq k+1$),
and let $\C_{k} = \{ x \mapsto 2\ind_{\{x_{i}\}}(x)-1 : i \in \{1,\ldots,k\}\}$,
a set of functions mapping $\X$ to $\{-1,+1\}$.
Let $\Px_{k,\zeta}$ be a probability measure over $\X$
with $\Px_{k,\zeta}(\{x_{i}\}) = \zeta$ for each $i \in \{1,\ldots,k\}$,
and $\Px_{k,\zeta}(\{x_{k+1}\}) = 1 - \zeta k$.
For each $t \in \{1,\ldots,k\}$, let $P_{k,\zeta,t}^{\prime}$ denote the probability measure over $\X \times \Y$
having marginal distribution $\Px_{k,\zeta}$ over $\X$, such that if $(X,Y) \sim P_{k,\zeta,t}^{\prime}$,
then every $i \in \{1,\ldots,k\}$ has $\P(Y=2\ind_{\{x_{t}\}}(X)-1 | X=x_{i})=1-\bound$, and furthermore $\P(Y=-1|X=x_{k+1})=1$.
Finally, define $\RR^{\prime}(k,\zeta,\bound) = \left\{ P_{k,\zeta,t}^{\prime} : t \in \{1,\ldots,k\}\right\}$.
\citet*{raginsky:11} prove the following result (see the proof of their Theorem 2).\footnote{Technically,
the proof of \citet*[][Theorem 2]{raginsky:11} relies on a lemma (their Lemma 4), with various conditions
on both $k$ and a parameter ``$d$'' in their construction.  However, one can easily verify that the conclusions
of that lemma continue to hold (in fact, with improved constants) in our special case (corresponding to $d=1$ 
and arbitrary $k\in\nats$) by defining $\mathcal{M}_{k,1}=\{0,1\}^{k}_{1}$ in their construction.}

\begin{lemma}
\label{lem:rr11}
For $\zeta,\bound,k$ as above, if $k \geq 2$
and $\C_{k} \subseteq \C$, then for any $\conf \in (0,1/4)$,
\begin{equation*}
\LC_{\RR^{\prime}(k,\zeta,\bound)}((\zeta/2) (1-2\bound),\conf) 
\geq \frac{\bound k\ln\left(\frac{1}{4\conf}\right)}{3(1-2\bound)^{2}}.
\end{equation*}
\end{lemma}

This has the following immediate implication for general $\X$ and $\C$.
Fix any $\zeta \in (0,1]$ and $\bound \in [0,1/2)$, let $k \in \nats \cup \{0\}$ satisfy $k \leq \min\left\{\s-1,\lfloor 1/\zeta \rfloor \right\}$,
and let $x_{1},\ldots,x_{k+1}$ and $h_{0},h_{1},\ldots,h_{k}$ be as in Definition~\ref{def:star}.
Let $\Px_{k,\zeta}$ be as above (for this choice of $x_{1},\ldots,x_{k+1}$),
and for each $t \in \{1,\ldots,k\}$, let $P_{k,\zeta,t}$ denote the probability 
measure over $\X \times \Y$ having marginal distribution $\Px_{k,\zeta}$ over $\X$,
such that if $(X,Y) \sim P_{k,\zeta,t}$,
then every $i \in \{1,\ldots,k\}$ has $\P(Y=h_{t}(X)|X=x_{i})=1-\bound$, and furthermore $\P(Y=h_{t}(X)|X=x_{k+1})=1$.
Define $\RR(k,\zeta,\bound) = \left\{P_{k,\zeta,t} : t \in \{1,\ldots,k\}\right\}$.
We have the following result.

\begin{lemma}
\label{lem:rr11-star}
For $k,\zeta,\bound$ as above, for any $\conf \in (0,1/4)$,
\begin{equation*}
\LC_{\RR(k,\zeta,\bound)}((\zeta/2)(1-2\bound),\conf)
\geq \frac{\bound (k-1) \ln\left(\frac{1}{4\conf}\right)}{3(1-2\bound)^{2}}.
\end{equation*}
\end{lemma}
\begin{proof}
First note that if $k \leq 1$, then the lemma trivially holds (since $\LC_{\RR(k,\zeta,\bound)}(\cdot,\cdot) \geq 0$).  
For this same reason, the result also trivially holds if $\bound = 0$.
Otherwise, suppose $k \geq 2$ and $\bound > 0$, and 
fix any $n$ less than the right hand side of the above inequality.
Let $\alg$ be any active learning algorithm, and consider the following modification $\alg^{\prime}$ of $\alg$.
For any given sequence $X_{1},X_{2},\ldots$ of unlabeled data, $\alg^{\prime}(n)$ simulates the execution of $\alg(n)$,
except that when $\alg(n)$ would request the label $Y_{i}$ of a point $X_{i}$ in the sequence, 
$\alg^{\prime}(n)$ requests the label $Y_{i}$,
but proceeds as $\alg(n)$ would if the label value had been 
$- Y_{i} h_{0}(X_{i})$ instead of $Y_{i}$.
When the simulation of $\alg(n)$ concludes, if $\hat{h}$ is its return value, 
$\alg^{\prime}(n)$ instead returns the function $x \mapsto \hat{h}^{\prime}(x) = -\hat{h}(x) h_{0}(x)$.

Now fix a $P_{k,\zeta,t}^{\prime} \in \RR^{\prime}(k,\zeta,\bound)$ minimizing the probability that
$\er_{P_{k,\zeta,t}^{\prime}}(\hat{h}^{\prime}) - \inf_{h \in \C_{k}} \er_{P_{k,\zeta,t}^{\prime}}(h)$ $\leq (\zeta/2)(1-2\bound)$
when $\alg^{\prime}$ is run with $\PXY = P_{k,\zeta,t}^{\prime}$,
and let $(X,Y) \sim P_{k,\zeta,t}^{\prime}$.
Note that the marginal distribution of $P_{k,\zeta,t}^{\prime}$ over $\X$ is $\Px_{k,\zeta}$,
that for any $i \in \{1,\ldots,k\}$, 
$\P(-Y h_{0}(X) = h_{t}(X) | X=x_{i}) = \P( Y = 2\ind_{\{x_{t}\}}(X)-1 | X = x_{i} ) = 1-\bound$,
and that $\P(-Y h_{0}(X) = h_{t}(X) | X = x_{k+1}) = \P(Y = -1 | X = x_{k+1}) = 1$.
In particular, this implies $(X,-Y h_{0}(X)) \sim P_{k,\zeta,t}$.
Therefore, running the active learning algorithm $\alg^{\prime}(n)$ with a sequence 
$(X_{1},Y_{1}),(X_{2},Y_{2}),\ldots$ of independent $P_{k,\zeta,t}^{\prime}$-distributed samples,
the algorithm behaves as $\alg(n)$ would under $P_{k,\zeta,t}$, except that its returned classifier
is $\hat{h}^{\prime}$ instead of $\hat{h}$.
Next, note that 
\begin{align*}
\er_{P_{k,\zeta,t}^{\prime}}(\hat{h}^{\prime}) 
& = \P( - \hat{h}(X) h_{0}(X) \neq Y ) 
\\ & = \E[ \P( \hat{h}(X) \neq -Y|X)\ind[h_{0}(X)=1] + \P(\hat{h}(X) \neq Y|X)\ind[h_{0}(X)=-1] ]
\\ & = \P( \hat{h}(X) \neq -Y h_{0}(X)) 
= \er_{P_{k,\zeta,t}}(\hat{h}),
\end{align*}
and furthermore 
\begin{equation*}
\inf_{h \in \C_{k}} \er_{P_{k,\zeta,t}^{\prime}}(h) = \er_{P_{k,\zeta,t}^{\prime}}(2 \ind_{\{x_{t}\}}-1) = \bound \zeta k = \er_{P_{k,\zeta,t}}(h_{t}) = \inf_{h \in \C} \er_{P_{k,\zeta,t}}(h).
\end{equation*}
Thus, if $\er_{P_{k,\zeta,t}}(\hat{h}) - \inf_{h \in \C} \er_{P_{k,\zeta,t}}(h) \leq (\zeta/2)(1-2\bound)$, 
then we must also have $\er_{P_{k,\zeta,t}^{\prime}}(\hat{h}^{\prime}) - \inf_{h \in \C_{k}} \er_{P_{k,\zeta,t}^{\prime}}(h) \leq (\zeta/2)(1-2\bound)$.
Since $n < \frac{\bound k\ln\left(\frac{1}{4\conf}\right)}{3(1-2\bound)^{2}}$, 
Lemma~\ref{lem:rr11} implies that (for this choice of $P_{k,\zeta,t}^{\prime}$)
$\alg^{\prime}(n)$ achieves the latter guarantee with probability strictly less than $1-\conf$, 
and therefore the corresponding $P_{k,\zeta,t} \in \RR(k,\zeta,\bound)$ is such that $\alg(n)$ has probability strictly less than $1-\conf$
of achieving $\er_{P_{k,\zeta,t}}(\hat{h}) - \inf_{h \in \C} \er_{P_{k,\zeta,t}}(h) \leq (\zeta/2)(1-2\bound)$.  Since this argument applies 
to any active learning algorithm $\alg$, the result follows.
\end{proof}

\subsection{Finite Approximation of VC Classes}
\label{sec:approximability}

For a given probability measure $\Px$ over $\X$, \citet*{adams:12} have proven that for any $\tau > 0$, if $\vc < \infty$, 
there exist disjoint measurable sets $A_{1},\ldots,A_{k}$ (for some $k \in \nats$) with $\bigcup_{i} A_{i} = \X$
such that, $\forall h \in \C$, $\Px\left( \bigcup \{A_{i} : \exists x,y \in A_{i} \text{ s.t. } h(x) \neq h(y)\} \right) < \tau$:
that is, every $h \in \C$ is constant on all of the sets $A_{i}$, except a few of them whose total probability is at most $\tau$.
This property has implications for bracketing behavior in VC classes, and was proven in the context of establishing uniform
laws of large numbers for VC classes under stationary ergodic processes (see also \citealp*{adams:10,van-handel:13}).

For our purposes, this result has the appealing feature that it allows one to effectively \emph{discretize} the space $\X$
by partitioning it into subsets, with the guarantee that with high probability over the random choice of a point $x$, 
any other point $y$ in the same cell in the partition as $x$ will have $\target_{\PXY}(x) = \target_{\PXY}(y)$,
for any $\PXY \in \bigcup_{\nu \in [0,1/2)} \BE(\nu)$.
However, before we can make use of this property, we must first address the fact that the construction of these sets $A_{i}$ 
by \citet*{adams:12} requires a strong dependence on $\Px$, to the extent that it is not obvious that this dependence can
be supplanted by a data-dependent construction.  However, it turns out that if we relax the requirement that the classifiers
be \emph{constant} in these cells, instead settling for being \emph{nearly-constant}, then it is straightforward to 
construct a partition $A_{1},\ldots,A_{k}$ satisfying the requirement.
Specifically, we have the following result.

\begin{lemma}
\label{lem:empirical-approx}
Fix any $\tau,\conf \in (0,1)$, 
and let $m_{\tau,\conf} = \left\lceil \frac{c}{\tau} \!\left( \vc \Log\!\left(\frac{1}{\tau}\right) + \Log\!\left(\frac{1}{\conf}\right) \right) \right\rceil$
(for $c$ as in Lemma~\ref{lem:vc-cover}).
For any probability measure $\Px$ over $\X$, 
if $X_{1}^{\prime},\ldots,X_{m_{\tau,\conf}}^{\prime}$ are independent $\Px$-distributed random variables,
then with probability at least $1-\conf$, 
letting $\C_{\tau,\conf} = \C[(X_{1}^{\prime},\ldots,X_{m_{\tau,\conf}}^{\prime})]$ (as defined in Section~\ref{sec:xtd}),
the collection of disjoint sets 
\begin{equation*}
J_{\tau,\conf} = \left\{ \bigcap_{g \in \C[(X_{1}^{\prime},\ldots,X_{m_{\tau,\conf}}^{\prime})]} \X_{g} : \forall g \in \C_{\tau,\conf}, \X_{g} \in \{ \{x : g(x) = +1\}, \{x : g(x) = -1\} \}  \right\} 
\end{equation*}
is a partition of $\X$ with the property that, 
$\forall h \in \C$,
\begin{equation*}
\sum_{A \in J_{\tau,\conf}} \min_{y \in \Y} \Px( x \in A : h(x) = y ) \leq \tau,
\end{equation*}
and $\forall \eps > 0$, $\forall h \in \C$,
\begin{equation*}
\Px\left( \bigcup\left\{ A \in J_{\tau,\conf} : \min_{y \in \Y} \Px(x \in A : h(x) = y) > \eps \Px(A) \right\} \right) \leq \frac{\tau}{\eps}.
\end{equation*}
\end{lemma}
\begin{proof}
By Lemma~\ref{lem:vc-cover}, with probability at least $1-\conf$, 
$\C_{\tau,\conf}$ is a $\tau$-cover of $\C$.
Furthermore, note that for every $g \in \C_{\tau,\conf}$ and every $A \in J_{\tau,\conf}$, 
either every $x \in A$ has $g(x) = +1$ or every $x \in A$ has $g(x) = -1$ (i.e., $g$ is constant on $A$).
Therefore, $\forall h \in \C$,
\begin{multline*}
\sum_{A \in J_{\tau,\conf}} \min_{y \in \Y} \Px( x \in A : h(x) = y ) 
\leq \sum_{A \in J_{\tau,\conf}} \min_{g \in \C_{\tau,\conf}} \Px( x \in A : h(x) \neq g(x) ) 
\\ \leq \min_{g \in \C_{\tau,\conf}} \sum_{A \in J_{\tau,\conf}} \Px( x \in A : h(x) \neq g(x) ) 
= \min_{g \in \C_{\tau,\conf}} \Px( x : h(x) \neq g(x) )
\leq \tau.
\end{multline*}
The final claim follows by Markov's inequality, since on the above event, $\forall \eps > 0$, $\forall h \in \C$, 
\begin{align*}
& \Px\left( \bigcup\left\{ A \in J_{\tau,\conf} : \min_{y \in \Y} \Px(x \in A : h(x) = y) > \eps \Px(A) \right\} \right)
\\ & = \Px\left( \bigcup\left\{ A \in J_{\tau,\conf} : \Px(A) > 0, \min_{y \in \Y} \Px(x \in A : h(x) = y) > \eps \Px(A) \right\} \right)
\\ & = \Px\left( \bigcup\left\{ A \in J_{\tau,\conf} : \Px(A) > 0, \min_{y \in \Y} \frac{\Px(x \in A : h(x) = y)}{\Px(A)} > \eps\right\} \right)
\\ & \leq \frac{1}{\eps} \sum_{A \in J_{\tau,\conf}} \Px(A) \min_{y \in \Y} \frac{\Px(x \in A : h(x) = y)}{\Px(A)}
= \frac{1}{\eps} \sum_{A \in J_{\tau,\conf}} \min_{y \in \Y} \Px(x \in A : h(x) = y)
\leq \frac{\tau}{\eps}.
\end{align*}
\end{proof}

\section{Proofs for Results in Section~\ref{sec:main}}
\label{app:main-proofs}

This section provides proofs of the main results of this article.

\subsection{The Realizable Case}
\label{app:realizable}

We begin with the particularly-simple case of Theorem~\ref{thm:realizable}.

\begin{proof}[of Theorem~\ref{thm:realizable}]
The lower bounds proportional to $\vc$ and $\Log\left(\min\left\{\frac{1}{\eps},|\C|\right\}\right)$ 
are due to \citet*{kulkarni:93} (lower bound in terms of the covering numbers)
in conjunction with \citet*{kulkarni:89,kulkarni:93} (lower bounds on the worst-case covering numbers).
Specifically, \citet*{kulkarni:93} study the problem of learning from arbitrary binary-valued queries.
Since active learning receives binary responses in the binary classification setting, it is a special 
case of this type of algorithm.  In particular, 
for any probability measure $\Px$ over $\X$, and $\eps \in (0,1)$, 
let $\covering(\eps,\C,\Px)$ denote the minimum cardinality $|\H|$ over all $\eps$-covers $\H$ of $\C$ (under the $\Px(\DIS(\{\cdot,\cdot\}))$ pseudometric),
or else $\covering(\eps,\C,\Px) = \infty$ if no finite $\eps$-cover of $\C$ exists.
Then the lower bound of \citet*[][Theorem 3]{kulkarni:93}
implies that, $\forall \eps,\conf \in (0,1/2)$,
\begin{equation}
\label{eqn:kulkarni-lower}
\LC_{\RE}(\eps,\conf) \geq \sup_{\Px} \left\lceil \log_{2}\left( (1-\conf) \covering(2\eps,\C,\Px) \right) \right\rceil.
\end{equation}
Furthermore, the construction in the proof of \citet*[][Lemma 2]{kulkarni:93} implies that 
$\sup_{\Px} \covering(2\eps,\C,\Px) \geq \min\left\{ \left\lfloor \frac{1}{4\eps} \right\rfloor, |\C| \right\}$,
so that combined with \eqref{eqn:kulkarni-lower}, we have 
\begin{equation*}
\LC_{\RE}(\eps,\conf) \geq \left\lceil \log_{2}\left( (1-\conf) \min\left\{ \left\lfloor \frac{1}{4\eps} \right\rfloor, |\C| \right\} \right) \right\rceil.
\end{equation*}
For $\conf \in (0,1/3)$ and $\eps \in (0,1/8)$, and since $|\C| \geq 3$ (by assumption, intended to focus on nontrivial cases to simplify the expressions),
the right hand side is at least
$\frac{1}{4} \Log\left( \min\left\{ \frac{1}{\eps}, |\C| \right\} \right)$.
Furthermore, if $\vc < 162$, this already implies that for any $\eps \in (0,1/3)$ and $\conf \in (0,1/3)$, 
$\LC_{\RE}(\eps,\conf) \geq \frac{1}{4} \ln(3) \geq \frac{\vc}{648}$.
Otherwise, in the case that $\vc \geq 162$,
\citet*[][Proposition 3]{kulkarni:89} proves that, if $\eps \in (0,1/9)$, 
$\sup_{\Px} \covering(2\eps,\C,\Px) \geq \exp\left\{ 2 \left(\frac{1}{2}-4\eps\right)^{2} \vc \right\} \geq \exp\left\{ \vc / 162 \right\}$.
Combined with \eqref{eqn:kulkarni-lower}, this implies that for $\eps \in (0,1/9)$ and $\conf \in (0,1/3)$,
if $\vc \geq 162$, then
\begin{equation*}
\LC_{\RE}(\eps,\conf) 
\geq \left\lceil \log_{2}\left( \frac{2}{3} e^{\vc/162} \right) \right\rceil 
\geq \frac{\vc}{162} \log_{2}(e) - \log_{2}\left( \frac{3}{2} \right)
\geq \frac{\vc}{162} \log_{2}\left( \frac{2e}{3} \right)
\geq \frac{\vc}{189}.
\end{equation*}
Thus, regardless of the value of $\vc$, we have $\LC_{\RE}(\eps,\conf) \geq \frac{\vc}{648}$.

For the final part of the proof of the lower bound, a lower bound proportional to $\s \land \frac{1}{\eps}$ may be credited to \citet*{dasgupta:05,dasgupta:04}.
It can be proven as follows.  Let $x_{1},\ldots,x_{\s}$ and $h_{0},h_{1},\ldots,h_{\s}$ be as in Definition~\ref{def:star},
let $t = \s \land \left\lceil \frac{1-\eps}{\eps} \right\rceil$,
and let us restrict the discussion to those $t+1$ distributions $\PXY \in \RE$
such that the marginal distribution $\Px$ of $\PXY$ over $\X$ is uniform on $\{x_{1}, \ldots, x_{t}\}$,
and $\target_{\PXY} \in \{h_{0},h_{1},\ldots,h_{t}\}$.
Then for any active learning algorithm $\alg$, for any $n \leq t/2$, 
let $Q_{i}$ denote the (possibly random) set of (at most $n$) points $X_{i}$ that $\alg(n)$ requests the labels of, 
given that $\target_{\PXY} = h_{i}$ (for $i \in \{0,\ldots,t\}$), and let $\hat{h}_{i}$ denote the classifier returned by $\alg(n)$ in this case.
Since the marginal distribution of $\PXY$ over $\X$ is fixed to $\Px$ for all $t+1$ of these $\PXY$ distributions,
we may consider the sequence $X_{1},X_{2},\ldots$ of i.i.d. $\Px$-distributed random variables to be identical over these 
$t+1$ possible choices of $\PXY$, without affecting the distributions of $Q_{i}$ and $\hat{h}_{i}$ \citep*[see][]{kallenberg:02}.
Thus, we may note that $\hat{h}_{i} = \hat{h}_{0}$ whenever $x_{i} \notin Q_{0}$,
since $x_{i} \notin Q_{0}$ implies that all of the labels observed by the algorithm are identical 
to those that would be observed if $\target_{\PXY} = h_{0}$ instead of $\target_{\PXY} = h_{i}$.
Now, if it holds that $\P\left( \Px\left(x : \hat{h}_{0}(x) \neq h_{0}(x)\right) > \eps \right) \leq \conf$, 
then since every $x_{i}$ with $i \leq t$ has $\Px(\{x_{i}\}) > \eps$, we have that
$\P\left( \forall i \in \{1,\ldots,t\}, \hat{h}_{0}(x_{i}) = h_{0}(x_{i}) \right) \geq 1 - \conf$.
But if this holds, then it must also be true that
\begin{align*}
& \max_{i \in \{1,\ldots,t\}} \P\left( \Px(x : \hat{h}_{i}(x) \neq h_{i}(x)) > \eps \right)
\geq \frac{1}{t} \sum_{i = 1}^{t} \P\left( \Px(x : \hat{h}_{i}(x) \neq h_{i}(x)) > \eps \right)
\\ & \geq \frac{1}{t} \sum_{i = 1}^{t} \P\left( \hat{h}_{i}(x_{i}) = h_{0}(x_{i}) \right)
= \frac{1}{t} \E\left[ \sum_{i=1}^{t} \ind\left[ \hat{h}_{i}(x_{i}) = h_{0}(x_{i}) \right] \right]
\\ & \geq \frac{1}{t} \E\left[ \sum_{i=1}^{t} \ind\left[ x_{i} \notin Q_{0} \right] \ind\left[ \hat{h}_{i}(x_{i}) = h_{0}(x_{i}) \right] \right]
= \frac{1}{t} \E\left[ \sum_{i=1}^{t} \ind\left[ x_{i} \notin Q_{0} \right] \ind\left[ \hat{h}_{0}(x_{i}) = h_{0}(x_{i}) \right] \right]
\\ & \geq \frac{1}{t} \E\left[  \ind\left[\forall i \in \{1,\ldots,t\}, \hat{h}_{0}(x_{i}) = h_{0}(x_{i}) \right]  \sum_{i=1}^{t} \ind\left[ x_{i} \notin Q_{0} \right] \right]
\\ & \geq \frac{1}{t} \E\left[  \ind\left[\forall i \in \{1,\ldots,t\}, \hat{h}_{0}(x_{i}) = h_{0}(x_{i}) \right] (t-n) \right]
\\ & = \frac{t-n}{t} \P\left(  \forall i \in \{1,\ldots,t\}, \hat{h}_{0}(x_{i}) = h_{0}(x_{i}) \right)
\geq \frac{t-n}{t} (1-\conf) 
\geq \frac{1-\conf}{2} 
\geq \frac{1}{3}
> \conf.
\end{align*}
Thus, when $n \leq t/2$, at least one of these $t+1$ distributions $\PXY$ (all of which are in $\RE$) 
has $\P\left( \er_{\PXY}(\alg(n)) > \eps \right) > \conf$.
Since this argument holds for any $\alg$, we have that 
$\LC_{\RE}(\eps,\conf) > t/2 = \frac{1}{2} \min\left\{ \s, \left\lceil \frac{1-\eps}{\eps} \right\rceil \right\} \geq \frac{4}{9} \min\left\{ \s, \frac{1}{\eps} \right\}$.
Combined with the lower bounds proportional $\vc$ and $\Log\left( \min\left\{ \frac{1}{\eps}, |\C| \right\} \right)$ established above,
this completes the proof of the lower bound in Theorem~\ref{thm:realizable}.

The proof of the upper bound is in three parts.  The first part, establishing the $\frac{\vc}{\eps}\Log\left(\frac{1}{\eps}\right)$
upper bound, is a straightforward application of Lemma~\ref{lem:eps-net}.  The second part, 
establishing the $\frac{\s\vc}{\Log(\s)}\Log\left(\frac{1}{\eps}\right)$ upper bound, is directly
based on techniques of \citet*{hanneke:07a,hegedus:95}.  Finally, and most involved, is the third part,
establishing the $\s \Log\left(\frac{1}{\eps}\right)$ upper bound.  This part is partly based on a recent technique of \citet*{hanneke:14a}
for analyzing disagreement-based active learning (which refines an earlier technique of \citealp*{el-yaniv:10,el-yaniv:12}).
Here, we modify this technique by using an $\eps$-net in place of random samples,
thereby refining logarithmic factors, and entirely eliminating the dependence on $\conf$ in the label complexity.

Fix any $\eps,\conf \in (0,1)$.  
We begin with the $\frac{\vc}{\eps}\Log\left(\frac{1}{\eps}\right)$ upper bound.
Let $m = \left\lceil \frac{c^{\prime} \vc}{\eps}\Log\left(\frac{1}{\eps}\right) \right\rceil$
and $\ell = \left\lceil \frac{c^{\prime}}{\eps}\left( \vc \Log\left(\frac{1}{\eps}\right) + \Log\left(\frac{1}{\conf}\right) \right) \right\rceil$,
for $c^{\prime}$ as in Lemma~\ref{lem:eps-net}.
Define
\begin{equation*}
\hat{i} = \argmin_{i \in \{1,\ldots,\lceil \log_{2}(2/\conf)\rceil \}} \max_{\substack{h,g \in \C : \\ \sum_{j=m(i-1)+1}^{m i} \ind_{\DIS(\{h,g\})}(X_{j}) = 0}} \sum_{j=m\lceil \log_{2}(2/\conf)\rceil+1}^{m\lceil \log_{2}(2/\conf)\rceil+\ell} \ind_{\DIS(\{h,g\})}(X_{j}).
\end{equation*}
Consider an active learning algorithm which, given a budget $n \in \nats$, 
requests the labels $Y_{t}$ for $t \in \left\{ m \left(\hat{i}-1\right)+1, \ldots, m \left(\hat{i}-1\right)+\min\left\{m,n\right\}\right\}$,
and returns any classifier $\hat{h}_{n} \in \C$ with $\sum_{t= m \left(\hat{i}-1\right)+1}^{m\left(\hat{i}-1\right)+\min\left\{m, n\right\}} \ind\left[ \hat{h}_{n}(X_{t}) \neq Y_{t} \right] = 0$
if such a classifier exists (and otherwise returns an arbitrary classifier).
Note that, for $\PXY \in \RE$, $\sum_{t= m \left(\hat{i}-1\right)+1}^{m\left(\hat{i}-1\right)+\min\left\{m, n\right\}} \ind\left[ \target_{\PXY}(X_{t}) \neq Y_{t} \right] = 0$ with probability one,
and since $\target_{\PXY} \in \C$, the classifier $\hat{h}_{n}$ will have $\sum_{t= m \left(\hat{i}-1\right)+1}^{m\left(\hat{i}-1\right)+\min\left\{m, n\right\}} \ind\left[ \hat{h}_{n}(X_{t}) \neq Y_{t} \right] = 0$
with probability one.  Furthermore, this implies $\sum_{t= m \left(\hat{i}-1\right)+1}^{m\left(\hat{i}-1\right)+\min\left\{m, n\right\}} \ind\left[ \hat{h}_{n}(X_{t}) \neq \target_{\PXY}(X_{t}) \right] = 0$ with probability one.
Additionally, Lemma~\ref{lem:eps-net} implies that, with probability at least $1-\conf$, 
the set $\left\{ X_{t} : t \in \left\{ m \left(\hat{i}-1\right)+1, \ldots, m\hat{i} \right\} \right\}$ 
is an $\eps$-net of $\Px$ for $\{ \DIS(\{h,g\}) : h,g \in \C \}$.
Since both $\hat{h}_{n}, \target_{\PXY} \in \C$, this implies that if $n \geq m$, then with probability at least $1-\conf$,
$\Px\left( \DIS\left( \left\{ \hat{h}_{n}, \target_{\PXY} \right\} \right) \right) \leq \eps$.
Since $\PXY \in \RE$, $\er_{\PXY}\left( \hat{h}_{n} \right) = \Px\left( \DIS\left( \left\{ \hat{h}_{n}, \target_{\PXY} \right\} \right) \right)$.
Thus, if $n \geq m$, then with probability at least $1-\conf$, $\er_{\PXY}\left( \hat{h}_{n} \right) \leq \eps$.
Since this holds for any $\PXY \in \RE$, we have established that $\LC_{\RE}(\eps,\conf) \leq m \leq \frac{2 c^{\prime} \vc}{\eps} \Log\left(\frac{1}{\eps}\right)$.
This also completes the proof of the entire upper bound in Theorem~\ref{thm:realizable} in the case $\s=\infty$;
for this reason, for the remainder of the proof below, we restrict our attention to the case $\s < \infty$.

Next, we turn to proving the $\frac{\s\vc}{\Log(\s)}\Log\left(\frac{1}{\eps}\right)$ upper bound,
based on a technique of \citet*{hanneke:07a,hegedus:95} (see also \citealp*{hellerstein:96} for related ideas),
except using an $\eps$-net in place of the random samples used by \citet*{hanneke:07a}.
Let $m$ and $\hat{i}$ be defined as above, and denote $\U = \left\{ X_{t} : t \in \left\{ m \left(\hat{i}-1\right)+1,\ldots,m\hat{i} \right\} \right\}$.
The technique is based on using a general algorithm for \emph{Exact} learning with membership queries, 
treating $\U$ as the instance space, and $\C[\U]$ as the concept space (where $\C[\U]$ is as defined in Section~\ref{sec:xtd}).
Specifically, for any finite set $V \subseteq \C$ and any $x \in \X$, let $h_{{\rm maj}(V)}(x) = \argmax_{y \in \Y} |\{ h \in V : h(x) = y \}|$
(breaking ties arbitrarily);
$h_{{\rm maj}(V)}$ is called the \emph{majority vote classifier}.
In this context, the following algorithm is due to \citet*{hegedus:95} (see Section~\ref{sec:xtd} for the definition of ``specifying set'').

\begin{bigboxit}
\MembHalvingII\\
Input: label budget $n$\\
Output: classifier $\hat{h}_{n}$\\
{\vskip -2mm}\line(1,0){420}\\
0. $V \gets \C[\U]$, $t \gets 0$ \\
1. While $|V| \geq 2$ and $t < n$\\
2. \quad $\hat{h} \gets h_{{\rm maj}(V)}$\\
3. \quad Let $k = \TD(\hat{h},\C[\U],\U)$\\
4. \quad Let $\{X_{j_{1}},\ldots,X_{j_{k}}\} \in \U^{k}$ be a minimal specifying set for $\hat{h}$ on $\U$ with respect to $\C[\U]$\\
5. \quad Repeat\\
6. \qquad Let $\hat{j} = \argmin\limits_{j \in \{ j_{1},\ldots,j_{k} \}} | \{ g \in V : g(X_{j}) = \hat{h}(X_{j}) \}|$\\
7. \qquad Request the label $Y_{\hat{j}}$, let $t \gets t+1$\\
8. \qquad $V \gets \{ h \in V : h(X_{\hat{j}}) = Y_{\hat{j}} \}$\\ 
9. \quad Until $\hat{h}(X_{\hat{j}}) \neq Y_{\hat{j}}$ or $|V| \leq 1$ or $t=n$\\
10. Return any $\hat{h}_{n}$ in $V$ (or $\hat{h}_{n}$ arbitrary if $V = \emptyset$)
\end{bigboxit}

Fix any $\PXY \in \RE$, and note that we have $\target_{\PXY} \in \C$, so that $\exists h^{*} \in \C[\U]$ with $h^{*}(x) = \target_{\PXY}(x)$, $\forall x \in \U$.
Since $Y_{j} = \target_{\PXY}(X_{j})$ for every $j$ with probability one in this case, we have that with probability one
the set $V$ will be nonempty in Step 10, so that $\hat{h}_{n}$ is chosen from $V$; in particular, 
we have $h^{*}(X_{j}) = Y_{j}$ for every $X_{j} \in \U$, and hence $h^{*} \in V$ in Step 10.
Furthermore, when this is the case, \citet*{hegedus:95} proves that, letting $\XTD(\C[\U],\U) = \max\limits_{h : \X \to \Y} \TD(h,\C[\U],\U)$ (see Section~\ref{sec:xtd}),
if 
\begin{equation*}
n \geq 2 \frac{\XTD(\C[\U],\U)}{1 \lor \log_{2}( \XTD(\C[\U],\U) )} \log_{2}( |\C[\U]| ),
\end{equation*}
then the classifier $\hat{h}_{n}$ returned by \MembHalvingII~satisfies
$\hat{h}_{n} = h^{*}$, so that $\hat{h}_{n}(x) = \target_{\PXY}(x)$ for every $x \in \U$.\footnote{The two cases not
covered by the theorem of \citet*{hegedus:95} are the case $|\C[\U]| = 1$, for which 
the algorithm returns the sole element of $\C[\U]$ (which must agree with $\target_{\PXY}$ on $\U$) 
without requesting any labels, and the case $|\C[\U]| = 2$, for which one can easily verify that $\XTD(\C[\U],\U)=1$
and that the algorithm returns a classifier with the claimed property after requesting exactly one label.}
Since $\XTD(\C[\U],\U) \leq \XTD(\C,m)$, and Theorem~\ref{thm:xtd} implies $\XTD(\C,m) = \s \land m \leq \s$, 
and since $\Log(\XTD(\C[\U],\U)) \leq 1 \lor \log_{2}(\XTD(\C[\U],\U))$
and $x \mapsto \frac{x}{\Log(x)}$ is nondecreasing on $\nats\cup\{0\}$,
and the VC-Sauer Lemma \citep*{vapnik:71,sauer:72} implies $|\C[\U]| \leq \left(\frac{e m}{\vc}\right)^{\vc}$,
we have that for any $n \geq 2 \frac{\s \vc}{\Log( \s )} \log_{2}\left( \frac{e m}{\vc} \right)$,
if $\forall j, \target_{\PXY}(X_{j}) = Y_{j}$, then 
$\hat{h}_{n}(x) = \target_{\PXY}(x)$ for every $x \in \U$.
Thus,
for
$n \geq 2 \frac{\s \vc}{\Log( \s )} \log_{2}\left( \frac{e m}{\vc} \right)$,
with probability one
the classifier $\hat{h}_{n}$ returned by \MembHalvingII~has
$\hat{h}_{n}(x) = \target_{\PXY}(x)$ for every $x \in \U$.  
Furthermore, as proven above, with probability at least $1-\conf$, 
$\U$ is an $\eps$-net of $\Px$ for $\{\DIS(\{h,g\}) : h,g \in \C\}$.  Thus, since $\target_{\PXY},\hat{h}_{n} \in \C$,
by a union bound we have that for any $n \geq 2 \frac{\s \vc}{\Log( \s )} \log_{2}\left( \frac{e m}{\vc} \right)$, 
with probability at least $1-\conf$, $\Px( \DIS(\{ \target_{\PXY}, \hat{h}_{n} \}) ) \leq \eps$.
Since $\PXY \in \RE$, this implies $\er_{\PXY}(\hat{h}_{n}) = \Px( \DIS(\{\target_{\PXY},\hat{h}_{n}\})) \leq \eps$ as well.
Thus, since this reasoning holds for any $\PXY \in \RE$, we have established that 
\begin{equation*}
\LC_{\RE}(\eps,\conf) \leq 2 \frac{\s \vc}{\Log( \s )} \log_{2}\left( \frac{e m}{\vc} \right) 
\leq 16\Log\left( 2ec^{\prime} \right) \frac{\s \vc}{\Log(\s)} \Log\left(\frac{1}{\eps}\right).
\end{equation*}

Finally, we establish the $\s \Log\left(\frac{1}{\eps}\right)$ upper bound, as follows.
Note that, since $|\C| \geq 2$, we must have $\s \geq 1$.
Fix any $\PXY \in \RE$.
Let $\T = \{ \DIS(V_{S,h}) : S \in \bigcup_{m \in \nats} \X^{m}, h \text{ a classifier} \}$,
and for each $x_{1},\ldots,x_{\s} \in \X$ and $y_{1},\ldots,y_{\s} \in \Y$, define
\begin{equation*}
\phi_{\s}(x_1,\ldots,x_{\s},y_{1},\ldots,y_{\s}) = \DIS(\{g \in \C : \forall i \leq \s, g(x_{i}) = y_{i}\}) \in \T.
\end{equation*}
Let $\tilde{c}^{\prime}$ be as in Lemma~\ref{lem:compression-eps-net},
and define 
$\conf^{\prime} = \conf / \left( 2 \lceil \log_{2}(1/\eps) \rceil \right)$,
$\ell = \left\lceil 2\tilde{c}^{\prime} \left( \s \Log( 3 \tilde{c}^{\prime} ) + \Log(1 / \conf^{\prime}) \right) \right\rceil$,
$m = \left\lceil 2 \tilde{c}^{\prime} \s \right\rceil$,
and $\tilde{j} = \left\lceil (2 m \lceil \log_{2}(2/\conf^{\prime}) \rceil + 2 \ell ) / \eps \right\rceil$.
Consider the following algorithm.

\begin{bigboxit}
\RealizableAlg\\
Input: label budget $n$\\
Output: classifier $\hat{h}_{n}$\\
{\vskip -2mm}\line(1,0){420}\\
0. $V_{0} \gets \C$, $\bar{j}_{0} = 0$ \\
1. For $k = 1,2,\ldots, \lfloor n / m \rfloor$\\
2. \quad If $|\{ j \in \{\bar{j}_{k-1}+1,\ldots,\bar{j}_{k-1}+\tilde{j}\} : X_{j} \in \DIS(V_{k-1}) \}| < m \lceil \log_{2}( 2/\conf^{\prime} ) \rceil + \ell$\\
3. \qquad Return any $\hat{h}_{n} \in V_{k-1}$ (or an arbitrary classifier $\hat{h}_{n}$ if $V_{k-1} = \emptyset$)\\
4. \quad Let $j_{k,1},\ldots,j_{k,m\lceil\log_{2}(2/\conf^{\prime})\rceil+\ell}$ denote the $m\lceil\log_{2}(2/\conf^{\prime})\rceil+\ell$ smallest indices in the set\\
\qquad $\{ j \in \{ \bar{j}_{k-1} + 1, \ldots, \bar{j}_{k-1} + \tilde{j} \} : X_{j} \in \DIS(V_{k-1}) \}$ (in increasing order)\\
5. \quad Let $\bar{j}_{k} = j_{k, m \lceil \log_{2}(2/\conf^{\prime}) \rceil + \ell}$\\
6. \quad For each $i \in \nats$, let 
\begin{multline*}
I_{i} = \vast\{ (i_{1},\ldots,i_{\s},y_{1},\ldots,y_{\s}) \in \nats^{\s}\times\Y^{\s} : m (i-1) < i_{1} \leq \cdots \leq i_{\s} \leq m i,
\\ \sum_{t=m(i-1)+1}^{m i} \ind_{\phi_{\s}(X_{j_{k,i_{1}}},\ldots,X_{j_{k,i_{\s}}}, y_{1},\ldots,y_{\s})}(X_{j_{k,t}})=0 \vast\}
\end{multline*}
7. \quad Let 
\[\hat{i}_{k} = \argmin\limits_{i \in \{1,\ldots,\lceil \log_{2}(2/\conf^{\prime}) \rceil \}} \max\limits_{(i_{1},\ldots,i_{\s},y_{1},\ldots,y_{\s}) \in I_{i}} \sum\limits_{t= m \lceil \log_{2}(2/\conf^{\prime}) \rceil+1}^{m \lceil \log_{2}(2/\conf^{\prime}) \rceil + \ell} \ind_{\phi_{\s}(X_{j_{k,i_{1}}},\ldots,X_{j_{k,i_{\s}}},y_{1},\ldots,y_{\s})}(X_{j_{k,t}})\]
8. \quad Request the label $Y_{j_{k,t}}$ for each $t \in\left\{ m \left(\hat{i}_{k}-1\right)+1,\ldots, m \hat{i}_{k} \right\}$\\
9. \quad Let $V_{k} \gets \left\{ g \in V_{k-1} : \forall t \in \left\{ m \left(\hat{i}_{k}-1\right)+1,\ldots, m \hat{i}_{k} \right\}, g(X_{j_{k,t}}) = Y_{j_{k,t}} \right\}$\\
10. Return any $\hat{h}_{n} \in V_{\lfloor n / m \rfloor}$
\end{bigboxit}

Fix any $k \in \{1,\ldots,\lfloor n/m \rfloor\}$.
In the event that $V_{k-1}$ is defined, 
let 
\begin{equation*}
M_{k} = \left|\left\{ j \in \left\{ \bar{j}_{k-1}+1,\ldots,\bar{j}_{k-1}+\tilde{j} \right\} : X_{j} \in \DIS(V_{k-1}) \right\} \right|.
\end{equation*}
By a Chernoff bound (applied under the conditional distribution given $V_{k-1}$ and $\bar{j}_{k-1}$)
and the law of total probability (integrating out $V_{k-1}$ and $\bar{j}_{k-1}$), there is an event 
$E_{k}^{\prime}$ of probability at least $1-\conf^{\prime}$, on which,
if $V_{k-1}$ is defined and satisfies 
\begin{equation}
\label{eqn:Vk1-constraint}
\Px(\DIS(V_{k-1})) \geq 2 \tilde{j}^{-1}\left( m \lceil \log_{2}(2/\conf^{\prime}) \rceil + \ell \right),
\end{equation}
then 
$M_{k} \geq (1/2) \tilde{j}  \Px(\DIS(V_{k-1})) \geq m \lceil \log_{2}(2/\conf^{\prime}) \rceil + \ell$,
in which case the algorithm will execute Steps 4-9 for this particular value of $k$,
and in particular, the set $V_{k}$ is defined.
In this case, denote $\U_{k} = \left\{X_{j_{k,t}} : t \in \left\{ m \left( \hat{i}_{k} - 1\right) + 1, \ldots, m \hat{i}_{k} \right\}\right\}$,
which is well-defined in this case.

Next note that, on the event that $V_{k-1}$ is defined,
the $M_{k}$ samples
\begin{equation*}
\left\{ X_{j} : j \in \left\{ \bar{j}_{k-1} + 1, \ldots, \bar{j}_{k-1} + \tilde{j} \right\}, X_{j} \in \DIS(V_{k-1}) \right\}
\end{equation*}
are conditionally independent
given $V_{k-1}$, $\bar{j}_{k-1}$, and $M_{k}$,
each having conditional distribution $\Px(\cdot | \DIS(V_{k-1}))$.
Thus, applying Lemma~\ref{lem:compression-eps-net} under the conditional distribution given $V_{k-1}$, $\bar{j}_{k-1}$, and $M_{k}$,
combined with the law of total probability (integrating out $V_{k-1}$, $\bar{j}_{k-1}$, and $M_{k}$),
we have that there exists an event $E_{k}$ of probability at least $1-\conf^{\prime}$,
on which, if $V_{k-1}$ is defined, and $M_{k} \geq m \lceil \log_{2}(2/\conf^{\prime}) \rceil + \ell$, 
then $\U_{k}$
is a $\frac{1}{2}$-net of $\Px(\cdot | \DIS(V_{k-1}))$ for 
\begin{equation}
\label{eqn:phii-set-netted}
\left\{ \phi_{\s}(X_{j_{k,i_{1}}},\ldots,X_{j_{k,i_{\s}}},y_{1},\ldots,y_{\s}) : m\left( \hat{i}_{k} - 1 \right) + 1 < i_{1} \leq \cdots \leq i_{\s} \leq m\hat{i}_{k}, y_{1},\ldots,y_{\s} \in \Y \right\}.
\end{equation}
Together, we have that on $E_{k} \cap E_{k}^{\prime}$, if $V_{k-1}$ is defined and satisfies \eqref{eqn:Vk1-constraint},
then
$\U_{k}$
is a $\frac{1}{2}$-net of $\Px(\cdot | \DIS(V_{k-1}))$ for the collection \eqref{eqn:phii-set-netted}.

In particular, Theorem~\ref{thm:xtd} implies that, for any $x_1,\ldots,x_m \in \X^{m}$ and classifier $f \in \C$,
$\exists i_{1},\ldots,i_{\s} \in \{1,\ldots,m\}$
such that $\{ g \in \C : \forall j \leq \s, g(x_{i_{j}}) = f(x_{i_{j}}) \} = \{ g \in \C : \forall i \leq m, g(x_{i}) = f(x_{i}) \}$
(see the discussion in Section~\ref{sec:hatn}),
and since the left hand side is invariant to permutations of the $i_{j}$ values,
without loss of generality we may take $i_{1} \leq \cdots \leq i_{\s}$.
This implies that on $E_{k} \cap E_{k}^{\prime}$, if $V_{k-1}$ is defined and satisfies \eqref{eqn:Vk1-constraint},
then $\exists i_{1}^{\prime},\ldots,i_{\s}^{\prime} \in \left\{ m\left(\hat{i}_{k}-1\right)+1,\ldots,m\hat{i}_{k} \right\}$ with 
$i_{1}^{\prime} \leq \cdots \leq i_{\s}^{\prime}$ such that 
\begin{multline*}
\phi_{\s}(X_{j_{k,i_{1}^{\prime}}},\ldots,X_{j_{k,i_{\s}^{\prime}}},f(X_{j_{k,i_{1}^{\prime}}}),\ldots,f(X_{j_{k,i_{\s}^{\prime}}})) 
\\ = \DIS\left( \left\{ g \in \C : \forall t \in \left\{ m\left(\hat{i}_{k}-1\right)+1,\ldots,m\hat{i}_{k}\right\}, g(X_{j_{k,t}}) = f(X_{j_{k,t}})\right\} \right)
= \DIS(V_{\U_{k},f}),
\end{multline*}
so that 
\begin{multline*}
\DIS(V_{\U_{k},f}) \in
\\ \left\{ \phi_{\s}(X_{j_{k,i_{1}}},\ldots,X_{j_{k,i_{\s}}},y_{1},\ldots,y_{\s}) : m \left(\hat{i}_{k}-1\right) < i_{1} \leq \cdots \leq i_{\s} \leq m \hat{i}_{k}, y_{1},\ldots,y_{\s} \in \Y\right\}.
\end{multline*}
But we certainly have 
$\DIS( V_{\U_{k},f} ) \cap \U_{k} = \emptyset$.
Thus, by the $\frac{1}{2}$-net property,
on the event $E_{k} \cap E_{k}^{\prime}$, if $V_{k-1}$ is defined and satisfies \eqref{eqn:Vk1-constraint},
then every $f \in \C$ has 
\begin{equation}
\label{eqn:conditional-dis}
\Px\left(\DIS( V_{\U_{k},f} ) \Big| \DIS(V_{k-1}) \right) \leq \frac{1}{2}.
\end{equation}
Also note that, since $\PXY \in \RE$, we have $\target_{\PXY} \in \C$,
and furthermore that there is an event $E$ of probability one, on which $\forall j, Y_{j} = \target_{\PXY}(X_{j})$.
In particular, on $E$, if $V_{k-1}$ and $V_{k}$ are defined, then $V_{k} = V_{\U_{k},\target_{\PXY}} \cap V_{k-1}$,
which implies $\DIS(V_{k}) = \DIS\left(V_{\U_{k},\target_{\PXY}} \cap V_{k-1}\right) \subseteq \DIS(V_{k-1})$.
Thus, applying \eqref{eqn:conditional-dis} with $f = \target_{\PXY}$, 
we have that on the event $E \cap E_{k} \cap E_{k}^{\prime}$, if $V_{k-1}$ is defined and satisfies \eqref{eqn:Vk1-constraint},
then $V_{k}$ is defined and satisfies
\begin{align*}
\Px(\DIS(V_{k})) 
& = \Px(\DIS(V_{k}) | \DIS(V_{k-1})) \Px(\DIS(V_{k-1})) 
\\ & \leq \Px\left(\DIS\left(V_{\U_{k},\target_{\PXY}}\right) \Big| \DIS(V_{k-1})\right) \Px(\DIS(V_{k-1})) \leq \frac{1}{2} \Px(\DIS(V_{k-1})).
\end{align*}

Now suppose $\lfloor n/m \rfloor \geq \lceil \log_{2}(1/\eps) \rceil$.
Applying the above to every $k \leq \lceil \log_{2}(1/\eps) \rceil$, 
we have that there exist events $E_{k}^{\prime}$ and $E_{k}$ for each $k \in \left\{1,\ldots, \lceil \log_{2}(1/\eps) \rceil \right\}$,
each of probability at least $1-\conf^{\prime}$, such that 
on the event $E \cap \bigcap_{k=1}^{\lceil \log_{2}(1/\eps) \rceil} E_{k}^{\prime} \cap E_{k}$, 
every $k \in \left\{1,\ldots,\lceil \log_{2}(1/\eps) \rceil \right\}$ with $V_{k-1}$ defined
either has $\Px(\DIS(V_{k-1})) < 2 \tilde{j}^{-1} \left( m \lceil \log_{2}(2/\conf^{\prime}) \rceil + \ell \right)$
or else $V_{k}$ is defined and satisfies $\Px(\DIS(V_{k})) \leq \frac{1}{2} \Px(\DIS(V_{k-1}))$.
Since $V_{0} = \C$ is defined, by induction we have that on the event 
$E \cap \bigcap_{k=1}^{\lceil \log_{2}(1/\eps) \rceil} E_{k}^{\prime} \cap E_{k}$, 
either some $k \in \left\{ 1, \ldots, \lceil \log_{2}(1/\eps) \rceil \right\}$ has $V_{k-1}$ defined and satisfies 
$\Px(\DIS(V_{k-1})) < 2 \tilde{j}^{-1} \left( m \lceil \log_{2}(2/\conf^{\prime}) \rceil + \ell \right)$,
or else every $k \in \left\{ 1, \ldots, \lceil \log_{2}(1/\eps) \rceil \right\}$ has $V_{k}$ defined
and satisfying $\Px(\DIS(V_{k})) \leq \frac{1}{2} \Px(\DIS(V_{k-1}))$.  In particular, in this latter case, 
since $\Px(\DIS(V_{0})) \leq 1$, by induction we have $\Px(\DIS(V_{\lceil \log_{2}(1/\eps) \rceil})) \leq 2^{-\lceil \log_{2}(1/\eps) \rceil} \leq \eps$.

Also note that $2 \tilde{j}^{-1} \left( m \lceil \log_{2}(2/\conf^{\prime}) \rceil + \ell \right) \leq \eps$.
Thus, denoting by $\hat{k}$ the largest $k \leq \lfloor n/m \rfloor$ for which $V_{k}$ is defined
(which also implies $V_{k}$ is defined for every $k \in \{0,\ldots,\hat{k}\}$),
on the event $E \cap \bigcap_{k=1}^{\lceil \log_{2}(1/\eps) \rceil} E_{k}^{\prime} \cap E_{k}$, 
either some $k \leq (\hat{k}+1) \land \lceil \log_{2}(1/\eps) \rceil$ has $\Px(\DIS(V_{k-1})) < \eps$,
so that (since $k \mapsto V_{k}$ is nonincreasing for $k \leq \hat{k}$) $\Px(\DIS(V_{\hat{k}})) \leq \Px(\DIS(V_{k-1})) < \eps$,
or else $\hat{k} \geq \lceil \log_{2}(1/\eps) \rceil$, so that $\Px(\DIS(V_{\hat{k}})) \leq \Px(\DIS(V_{\lceil \log_{2}(1/\eps) \rceil})) \leq \eps$.
Thus, on the event $E \cap \bigcap_{k=1}^{\lceil \log_{2}(1/\eps) \rceil} E_{k}^{\prime} \cap E_{k}$,
in any case we have $\Px(\DIS(V_{\hat{k}})) \leq \eps$.
Furthermore, by the realizable case assumption, we have $\target_{\PXY} \in V_{0}$,
and if $\target_{\PXY} \in V_{k-1}$ in Step 9, then (on the event $E$) $\target_{\PXY} \in V_{k}$ as well.  Thus, by induction, on the event $E$, 
$\target_{\PXY} \in V_{\hat{k}}$.
In particular, this also implies $V_{\hat{k}} \neq \emptyset$ on $E$, so that there exist valid choices of $\hat{h}_{n}$ in $V_{\hat{k}}$ upon reaching the ``Return'' step 
(Step 3, if $\hat{k} < \lfloor n/m \rfloor$, or Step 10, if $\hat{k} = \lfloor n/m \rfloor$).
Thus, $\hat{h}_{n} \in V_{\hat{k}}$ as well on $E$, so that on the event $E$ we have $\left\{x : \hat{h}_{n}(x) \neq \target_{\PXY}(x)\right\} \subseteq \DIS(V_{\hat{k}})$.
Therefore, on the event $E \cap \bigcap_{k=1}^{\lceil \log_{2}(1/\eps) \rceil} E_{k}^{\prime} \cap E_{k}$,
we have 
\begin{equation*}
\er_{\PXY}(\hat{h}_{n}) = \Px\left( x : \hat{h}_{n}(x) \neq \target_{\PXY}(x) \right) \leq \Px\left( \DIS\left( V_{\hat{k}} \right) \right) \leq \eps.
\end{equation*}

Finally, by a union bound, the event $E \cap \bigcap_{k=1}^{\lceil \log_{2}(1/\eps) \rceil} E_{k}^{\prime} \cap E_{k}$ has probability at least
$1 - \lceil \log_{2}(1/\eps) \rceil 2 \conf^{\prime} = 1 - \conf$.
Noting that the above argument holds for any $\PXY \in \RE$,
and that the condition $\lfloor n/m \rfloor \geq \lceil \log_{2}(1/\eps) \rceil$ is satisfied for any $n \geq 9 \tilde{c}^{\prime} \s \Log(1/\eps)$,
this completes the proof that $\LC_{\RE}(\eps,\conf) \leq 9 \tilde{c}^{\prime} \s \Log(1/\eps) \lesssim \s \Log(1/\eps)$.
\end{proof}

\subsection{The Noisy Cases}
\label{app:noisy-proofs}

To extend the above ideas to noisy settings, 
we make use of a novel modification of a technique of \citet*{kaariainen:06}.
We first partition the data sequence into three parts.
For $m \in \nats$, let $X_{m}^{1} = X_{3(m-1)+1}$, $X_{m}^{2} = X_{3(m-1)+2}$, and let $X_{m}^{3} = X_{3m}$ and $Y_{m}^{3} = Y_{3m}$;
also denote $\DataX_{1} = \{X_{m}^{1}\}_{m=1}^{\infty}$, $\DataX_{2} = \{X_{m}^{2}\}_{m=1}^{\infty}$, $\DataX_{3} = \{X_{m}^{3}\}_{m=1}^{\infty}$, $\DataY_{3} = \{Y_{m}^{3}\}_{m=1}^{\infty}$,
and $\Data = \{(X_{m},Y_{m})\}_{m=1}^{\infty}$.
Additionally, it will simplify some of the proofs to further partition $\DataX_{3}$ and $\DataY_{3}$, as follows.
Fix any bijection $\phi : \nats^{2} \to \nats$, and for each $m,\ell \in \nats$, let $X_{m,\ell}^{3} = X_{\phi(m,\ell)}^{3}$ and $Y_{m,\ell}^{3} = Y_{\phi(m,\ell)}^{3}$.

Fix values $\eps,\conf \in (0,1)$, and let $\hat{\gamma}_{\eps}$ be a value in $[\eps/2,1]$.
Let $k_{\eps} = \lceil \log_{2}(8/\hat{\gamma}_{\eps}) \rceil$,
and for each $k \in \{2,\ldots,k_{\eps}\}$, define
\begin{equation*}
\tilde{m}_{k} = \left\lceil \frac{16\max\{c,8\} k_{\eps}}{2^{k}\eps} \left( \vc \Log\left(\frac{2 k_{\eps}}{\eps}\right) + \Log\left(\frac{64 k_{\eps}}{\conf}\right) \right) \right\rceil,
\end{equation*}
for $c$ as in Lemma~\ref{lem:vc-cover}.
Also define $\tilde{m}_{k_{\eps}+1} = 0$, $\tilde{m} = \tilde{m}_{2}$.
and $q_{\eps,\conf} = 2 + \left\lceil 2^{2k_{\eps}+4} \ln\left( \frac{32 \tilde{m} 2^{2k_{\eps}+3}}{\conf} \right) \right\rceil$.
Also, for each $m \in \{1,\ldots,\tilde{m}\}$, 
define $\tilde{k}_{m} = \max\left\{ k \in \{2,\ldots,k_{\eps}\} : m \leq \tilde{m}_{k} \right\}$
and let $\tilde{q}_{m} = 2^{3+2\tilde{k}_{m}} \ln(32 \tilde{m} q_{\eps,\conf} / \conf)$.
Fix a value $\tau = \frac{\conf \eps}{512 \tilde{m}}$.
Let $J_{\tau,\conf/2}$ be as in Lemma~\ref{lem:empirical-approx}, applied to the sequence $X_{m}^{\prime} = X_{m}^{1}$;
to simplify notation, in this section we abbreviate $J = J_{\tau,\conf/2}$.
Also, for each $x \in \X$, denote by $J(x)$ the (unique) set $A \in J$ with $x \in A$,
and for each $m \in \{1,\ldots,\tilde{m}\}$, we abbreviate $J_{m} = J(X_{m}^{2})$.
Now consider the following algorithm.

\begin{bigboxit}
\RQCAL\\
Input: label budget $n$\\
Output: classifier $\hat{h}_{n}$\\
{\vskip -2mm}\line(1,0){420}\\
0. $V_{0} \gets \C$, $t \gets 0$, $m \gets 0$\\
1. While $t < n$ and $m < \tilde{m}$\\
2. \quad $m \gets m+1$\\
3. \quad If $X_{m}^{2} \in \DIS(V_{m-1})$\\
4. \qquad Run \Filter~with arguments $(n-t,m)$; \\
\qquad\quad let $(q,y)$ be the returned values; let $t \gets t+q$\\
5. \qquad If $y \neq 0$ and $\exists h \in V_{m-1}$ with $h(X_{m}^{2}) = y$\\
6. \quad\qquad Let $V_{m} \gets \{h \in V_{m-1} : h(X_{m}^{2}) = y\}$\\
7. \qquad Else let $V_{m} \gets V_{m-1}$\\
8. \quad Else let $V_{m} \gets V_{m-1}$\\
9. Return any $\hat{h}_{n} \in V_{m}$
\end{bigboxit}

\begin{bigboxit}
\Filter\\
Input: label budget $n$, data point index $m$ \\
Output: query counter $q$, value $y$\\
{\vskip -2mm}\line(1,0){420}\\
0. $\sigma_{m,0} \gets 0$, $q \gets 0$, $\ell_{m,0} \gets 0$ \\  
1. Repeat\\
2. \quad Let $\ell_{m,q+1} \gets \min\{ \ell > \ell_{m,q} : X_{m,\ell}^{3} \in J_{m} \}$ (or $\ell_{m,q+1} \gets 1$ if this set is empty) \\
3. \quad Request the label $Y_{m,\ell_{m,q+1}}^{3}$; let $\sigma_{m,q+1} \gets \sigma_{m,q} + Y_{m,\ell_{m,q+1}}^{3}$; let $q \gets q+1$  \\
4. \quad If $|\sigma_{m,q}| \geq 3 \sqrt{ 2 q \ln( 32 \tilde{m} q_{\eps,\conf} / \conf ) }$\\
5. \qquad Return $(q,\sign(\sigma_{m,q}))$\\
6. \quad Else if $q \geq \min\{ n, \tilde{q}_{m} \}$\\
7. \qquad Return $(q,0)$
\end{bigboxit}

In this algorithm, the first part of the data (namely, $\DataX_{1}$) is used to partition the space via Lemma~\ref{lem:empirical-approx},
so that each cell of the partition has $\target_{\PXY}$ nearly-constant within it (assuming $\target_{\PXY} \in \C$).
The second part, $\DataX_{2}$, is used to simulate a commonly-studied active learning algorithm for the realizable case (namely, the algorithm of \citealp*{cohn:94}),
with two significant modifications.  First, instead of directly requesting the label of a point, we use samples
from the third part of the data (i.e., $\DataX_{3}$) that co-occur in the same cell of the partition as the would-be query point,
repeatedly requesting for labels from that cell and using the majority vote of these returned labels in place of
the label of the original point.  Second, we discard a point $X_{m}^{2}$  if we cannot identify a clear majority 
label within a certain number of queries, which decreases as the algorithm runs.  Since this second modification 
often ends up rejecting more samples in cells with higher noise rates than those with lower noise rates,
this effectively alters the marginal distribution over $\X$, shifting the distribution to favor less-noisy regions.

For the remainder of Appendix~\ref{app:noisy-proofs}, we fix an arbitrary probability measure $\PXY$ over $\X\times\Y$ with $\target_{\PXY} \in \C$,
and as usual, we denote by $\Px(\cdot) = \PXY(\cdot\times\Y)$ the marginal of $\PXY$ over $\X$.
For any $x \in \X$, define $\gamma_{x} = \left| \eta(x;\PXY) - \frac{1}{2} \right|$,
and define 
\begin{equation*}
\gamma_{\eps} = \sup\left\{ \gamma \in (0,1/2] : \gamma \Px(x : \gamma_{x} \leq \gamma) \leq \eps/2 \right\}.
\end{equation*}
Also, for the remainder of Appendix~\ref{app:noisy-proofs}, we suppose $\hat{\gamma}_{\eps}$ is chosen to be in the range $[\eps/2, \gamma_{\eps}]$.
For each $A \in J$, define 
\begin{equation*}
y_{A} = \argmax_{y \in \Y} \Px\left( x \in A : \target_{\PXY}(x) = y \right) = \sign\left( \int_{A} \target_{\PXY} {\rm d}\Px \right),
\end{equation*}
and if $\Px(A) > 0$, define $\eta(A;\PXY) = \PXY( A \times \{1\} | A \times \Y )$ (i.e., the average value of $\eta(x;\PXY)$ over $x \in A$),
and let $\gamma_{A} = \left| \eta(A;\PXY) - \frac{1}{2} \right|$.  For completeness, for any $A \in J$ with $\Px(A) = 0$, 
define $\eta(A;\PXY) = 1/2$ and $\gamma_{A} = 0$.
Additionally, for each $n \in \nats \cup \{\infty\}$ and $m \in \nats$,
let $(\hat{q}_{n,m},\hat{y}_{n,m})$ denote the return values of \Filter~when run with 
arguments $(n,m)$.

Denote by $E_{1}$ the $\DataX_{1}$-measurable event of probability at least $1-\conf/2$ implied by Lemma~\ref{lem:empirical-approx}, on which
every $h \in \C$ has
\begin{equation}
\label{eqn:abstract-repeated-queries-J-approx}
\sum_{A \in J} \min_{y \in \Y} \Px\left( x \in A : h(x) = y \right) \leq \tau
\end{equation}
and $\forall \gamma > 0$,
\begin{equation}
\label{eqn:abstract-repeated-queries-J-conditional-approx}
\Px\left( \bigcup \left\{ A \in J : \min_{y \in \Y} \Px\left( x \in A : h(x) = y \right) > \gamma \Px(A) \right\} \right) \leq \frac{\tau}{\gamma}.
\end{equation}

We now proceed to characterize the behaviors of \Filter~and \RQCAL~via the following sequence of lemmas.

\begin{lemma}
\label{lem:E0}
There exists a $(\DataX_{1},\DataX_{2},\DataX_{3})$-measurable event $E_{0}$ of probability $1$, on which $\forall m \in \{1,\ldots,\tilde{m}\}$, 
$\Px(J_{m}) > 0$ and $|\{ \ell \in \nats : X_{m,\ell}^{3} \in J_{m} \}| = \infty$.
\end{lemma}
\begin{proof}
For each $m$, since each $A \in J$ with $\Px(A) = 0$ has $\P( X_{m}^{2} \in A ) = 0$,
and $J$ has finite size, a union bound implies $\P( \Px(J_{m})=0 ) = 0$.
The strong law of large numbers (applied under the conditional distribution given $J_{m}$) and 
the law of total probability implies that $\frac{1}{\ell} \sum_{j=1}^{\ell} \ind_{J_{m}}(X_{m,j}^{3}) \to \Px(J_{m})$ with probability $1$,
so that when $\Px(J_{m}) > 0$, $\sum_{j=1}^{\ell} \ind_{J_{m}}(X_{m,j}^{3}) \to \infty$.
Finally, a union bound implies $\P( \exists m \leq \tilde{m} : \Px(J_{m})=0 \text{ or } |\{\ell \in \nats : X_{m,\ell}^{3} \in J_{m}\}| < \infty) 
\leq \sum_{m=1}^{\tilde{m}} \P( \Px(J_{m}) = 0 ) + \P( \Px(J_{m})>0 \text{ and } |\{\ell \in \nats : X_{m,\ell}^{3} \in J_{m}\}| < \infty) = 0$.
\end{proof}

\begin{lemma}
\label{lem:E2}
There exists a $(\DataX_{1},\DataX_{2})$-measurable event $E_{2}$ of probability at least $1-\tau \tilde{m} \geq 1-\conf/512$ such that, on $E_{1} \cap E_{2}$,
every $m \in \{1,\ldots,\tilde{m}\}$ has $\target_{\PXY}(X_{m}^{2}) = y_{J_{m}}$.
\end{lemma}
\begin{proof}
Noting that, on $E_{1}$, \eqref{eqn:abstract-repeated-queries-J-approx} implies that
\begin{align*}
\Px\left( x : \target_{\PXY}(x) \neq y_{J(x)} \right) 
& = \sum_{A \in J} \Px\left( x \in A : \target_{\PXY}(x) \neq y_{A} \right) 
\\ & = \sum_{A \in J} \min_{y \in \Y} \Px\left( x \in A : \target_{\PXY}(x) = y \right)
\leq \tau,
\end{align*}
the result follows by a union bound.
\end{proof}

\begin{lemma}
\label{lem:E3}
There exists a $(\DataX_{1},\DataX_{2})$-measurable event $E_{3}$ of probability at least $1-\frac{128 \tau}{\eps}\tilde{m} \geq 1 - \conf/4$ such that, on $E_{1} \cap E_{3}$,
every $m \in \{1,\ldots,\tilde{m}\}$ has $\Px\left( x \in J_{m} : \target_{\PXY}(x) \neq y_{J_{m}} \right) \leq \frac{\eps}{128} \Px( J_{m})$.
\end{lemma}
\begin{proof}
Noting that, on $E_{1}$, \eqref{eqn:abstract-repeated-queries-J-conditional-approx} implies that
\begin{align*}
& \Px\left( x : \Px\left( x^{\prime} \in J(x) : \target_{\PXY}(x^{\prime}) \neq y_{J(x)} \right) > \frac{\eps}{128} \Px(J(x)) \right) 
\\ & = \Px\left( \bigcup \left\{ A \in J : \Px\left( x^{\prime} \in A : \target_{\PXY}(x^{\prime}) \neq y_{A} \right) > \frac{\eps}{128} \Px(A) \right\} \right)
\\ & =  \Px\left( \bigcup \left\{ A \in J : \min_{y \in \Y} \Px\left( x^{\prime} \in A : \target_{\PXY}(x^{\prime}) = y \right) > \frac{\eps}{128} \Px(A) \right\} \right)
\leq \frac{128 \tau}{\eps},
\end{align*}
the result follows by a union bound.
\end{proof}

\begin{lemma}
\label{lem:maj-label-noisy-vote}
$\forall A \in J$, 
\begin{equation*}
\PXY\left( A \times \{ y_{A} \} \right)
\geq \frac{1}{2} \Px(A) + \int_{A} \gamma_{x} \Px( {\rm d} x ) - \Px\left( x \in A : \target_{\PXY}(x) \neq y_{A} \right).
\end{equation*}
\end{lemma}
\begin{proof}
Any $A \in J$ has
\begin{align*}
\PXY\left( A \times \{ y_{A} \} \right)
& \geq 
\int_{A} \ind[\target_{\PXY}(x) = y_{A}] \left( \frac{1}{2} + \gamma_{x} \right) \Px( {\rm d} x )
\\ & \geq  
\int_{A} \left( \frac{1}{2} + \gamma_{x} \right) \Px( {\rm d} x )
- \Px\left( x \in A : \target_{\PXY}(x) \neq y_{A} \right) \notag
\\ & = \frac{1}{2} \Px(A) + \int_{A} \gamma_{x} \Px( {\rm d} x ) - \Px\left( x \in A : \target_{\PXY}(x) \neq y_{A} \right).
\end{align*}
\end{proof}

\begin{lemma}
\label{lem:good-gammas}
On the event $E_{0} \cap E_{1} \cap E_{3}$, 
every $m \in \{1,\ldots,\tilde{m}\}$ with $\gamma_{J_{m}} > \eps/128$ has $\PXY(J_{m} \times \{y_{J_{m}}\}) > \PXY( J_{m} \times \{-y_{J_{m}}\} )$,
and every $m \in \{1,\ldots,\tilde{m}\}$ with $\int_{J_{m}} \gamma_{x} \Px( {\rm d} x ) > (\eps/2) \Px(J_{m})$ has 
\begin{equation}
\label{eqn:abstract-repeated-queries-gamma-batching}
\int_{J_{m}} \gamma_{x} \Px({\rm d}x) 
\geq \gamma_{J_{m}}\Px(J_{m}) 
\geq \frac{63}{64} \int_{J_{m}} \gamma_{x} \Px({\rm d}x) 
> \frac{63}{128} \eps \Px(J_{m}).
\end{equation}
\end{lemma}
\begin{proof}
Jensen's inequality implies we always have $\gamma_{A} \Px(A) \leq \int_{A} \gamma_{x} \Px( {\rm d} x )$.
In particular, this implies that any $A \in J$ with $\Px(A) > 0$ and $\Px( x \in A : \target_{\PXY}(x) \neq y_{A} ) \leq \frac{\eps}{128} \Px(A)$
and $\gamma_{A} > \eps/128$ has $\int_{A} \gamma_{x} \Px( {\rm d} x ) - \Px( x \in A : \target_{\PXY}(x) \neq y_{A} ) \geq \gamma_{A} \Px(A) - \Px( x \in A : \target_{\PXY}(x) \neq y_{A} ) > (\eps/128) \Px(A) - (\eps/128) \Px(A) = 0$,
so that Lemma~\ref{lem:maj-label-noisy-vote} implies $\PXY( A \times \{y_{A}\} ) > \frac{1}{2}\Px(A)$, and therefore $\PXY( A \times \{y_{A}\} ) > \PXY( A \times \{-y_{A}\} )$.
Since Lemmas~\ref{lem:E0} and \ref{lem:E3} imply that, on $E_{0} \cap E_{1} \cap E_{3}$, 
for every $m \in \{1,\ldots,\tilde{m}\}$, $\Px(J_{m}) > 0$ and $\Px(x \in J_{m} : \target_{\PXY}(x) \neq y_{J_{m}} ) \leq \frac{\eps}{128} \Px(J_{m})$,
we have established the first claim in the lemma statement.

For the second claim, the first inequality follows by Jensen's inequality.
For the second inequality, note that any $A \in J$ has $\gamma_{A} \Px(A) \geq \PXY(A \times \{y_{A}\}) - \frac{1}{2}\Px(A)$,
so that Lemma~\ref{lem:maj-label-noisy-vote} implies $\gamma_{A} \Px(A) \geq \int_{A} \gamma_{x} \Px( {\rm d} x ) - \Px( x \in A : \target_{\PXY}(x) \neq y_{A} )$.
Therefore, since Lemma~\ref{lem:E3} implies that, on $E_{1} \cap E_{3}$, every $m \in \{1,\ldots,\tilde{m}\}$ has $\Px( x \in J_{m} : \target_{\PXY}(x) \neq y_{J_{m}} ) \leq \frac{\eps}{128} \Px(J_{m})$,
we have that on $E_{1} \cap E_{3}$, any $m \in \{1,\ldots,\tilde{m}\}$ with $\int_{J_{m}} \gamma_{x} \Px( {\rm d} x ) > (\eps/2) \Px(J_{m})$
has $\Px( x \in J_{m} : \target_{\PXY}(x) \neq y_{J_{m}} ) \leq \frac{1}{64} \int_{J_{m}} \gamma_{x} \Px( {\rm d} x )$,
so that $\gamma_{J_{m}} \Px(J_{m}) \geq \int_{J_{m}} \gamma_{x} \Px( {\rm d} x ) - \Px( x \in J_{m} : \target_{\PXY}(x) \neq y_{J_{m}} ) \geq \frac{63}{64} \int_{J_{m}} \gamma_{x} \Px( {\rm d} x )$.
The final inequality then follows by the assumption that $\int_{J_{m}} \gamma_{x} \Px( {\rm d} x ) > (\eps/2) \Px(J_{m})$.
\end{proof}

\begin{lemma}
\label{lem:bucketed-gamma-probs}
On $E_{1}$, $\forall \gamma > (1/4)\gamma_{\eps}$,
\begin{equation*}
\Px\left( \bigcup \left\{ A \in J : \gamma_{A} \leq \gamma \right\} \right) 
\leq 3 \Px(x : \gamma_{x} < 4 \gamma),
\end{equation*}
and $\forall \gamma \in (0, (1/4)\gamma_{\eps}]$,
\begin{equation*}
\Px\left( \bigcup \left\{ A \in J : \gamma_{A} \leq \gamma \right\} \right) 
\leq \frac{3 \eps}{2 \gamma_{\eps}}.
\end{equation*}
\end{lemma}
\begin{proof}
By Markov's inequality, for any $\gamma > 0$, any $A \in J$ with $\int_{A} \gamma_{x} \Px( {\rm d} x ) \leq \gamma \Px(A)$ 
must have $\Px( x \in A : \gamma_{x} \geq 2\gamma ) \leq \frac{1}{2}\Px(A)$, which implies $\Px( x \in A : \gamma_{x} < 2\gamma ) \geq \frac{1}{2}\Px(A)$.
Therefore, 
\begin{multline}
\label{eqn:int-gamma-margin}
\Px\left( \bigcup \left\{ A \in J : \int_{A} \gamma_{x} \Px( {\rm d} x ) \leq \gamma \Px(A) \right\} \right)
\\ \leq \Px\left( \bigcup \left\{ A \in J : \Px(x \in A : \gamma_{x} < 2\gamma) \geq \frac{1}{2}\Px(A) \right\} \right)
\leq 2\Px(x : \gamma_{x} < 2\gamma),
\end{multline}
where the last inequality is due to Markov's inequality.

Also, for every $\gamma > 0$,
since $\gamma_{A} \Px(A) \geq \PXY(A \times \{y_{A}\}) - \frac{1}{2}\Px(A)$,
\begin{align*}
\Px\left( \bigcup \left\{ A \in J : \gamma_{A} \leq \gamma \right\} \right) 
& = \Px\left( \bigcup \left\{ A \in J : \gamma_{A} \Px(A) \leq \gamma \Px(A) \right\} \right)
\\ & \leq \Px\left( \bigcup \left\{ A \in J : \PXY(A \times \{y_{A}\}) - \frac{1}{2}\Px(A) \leq \gamma \Px(A) \right\} \right).
\end{align*}
Lemma~\ref{lem:maj-label-noisy-vote} implies $\PXY(A \times \{y_{A}\}) - \frac{1}{2}\Px(A) \geq \int_{A} \gamma_{x} \Px( {\rm d} x ) - \Px( x \in A : \target_{\PXY}(x) \neq y_{A} )$,
so that the above is at most
\begin{equation*}
\Px\left( \bigcup \left\{ A \in J : \int_{A} \gamma_{x} \Px( {\rm d} x ) \leq \gamma \Px(A) + \Px( x \in A : \target_{\PXY}(x) \neq y_{A} ) \right\} \right).
\end{equation*}
By a union bound, this is at most
\begin{multline}
\label{eqn:bucketed-gamma-probs-intermediate}
\Px\left( \bigcup \left\{ A \in J : \int_{A} \gamma_{x} \Px( {\rm d} x ) \leq 2 \gamma \Px(A) \right\} \right)
\\ + \Px\left( \bigcup \left\{ A \in J : \Px( x \in A : \target_{\PXY}(x) \neq y_{A} ) > \gamma \Px(A) \right\} \right).
\end{multline}
On $E_{1}$, \eqref{eqn:abstract-repeated-queries-J-conditional-approx} implies that
\begin{equation*}
\Px\left( \bigcup \left\{ A \in J : \Px( x \in A : \target_{\PXY}(x) \neq y_{A} ) > \gamma \Px(A) \right\} \right)
\leq \frac{\tau}{\gamma}
< \frac{\eps}{8\gamma}.
\end{equation*}
Furthermore, by \eqref{eqn:int-gamma-margin}, 
\begin{equation*}
\Px\left( \bigcup \left\{ A \in J : \int_{A} \gamma_{x} \Px( {\rm d} x ) \leq 2 \gamma \Px(A) \right\} \right)
\leq 2 \Px(x : \gamma_{x} < 4 \gamma).
\end{equation*}
Using these two inequalities to bound the two terms in \eqref{eqn:bucketed-gamma-probs-intermediate}, we have that
\begin{equation*}
\Px\left( \bigcup \left\{ A \in J : \gamma_{A} \leq \gamma \right\} \right)
\leq 2 \Px(x : \gamma_{x} < 4 \gamma) + \frac{\eps}{8\gamma}.
\end{equation*}
By definition of $\gamma_{\eps}$, if $\gamma > (1/4) \gamma_{\eps}$,
we must have $4 \gamma \Px(x : \gamma_{x} < 4 \gamma) \geq \gamma_{\eps}\Px(x : \gamma_{x} \leq \gamma_{\eps}) \geq \eps/2$,
so that $\frac{\eps}{8\gamma} \leq \Px(x : \gamma_{x} < 4 \gamma)$,
which implies 
\begin{equation*}
2 \Px(x : \gamma_{x} < 4 \gamma) + \frac{\eps}{8\gamma}
\leq 3 \Px(x : \gamma_{x} < 4 \gamma),
\end{equation*}
which establishes the first claim.
On the other hand, if $0 < \gamma \leq (1/4) \gamma_{\eps}$, we have
$4 \gamma \Px(x : \gamma_{x} < 4 \gamma) \leq \eps/2$, 
so that $2 \Px(x : \gamma_{x} < 4 \gamma) \leq \frac{\eps}{4 \gamma}$,
which implies
\begin{equation*}
2 \Px(x : \gamma_{x} < 4 \gamma) + \frac{\eps}{8\gamma}
\leq \frac{3 \eps}{8 \gamma}.
\end{equation*}
This establishes the second claim, since 
(combined with monotonicity of probabilities) it implies
\begin{equation*}
\Px\left( \bigcup \left\{ A \in J : \gamma_{A} \leq \gamma \right\} \right)
\leq \Px\left( \bigcup \left\{ A \in J : \gamma_{A} \leq (1/4) \gamma_{\eps} \right\} \right)
\leq \frac{3 \eps}{2 \gamma_{\eps}}.
\end{equation*}
\end{proof}

\begin{lemma}
\label{lem:excess-error-gamma-swap}
On $E_{1}$, $\forall h \in \C$, 
\begin{equation*}
\er_{\PXY}(h) - \er_{\PXY}(\target_{\PXY})
\leq 5\tau + \int \ind[ h(x) \neq \target_{\PXY}(x) ] 2\gamma_{J(x)} \Px({\rm d}x).
\end{equation*}
\end{lemma}
\begin{proof}
For any $h \in \C$, we generally have
\begin{equation*}
\er_{\PXY}(h) - \er_{\PXY}(\target_{\PXY})
= \int \ind[ h(x) \neq \target_{\PXY}(x) ] 2 \gamma_{x} \Px({\rm d}x).
\end{equation*}
For each $A \in J$, let $y_{A}^{h} = \argmax_{y \in \Y} \Px(x : h(x)=y)$.
$\forall x \in \X$, $\ind[ h(x) \neq \target_{\PXY}(x)] 2\gamma_{x} \leq 1$.  Therefore, 
\begin{multline}
\label{eqn:excess-error-gamma-swap-1}
\int \ind[ h(x) \neq \target_{\PXY}(x) ] 2 \gamma_{x} \Px({\rm d}x)
 \leq \Px\left(x : h(x) \neq y_{J(x)}^{h} \text{ or } \target_{\PXY}(x) \neq y_{J(x)}\right)
\\+ \int_{\left\{x : h(x) = y_{J(x)}^{h}, \target_{\PXY}(x) = y_{J(x)}\right\}} \ind[ y_{J(x)}^{h} \neq y_{J(x)} ] 2 \gamma_{x} \Px({\rm d}x).
\end{multline}
By a union bound, 
\begin{equation*}
\Px\left(x : h(x) \neq y_{J(x)}^{h} \text{ or } \target_{\PXY}(x) \neq y_{J(x)}\right)
\leq \Px\left(x : h(x) \neq y_{J(x)}^{h}\right) + \Px\left(x : \target_{\PXY}(x) \neq y_{J(x)}\right).
\end{equation*}
Furthermore, on $E_{1}$, \eqref{eqn:abstract-repeated-queries-J-approx} implies
the right hand side is at most $2\tau$.
Combining this with \eqref{eqn:excess-error-gamma-swap-1} implies
\begin{equation}
\label{eqn:excess-error-gamma-swap-2}
\er_{\PXY}(h) - \er_{\PXY}(\target_{\PXY}) \leq 2\tau + \int_{\left\{x : h(x) = y_{J(x)}^{h}, \target_{\PXY}(x) = y_{J(x)}\right\}} \ind[ y_{J(x)}^{h} \neq y_{J(x)} ] 2 \gamma_{x} \Px({\rm d}x).
\end{equation}
Also, 
\begin{align*}
& \int_{\left\{x : h(x) = y_{J(x)}^{h}, \target_{\PXY}(x) = y_{J(x)}\right\}} \ind[ y_{J(x)}^{h} \neq y_{J(x)} ] 2 \gamma_{x} \Px({\rm d}x)
\\ & = \sum_{A \in J : y_{A}^{h} \neq y_{A}} \int_{\left\{x \in A : h(x) = y_{A}^{h}, \target_{\PXY}(x)=y_{A}\right\}} 2 \gamma_{x} \Px({\rm d}x)
\leq \!\!\sum_{A \in J : y_{A}^{h} \neq y_{A}} \int_{\left\{x \in A : \target_{\PXY}(x)=y_{A}\right\}} 2 \gamma_{x} \Px({\rm d}x).
\end{align*}
Since $\target_{\PXY}(x) = \sign(2 \eta(x;\PXY) - 1)$ for every $x \in \X$, 
any measurable $C \subseteq \X$ has 
\begin{equation*}
\PXY\left( (x,y) : x \in C, y = \target_{\PXY}(x) \right) = \int_{C} \left( \frac{1}{2} + \gamma_{x} \right) \Px({\rm d}x).
\end{equation*}
Therefore, for each $A \in J$, 
\begin{align*}
\gamma_{A} \Px(A) 
& \geq \PXY(A \times \{y_{A}\}) - \frac{1}{2}\Px(A)
\geq \PXY\left( \left\{ x \in A : \target_{\PXY}(x) = y_{A} \right\} \times \{y_{A}\} \right) - \frac{1}{2}\Px(A)
\\ & = \int_{\left\{x \in A : \target_{\PXY}(x)=y_{A}\right\}} \left(\frac{1}{2}+\gamma_{x}\right)\Px({\rm d}x) - \frac{1}{2}\Px(A)
\\ & = \int_{\left\{x \in A : \target_{\PXY}(x)=y_{A}\right\}} \gamma_{x}\Px({\rm d}x) - \frac{1}{2}\Px\left(x \in A : \target_{\PXY}(x) \neq y_{A}\right).
\end{align*}
Therefore, 
\begin{equation*}
\sum_{A \in J : y_{A}^{h} \neq y_{A}} \int_{\left\{x \in A : \target_{\PXY}(x)=y_{A}\right\}} 2 \gamma_{x} \Px({\rm d}x)
\leq \sum_{A \in J : y_{A}^{h} \neq y_{A}} \Px\left(x \in A : \target_{\PXY}(x) \neq y_{A}\right) + 2 \gamma_{A} \Px(A).
\end{equation*}
On $E_{1}$, \eqref{eqn:abstract-repeated-queries-J-approx} implies that the right hand side is at most
\begin{equation*}
\tau + \sum_{A \in J : y_{A}^{h} \neq y_{A}} 2 \gamma_{A} \Px(A).
\end{equation*}
Combining this with \eqref{eqn:excess-error-gamma-swap-2}, we have that on $E_{1}$,
\begin{equation}
\label{eqn:excess-error-gamma-swap-3}
\er_{\PXY}(h) - \er_{\PXY}(\target_{\PXY}) \leq 3\tau + \sum_{A \in J : y_{A}^{h} \neq y_{A}} 2 \gamma_{A} \Px(A).
\end{equation}
For each $A \in J$ and $x \in A$, if $y_{A}^{h} \neq y_{A}$, then either $h(x) \neq \target_{\PXY}(x)$ holds,
or else one of $h(x) \neq y_{A}^{h}$ or $\target_{\PXY}(x) \neq y_{A}$ holds.
Thus, any $A \in J$ with $y_{A}^{h} \neq y_{A}$ has
\begin{align*}
& \Px(A) 
\leq \int_{A} \left( \ind\left[h(x) \neq \target_{\PXY}(x)\right] + \ind\left[h(x) \neq y_{A}^{h}\right] + \ind\left[\target_{\PXY}(x) \neq y_{A}\right] \right) \Px({\rm d}x)
\\ & = \Px\left(x \in A : h(x) \neq y_{A}^{h}\right) + \Px\left(x \in A : \target_{\PXY}(x) \neq y_{A}\right) + \int_{A} \ind\left[h(x) \neq \target_{\PXY}(x)\right] \Px({\rm d}x).
\end{align*}
Combined with \eqref{eqn:excess-error-gamma-swap-3}, this implies that on $E_{1}$,
\begin{multline*}
\er_{\PXY}(h) - \er_{\PXY}(\target_{\PXY})
\\ \leq 3\tau + \sum_{A \in J : y_{A}^{h} \neq y_{A}} 2 \gamma_{A} \bigg( \Px\left(x \in A : h(x) \neq y_{A}^{h}\right) + \Px\left(x \in A : \target_{\PXY}(x) \neq y_{A}\right) 
\\ + \int_{A} \ind\left[h(x) \neq \target_{\PXY}(x)\right] \Px({\rm d}x) \bigg).
\end{multline*}
Since $2\gamma_{A} \leq 1$, the right hand side is at most
\begin{multline*}
3\tau + \sum_{A \in J} \Px\left(x \in A : h(x) \neq y_{A}^{h}\right) + \sum_{A \in J} \Px\left(x \in A : \target_{\PXY}(x) \neq y_{A}\right) 
\\ + \sum_{A \in J : y_{A}^{h} \neq y_{A}} 2 \gamma_{A} \int_{A} \ind\left[h(x) \neq \target_{\PXY}(x)\right] \Px({\rm d}x),
\end{multline*}
and on $E_{1}$, \eqref{eqn:abstract-repeated-queries-J-approx} implies this is at most
\begin{align*}
& 5\tau + \!\!\!\!\sum_{A \in J : y_{A}^{h} \neq y_{A}} 2 \gamma_{A} \int_{A} \ind\left[h(x) \neq \target_{\PXY}(x)\right] \Px({\rm d}x)
\\ & \leq 5 \tau + \sum_{A \in J} \int_{A} \ind\left[h(x) \neq \target_{\PXY}(x)\right] 2 \gamma_{A} \Px({\rm d}x)
= 5 \tau + \int \ind\left[h(x) \neq \target_{\PXY}(x)\right] 2 \gamma_{J(x)} \Px({\rm d}x).
\end{align*}
\end{proof}

\begin{lemma}
\label{lem:filter-good-labels}
There is a $\Data$-measurable event $E_{4}$ of probability at least $1-\conf/32$ such that,
on $\bigcap_{j=0}^{4} E_{j}$, $\forall k \in \{2,\ldots,k_{\eps}\}$,
$\forall m \in \left\{ \tilde{m}_{k+1}+1, \ldots, \tilde{m}_{k}\right\}$,
$\forall n \in \nats\cup\{\infty\}$, $\hat{y}_{n,m} \in \{0,\target_{\PXY}(X_{m}^{2})\}$,
$\hat{q}_{n,m} \leq \left\lceil \frac{8}{\max\{ \gamma_{J_{m}}^{2}, 2^{-2k} \}} \ln\left(\frac{32 \tilde{m} q_{\eps,\conf}}{\conf}\right) \right\rceil$,
and if $\gamma_{J_{m}} \geq 2^{-k}$ then $\hat{y}_{\infty,m} = \target_{\PXY}(X_{m}^{2})$.
\end{lemma}
\begin{proof}
Since $\hat{q}_{n,m} \leq \hat{q}_{\infty,m}$, and $\hat{y}_{n,m} = 0$ whenever $\hat{q}_{n,m} < \hat{q}_{\infty,m}$,
it suffices to show the claims hold for $\hat{q}_{\infty,m}$ and $\hat{y}_{\infty,m}$ for each $m \in \left\{1,\ldots,\tilde{m}\right\}$.

For each $m \in \left\{1,\ldots,\tilde{m}\right\}$,
let $\ell_{m,1},\ell_{m,2},\ldots$ denote the increasing infinite subsequence of values $\ell \in \nats$ with $X_{m,\ell}^{3} \in J_{m}$,
guaranteed to exist by Lemma~\ref{lem:E0} on $E_{0}$;
also, for each $q \in \nats$, define $\sigma_{m,q} = \sum_{j=1}^{q} Y_{m,\ell_{m,j}}^{3}$.
Note that these definitions of $\ell_{m,q}$ and $\sigma_{m,q}$ agree with those defined in \Filter~for each $q \leq \hat{q}_{\infty,m}$.
Let $E_{4}$ denote the event that $E_{0}$ occurs and that $\forall m \in \left\{1,\ldots,\tilde{m}\right\}$, $\forall q \in \{1,\ldots,q_{\eps,\conf}\}$,
\begin{equation}
\label{eqn:sigma-concentration}
\left| \sigma_{m,q} - q (2 \eta(J_{m};\PXY) - 1)\right| \leq \sqrt{2q\ln\left(\frac{32 \tilde{m} q_{\eps,\conf}}{\conf}\right)}.
\end{equation}
For each $m \in \left\{1,\ldots,\tilde{m}\right\}$ and $q \in \{1,\ldots,q_{\eps,\conf}\}$,
Lemma~\ref{lem:E0} and Hoeffding's inequality imply that \eqref{eqn:sigma-concentration} holds with conditional probability (given $J_{m}$) 
at least $1 - \conf / (32 \tilde{m} q_{\eps,\conf})$.
The law of total probability and a union bound over values of $m$ and $q$ then imply that $E_{4}$ has probability at least 
$1 - \conf / 32$.

Now fix any $k \in \{2,\ldots,k_{\eps}\}$ and $m \in \left\{\tilde{m}_{k+1}+1,\ldots,\tilde{m}_{k}\right\}$.
Since $\tilde{k}_{m} = k$,
the condition in Step 6 guarantees
$\hat{q}_{\infty,m} \leq \left\lceil 2^{2k+3} \ln\left(\frac{32 \tilde{m} q_{\eps,\conf}}{\conf}\right) \right\rceil$.
Furthermore, if $\gamma_{J_{m}} \geq 2^{-k}$,
then for 
\begin{equation*}
q = \left\lceil \frac{8}{\gamma_{J_{m}}^{2}} \ln\left(\frac{32 \tilde{m} q_{\eps,\conf}}{\conf}\right) \right\rceil,
\end{equation*} 
we have 
\begin{equation*}
2q \gamma_{m} \geq 4 \sqrt{ 2 q \ln\left(\frac{32 \tilde{m} q_{\eps,\conf}}{\conf}\right) }.
\end{equation*}
In particular, recalling that $2q\gamma_{J_{m}} = |q(2\eta(J_{m};\PXY)-1)|$,
we have
\begin{equation}
\label{eqn:eta-Jm-lower-bound}
|q(2\eta(J_{m};\PXY)-1)| \geq 4\sqrt{ 2 q \ln\left(\frac{32 \tilde{m} q_{\eps,\conf}}{\conf}\right) }.
\end{equation}
Since $q_{\eps,\conf} \geq \left\lceil 2^{2k+3} \ln\left(\frac{32 \tilde{m} q_{\eps,\conf}}{\conf}\right) \right\rceil \geq q$,
the event $E_{4}$ implies that \eqref{eqn:sigma-concentration} holds, so that 
\begin{equation*}
\sigma_{m,q}
\geq q(2\eta(J_{m};\PXY)-1) - \sqrt{2q\ln\left(\frac{32 \tilde{m} q_{\eps,\conf}}{\conf}\right)}.
\end{equation*}
Thus, if $q(2\eta(J_{m};\PXY)-1) \geq 4\sqrt{2q\ln\left(\frac{32 \tilde{m} q_{\eps,\conf}}{\conf}\right)}$, 
the condition in Step 4 will imply $\hat{q}_{\infty,m} \leq q$, and since $q \leq \tilde{q}_{m}$, that $\hat{y}_{\infty,m} \in \Y$.
Likewise, \eqref{eqn:sigma-concentration} implies 
\begin{equation*}
\sigma_{m,q} \leq q(2\eta(J_{m};\PXY)-1) + \sqrt{2q\ln\left(\frac{32 \tilde{m} q_{\eps,\conf}}{\conf}\right)},
\end{equation*}
so that $q(2\eta(J_{m};\PXY)-1) \leq -4\sqrt{2q\ln\left(\frac{32 \tilde{m} q_{\eps,\conf}}{\conf}\right)}$
would also suffice to imply $\hat{q}_{\infty,m} \leq q$ and $\hat{y}_{\infty,m} \in \Y$ via the condition in Step 4.
Thus, since \eqref{eqn:eta-Jm-lower-bound} implies one of these two conditions holds,
we have that on $E_{4}$, 
if $\gamma_{J_{m}} \geq 2^{-k}$ then
$\hat{q}_{\infty,m} \leq \left\lceil \frac{8}{\gamma_{J_{m}}^{2}} \ln\left(\frac{32 \tilde{m} q_{\eps,\conf}}{\conf}\right) \right\rceil$
and $\hat{y}_{\infty,m} \in \Y$.

It remains only to show that $\hat{y}_{\infty,m} \in \{0,\target_{\PXY}(X_{m}^{2})\}$.
This clearly holds if the return value originates in Step 7, so we need only consider
the case where \Filter~reaches Step 5.  Due to the condition in Step 6, this cannot
occur for a value of $q > q_{\eps,\conf}$ (since $\tilde{q}_{m} \leq \tilde{q}_{1} \leq q_{\eps,\conf}$), 
so let us consider any value of $q \in \{1,\ldots,q_{\eps,\conf}\}$,
and suppose $|\sigma_{m,q}| \geq 3\sqrt{2q\ln\left(\frac{32 \tilde{m} q_{\eps,\conf}}{\conf}\right)}$.
On the event $E_{4}$, \eqref{eqn:sigma-concentration} implies
that if $\sigma_{m,q} \geq 3\sqrt{2q\ln\left(\frac{32 \tilde{m} q_{\eps,\conf}}{\conf}\right)}$,
then $q (2 \eta(J_{m};\PXY)-1) \geq \sigma_{m,q} - \sqrt{2q\ln\left(\frac{32 \tilde{m} q_{\eps,\conf}}{\conf}\right)} \geq 2\sqrt{2q\ln\left(\frac{32 \tilde{m} q_{\eps,\conf}}{\conf}\right)} > 0$,
and likewise if $\sigma_{m,q} \leq -3\sqrt{2q\ln\left( \frac{32 \tilde{m} q_{\eps,\conf}}{\conf}\right)}$, 
then $q (2 \eta(J_{m};\PXY)-1) \leq \sigma_{m,q} + \sqrt{2q\ln\left(\frac{32 \tilde{m} q_{\eps,\conf}}{\conf}\right)} \leq -2\sqrt{2q\ln\left(\frac{32 \tilde{m} q_{\eps,\conf}}{\conf}\right)} < 0$;
thus, since $|2\eta(J_{m};\PXY)-1|=2\gamma_{J_{m}}$,
if $|\sigma_{m,q}| \geq 3\sqrt{2q\ln\left(\frac{32 \tilde{m} q_{\eps,\conf}}{\conf}\right)}$,
then 
\begin{equation}
\label{eqn:gamma-Jm-lower-bound}
\gamma_{J_{m}} \geq \sqrt{\frac{2}{q}\ln\left(\frac{32 \tilde{m} q_{\eps,\conf}}{\conf}\right)}
\end{equation}
and $\sign(2\eta(J_{m};\PXY)-1) = \sign(\sigma_{m,q})$.
In particular, since $q \leq q_{\eps,\conf} \leq 2^{2k_{\eps}+5} \ln\left( \frac{32 \tilde{m} q_{\eps,\conf}}{\conf} \right)$, 
this implies
\begin{equation*}
\gamma_{J_{m}} \geq \sqrt{\frac{2}{q_{\eps,\conf}}\ln\left(\frac{32 \tilde{m} q_{\eps,\conf}}{\conf}\right)}
\geq 2^{-k_{\eps}-2} > \eps/128.
\end{equation*}
Therefore, Lemma~\ref{lem:good-gammas} implies that on $\bigcap_{j=0}^{4} E_{j}$, $\sign(2\eta(J_{m};\PXY)-1)=y_{J_{m}}$;
combined with the above, this implies $\sign(\sigma_{m,q}) = y_{J_{m}}$.
Furthermore, Lemma~\ref{lem:E2} implies that on $\bigcap_{j=0}^{4} E_{j}$, $y_{J_{m}} = \target_{\PXY}(X_{m}^{2})$,
so that $\sign(\sigma_{m,q}) = \target_{\PXY}(X_{m}^{2})$.
In particular, recall that if $\hat{y}_{\infty,m} \in \Y$, then 
$|\sigma_{m,\hat{q}_{\infty,m}}| \geq 3\sqrt{2\hat{q}_{\infty,m}\ln\left(\frac{32 \tilde{m} q_{\eps,\conf}}{\conf}\right)}$.
Thus, since the condition in Step 6 implies $\hat{q}_{\infty,m} \leq \tilde{q}_{m} \leq q_{\eps,\conf}$, 
we have that on $\bigcap_{j=0}^{4} E_{j}$, if $\hat{y}_{\infty,m} \in \Y$, 
then $\hat{y}_{\infty,m} = \target_{\PXY}(X_{m}^{2})$.  
This completes the proof that $\hat{y}_{\infty,m} \in \{0,\target_{\PXY}(X_{m}^{2})\}$
on $\bigcap_{j=0}^{4} E_{j}$.
Since we established above that $\hat{y}_{\infty,m} \in \Y$ if $\gamma_{J_{m}} \geq 2^{-k}$ on $E_{4}$,
this also completes the proof that $\hat{y}_{\infty,m} = \target_{\PXY}(X_{m}^{2})$ when $\gamma_{J_{m}} \geq 2^{-k}$ on $\bigcap_{j=0}^{4} E_{j}$.
\end{proof}

\begin{lemma}
\label{lem:lb-fraction-in-gamma-range}
There exists an $(\DataX_{1},\DataX_{2})$-measurable event $E_{5}$ of probability at least $1-\conf/64$ such that, 
on $E_{5}$, 
for every $k \in \{2,\ldots,k_{\eps}\}$ with $\Px\left( \bigcup \left\{ A \in J : \gamma_{A} \in \left[2^{-k}, 2^{1-k}\right] \right\} \right) \geq 2^{k-3} \eps / k_{\eps}$, 
\begin{equation*}
\left|\left\{ m \in \left\{1,\ldots,\tilde{m}_{k}\right\} : \gamma_{J_{m}} \in \left[2^{-k}, 2^{1-k}\right] \right\}\right|
\geq (1/2) \tilde{m}_{k} \Px\left( \bigcup \left\{ A \in J : \gamma_{A} \in \left[2^{-k}, 2^{1-k}\right] \right\} \right).
\end{equation*}
\end{lemma}
\begin{proof}
Fix any $k \in \{2,\ldots, k_{\eps}\}$.
First, note that a Chernoff bound (under the conditional distribution given $J$) implies that,
with conditional probability (given $J$) at least 
\begin{equation*}
1 - \exp\left\{ - \frac{\tilde{m}_{k}}{8} \Px\left( \bigcup\left\{ A \in J : \gamma_{A} \in \left[2^{-k},2^{1-k}\right] \right\} \right) \right\}, 
\end{equation*}
we have
\begin{equation}
\label{eqn:lb-fraction-in-gamma-range-tilde-m}
\left| \left\{ m \in \left\{1,\ldots,\tilde{m}_{k}\right\} : \gamma_{J_{m}} \in \left[2^{-k},2^{1-k}\right] \right\} \right| 
\geq \frac{\tilde{m}_{k}}{2} \Px\left( \bigcup \left\{ A \in J : \gamma_{A} \in \left[2^{-k},2^{1-k}\right] \right\} \right).
\end{equation}
If $\Px\left( \bigcup\left\{ A \in J : \gamma_{A} \in \left[2^{-k},2^{1-k}\right] \right\} \right) \geq 2^{k-3} \eps / k_{\eps}$, 
then 
\begin{align*}
&\exp\left\{ - \frac{\tilde{m}_{k}}{8} \Px\left( \bigcup\left\{ A \in J : \gamma_{A} \in \left[2^{-k},2^{1-k}\right] \right\} \right) \right\}
\\ & \leq \exp\left\{ - \frac{8 k_{\eps}}{2^{k}\eps} \Log\left( \frac{64 k_{\eps}}{\conf} \right) 2^{k-3} \eps / k_{\eps} \right\}
= \exp\left\{ - \Log\left( \frac{64 k_{\eps}}{\conf} \right) \right\}
= \frac{\conf}{64 k_{\eps}}.
\end{align*}
Thus, by the law of total probability, there is an event $G_{5}(k)$ of probability at least $1 - \conf / (64 k_{\eps})$ such that, 
on $G_{5}(k)$, if $\Px\left( \bigcup\left\{ A \in J : \gamma_{A} \in \left[2^{-k},2^{1-k}\right] \right\} \right) \geq 2^{k-3} \eps / k_{\eps}$,
then \eqref{eqn:lb-fraction-in-gamma-range-tilde-m} holds.
This holds for all $k \in \{2,\ldots,k_{\eps}\}$ on the event $E_{5} = \bigcap_{k=2}^{k_{\eps}} G_{5}(k)$,
which has probability at least $1-\conf/64$ by a union bound. 
\end{proof}

We are now ready to apply the above results to characterize the behavior of \RQCAL.
For simplicity, we begin with the case of an infinite budget $n$, so that the algorithm proceeds until $m = \tilde{m}$;
later, we discuss sufficient finite sizes of $n$ to retain this behavior.

\begin{lemma}
\label{lem:V-consistent}
Consider running \RQCAL~with budget $\infty$.
On the event $\bigcap_{j=0}^{4} E_{j}$, 
$\forall k \!\in\! \{2,\ldots,k_{\eps}\}$, $\forall m \in \left\{ 1,\ldots,\tilde{m}_{k} \right\}$,
$\target_{\PXY} \in V_{m}$ and 
\begin{equation*}
V_{m} \subseteq \left\{ h \in \C : \forall m^{\prime} \leq m \text{ with } \gamma_{J_{m^{\prime}}} \geq 2^{-k}, h(X_{m^{\prime}}^{2}) = \target_{\PXY}(X_{m^{\prime}}^{2}) \right\}.
\end{equation*}
\end{lemma}
\begin{proof}
Fix any $k \in \{2,\ldots,k_{\eps}\}$.
We proceed by induction.  The claim is clearly satisfied for $V_{0} = \C$.
Now take as the inductive hypothesis that, for some $m \in \left\{1,\ldots,\tilde{m}_{k}\right\}$,
$\target_{\PXY} \in V_{m-1} \subseteq \left\{h \in \C : \forall m^{\prime} \leq m-1 \text{ with } \gamma_{J_{m^{\prime}}} \geq 2^{-k}, h(X_{m^{\prime}}^{2}) = \target_{\PXY}(X_{m^{\prime}}^{2})\right\}$.

If $X_{m}^{2} \notin \DIS(V_{m-1})$, then we have $V_{m} = V_{m-1}$, so that $\target_{\PXY} \in V_{m}$ as well.
Furthermore, since $\target_{\PXY} \in V_{m-1}$, the fact that $X_{m}^{2} \notin \DIS(V_{m-1})$ implies 
that every $h \in V_{m}$ has $h(X_{m}^{2}) = \target_{\PXY}(X_{m}^{2})$.  Therefore, 
\begin{align*}
& V_{m} 
= V_{m-1} \cap \left\{h \in \C : h(X_{m}^{2}) = \target_{\PXY}(X_{m}^{2}) \right\} 
\\ & \subseteq \left\{h \in \C : \forall m^{\prime} \leq m-1 \text{ with } \gamma_{J_{m^{\prime}}} \geq 2^{-k}, h(X_{m^{\prime}}^{2}) = \target_{\PXY}(X_{m^{\prime}}^{2})\right\} 
\\ & {\hskip 5.4cm}\cap \left\{h \in \C : h(X_{m}^{2}) = \target_{\PXY}(X_{m}^{2})\right\}
\\ & \subseteq \left\{h \in \C : \forall m^{\prime} \leq m \text{ with } \gamma_{J_{m^{\prime}}} \geq 2^{-k}, h(X_{m^{\prime}}^{2}) = \target_{\PXY}(X_{m^{\prime}}^{2})\right\}.
\end{align*}

Next, consider the case that $X_{m}^{2} \in \DIS(V_{m-1})$.
Lemma~\ref{lem:filter-good-labels} implies that on $\bigcap_{j=0}^{4} E_{j}$, 
$\hat{y}_{\infty,m} \in \{0,\target_{\PXY}(X_{m}^{2})\}$.  If $\hat{y}_{\infty,m} = 0$,
then $V_{m} = V_{m-1}$, so that $\target_{\PXY} \in V_{m}$ by the inductive hypothesis.
Furthermore, since $k \leq \tilde{k}_{m}$, Lemma~\ref{lem:filter-good-labels} implies that on $\bigcap_{j=0}^{4} E_{j}$,
if $\gamma_{J_{m}} \geq 2^{-k}$ then $\hat{y}_{\infty,m} \neq 0$; thus, if $\hat{y}_{\infty,m} = 0$,
we have $\gamma_{J_{m}} < 2^{-k}$, so that
\begin{align*}
V_{m} = V_{m-1} 
& \subseteq \left\{h \in \C : \forall m^{\prime} \leq m-1 \text{ with } \gamma_{J_{m^{\prime}}} \geq 2^{-k}, h(X_{m^{\prime}}^{2}) = \target_{\PXY}(X_{m^{\prime}}^{2})\right\}
\\ & = \left\{h \in \C : \forall m^{\prime} \leq m \text{ with } \gamma_{J_{m^{\prime}}} \geq 2^{-k}, h(X_{m^{\prime}}^{2}) = \target_{\PXY}(X_{m^{\prime}}^{2})\right\}.
\end{align*}

On the other hand, if $\hat{y}_{\infty,m} = \target_{\PXY}(X_{m}^{2})$, then since
$\target_{\PXY} \in V_{m-1}$ by the inductive hypothesis, the condition in Step 5 
will be satisfied, so that we have $V_{m} = \left\{h \!\in\! V_{m-1} : h(X_{m}^{2}) = \target_{\PXY}(X_{m}^{2})\right\}$.
In particular, this implies $\target_{\PXY} \in V_{m}$ as well, and combined with the 
inductive hypothesis,
we have 
\begin{align*}
V_{m} & = V_{m-1} \cap \left\{h \in \C : h(X_{m}^{2}) = \target_{\PXY}(X_{m}^{2})\right\}
\\ & \subseteq \left\{h \in \C : \forall m^{\prime} \leq m \text{ with } \gamma_{J_{m^{\prime}}} \geq 2^{-k}, h(X_{m^{\prime}}^{2}) = \target_{\PXY}(X_{m^{\prime}}^{2})\right\}.
\end{align*}
The result follows by the principle of induction.
\end{proof}

In particular, this implies the following result.

\begin{lemma}
\label{lem:infinite-budget-error}
There exists an event $E_{6}$ of probability at least $1-\conf/64$ such that,
on $\bigcap_{j=0}^{6} E_{j}$,
the classifier $\hat{h}_{\infty}$ produced by running \RQCAL~with budget $\infty$
has $\er_{\PXY}(\hat{h}_{\infty}) - \er_{\PXY}(\target_{\PXY}) \leq \eps$.
\end{lemma}
\begin{proof}
Fix any $k \in \{2,\ldots,k_{\eps}\}$
and let $\hat{\ell}_{k} = \left\lceil (1/2) \tilde{m}_{k} \Px\left( \bigcup\left\{ A \in J : \gamma_{A} \in \left[2^{-k},2^{1-k}\right] \right\} \right) \right\rceil$.
Note that 
\begin{align*}
\hat{\ell}_{k} \geq \frac{8 c k_{\eps} \Px\left( \bigcup\left\{A \!\in\! J : \gamma_{A} \in \left[2^{-k},2^{1-k}\right]\right\}\right) }{2^{k} \eps} 
\Bigg( \vc \Log&\left( \frac{8 k_{\eps} \Px\left( \bigcup\left\{A \!\in\! J : \gamma_{A} \in \left[2^{-k},2^{1-k}\right]\right\}\right) }{2^{k} \eps} \right) 
\\ & + \Log\left( \frac{64 k_{\eps}}{\conf} \right) \Bigg),
\end{align*}
for $c$ as in Lemma~\ref{lem:vc-cover}.
Let $\hat{m}_{k} = \min\left\{ m \in \nats : \sum_{m^{\prime} = 1}^{m} \ind_{\left[2^{-k},2^{1-k}\right]}(\gamma_{J_{m^{\prime}}}) = \hat{\ell}_{k} \right\} \cup \{\infty\}$.
Note that, if $\hat{m}_{k} < \infty$, then the sequence $\left\{ X_{m}^{2} : 1 \leq m \leq \hat{m}_{k}, \gamma_{J_{m}} \in \left[2^{-k},2^{1-k}\right] \right\}$ is
conditionally i.i.d. (given $J$ and $\hat{m}_{k}$), with conditional distributions $\Px\left( \cdot \Big| \bigcup \left\{ A \in J : \gamma_{A} \in \left[2^{-k},2^{1-k}\right] \right\} \right)$.
Applying Lemma~\ref{lem:vc-cover} to these samples implies that there exists an event of conditional probability (given $J$ and $\hat{m}_{k}$) at least $1-\conf / (64 k_{\eps})$ on which,
if $\hat{m}_{k} < \infty$ and $\Px\left( \bigcup \left\{ A \in J : \gamma_{A} \in \left[2^{-k},2^{1-k}\right]\right\}\right) > \frac{2^{k} \eps}{8 k_{\eps}}$, 
letting 
\begin{equation*}
\H_{k} = \left\{ h \in \C : \forall m \leq \hat{m}_{k} \text{ with } \gamma_{J_{m}} \in \left[2^{-k},2^{1-k}\right], h(X_{m}^{2}) = \target_{\PXY}(X_{m}^{2}) \right\},
\end{equation*}
every $h \in \H_{k}$ has
\begin{equation*}
\Px\left(x : h(x) \neq \target_{\PXY}(x) \Big| \gamma_{J(x)} \in \left[2^{-k},2^{1-k}\right]\right)
\leq \frac{2^{k} \eps}{8 k_{\eps} \Px\left( \bigcup\left\{ A \in J : \gamma_{A} \in \left[2^{-k},2^{1-k}\right]\right\} \right)},
\end{equation*}
which implies
\begin{equation*}
\Px\left(x : h(x) \neq \target_{\PXY}(x) \text{ and } \gamma_{J(x)} \in \left[2^{-k},2^{1-k}\right]\right)
\leq \frac{2^{k} \eps}{8 k_{\eps}}.
\end{equation*}
By the law of total probability and a union bound, there exists an event $E_{6}$ of probability at least $1-\conf/64$ 
on which this holds for every $k \in \{2,\ldots,k_{\eps}\}$.

Lemma~\ref{lem:V-consistent} implies that, on $\bigcap_{j=0}^{4} E_{j}$, $\forall k \in \{2,\ldots,k_{\eps}\}$,
\begin{equation*}
V_{\tilde{m}} \subseteq V_{\tilde{m}_{k}} \subseteq \left\{ h \in \C : \forall m \leq \tilde{m}_{k} \text{ with } \gamma_{J_{m}} \geq 2^{-k}, h(X_{m}^{2}) = \target_{\PXY}(X_{m}^{2}) \right\}.
\end{equation*}
Lemma~\ref{lem:lb-fraction-in-gamma-range} implies that, on $E_{5}$, $\forall k \in \{2,\ldots,k_{\eps}\}$,
if $\Px\left(\bigcup \left\{A \in J : \gamma_{A} \in \left[2^{-k},2^{1-k}\right]\right\}\right) > \frac{2^{k} \eps}{8 k_{\eps}}$,
then
$\left|\left\{ m \in \left\{1,\ldots,\tilde{m}_{k}\right\} : \gamma_{J_{m}} \in \left[2^{-k}, 2^{1-k}\right] \right\}\right|
\geq (1/2) \tilde{m}_{k} \Px\left( \bigcup \left\{ A \in J : \gamma_{A} \in \left[2^{-k}, 2^{1-k}\right] \right\} \right)$,
so that $\hat{m}_{k} \leq \tilde{m}_{k}$.  In particular, this implies $\hat{m}_{k} < \infty$ and 
\begin{equation*}
\left\{ h \in \C : \forall m \leq \tilde{m}_{k} \text{ with } \gamma_{J_{m}} \geq 2^{-k}, h(X_{m}^{2}) = \target_{\PXY}(X_{m}^{2}) \right\} \subseteq \H_{k}.
\end{equation*}
Combining the above three results, we have that on $\bigcap_{j=0}^{6} E_{j}$, 
for every $k \in \{2,\ldots,k_{\eps}\}$ with $\Px\left(\bigcup \left\{A \in J : \gamma_{A} \in \left[2^{-k},2^{1-k}\right]\right\}\right) > \frac{2^{k} \eps}{8 k_{\eps}}$,
$V_{\tilde{m}} \subseteq \H_{k}$, and therefore every $h \in V_{\tilde{m}}$ has 
\begin{equation*}
\Px\left(x : h(x) \neq \target_{\PXY}(x) \text{ and } \gamma_{J(x)} \in \left[2^{-k},2^{1-k}\right]\right) \leq \frac{2^{k} \eps}{8 k_{\eps}}.
\end{equation*}
Furthermore, for every $k \in \{2,\ldots,k_{\eps}\}$ with $\Px\left(\bigcup \left\{A \in J : \gamma_{A} \in \left[2^{-k},2^{1-k}\right]\right\}\right) \leq \frac{2^{k} \eps}{8 k_{\eps}}$,
we also have that every $h \in V_{\tilde{m}}$ satisfies
\begin{multline*}
\Px\left(x : h(x) \neq \target_{\PXY}(x) \text{ and } \gamma_{J(x)} \in \left[2^{-k},2^{1-k}\right]\right) 
\\ \leq \Px\left(\bigcup \left\{A \in J : \gamma_{A} \in \left[2^{-k},2^{1-k}\right]\right\}\right) \leq \frac{2^{k} \eps}{8 k_{\eps}}.
\end{multline*}
Combined with Lemma~\ref{lem:excess-error-gamma-swap}, we have that on $\bigcap_{j=0}^{6} E_{j}$, every $h \in V_{\tilde{m}}$ has
\begin{align}
& \er_{\PXY}(h) - \er_{\PXY}(\target_{\PXY})
\leq 5\tau + \int \ind\left[h(x) \neq \target_{\PXY}(x)\right] 2\gamma_{J(x)} \Px({\rm d}x) \notag
\\ & 
\leq 5\tau + 2^{1-k_{\eps}} \Px\left( x : h(x) \neq \target_{\PXY}(x) \text{ and } \gamma_{J(x)} \leq 2^{-k_{\eps}} \right) \notag
\\ & \phantom{\leq } +
\sum_{k=2}^{k_{\eps}} 2^{2-k} \Px\left( x : h(x) \neq \target_{\PXY}(x) \text{ and } \gamma_{J(x)} \in \left[2^{-k},2^{1-k}\right] \right) \notag
\\ & \leq
5\tau + 2^{1-k_{\eps}} \Px\left( \bigcup \left\{ A \in J : \gamma_{A} \leq 2^{-k_{\eps}} \right\} \right) 
+ \sum_{k=2}^{k_{\eps}} 2^{2-k} \frac{2^{k} \eps}{8 k_{\eps}}. \label{eqn:excess-error-epsilon-prebound}
\end{align}
Next, note that $\sum_{k=2}^{k_{\eps}} 2^{2-k} \frac{2^{k} \eps}{8 k_{\eps}} = (k_{\eps}-1) \frac{\eps}{2 k_{\eps}} \leq \frac{\eps}{2}$.
Furthermore, since $2^{-k_{\eps}} \leq \hat{\gamma}_{\eps}/8 < \gamma_{\eps}/4$,
Lemma~\ref{lem:bucketed-gamma-probs} implies that, on $E_{1}$, 
\begin{equation*}
\Px\left( \bigcup \left\{ A \in J : \gamma_{A} \leq 2^{-k_{\eps}} \right\} \right)
\leq \frac{3 \eps}{2 \gamma_{\eps}}.
\end{equation*}
Plugging these facts into \eqref{eqn:excess-error-epsilon-prebound} reveals that, on $\bigcap_{j=0}^{6} E_{j}$, $\forall h \in V_{\tilde{m}}$,
\begin{equation*}
\er_{\PXY}(h) - \er_{\PXY}(\target_{\PXY})
\leq 5\tau + 2^{1-k_{\eps}} \frac{3 \eps}{2 \gamma_{\eps}} +  \frac{\eps}{2}
\leq 
5\tau + \frac{3}{8} \eps + \frac{\eps}{2}
\leq \frac{453}{512} \eps < \eps.
\end{equation*}
The result follows by noting that, when the budget is set to $\infty$,
\RQCAL~definitely reaches $m=\tilde{m}$ before halting, so that $\hat{h}_{\infty} \in V_{\tilde{m}}$.
\end{proof}

The only remaining question is how many label requests the algorithm makes 
in the process of producing this $\hat{h}_{\infty}$, so that taking a budget $n$ of at least
this size is equivalent to having an infinite budget.  
This question is addressed by the following sequence of lemmas.

\begin{lemma}
\label{lem:rc-simulated-query-bound-basic}
Consider running \RQCAL~with budget $\infty$.
There exists
an event $E_{7}$ of probability at least $1-\conf/64$ such that,
on $E_{1} \cap E_{7}$, $\forall k \in \{2,\ldots,k_{\eps}\}$,
\begin{multline*}
\left| \left\{ m \in \left\{1,\ldots,\tilde{m}_{k}\right\} : \gamma_{J_{m}} \leq 2^{1-k}, X_{m}^{2} \in \DIS(V_{m-1}) \right\}\right|
\\ \leq
17 \max\left\{\Px\left(x : \gamma_{x} < 2^{3-k}\right), \frac{\eps}{2\hat{\gamma}_{\eps}}\right\} \tilde{m}_{k}.
\end{multline*}
\end{lemma}
\begin{proof}
Fix any $k \in \{2,\ldots,k_{\eps}\}$.
By a Chernoff bound (applied under the conditional given $J$) and the law of total probability, 
there is an event $G_{7}(k)$ of probability at least $1 - \frac{\conf}{64 k_{\eps}}$, on which
\begin{equation*}
\left| \left\{ m \in \left\{1,\ldots,\tilde{m}_{k}\right\} : \gamma_{J_{m}} \leq 2^{1-k} \right\} \right|
\leq \log_{2}\left(\frac{64 k_{\eps}}{\conf}\right) + 2e \Px\left( \bigcup\left\{ A \in J : \gamma_{A} \leq 2^{1-k} \right\} \right) \tilde{m}_{k}.
\end{equation*}
Lemma~\ref{lem:bucketed-gamma-probs} implies that, on $E_{1}$, 
\begin{equation*}
\Px\left( \bigcup\left\{ A \in J : \gamma_{A} \leq 2^{1-k}\right\}\right) 
\leq \max\left\{ 3 \Px\left(x : \gamma_{x} < 2^{3-k}\right), \frac{3\eps}{2\gamma_{\eps}} \right\}.
\end{equation*}
Therefore, on $E_{1} \cap G_{7}(k)$, 
\begin{align}
& \left| \left\{ m \in \left\{1,\ldots,\tilde{m}_{k}\right\} : \gamma_{J_{m}} \leq 2^{1-k}, X_{m}^{2} \in \DIS(V_{m-1}) \right\} \right|
\leq \left| \left\{ m \in \left\{1,\ldots,\tilde{m}_{k}\right\} : \gamma_{J_{m}} \leq 2^{1-k} \right\} \right| \notag
\\ & \leq \log_{2}\left(\frac{64 k_{\eps}}{\conf}\right) + 6 e \max\left\{ \Px\left(x : \gamma_{x} < 2^{3-k}\right), \frac{\eps}{2\gamma_{\eps}} \right\} \tilde{m}_{k}.
\label{eqn:rc-simulated-query-bound-basic-1}
\end{align}
Furthermore, since $\hat{\gamma}_{\eps} \leq \gamma_{\eps}$, and
\begin{equation*}
\frac{\eps}{2\hat{\gamma}_{\eps}} \tilde{m}_{k}
\geq \frac{64}{2^{k_{\eps}} \hat{\gamma}_{\eps}} \Log\left( \frac{64 k_{\eps}}{\conf} \right)
\geq 4 \Log\left( \frac{64 k_{\eps}}{\conf} \right)
\geq 2 \log_{2}\left( \frac{64 k_{\eps}}{\conf} \right),
\end{equation*}
\eqref{eqn:rc-simulated-query-bound-basic-1} is at most
\begin{equation*}
\left(6 e + \frac{1}{2}\right) \max\left\{ \Px\left(x : \gamma_{x} < 2^{3-k}\right), \frac{\eps}{2\hat{\gamma}_{\eps}} \right\} \tilde{m}_{k}
\leq 17 \max\left\{ \Px\left(x : \gamma_{x} < 2^{3-k}\right), \frac{\eps}{2\hat{\gamma}_{\eps}} \right\} \tilde{m}_{k}.
\end{equation*} 
Defining $E_{7} = \bigcap_{k=2}^{k_{\eps}} G_{7}(k)$, a union bound implies $E_{7}$ has probability at least $1-\conf/64$,
and the result follows.
\end{proof}

\begin{lemma}
\label{lem:rc-simulated-query-bound-star}
Consider running \RQCAL~with budget $\infty$.
There exists
an event $E_{8}$ of probability at least $1-\conf/64$ such that,
on $E_{8} \cap \bigcap_{j=0}^{4} E_{j}$, 
$\forall \bar{k} \in \{3,\ldots,k_{\eps}\}$, $\forall k \in \{2,\ldots,\bar{k}-1\}$, 
\begin{multline*}
\left| \left\{ m \in \left\{1,\ldots,\tilde{m}_{k}\right\} : X_{m}^{2} \in \DIS(V_{m-1}) \right\} \right|
\\\leq 
6e \max\left\{ \Px\left( x : \gamma_{x} < 2^{2-\bar{k}} \right), \frac{\eps}{2\gamma_{\eps}} \right\} \tilde{m}_{k}
\\ + 91 \tilde{c} \left( 2^{1+\bar{k}-k} + \Log\left( \frac{64 c}{\eps} \right) \right) \left( 6 \s \Log\left( \frac{128 c}{\eps} \right) + \Log\left(\frac{1}{\conf}\right) \right),
\end{multline*}
for $c$ as in Lemma~\ref{lem:vc-cover} and $\tilde{c}$ as in Lemma~\ref{lem:compression-bound}.
\end{lemma}
\begin{proof}
The claim trivially holds if $\s = \infty$, so for the remainder of the proof we suppose $\s < \infty$.
Fix any $\bar{k} \in \{3,\ldots,k_{\eps}\}$ and $k \in \{2,\ldots,\bar{k}-1\}$, and note that
\begin{align}
\big| \big\{ m \in \left\{1,\ldots,\tilde{m}_{k}\right\} & : X_{m}^{2} \in \DIS(V_{m-1}) \big\} \big| \notag
\\ \leq & \left| \left\{ m \in \left\{1,\ldots,\tilde{m}_{k}\right\} : \gamma_{J_{m}} \leq 2^{-\bar{k}}, X_{m}^{2} \in \DIS(V_{m-1}) \right\} \right| \notag
\\ + & \left| \left\{ m \in \left\{1,\ldots,\tilde{m}_{k}\right\} : \gamma_{J_{m}} \geq 2^{-\bar{k}}, X_{m}^{2} \in \DIS(V_{m-1}) \right\} \right|.
\label{eqn:rc-sim-queries-bark-split}
\end{align}
We proceed to bound each term on the right hand side.
A Chernoff bound (applied under the conditional distribution given $J$) and the law of total probability imply that, 
on an event $G_{8}^{(i)}(\bar{k},k)$ of probability at least $1-\frac{\conf}{256 k_{\eps}^{2}}$,
\begin{multline*}
\left| \left\{ m \in \left\{1,\ldots,\tilde{m}_{k}\right\} : \gamma_{J_{m}} \leq 2^{-\bar{k}}, X_{m}^{2} \in \DIS(V_{m-1}) \right\} \right|
\\ \leq \log_{2}\left(\frac{256 k_{\eps}^{2}}{\conf}\right) + 2e \Px\left( \bigcup\left\{ A \in J : \gamma_{A} \leq 2^{-\bar{k}} \right\}\right) \tilde{m}_{k},
\end{multline*}
and Lemma~\ref{lem:bucketed-gamma-probs} implies that, on $E_{1}$, this is at most
\begin{equation*}
\log_{2}\left(\frac{256 k_{\eps}^{2}}{\conf}\right) + 6e \max\left\{ \Px\left( x : \gamma_{x} < 2^{2-\bar{k}} \right), \frac{\eps}{2\gamma_{\eps}} \right\} \tilde{m}_{k}.
\end{equation*}

Now we turn to bounding the second term on the right hand side of \eqref{eqn:rc-sim-queries-bark-split}.
We proceed in two steps, noting that monotonicity of $m \mapsto \DIS(V_{m})$ implies
\begin{align}
\Big| \Big\{ m \in \left\{1,\ldots,\tilde{m}_{k}\right\} & : \gamma_{J_{m}} \geq 2^{-\bar{k}}, X_{m}^{2} \in \DIS(V_{m-1}) \Big\} \Big| \notag
\\ \leq & \left| \left\{ m \in \left\{1,\ldots,\tilde{m}_{\bar{k}}\right\} : \gamma_{J_{m}} \geq 2^{-\bar{k}}, X_{m}^{2} \in \DIS(V_{m-1})\right\}\right| \notag
\\ + & \left| \left\{ m \in \left\{\tilde{m}_{\bar{k}}+1,\ldots,\tilde{m}_{k}\right\} : \gamma_{J_{m}} \geq 2^{-\bar{k}}, X_{m}^{2} \in \DIS(V_{\tilde{m}_{\bar{k}}})\right\}\right|.
\label{eqn:rc-sim-queries-mbark-split}
\end{align}
We start with the first term on the right side of \eqref{eqn:rc-sim-queries-mbark-split}.
Let 
$L = \left| \left\{ m \in \left\{1,\ldots,\tilde{m}_{\bar{k}}\right\} : \gamma_{J_{m}} \geq 2^{-\bar{k}} \right\} \right|$,
and let $\ell_{1},\ldots,\ell_{L}$ denote the increasing subsequence of values $\ell \in \{1,\ldots,\tilde{m}_{\bar{k}}\}$ 
with $\gamma_{J_{\ell}} \geq 2^{-\bar{k}}$.
Also, let $\tilde{j}_{\bar{k}} = \max\left\{1, \left\lceil \log_{2}\left( \tilde{m}_{\bar{k}} / (\s + \Log(1/\conf)) \right) \right\rceil\right\}$,
let $M_{0} = 0$, and for each $j \in \nats$, let
\begin{equation*}
M_{j} = \left\lceil \tilde{c} 2^{j} \left( \s \Log\left( 2^{j} \right) + \Log\left(\frac{256 k_{\eps}^{2} \tilde{j}_{\bar{k}}}{\conf}\right) \right) \right\rceil,
\end{equation*}
for $\tilde{c}$ as in Lemma~\ref{lem:compression-bound}.
Let $\truV_{0} = \C$, and for each $i \leq L$, let 
\begin{equation*}
\truV_{i} = \left\{ h \in \C : \forall j \in \{1,\ldots,i\}, h(X_{\ell_{j}}^{2}) = \target_{\PXY}(X_{\ell_{j}}^{2}) \right\}.
\end{equation*}
Let $\phi_{\s}$ be the function mapping any $\U \in \X^{\s}$ to the set $\DIS(\{h \in \C : \forall x \in \U, h(x) = \target_{\PXY}(x)\})$.
Fix any $j \in \nats$.
By Theorem~\ref{thm:xtd}, if $M_{j} \leq L$, then there exist $i_{1},\ldots,i_{\s} \in \{1,\ldots,M_{j}\}$ 
such that $\{h \in \C : \forall r \in \{1,\ldots,\s\}, h(X_{\ell_{i_{r}}}^{2}) = \target_{\PXY}(X_{\ell_{i_{r}}}^{2})\} = \truV_{M_{j}}$
(see the discussion in Section~\ref{sec:hatn}).
In particular, for this choice of $i_{1},\ldots,i_{\s}$, we have $\phi_{\s}(X_{\ell_{i_{1}}}^{2},\ldots,X_{\ell_{i_{\s}}}^{2}) = \DIS(\truV_{M_{j}})$;
furthermore, since $\phi_{\s}$ is permutation-invariant, we can take $i_{1} \leq \cdots \leq i_{\s}$ without loss of generality.
Also note that $X_{\ell_{1}}^{2},\ldots,X_{\ell_{M_{j} \land L}}^{2}$ are conditionally independent (given $L$ and $J$),
each with conditional distribution $\Px\left(\cdot \Big| \bigcup\left\{ A \in J : \gamma_{A} \geq 2^{-\bar{k}} \right\} \right)$.
Since (when $M_{j} \leq L$) $\{X_{\ell_{1}}^{2},\ldots,X_{\ell_{M_{j}}}^{2}\} \cap \phi_{\s}(X_{\ell_{i_{1}}}^{2},\ldots,X_{\ell_{i_{\s}}}^{2}) = \{X_{\ell_{1}}^{2},\ldots,X_{\ell_{M_{j}}}^{2}\} \cap \DIS(\truV_{M_{j}}) = \emptyset$,
Lemma~\ref{lem:compression-bound} (applied under the conditional distribution given $L$ and $J$) and the law of total probability imply that,
on an event $G_{8}^{(ii)}(\bar{k},k,j)$ of probability at least $1-\frac{\conf}{256 k_{\eps}^{2} \tilde{j}_{\bar{k}}}$,
if $M_{j} \leq L$, then 
\begin{equation}
\label{eqn:rc-sim-PDIS-bound}
\Px\left( \DIS(\truV_{M_{j}}) \Big| \bigcup\left\{ A \in J : \gamma_{A} \geq 2^{-\bar{k}} \right\} \right)
\leq 2^{-j}.
\end{equation}
Furthermore, this clearly holds for $j=0$ as well.
Since $\Px\left(\DIS\left(\truV_{i-1}\right)\Big|\bigcup\left\{A \in J : \gamma_{A} \geq 2^{-\bar{k}}\right\}\right)$ is nonincreasing in $i$,
for every $j \geq 0$ with $M_{j} < L$, and every $i \in \{M_{j}+1,\ldots,M_{j+1} \land L\}$,
on $G_{8}^{(ii)}(\bar{k},k,j)$,
$\Px\left(\DIS\left(\truV_{i-1}\right)\Big|\bigcup\left\{A \in J : \gamma_{A} \geq 2^{-\bar{k}}\right\}\right) \leq 2^{-j}$.
Since every $j \geq \tilde{j}_{\bar{k}}$ has $M_{j} \geq \tilde{m}_{\bar{k}} \geq L$,
this holds simultaneously for every $j$ with $M_{j} < L$ on $\bigcap_{j=1}^{\tilde{j}_{\bar{k}}-1} G_{8}^{(ii)}(\bar{k},k,j)$.

Now note that, conditioned on $J$ and $L$, 
\begin{equation*}
\left\{ \ind_{\DIS(\truV_{i-1})}\left(X_{\ell_{i}}^{2}\right) - \Px\left(\DIS\left(\truV_{i-1}\right) \Big| \bigcup \left\{A \in J : \gamma_{A} \geq 2^{-\bar{k}}\right\}\right) \right\}_{i=1}^{L}
\end{equation*}
is a martingale difference sequence with respect to $X_{\ell_{1}}^{2},\ldots,X_{\ell_{L}}^{2}$.  
Therefore, Bernstein's inequality for martingales 
\citep*[e.g.,][Theorem 3.12]{mcdiarmid:98},
applied under the conditional distribution given $J$ and $L$, 
along with the law of total probability, imply that there exists an event $G_{8}^{(iii)}(\bar{k},k)$
of probability at least $1 - \frac{\conf}{256 k_{\eps}^{2}}$
such that, on $G_{8}^{(iii)}(\bar{k},k) \cap \bigcap_{j=1}^{\tilde{j}_{\bar{k}}-1} G_{8}^{(ii)}(\bar{k},k,j)$,
\begin{align*}
\sum_{i=1}^{L} \ind_{\DIS(\truV_{i-1})}\left(X_{\ell_{i}}^{2}\right)
& \leq \log_{2}\left(\frac{256 k_{\eps}^{2}}{\conf}\right) + 2 e \sum_{j  = 0}^{\tilde{j}_{\bar{k}}-1} 2^{-j} (M_{j+1} - M_{j})
\\ & \leq \log_{2}\left(\frac{256 k_{\eps}^{2}}{\conf}\right) + 4 e
+ 4 e \tilde{c} \left( \s \Log\left( 2^{\tilde{j}_{\bar{k}}}\right) + \Log\left(\frac{256 k_{\eps}^{2} \tilde{j}_{\bar{k}}}{\conf}\right) \right) \tilde{j}_{\bar{k}}
\\ & \leq 8 e \tilde{c} \left( \s \tilde{j}_{\bar{k}} + \Log\left(\frac{256 k_{\eps}^{2}}{\conf}\right) \right) \tilde{j}_{\bar{k}}.
\end{align*}
By Lemma~\ref{lem:V-consistent}, on $\bigcap_{j=0}^{4} E_{j}$, $\forall m \in \left\{1,\ldots,\tilde{m}_{\bar{k}}\right\}$,
\begin{equation*}
V_{m} \subseteq \left\{h \in \C : \forall m^{\prime} \leq m \text{ with } \gamma_{J_{m^{\prime}}} \geq 2^{-\bar{k}}, h(X_{m^{\prime}}^{2}) = \target_{\PXY}(X_{m^{\prime}}^{2})\right\}.
\end{equation*}
In particular, this implies $V_{\ell_{i}-1} \subseteq \truV_{i-1}$ for all $i \leq L$.
Therefore, on $\bigcap_{j=0}^{4} E_{j} \cap G_{8}^{(iii)}(\bar{k},k) \cap \bigcap_{j=1}^{\tilde{j}_{\bar{k}}-1} G_{8}^{(ii)}(\bar{k},k,j)$,
\begin{multline}
\label{eqn:rc-sim-queries-mbark-queries-bound}
\left| \left\{ m \in \left\{1,\ldots,\tilde{m}_{\bar{k}}\right\} : \gamma_{J_{m}} \geq 2^{-\bar{k}}, X_{m}^{2} \in \DIS(V_{m-1}) \right\} \right|
= \sum_{i=1}^{L} \ind_{\DIS(V_{\ell_{i}-1})}\left(X_{\ell_{i}}^{2}\right)
\\ \leq \sum_{i=1}^{L} \ind_{\DIS(\truV_{i-1})}\left(X_{\ell_{i}}^{2}\right)
\leq 8 e \tilde{c} \left( \s \tilde{j}_{\bar{k}} + \Log\left(\frac{256 k_{\eps}^{2}}{\conf}\right) \right) \tilde{j}_{\bar{k}}.
\end{multline}

Next, we turn to bounding the second term on the right hand side of \eqref{eqn:rc-sim-queries-mbark-split}.
A Chernoff bound (applied under the conditional distribution given $V_{\tilde{m}_{\bar{k}}}$ and $J$) and the law of total probability imply that
there is an event $G_{8}^{(iv)}(\bar{k},k)$ of probability at least $1 - \frac{\conf}{256 k_{\eps}^{2}}$, on which
\begin{multline}
\label{eqn:rc-sim-queries-frozen-dis-bound}
\left|\left\{m \in \left\{\tilde{m}_{\bar{k}}+1,\ldots,\tilde{m}_{k}\right\} : \gamma_{J_{m}} \geq 2^{-\bar{k}}, X_{m}^{2} \in \DIS(V_{\tilde{m}_{\bar{k}}})\right\}\right|
\\ \leq \log_{2}\left(\frac{256 k_{\eps}^{2}}{\conf}\right) + 2 e \Px\left(\DIS\left(V_{\tilde{m}_{\bar{k}}}\right) \cap \bigcup\left\{ A \in J : \gamma_{A} \geq 2^{-\bar{k}} \right\}\right) \tilde{m}_{k}.
\end{multline}
Also, by a Chernoff bound (applied under the conditional distribution given $J$), with probability at least
\begin{equation*}
1 - \exp\left\{ - (1/8) \Px\left( \bigcup\left\{ A \in J : \gamma_{A} \geq 2^{-\bar{k}} \right\}\right) \tilde{m}_{\bar{k}} \right\},
\end{equation*}
we have 
\begin{equation}
\label{eqn:rc-sim-queries-L-lower-bound}
L \geq (1/2) \tilde{m}_{\bar{k}} \Px\left( \bigcup\left\{ A \in J : \gamma_{A} \geq 2^{-\bar{k}} \right\} \right).
\end{equation}
If $\Px\left( \bigcup\left\{ A \in J : \gamma_{A} \geq 2^{-\bar{k}} \right\}\right) \geq \frac{8}{\tilde{m}_{\bar{k}}} \Log\left( \frac{256 k_{\eps}}{\conf} \right)$,
then
\begin{equation*}
\exp\left\{ - (1/8) \Px\left( \bigcup\left\{ A \in J : \gamma_{A} \geq 2^{-\bar{k}} \right\}\right) \tilde{m}_{\bar{k}} \right\}
\leq \frac{\conf}{256 k_{\eps}}.
\end{equation*}
Thus, by the law of total probability, there is an event $G_{8}^{(v)}(\bar{k})$ of probability at least $1 - \frac{\conf}{256 k_{\eps}}$,
on which, if $\Px\left( \bigcup\left\{ A \in J : \gamma_{A} \geq 2^{-\bar{k}} \right\}\right) \geq \frac{8}{\tilde{m}_{\bar{k}}} \Log\left( \frac{256 k_{\eps}}{\conf} \right)$,
then \eqref{eqn:rc-sim-queries-L-lower-bound} holds.
Let
\begin{equation*}
\hat{j} = \max\left\{ j \in \left\{0,1,\ldots,\tilde{j}_{\bar{k}}-1\right\} : M_{j} \leq (1/2) \tilde{m}_{\bar{k}} \Px\left( \bigcup\left\{ A \in J : \gamma_{A} \geq 2^{-\bar{k}} \right\} \right) \right\},
\end{equation*}
and note that
\begin{equation}
\label{eqn:hatj-lb}
\hat{j} \geq \left\lfloor \log_{2}\left(   \frac{\tilde{m}_{\bar{k}}\Px\left( \bigcup\left\{ A \in J : \gamma_{A} \geq 2^{-\bar{k}} \right\} \right)}{4\tilde{c}\left( 2 \s \Log\left( 2^{\tilde{j}_{\bar{k}}} \right) + \Log\left(\frac{256 k_{\eps}^{2}}{\conf}\right) \right)}\right) \right\rfloor.
\end{equation}
\eqref{eqn:rc-sim-PDIS-bound} implies that on $\bigcap_{j=1}^{\tilde{j}_{\bar{k}}-1} G_{8}^{(ii)}(\bar{k},k,j)$,
if \eqref{eqn:rc-sim-queries-L-lower-bound} holds,
we have 
\begin{equation*}
\Px\left( \DIS\left( \truV_{L} \right) \Big| \bigcup \left\{ A \in J : \gamma_{A} \geq 2^{-\bar{k}} \right\} \right) \leq 2^{-\hat{j}}.
\end{equation*}
Furthermore, Lemma~\ref{lem:V-consistent} implies that, on $\bigcap_{j=0}^{4} E_{j}$, 
$V_{\tilde{m}_{\bar{k}}} \subseteq \truV_{L}$.
Altogether, on $\bigcap_{j=0}^{4} E_{j} \cap G_{8}^{(v)}(\bar{k}) \cap \bigcap_{j=1}^{\tilde{j}_{\bar{k}}-1} G_{8}^{(ii)}(\bar{k},k,j)$,
if $\Px\left( \bigcup\left\{ A \in J : \gamma_{A} \geq 2^{-\bar{k}} \right\}\right) \geq \frac{8}{\tilde{m}_{\bar{k}}} \Log\left( \frac{256 k_{\eps}}{\conf} \right)$,
then
\begin{align*}
\Px\left( \DIS\left( V_{\tilde{m}_{\bar{k}}} \right) \cap \bigcup\left\{ A \in J : \gamma_{A} \geq 2^{-\bar{k}} \right\}\right)
& \leq 2^{-\hat{j}} \Px\left( \bigcup \left\{ A \in J : \gamma_{A} \geq 2^{-\bar{k}} \right\} \right)
\\ & \leq \frac{8\tilde{c}}{\tilde{m}_{\bar{k}}}\left( 2 \s \Log\left( 2^{\tilde{j}_{\bar{k}}} \right) + \Log\left(\frac{256 k_{\eps}^{2}}{\conf}\right) \right),
\end{align*}
where the last inequality is by \eqref{eqn:hatj-lb}.
Otherwise, if $\Px\left( \bigcup\left\{ A \in J : \gamma_{A} \geq 2^{-\bar{k}} \right\}\right) < \frac{8}{\tilde{m}_{\bar{k}}} \Log\left( \frac{256 k_{\eps}}{\conf} \right)$,
then in any case we have
\begin{multline*}
\Px\left( \DIS\left( V_{\tilde{m}_{\bar{k}}} \right) \cap \bigcup\left\{ A \in J : \gamma_{A} \geq 2^{-\bar{k}} \right\}\right)
\leq \Px\left( \bigcup\left\{ A \in J : \gamma_{A} \geq 2^{-\bar{k}} \right\}\right)
\\ < \frac{8}{\tilde{m}_{\bar{k}}} \Log\left( \frac{256 k_{\eps}}{\conf} \right)
\leq \frac{8\tilde{c}}{\tilde{m}_{\bar{k}}}\left( 2 \s \Log\left( 2^{\tilde{j}_{\bar{k}}}\right) + \Log\left(\frac{256 k_{\eps}^{2}}{\conf}\right) \right).
\end{multline*}
Combined with \eqref{eqn:rc-sim-queries-frozen-dis-bound}, this implies that on 
$\bigcap_{j=0}^{4} E_{j} \cap G_{8}^{(iv)}(\bar{k},k) \cap G_{8}^{(v)}(\bar{k}) \cap \bigcap_{j=1}^{\tilde{j}_{\bar{k}}-1} G_{8}^{(ii)}(\bar{k},k,j)$,
\begin{align*}
& \left|\left\{m \in \left\{\tilde{m}_{\bar{k}}+1,\ldots,\tilde{m}_{k}\right\} : \gamma_{J_{m}} \geq 2^{-\bar{k}}, X_{m}^{2} \in \DIS(V_{\tilde{m}_{\bar{k}}})\right\}\right|
\\ & \leq 
\log_{2}\left(\frac{256 k_{\eps}^{2}}{\conf}\right) + 
16 e \tilde{c} \frac{\tilde{m}_{k}}{\tilde{m}_{\bar{k}}} \left( 2 \s \Log\left( 2^{\tilde{j}_{\bar{k}}} \right) + \Log\left(\frac{256 k_{\eps}^{2}}{\conf}\right) \right)
\\ & \leq 
32 e \tilde{c} \frac{\tilde{m}_{k}}{\tilde{m}_{\bar{k}}} \left( \s \tilde{j}_{\bar{k}} + \Log\left(\frac{256 k_{\eps}^{2}}{\conf}\right) \right)
\leq 
64 e \tilde{c} 2^{\bar{k}-k} \left( \s \tilde{j}_{\bar{k}} + \Log\left(\frac{256 k_{\eps}^{2}}{\conf}\right) \right).
\end{align*}
Plugging this and \eqref{eqn:rc-sim-queries-mbark-queries-bound} into \eqref{eqn:rc-sim-queries-mbark-split},
we have that on $\bigcap_{j=0}^{4} E_{j} \cap G_{8}^{(iii)}(\bar{k},k) \cap G_{8}^{(iv)}(\bar{k},k) \cap G_{8}^{(v)}(\bar{k}) \cap \bigcap_{j=1}^{\tilde{j}_{\bar{k}}-1} G_{8}^{(ii)}(\bar{k},k,j)$,
\begin{align*}
& \left| \left\{ m \in \left\{1,\ldots,\tilde{m}_{k}\right\} : \gamma_{J_{m}} \geq 2^{-\bar{k}}, X_{m}^{2} \in \DIS(V_{m-1}) \right\} \right|
\\ & \leq 8 e \tilde{c} \left( \s \tilde{j}_{\bar{k}} + \Log\left(\frac{256 k_{\eps}^{2}}{\conf}\right) \right) \tilde{j}_{\bar{k}}
+ 64 e \tilde{c} 2^{\bar{k}-k} \left( \s \tilde{j}_{\bar{k}} + \Log\left(\frac{256 k_{\eps}^{2}}{\conf}\right) \right)
\\ & = 8 e \tilde{c} \left( 2^{3+\bar{k}-k} + \tilde{j}_{\bar{k}} \right) \left( \s \tilde{j}_{\bar{k}} + \Log\left(\frac{256 k_{\eps}^{2}}{\conf}\right) \right).
\end{align*}
Combined with the above result bounding the first term in \eqref{eqn:rc-sim-queries-bark-split},
we have that on $\bigcap_{j=0}^{4} E_{j} \cap G_{8}^{(i)}(\bar{k},k)$ $\cap G_{8}^{(iii)}(\bar{k},k) \cap G_{8}^{(iv)}(\bar{k},k) \cap G_{8}^{(v)}(\bar{k}) \cap \bigcap_{j=1}^{\tilde{j}_{\bar{k}}-1} G_{8}^{(ii)}(\bar{k},k,j)$,
\begin{align}
& \left| \left\{ m \in \left\{1,\ldots,\tilde{m}_{k}\right\} : X_{m}^{2} \in \DIS(V_{m-1}) \right\} \right| \notag
\\ & \leq \log_{2}\left(\frac{256 k_{\eps}^{2}}{\conf}\right) + 6e \max\left\{ \Px\left( x : \gamma_{x} < 2^{2-\bar{k}} \right), \frac{\eps}{2\gamma_{\eps}} \right\} \tilde{m}_{k} \notag
\\ & {\hskip 2.85cm}+ 8 e \tilde{c} \left( 2^{3+\bar{k}-k} + \tilde{j}_{\bar{k}} \right) \left( \s \tilde{j}_{\bar{k}} + \Log\left(\frac{256 k_{\eps}^{2}}{\conf}\right) \right) \notag
\\ & \leq 6e \max\!\left\{ \Px\!\left( x : \gamma_{x} < 2^{2-\bar{k}} \right), \frac{\eps}{2\gamma_{\eps}} \right\} \tilde{m}_{k}
+ \left(1+8 e \tilde{c}\right) \left( 2^{3+\bar{k}-k} + \tilde{j}_{\bar{k}} \right) \left( \s \tilde{j}_{\bar{k}} + \Log\!\left(\frac{256 k_{\eps}^{2}}{\conf}\right) \right)\!.
\label{eqn:rc-sim-queries-star-prefinal-bound}
\end{align}
Noting that $\s \geq \vc$, a bit of algebra reveals that 
\begin{equation*}
\frac{\tilde{m}_{\bar{k}}}{\s+\Log(1/\conf)} \leq 
\frac{32 c k_{\eps}}{\eps} \Log\left(\frac{128 k_{\eps}^{2}}{\eps}\right)
\leq \frac{2^{9} c k_{\eps}^{2}}{\eps^{3/2}},
\end{equation*}
so that 
\begin{equation*}
\tilde{j}_{\bar{k}} 
\leq \log_{2}\left( \frac{2^{10} c k_{\eps}^{2}}{\eps^{3/2}} \right)
\leq \frac{3}{2} \Log\left( \frac{2^{10} c k_{\eps}^{2}}{\eps^{3/2}} \right),
\end{equation*}
and therefore
\begin{align*}
& \left(1+8 e \tilde{c}\right) \left( 2^{3+\bar{k}-k} + \tilde{j}_{\bar{k}} \right) \left( \s \tilde{j}_{\bar{k}} + \Log\left(\frac{256 k_{\eps}^{2}}{\conf}\right) \right)
\\ & \leq \left(1+8 e \tilde{c}\right) \left( 2^{3+\bar{k}-k} + \frac{3}{2} \Log\left( \frac{2^{10} c k_{\eps}^{2}}{\eps^{3/2}} \right) \right) \left( \frac{3}{2} \s \Log\left( \frac{2^{10} c k_{\eps}^{2}}{\eps^{3/2}} \right) + \Log\left(\frac{256 k_{\eps}^{2}}{\conf}\right) \right)
\\ & \leq \left(1+8 e \tilde{c}\right) \left( 2^{3+\bar{k}-k} + \frac{3}{2} \Log\left( \frac{2^{10} c k_{\eps}^{2}}{\eps^{3/2}} \right) \right) \left( \frac{3}{2} \s \Log\left( \frac{2^{16} c k_{\eps}^{4}}{\eps^{3/2}} \right) + \Log\left(\frac{1}{\conf}\right) \right).
\end{align*}
Furthermore, since $k_{\eps} \leq \sqrt{32/\eps}$, this is at most
\begin{align*}
& \left(1+8 e \tilde{c}\right) \left( 2^{3+\bar{k}-k} + \frac{3}{2} \Log\left( \frac{2^{15} c}{\eps^{5/2}} \right) \right) \left( \frac{3}{2} \s \Log\left( \frac{2^{26} c}{\eps^{7/2}} \right) + \Log\left(\frac{1}{\conf}\right) \right)
\\ & \leq 91 \tilde{c} \left( 2^{1+\bar{k}-k} + \Log\left( \frac{64 c}{\eps} \right) \right) \left( 6 \s \Log\left( \frac{128 c}{\eps} \right) + \Log\left(\frac{1}{\conf}\right) \right).
\end{align*}
Plugging this into \eqref{eqn:rc-sim-queries-star-prefinal-bound}, we have that on
$\bigcap_{j=0}^{4} E_{j} \cap G_{8}^{(i)}(\bar{k},k) \cap G_{8}^{(iii)}(\bar{k},k) \cap G_{8}^{(iv)}(\bar{k},k) \cap G_{8}^{(v)}(\bar{k}) \cap \bigcap_{j=1}^{\tilde{j}_{\bar{k}}-1} G_{8}^{(ii)}(\bar{k},k,j)$,
\begin{multline}
\label{eqn:rc-sim-queries-star-final-bound} 
\left| \left\{ m \in \left\{1,\ldots,\tilde{m}_{k}\right\} : X_{m}^{2} \in \DIS(V_{m-1}) \right\} \right|
\\ \leq 
6e \max\left\{ \Px\left( x : \gamma_{x} < 2^{2-\bar{k}} \right), \frac{\eps}{2\gamma_{\eps}} \right\} \tilde{m}_{k}
\\ + 91 \tilde{c} \left( 2^{1+\bar{k}-k} + \Log\left( \frac{64 c}{\eps} \right) \right) \left( 6 \s \Log\left( \frac{128 c}{\eps} \right) + \Log\left(\frac{1}{\conf}\right) \right).
\end{multline}

Letting 
\begin{equation*}
E_{8} = \bigcap_{\bar{k}=3}^{k_{\eps}} \left(G_{8}^{(v)}(\bar{k}) \cap \bigcap_{k=2}^{\bar{k}-1} G_{8}^{(i)}(\bar{k},k) \cap G_{8}^{(iii)}(\bar{k},k) \cap G_{8}^{(iv)}(\bar{k},k) \cap \bigcap_{j=1}^{\tilde{j}_{\bar{k}}-1} G_{8}^{(ii)}(\bar{k},k,j)\right),
\end{equation*}
we have that \eqref{eqn:rc-sim-queries-star-final-bound} holds for all $\bar{k} \in \{3,\ldots,k_{\eps}\}$ and $k \in \{2,\ldots,\bar{k}-1\}$ 
on the event $E_{8} \cap \bigcap_{j=0}^{4} E_{j}$.
A union bound implies that $E_{8}$ has probability at least 
\begin{multline*}
1 - \sum_{\bar{k}=3}^{k_{\eps}} \left( \frac{\conf}{256 k_{\eps}} + \sum_{k=2}^{\bar{k}-1} \left( 3\frac{\conf}{256 k_{\eps}^{2}} + \sum_{j=1}^{\tilde{j}_{\bar{k}}-1} \frac{\conf}{256 k_{\eps}^{2} \tilde{j}_{\bar{k}}} \right)\right)
\\ \geq
1 - \frac{\conf}{256} - \sum_{\bar{k}=3}^{k_{\eps}} (\bar{k}-2) \frac{\conf}{64 k_{\eps}^{2}}
\geq 1 - \frac{\conf}{256} - \frac{\conf}{128}
> 1 - \frac{\conf}{64}.
\end{multline*}
\end{proof}

We can now state a sufficient size on the budget $n$ so that, with high probability,
\RQCAL~reaches $m=\tilde{m}$, so that the returned $\hat{h}_{n}$ is equivalent 
to the $\hat{h}_{\infty}$ classifier from Lemma~\ref{lem:infinite-budget-error},
which therefore satisfies the same guarantee on its error rate.

\begin{lemma}
\label{lem:total-queries}
There exists a finite universal constant $\bar{c} \geq 1$ such that,
on the event $\bigcap_{j=0}^{8} E_{j}$, for any $\bar{k} \in \{2,\ldots,k_{\eps}\}$,
for any $n$ of size at least
\begin{align}
& \bar{c} \ind[\bar{k}>2] 2^{2\bar{k}} \left( \s \Log\left( \frac{1}{\eps} \right) + \Log\left(\frac{1}{\conf}\right) \right) \Log\left(\frac{\vc}{\eps\conf}\right) \Log\left( \frac{1}{\eps} \right) \notag
\\ & + \bar{c} \sum_{k=\bar{k}}^{k_{\eps}} \max\left\{\Px\left(x : \gamma_{x} < 2^{3-k}\right), \frac{\eps}{\hat{\gamma}_{\eps}}\right\} \frac{2^{k}}{\eps} \left( \vc \Log\left(\frac{1}{\eps} \right) + \Log\left(\frac{1}{\conf}\right)\right) \Log\left(\frac{\vc}{\eps\conf}\right) \Log\left(\frac{1}{\hat{\gamma}_{\eps}}\right), \label{eqn:total-queries}
\end{align}
running \RQCAL~with budget $n$ results in at most $n$ label requests, and the returned classifier 
$\hat{h}_{n}$ satisfies 
$\er_{\PXY}(\hat{h}_{n}) - \er_{\PXY}(\target_{\PXY}) \leq \eps$.
Furthermore, the event $\bigcap_{j=0}^{8} E_{j}$ has probability at least $1-\conf$.
\end{lemma}
\begin{proof}
The value of $t$ keeps the running total of the number of label requests made by the algorithm after each call to \Filter.
Furthermore, within each execution of \Filter, the value $t+q$ represents the running total of the number of label requests 
made by the algorithm so far.  Since the $n-t$ budget argument to \Filter~ensures that it halts (in Step 6) if ever $t+q = n$,
and since the first condition in Step 1 of \RQCAL~ensures that \RQCAL~halts if ever $t=n$, we are guaranteed that the 
algorithm never requests a number of labels larger than the budget $n$.  

We will show that taking $n$ of the stated size suffices for the result by showing that this size 
suffices to reproduce the behavior of the infinite budget execution of \RQCAL.
Due to the condition $m < \tilde{m}$ in Step 1 of \RQCAL,
the final value of $t$ obtained when running \RQCAL~with budget $\infty$ may be expressed as
\begin{equation*}
\sum_{m=1}^{\tilde{m}} \hat{q}_{\infty,m} \ind_{\DIS(V_{m-1})}\left(X_{m}^{2}\right).
\end{equation*}
Lemma~\ref{lem:filter-good-labels} implies that, on $\bigcap_{j=0}^{8} E_{j}$, this is at most
\begin{align*}
& \sum_{m=1}^{\tilde{m}} \left\lceil \frac{8}{\max\{ \gamma_{J_{m}}^{2}, 2^{-2 \tilde{k}_{m}} \}} \ln\left(\frac{32 \tilde{m} q_{\eps,\conf}}{\conf}\right) \right\rceil \ind_{\DIS(V_{m-1})}\left(X_{m}^{2}\right)
\\ & \leq \sum_{m=1}^{\tilde{m}} \sum_{k=2}^{\tilde{k}_{m}} \ind\left[ \gamma_{J_{m}} \leq 2^{1-k} \right] 2^{2k+4} \ln\left(\frac{32 \tilde{m} q_{\eps,\conf}}{\conf}\right) \ind_{\DIS(V_{m-1})}\left(X_{m}^{2}\right).
\end{align*}
The summation in this last expression is over all $m \in \left\{1,\ldots,\tilde{m}\right\}$ and $k \in \{2,\ldots,k_{\eps}\}$ such that $k \leq \tilde{k}_{m}$,
which is equivalent to those $m \in \left\{1,\ldots,\tilde{m}\right\}$ and $k \in \{2,\ldots,k_{\eps}\}$ such that $m \leq \tilde{m}_{k}$.
Therefore, exchanging the order of summation, this expression is equal to
\begin{align}
& \sum_{k=2}^{k_{\eps}} \sum_{m=1}^{\tilde{m}_{k}} \ind\left[ \gamma_{J_{m}} \leq 2^{1-k} \right] 2^{2k+4} \ln\left(\frac{32 \tilde{m} q_{\eps,\conf}}{\conf}\right) \ind_{\DIS(V_{m-1})}\left(X_{m}^{2}\right) \notag
\\ & = \sum_{k=2}^{k_{\eps}} 2^{2k+4} \ln\left(\frac{32 \tilde{m} q_{\eps,\conf}}{\conf}\right) \left| \left\{ m \in \left\{1,\ldots,\tilde{m}_{k}\right\} : \gamma_{J_{m}} \leq 2^{1-k}, X_{m}^{2} \in \DIS(V_{m-1}) \right\} \right|. \label{eqn:total-queries-basic-sum}
\end{align}
Fix any value $\bar{k} \in \{2,\ldots,k_{\eps}\}$.
For any $k \in \left\{\bar{k},\ldots,k_{\eps}\right\}$, 
Lemma~\ref{lem:rc-simulated-query-bound-basic} implies that, on $\bigcap_{j=0}^{8} E_{j}$,
\begin{multline*}
\left| \left\{ m \in \left\{1,\ldots,\tilde{m}_{k}\right\} : \gamma_{J_{m}} \leq 2^{1-k}, X_{m}^{2} \in \DIS(V_{m-1}) \right\} \right|
\\ \leq 
17 \max\left\{\Px\left(x : \gamma_{x} < 2^{3-k}\right), \frac{\eps}{2\hat{\gamma}_{\eps}}\right\} \tilde{m}_{k}.
\end{multline*}
This implies
\begin{align}
& \sum_{k=\bar{k}}^{k_{\eps}} 2^{2k+4} \ln\left(\frac{32 \tilde{m} q_{\eps,\conf}}{\conf}\right)\left| \left\{ m \in \left\{1,\ldots,\tilde{m}_{k}\right\} : \gamma_{J_{m}} \leq 2^{1-k}, X_{m}^{2} \in \DIS(V_{m-1}) \right\} \right| \notag
\\ & \leq \sum_{k=\bar{k}}^{k_{\eps}} 2^{2k+9} \ln\left(\frac{32 \tilde{m} q_{\eps,\conf}}{\conf}\right) \max\left\{\Px\left(x : \gamma_{x} < 2^{3-k}\right), \frac{\eps}{2\hat{\gamma}_{\eps}}\right\} \tilde{m}_{k} \notag
\\ & \leq \sum_{k=\bar{k}}^{k_{\eps}} \max\!\left\{\!\Px\!\left(x : \gamma_{x} < 2^{3-k}\right), \frac{\eps}{2\hat{\gamma}_{\eps}}\!\right\} \frac{2^{k+17} c k_{\eps}}{\eps} \left(\! \vc \Log\!\left(\frac{2 k_{\eps}}{\eps} \right) \!+\! \Log\!\left(\frac{64 k_{\eps}}{\conf}\right)\!\right)
\!\Log\!\left(\frac{32 \tilde{m} q_{\eps,\conf}}{\conf}\right)\notag
\\ & \leq \sum_{k=\bar{k}}^{k_{\eps}} \max\!\left\{\!\Px\!\left(x : \gamma_{x} < 2^{3-k}\right), \frac{\eps}{2\hat{\gamma}_{\eps}}\!\right\} \frac{2^{k+25} c \Log\!\left(\frac{1}{\hat{\gamma}_{\eps}}\right)}{\eps} \left(\! \vc \Log\!\left(\frac{64}{\eps} \right) \!+\! \Log\!\left(\frac{1}{\conf}\right)\!\right) \!\Log\!\left(\frac{32 c \vc}{\eps\conf}\right),
\label{eqn:total-queries-basic-high-k}
\end{align}
where this last inequality is based on the fact that $k_{\eps} \leq \sqrt{32/\eps}$,
combined with some simple algebra.
If $\bar{k} > 2$, for any $k \in \left\{2,\ldots,\bar{k}-1\right\}$, 
Lemma~\ref{lem:rc-simulated-query-bound-star} implies that, 
on $\bigcap_{j=0}^{8} E_{j}$, 
\begin{multline*}
\left| \left\{ m \in \left\{1,\ldots,\tilde{m}_{k}\right\} : \gamma_{J_{m}} \leq 2^{1-k}, X_{m}^{2} \in \DIS(V_{m-1}) \right\} \right|
\\ \leq 6e \max\left\{ \Px\left( x : \gamma_{x} < 2^{2-\bar{k}} \right), \frac{\eps}{2\gamma_{\eps}} \right\} \tilde{m}_{k}
\\ + 91\tilde{c}\left( 2^{1+\bar{k}-k} + \Log\left( \frac{64 c}{\eps} \right) \right) \left( 6 \s \Log\left( \frac{128 c}{\eps} \right) + \Log\left(\frac{1}{\conf}\right) \right). 
\end{multline*}
This implies
\begin{align*}
& \sum_{k=2}^{\bar{k}-1} 2^{2k+4} \ln\left(\frac{32 \tilde{m} q_{\eps,\conf}}{\conf}\right)\left| \left\{ m \in \left\{1,\ldots,\tilde{m}_{k}\right\} : \gamma_{J_{m}} \leq 2^{1-k}, X_{m}^{2} \in \DIS(V_{m-1}) \right\} \right|
\\ & \leq \sum_{k=2}^{\bar{k}-1} 2^{2k+9} \ln\left(\frac{32 \tilde{m} q_{\eps,\conf}}{\conf}\right) \max\left\{ \Px\left( x : \gamma_{x} < 2^{2-\bar{k}} \right), \frac{\eps}{2\gamma_{\eps}} \right\} \tilde{m}_{k}
\\ & + \sum_{k=2}^{\bar{k}-1} 2^{2k+11} \tilde{c} \ln\left(\frac{32 \tilde{m} q_{\eps,\conf}}{\conf}\right) \left( 2^{1+\bar{k}-k} + \Log\left( \frac{64 c}{\eps} \right) \right) \left( 6 \s \Log\left( \frac{128 c}{\eps} \right) + \Log\left(\frac{1}{\conf}\right) \right).
\end{align*}
Since
\begin{align*}
\sum_{k=2}^{\bar{k}-1} 2^{2k} \tilde{m}_{k}
& \leq \sum_{k=2}^{\bar{k}-1} \frac{2^{k+8} c k_{\eps}}{\eps} \left( \vc \Log\left(\frac{2 k_{\eps}}{\eps}\right) + \Log\left(\frac{64 k_{\eps}}{\conf}\right) \right)
\\ & \leq \frac{2^{\bar{k}+8} c k_{\eps}}{\eps} \left( \vc \Log\left(\frac{2 k_{\eps}}{\eps}\right) + \Log\left(\frac{64 k_{\eps}}{\conf}\right) \right)
\\ & \leq \frac{2^{\bar{k}+12} c \Log(1/\hat{\gamma}_{\eps})}{\eps} \left( \vc \Log\left(\frac{64}{\eps}\right) + \Log\left(\frac{1}{\conf}\right) \right)
\end{align*}
and
\begin{equation*}
\sum_{k=2}^{\bar{k}-1} 2^{2k} \left( 2^{1+\bar{k}-k} + \Log\left( \frac{64 c}{\eps} \right) \right) 
\leq 2^{2\bar{k}} \left( 2 + \Log\left( \frac{64 c}{\eps} \right)\right)
\leq 2^{2\bar{k}+1} \Log\left( \frac{64 c}{\eps} \right),
\end{equation*}
we have that
\begin{align*}
& \sum_{k=2}^{\bar{k}-1} 2^{2k+4} \ln\left(\frac{32 \tilde{m} q_{\eps,\conf}}{\conf}\right)\left| \left\{ m \in \left\{1,\ldots,\tilde{m}_{k}\right\} : \gamma_{J_{m}} \leq 2^{1-k}, X_{m}^{2} \in \DIS(V_{m-1}) \right\} \right|
\\ & \leq 
2^{9} \ln\left(\frac{32 \tilde{m} q_{\eps,\conf}}{\conf}\right) \max\left\{ \Px\left( x : \gamma_{x} < 2^{2-\bar{k}} \right), \frac{\eps}{2\gamma_{\eps}} \right\}\sum_{k=2}^{\bar{k}-1} 2^{2k} \tilde{m}_{k}
\\ & \phantom{\leq} + 2^{11} \tilde{c} \ln\left(\frac{32 \tilde{m} q_{\eps,\conf}}{\conf}\right) \left( 6 \s \Log\left( \frac{128 c}{\eps} \right) + \Log\left(\frac{1}{\conf}\right) \right)\sum_{k=2}^{\bar{k}-1} 2^{2k} \left( 2^{1+\bar{k}-k} + \Log\left( \frac{64 c}{\eps} \right) \right) 
\\ & \leq 2^{9} \ln\!\left(\frac{32 \tilde{m} q_{\eps,\conf}}{\conf}\right) \max\!\left\{\! \Px\!\left( x : \gamma_{x} < 2^{2-\bar{k}} \right), \frac{\eps}{2\gamma_{\eps}} \!\right\} \frac{2^{\bar{k}+12} c \Log\!\left(\frac{1}{\hat{\gamma}_{\eps}}\right)}{\eps} \left(\! \vc \Log\!\left(\frac{64}{\eps}\right) + \Log\!\left(\frac{1}{\conf}\right) \!\right)
\\ & \phantom{\leq} + 2^{11} \tilde{c} \ln\left(\frac{32 \tilde{m} q_{\eps,\conf}}{\conf}\right) \left( 6 \s \Log\left( \frac{128 c}{\eps} \right) + \Log\left(\frac{1}{\conf}\right) \right) 2^{2\bar{k}+1} \Log\left( \frac{64 c}{\eps} \right)
\\ & \leq 
\max\left\{ \Px\left( x : \gamma_{x} < 2^{2-\bar{k}} \right), \frac{\eps}{2\gamma_{\eps}} \right\} \frac{2^{\bar{k}+25} c \Log\!\left(\frac{1}{\hat{\gamma}_{\eps}}\right)}{\eps} \left( \vc \Log\!\left(\frac{64}{\eps}\right) + \Log\!\left(\frac{1}{\conf}\right) \right) \Log\left(\frac{32 c \vc}{\eps\conf}\right)
\\ & \phantom{\leq} + 2^{2\bar{k}+16} \tilde{c} \left( 6 \s \Log\left( \frac{128 c}{\eps} \right) + \Log\left(\frac{1}{\conf}\right) \right) \Log\left( \frac{64 c}{\eps} \right) \Log\left(\frac{32 c \vc}{\eps\conf}\right).
\end{align*}
Plugging this and \eqref{eqn:total-queries-basic-high-k} into \eqref{eqn:total-queries-basic-sum} reveals that, on $\bigcap_{j=0}^{8} E_{j}$, if $\bar{k} > 2$,

\begin{align*}
& \sum_{m=1}^{\tilde{m}} \hat{q}_{\infty,m} \ind_{\DIS(V_{m-1})}\left(X_{m}^{2}\right)
\\ & \leq 
\max\left\{ \Px\left( x : \gamma_{x} < 2^{2-\bar{k}} \right), \frac{\eps}{2\gamma_{\eps}} \right\} \frac{2^{\bar{k}+25} c \Log\left(\frac{1}{\hat{\gamma}_{\eps}}\right)}{\eps} \left( \vc \Log\left(\frac{64}{\eps}\right) + \Log\left(\frac{1}{\conf}\right) \right) \Log\left(\frac{32 c \vc}{\eps\conf}\right)
\\ & + 2^{2\bar{k}+16} \tilde{c} \left( 6 \s \Log\left( \frac{128 c}{\eps} \right) + \Log\left(\frac{1}{\conf}\right) \right) \Log\left( \frac{64 c}{\eps} \right) \Log\left(\frac{32 c \vc}{\eps\conf}\right)
\\ & + \sum_{k=\bar{k}}^{k_{\eps}} \max\!\left\{\!\Px\!\left(x : \gamma_{x} < 2^{3-k}\right), \frac{\eps}{2\hat{\gamma}_{\eps}}\!\right\} \frac{2^{k+25} c \Log\left(\frac{1}{\hat{\gamma}_{\eps}}\right)}{\eps} \left(\! \vc \Log\!\left(\frac{64}{\eps} \right) \!+\! \Log\!\left(\frac{1}{\conf}\right)\!\right) \!\Log\!\left(\frac{32 c \vc}{\eps\conf}\right).
\\ & \leq \bar{c} 2^{2\bar{k}} \left( \s \Log\left( \frac{1}{\eps} \right) + \Log\left(\frac{1}{\conf}\right) \right) \Log\left(\frac{\vc}{\eps\conf}\right) \Log\left( \frac{1}{\eps} \right)
\\ & + \bar{c} \sum_{k=\bar{k}}^{k_{\eps}} \max\left\{\Px\left(x : \gamma_{x} < 2^{3-k}\right), \frac{\eps}{\hat{\gamma}_{\eps}}\right\} \frac{2^{k}}{\eps} \left( \vc \Log\left(\frac{1}{\eps} \right) + \Log\left(\frac{1}{\conf}\right)\right) \Log\left(\frac{\vc}{\eps\conf}\right) \Log\left(\frac{1}{\hat{\gamma}_{\eps}}\right),
\end{align*}
for an appropriate finite universal constant $\bar{c} \geq 1$.
Furthermore, if $\bar{k} = 2$, \eqref{eqn:total-queries-basic-high-k} and \eqref{eqn:total-queries-basic-sum} already imply that, on $\bigcap_{j=0}^{8} E_{j}$, 
\begin{align*}
& \sum_{m=1}^{\tilde{m}} \hat{q}_{\infty,m} \ind_{\DIS(V_{m-1})}\left(X_{m}^{2}\right)
\\ & \leq \bar{c} \sum_{k=\bar{k}}^{k_{\eps}} \max\left\{\Px\left(x : \gamma_{x} < 2^{3-k}\right), \frac{\eps}{\hat{\gamma}_{\eps}}\right\} \frac{2^{k}}{\eps} \left( \vc \Log\left(\frac{1}{\eps} \right) + \Log\left(\frac{1}{\conf}\right)\right) \Log\left(\frac{\vc}{\eps\conf}\right) \Log\left(\frac{1}{\hat{\gamma}_{\eps}}\right),
\end{align*}
again for $\bar{c} \geq 1$ chosen appropriately large.

Therefore, for a choice of $\bar{c}$ as above, on $\bigcap_{j=0}^{8} E_{j}$, for any $\bar{k} \in \{2,\ldots,k_{\eps}\}$, 
the final value of $t$ obtained when running \RQCAL~with budget $\infty$ is at most \eqref{eqn:total-queries}.
Since running \RQCAL~with a finite budget $n$ only returns a different $\hat{h}_{n}$ from the $\hat{h}_{\infty}$
returned by the infinite-budget execution if $t$ would exceed $n$ in the infinite-budget execution,
this implies that taking any $n$ of size at least \eqref{eqn:total-queries} suffices to produce identical 
output to the infinite-budget execution, on the event $\bigcap_{j=0}^{8} E_{j}$:
that is, $\hat{h}_{n} = \hat{h}_{\infty}$.
Therefore, since Lemma~\ref{lem:infinite-budget-error} implies that, on $\bigcap_{j=0}^{8} E_{j}$, $\er_{\PXY}(\hat{h}_{\infty}) - \er_{\PXY}(\target_{\PXY}) \leq \eps$,
we conclude that for $n$ of size at least \eqref{eqn:total-queries}, on $\bigcap_{j=0}^{8} E_{j}$, $\er_{\PXY}(\hat{h}_{n}) - \er_{\PXY}(\target_{\PXY}) \leq \eps$.

Finally, by a union bound, the event $\bigcap_{j=0}^{8} E_{j}$ has probability at least
\begin{equation*}
1 - 0 - \frac{\conf}{2} - \frac{\conf}{512} - \frac{\conf}{4} - \frac{\conf}{32} - 4\frac{\conf}{64}
> 1 - \conf.
\end{equation*}
\end{proof}

We can obtain the upper bounds for Theorems~\ref{thm:bounded}, \ref{thm:tsybakov}, and \ref{thm:benign} 
from Section~\ref{sec:main} by straightforward applications of Lemma~\ref{lem:total-queries}.  
Note that, due to the choice of $\hat{\gamma}_{\eps}$ in each of these proofs, \RQCAL~is not 
adaptive to the noise parameters.  It is conceivable that this dependence can be removed by 
a model selection procedure \citep*[see][for discussions related to this]{hanneke:12c,hanneke:11a}.
However, we do not discuss this further here, leaving this important issue for future work.
The upper bounds for 
Theorems~\ref{thm:diameter-localization} and \ref{thm:agnostic} are based on known results for other 
algorithms in the literature, though the lower bound for Theorem~\ref{thm:diameter-localization} is
new here.  The remainder of this section provides the details of these proofs.

\begin{proof}[of Theorem~\ref{thm:bounded}]
Fix any $\bound \in [0,1/2)$, $\eps,\conf \in (0,1)$, and $\PXY \in \BN(\bound)$.
Any $\gamma < 1/2-\bound$ has $\Px(x : \gamma_{x} \leq \gamma) = 0$,
and since we always have $\gamma_{\eps} \geq \eps/2$,
we must have $\gamma_{\eps} \geq \max\{1/2-\bound,\eps/2\}$.  We may therefore
take $\hat{\gamma}_{\eps} = \max\{1/2-\bound,\eps/2\}$.  
Therefore, taking $\bar{k} = k_{\eps}$ in Lemma~\ref{lem:total-queries},
the first term in \eqref{eqn:total-queries} is at most
\begin{equation*}
\frac{2^{10} \bar{c}}{(1-2\bound)^{2}} \left( \s \Log\left(\frac{1}{\eps}\right) + \Log\left(\frac{1}{\conf}\right) \right) \Log\left(\frac{\vc}{\eps\conf}\right)\Log\left(\frac{1}{\eps}\right),
\end{equation*}
while the second term in \eqref{eqn:total-queries} is at most
\begin{equation*}
\bar{c} \max\left\{ \Px\left(x : \gamma_{x} < \hat{\gamma}_{\eps} \right), \frac{\eps}{\hat{\gamma}_{\eps}} \right\} \frac{16}{\hat{\gamma}_{\eps} \eps} \left( \vc \Log\left(\frac{1}{\eps}\right) + \Log\left(\frac{1}{\conf}\right)\right) \Log\left(\frac{\vc}{\eps\conf}\right)\Log\left(\frac{1}{\hat{\gamma}_{\eps}}\right).
\end{equation*}
Since $\Px\left(x : \gamma_{x} < 1/2-\bound \right) = 0 < \frac{\eps}{1/2-\bound}$
and $\Px\left(x : \gamma_{x} < \eps/2\right) \leq 1 < 2 = \frac{\eps}{\eps/2}$,
we have that $\Px\left(x : \gamma_{x} < \hat{\gamma}_{\eps}\right) < \frac{\eps}{\hat{\gamma}_{\eps}}$,
so that the above is at most
\begin{equation*}
\frac{64 \bar{c}}{(1-2\bound)^{2}} \left( \vc \Log\left(\frac{1}{\eps}\right) + \Log\left(\frac{1}{\conf}\right)\right) \Log\left(\frac{\vc}{\eps\conf}\right)\Log\left(\frac{2}{(1-2\bound) \lor \eps}\right).
\end{equation*}
Therefore, recalling that $\s \geq \vc$,
since Lemma~\ref{lem:total-queries} implies that, with any budget $n$ at least the size of the sum of these two terms, 
\RQCAL~produces a classifier $\hat{h}_{n}$ with $\er_{\PXY}(\hat{h}_{n})-\er_{\PXY}(\target_{\PXY}) \leq \eps$ with probability at least $1-\conf$,
and requests a number of labels at most $n$, we have that
\begin{align*}
\LC_{\BN(\bound)}(\eps,\conf) & \leq
\frac{2^{10} \bar{c}}{(1-2\bound)^{2}} \left( \s \Log\left(\frac{1}{\eps}\right) + \Log\left(\frac{1}{\conf}\right) \right) \Log\left(\frac{\vc}{\eps\conf}\right)\Log\left(\frac{1}{\eps}\right)
\\ & + \frac{64 \bar{c}}{(1-2\bound)^{2}} \left( \vc \Log\left(\frac{1}{\eps}\right) + \Log\left(\frac{1}{\conf}\right)\right) \Log\left(\frac{\vc}{\eps\conf}\right)\Log\left(\frac{2}{(1-2\bound)\lor\eps}\right)
\\ & \lesssim \frac{1}{(1-2\bound)^{2}} \left( \s \Log\left(\frac{1}{\eps}\right) + \Log\left(\frac{1}{\conf}\right) \right) \Log\left(\frac{\vc}{\eps\conf}\right)\Log\left(\frac{1}{\eps}\right).
\end{align*}

On the other hand, \citet*{gine:06} have shown that for the passive learning method 
of \emph{empirical risk minimization}, producing the classifier 
$\check{h}_{n} = \argmin_{h \in \C} \sum_{m=1}^{n} \ind[ h(X_{m}) \neq Y_{m} ]$,
if $n$ is of size at least
\begin{equation*}
\frac{\check{c}}{(1-2\bound) \eps} \left( \vc \Log\left( \dc_{\PXY}\left( \frac{\eps}{1-2\bound} \right) \right) + \Log\left(\frac{1}{\conf}\right) \right),
\end{equation*}
for an appropriate finite universal constant $\check{c}$,
then with probability at least $1-\conf$, $\er_{\PXY}(\check{h}_{n}) - \er_{\PXY}(\target_{\PXY}) \leq \eps$.
Therefore, since Theorem~\ref{thm:dc-star} implies $\dc_{\PXY}(\eps / (1-2\bound)) \leq \dc_{\PXY}((\eps / (1-2\bound)) \land 1) \leq \min\left\{ \s, \frac{1-2\bound}{\eps} \lor 1\right\}$, we have
\begin{equation*}
\LC_{\BN(\bound)}(\eps,\conf) \lesssim \frac{1}{(1-2\bound) \eps} \left( \vc \Log\left( \min\left\{ \s,  \frac{1-2\bound}{\eps} \right\} \right) + \Log\left(\frac{1}{\conf}\right) \right).
\end{equation*}
Together, these two bounds on $\LC_{\BN(\bound)}(\eps,\conf)$ imply the following upper bound,
simply by choosing whichever of these two methods has the smaller corresponding bound for the given values of $\eps$, $\conf$, $\bound$, $\vc$, and $\s$.
\begin{equation*}
\LC_{\BN(\bound)}(\eps,\conf) \lesssim 
\min\begin{cases}
\frac{1}{(1-2\bound)^{2}} \left( \s \Log\left(\frac{1}{\eps}\right) + \Log\left(\frac{1}{\conf}\right) \right) \Log\left(\frac{\vc}{\eps\conf}\right)\Log\left(\frac{1}{\eps}\right)
\\ \frac{1}{(1-2\bound) \eps} \left( \vc \Log\left( \min\left\{ \s,  \frac{1-2\bound}{\eps} \right\} \right) + \Log\left(\frac{1}{\conf}\right) \right)
\end{cases}.
\end{equation*}
The statement of the upper bound in Theorem~\ref{thm:bounded} represents a relaxation of this,
in that it is slightly larger (in the logarithmic factors), the intention being that it is a simpler expression to state.
To arrive at this relaxation, we note that $\s \Log\left(\frac{1}{\eps}\right) + \Log\left(\frac{1}{\conf}\right) \leq \s \Log\left(\frac{1}{\eps\conf}\right)$,
and $\vc\Log\left(\min\left\{\s, \frac{1-2\bound}{\eps}\right\}\right) + \Log\left(\frac{1}{\conf}\right) \leq \vc \Log\left(\frac{1}{\eps\conf}\right) \Log\left(\frac{\vc}{\eps\conf}\right) \Log\left(\frac{1}{\eps}\right)$,
so that the above is at most
\begin{equation*}
\frac{1}{(1-2\bound)^{2}} \min\left\{ \s, \frac{(1-2\bound) \vc}{\eps} \right\} \Log\left(\frac{\vc}{\eps\conf}\right) \Log\left(\frac{1}{\eps\conf}\right) \Log\left(\frac{1}{\eps}\right).
\end{equation*}

Next, we turn to establishing the lower bound.
Fix $\eps \in (0,(1-2\bound)/24)$ and $\conf \in (0,1/24]$.
First note that taking $\zeta = \frac{2\eps}{1-2\bound}$ and $k = \min\left\{\s-1,\lfloor 1/\zeta \rfloor\right\}$ in Lemma~\ref{lem:rr11-star},
we have $\RR(k,\zeta,\bound) \subseteq \BN(\bound)$, so that Lemma~\ref{lem:rr11-star} implies
\begin{multline}
\label{eqn:bounded-rr-lb}
\LC_{\BN(\bound)}(\eps,\conf) 
\geq \LC_{\RR(k,\zeta,\bound)}(\eps,\conf) 
= \LC_{\RR(k,\zeta,\bound)}((\zeta/2)(1-2\bound),\conf) 
\geq \frac{\bound (k-1) \ln\left(\frac{1}{4\conf}\right)}{3(1-2\bound)^{2}}
\\ \geq \min\left\{ \s - 2, \frac{1-2\zeta}{\zeta}\right\} \frac{\bound \ln\left(\frac{1}{4\conf}\right)}{3(1-2\bound)^{2}}
= \frac{\bound}{(1-2\bound)^{2}} \min\left\{ \s-2, \frac{1-2\bound-4\eps}{2\eps} \right\} \ln\left(\frac{1}{4\conf}\right).
\\ \geq \frac{\bound}{8(1-2\bound)^{2}} \min\left\{ \s-2, \frac{1-2\bound}{\eps} \right\} \Log\left(\frac{1}{\conf}\right).
\end{multline}

Additionally, based on techniques of \citet*{kaariainen:06,beygelzimer:09,hanneke:11a},
the recent article of \citet*{hanneke:survey} contains the following lower bound (in the proof of Theorem 4.3 there),
for $\eps \in (0,(1-2\bound)/24)$ and $\conf \in (0,1/24]$.
\begin{align*}
\LC_{\BN(\bound)}(\eps,\conf) 
& \geq \max\left\{ 2 \left\lfloor \frac{1 - (1-2\bound)^{2}}{2(1-2\bound)^{2}} \ln\left( \frac{1}{8\conf(1-2\conf)} \right) \right\rfloor, \frac{\vc-1}{6} \left\lfloor \frac{1-(1-2\bound)^{2}}{2(1-2\bound)^{2}} \ln\left(\frac{9}{8}\right) \right\rfloor \right\}
\\ & \geq \max\left\{ 2 \left\lfloor \frac{\bound}{(1-2\bound)^{2}} \Log\left( \frac{1}{8\conf} \right) \right\rfloor, \frac{\vc-1}{6} \left\lfloor \frac{\bound}{10(1-2\bound)^{2}} \right\rfloor \right\}.
\end{align*}
If $\frac{\bound}{(1-2\bound)^{2}} \Log\left(\frac{1}{8\conf}\right) \geq 1$, then 
$2 \left\lfloor \frac{\bound}{(1-2\bound)^{2}} \Log\left( \frac{1}{8\conf} \right) \right\rfloor \geq \frac{\bound}{(1-2\bound)^{2}} \Log\left(\frac{1}{8\conf}\right) \geq \frac{\bound}{3 (1-2\bound)^{2}} \Log\left(\frac{1}{\conf}\right)$,
so that $\LC_{\BN(\bound)}(\eps,\conf) \gtrsim \frac{\bound}{(1-2\bound)^{2}} \Log\left(\frac{1}{\conf}\right)$.
Otherwise, if $\frac{\bound}{(1-2\bound)^{2}} \Log\left(\frac{1}{8\conf}\right) < 1$,
then since $\RE \subseteq \BN(\bound)$, and $|\C| \geq 2$ implies $\vc \geq 1 > \frac{\bound}{(1-2\bound)^{2}} \Log\left(\frac{1}{8\conf}\right)$, 
Theorem~\ref{thm:realizable} (proven above) implies we still have
$\LC_{\BN(\bound)}(\eps,\conf) \geq \LC_{\RE}(\eps,\conf) \gtrsim \frac{\bound}{(1-2\bound)^{2}} \Log\left(\frac{1}{\conf}\right)$
in this case.
When $\vc = 1$, these observations further imply $\LC_{\BN(\bound)} \gtrsim \frac{\vc \bound}{(1-2\bound)^{2}}$.
On the other hand, if $\vc > 1$, 
and if $\frac{\bound}{10(1-2\bound)^{2}} \geq 1$, then 
$\frac{\vc-1}{6} \left\lfloor \frac{\bound}{10(1-2\bound)^{2}} \right\rfloor \geq \frac{\vc}{240} \frac{\bound}{(1-2\bound)^{2}}$,
so that $\LC_{\BN(\bound)}(\eps,\conf) \gtrsim \frac{\vc \bound}{(1-2\bound)^{2}}$.
Otherwise, if $\frac{\bound}{10(1-2\bound)^{2}} < 1$, then since $\RE \subseteq \BN(\bound)$, 
Theorem~\ref{thm:realizable} implies we still have
$\LC_{\BN(\bound)}(\eps,\conf) \geq \LC_{\RE}(\eps,\conf) \gtrsim \vc \gtrsim \frac{\vc \bound}{(1-2\bound)^{2}}$
in this case as well.
If $\bound > 1/4$, then $\frac{\vc \bound}{(1-2\bound)^{2}} \geq \frac{\vc}{4 (1-2\bound)^{2}} \gtrsim \frac{\vc}{(1-2\bound)^{2}}$, so that $\LC_{\BN(\bound)}(\eps,\conf) \gtrsim \frac{\vc}{(1-2\bound)^{2}}$.
Otherwise, if $\bound \leq 1/4$, then $\frac{1}{(1-2\bound)^{2}} \leq 4$, so that Theorem~\ref{thm:realizable} implies
$\LC_{\BN(\bound)}(\eps,\conf) \geq \LC_{\RE}(\eps,\conf) \gtrsim \vc \gtrsim \frac{\vc}{(1-2\bound)^{2}}$.
Altogether, we have that
\begin{equation}
\label{eqn:bounded-survey-lb}
\LC_{\BN(\bound)}(\eps,\conf) \gtrsim \frac{1}{(1-2\bound)^{2}} \max\left\{ \bound \Log\left(\frac{1}{\conf}\right), \vc \right\}.
\end{equation}

When $\s \leq 2$, $\min\left\{ \s, \frac{1-2\bound}{\eps} \right\} \leq 2$, so that \eqref{eqn:bounded-survey-lb} trivially implies
\begin{equation}
\label{eqn:bn-lb-conclusion}
\LC_{\BN(\bound)}(\eps,\conf) \gtrsim \frac{1}{(1-2\bound)^{2}} \max\left\{ \min\left\{ \s, \frac{1-2\bound}{\eps} \right\} \bound \Log\left(\frac{1}{\conf}\right), \vc \right\}.
\end{equation}
Otherwise, when $\s \geq 3$, we have $\s - 2 \geq \s/3$, so that $\min\left\{ \s-2, \frac{1-2\bound}{\eps} \right\} \geq \frac{1}{3} \min\left\{ \s, \frac{1-2\bound}{\eps} \right\}$.
Combined with \eqref{eqn:bounded-rr-lb} and \eqref{eqn:bounded-survey-lb}, this implies \eqref{eqn:bn-lb-conclusion} holds
in this case as well.
\end{proof}

\begin{proof}[of Theorem~\ref{thm:tsybakov}]
We begin with the upper bounds.
Fix any $\tsybca \in [1,\infty)$, $\tsyba \in (0,1)$, $\eps,\conf \in (0,1)$, and $\PXY \in \TN(\tsybca,\tsyba)$.
For any $\gamma \leq \left(\frac{\eps}{2\tsybca^{\prime}}\right)^{1-\tsyba}$, by definition of $\TN(\tsybca,\tsyba)$, we have
$\gamma \Px\left( x : \gamma_{x} \leq \gamma \right) \leq \tsybca^{\prime} \gamma^{1/(1-\tsyba)} \leq \eps/2$.
Therefore, since we always have $\gamma_{\eps} \geq \eps/2$, we have
$\gamma_{\eps} \geq \max\left\{ \left(\frac{\eps}{2\tsybca^{\prime}}\right)^{1-\tsyba}, \frac{\eps}{2} \right\}$,
so that we can take $\hat{\gamma}_{\eps} = \max\left\{\left(\frac{\eps}{2\tsybca^{\prime}}\right)^{1-\tsyba}, \frac{\eps}{2}\right\}$.

Therefore, taking $\bar{k} = 2$ in Lemma~\ref{lem:total-queries} implies that, with any budget $n$ of size at least
\begin{equation}
\label{eqn:tsybakov-basic-initial}
\bar{c} \sum_{k=2}^{k_{\eps}} \max\!\left\{\min\!\left\{\tsybca^{\prime} 2^{(3-k) \tsyba / (1-\tsyba)},1\right\}, \frac{\eps}{\hat{\gamma}_{\eps}}\right\} \frac{2^{k}}{\eps} \left( \vc \Log\!\left(\frac{1}{\eps} \right) + \Log\!\left(\frac{1}{\conf}\right)\right) \Log\!\left(\frac{\vc}{\eps\conf}\right) \Log\!\left(\frac{1}{\hat{\gamma}_{\eps}}\right),
\end{equation}
\RQCAL~produces a classifier $\hat{h}_{n}$ with $\er_{\PXY}(\hat{h}_{n}) - \er_{\PXY}(\target_{\PXY}) \leq \eps$
with probability at least $1-\conf$, and requests a number of labels at most $n$.
This implies $\LC_{\TN(\tsybca,\tsyba)}(\eps,\conf)$ is at most \eqref{eqn:tsybakov-basic-initial}.

First note that
\begin{align}
\sum_{k=2}^{k_{\eps}} \frac{\eps}{\hat{\gamma}_{\eps}} \frac{2^{k}}{\eps}
& \leq \frac{2^{1+k_{\eps}}}{\hat{\gamma}_{\eps}}
= \frac{2^{ \left\lceil \log_{2}\left( 16 / \hat{\gamma}_{\eps} \right) \right\rceil }}{\hat{\gamma}_{\eps}}
\leq \frac{32}{\hat{\gamma}_{\eps}^{2}}
\leq 32 \min\left\{ \left( 2 \tsybca^{\prime} \right)^{2-2\tsyba} \eps^{2\tsyba-2}, 4 \eps^{-2} \right\} \notag
\\ & = 32 \min\left\{ (2-2\tsyba)^{2-2\tsyba} (2\tsyba)^{2\tsyba} \tsybca^{2} \eps^{2\tsyba-2}, 4 \eps^{-2} \right\}
\leq 128 \min\left\{ \tsybca^{2} \eps^{2\tsyba-2}, \eps^{-2} \right\}. \label{eqn:tsybakov-first-note}
\end{align}
Furthermore, since $\eps^{-2} < \tsybca^{2} \eps^{2\tsyba-2}$ only if $\eps > \tsybca^{-1/\tsyba}$, 
this is at most $128 \min\left\{ \tsybca^{2} \eps^{2\tsyba-2}, \tsybca^{1/\tsyba} \eps^{-1} \right\}$.
Also, for $\tsyba \geq 1/2$, letting $k_{(\tsybca,\tsyba)} = \left\lceil \log_{2}\left( 8 \left(\tsybca^{\prime}\right)^{(1-\tsyba)/\tsyba} \right) \right\rceil$,
we have $k_{(\tsybca,\tsyba)} \geq 2$.
Additionally, for $\tsyba \geq 1/2$, $2^{k \frac{1-2\tsyba}{1-\tsyba}}$ is nonincreasing in $k$.
In particular, if $k_{(\tsybca,\tsyba)} = 2$, then 
\begin{equation*}
\sum_{k=2}^{k_{\eps}} \min\left\{ \tsybca^{\prime} 2^{(3-k) \tsyba / (1-\tsyba)}, 1 \right\} \frac{2^{k}}{\eps}
\leq \sum_{k=k_{(\tsybca,\tsyba)}}^{k_{\eps}} \frac{8 \tsybca^{\prime}}{\eps} 2^{(k-3) \frac{1-2\tsyba}{1-\tsyba}}
\leq \frac{8 k_{\eps}}{\eps} (\tsybca^{\prime})^{\frac{1-\tsyba}{\tsyba}}.
\end{equation*}
Otherwise, if $k_{(\tsybca,\tsyba)} \geq 3$, then
\begin{align*}
\sum_{k=2}^{k_{\eps}} \min\left\{ \tsybca^{\prime} 2^{(3-k) \tsyba / (1-\tsyba)}, 1 \right\} \frac{2^{k}}{\eps}
& \leq \sum_{k=2}^{k_{(\tsybca,\tsyba)}-1} \frac{2^{k}}{\eps}
+ \sum_{k=k_{(\tsybca,\tsyba)}}^{k_{\eps}} \frac{8 \tsybca^{\prime}}{\eps} 2^{(k-3) \frac{1-2\tsyba}{1-\tsyba}}.
\\ & \leq \frac{16}{\eps} (\tsybca^{\prime})^{\frac{1-\tsyba}{\tsyba}} + \frac{8 (k_{\eps}-2)}{\eps} (\tsybca^{\prime})^{\frac{1-\tsyba}{\tsyba}} 
= \frac{8 k_{\eps}}{\eps} (\tsybca^{\prime})^{\frac{1-\tsyba}{\tsyba}}.
\end{align*}
Furthermore, since $(1-\tsyba)^{\frac{1-\tsyba}{\tsyba}} \leq 1$, we have 
\begin{equation*}
\frac{8 k_{\eps}}{\eps} (\tsybca^{\prime})^{\frac{1-\tsyba}{\tsyba}}
= \frac{8 k_{\eps}}{\eps} (1-\tsyba)^{\frac{1-\tsyba}{\tsyba}} (2\tsyba) \tsybca^{1/\tsyba}
\leq \frac{16 k_{\eps}}{\eps} \tsybca^{1/\tsyba}.
\end{equation*}
Therefore, in either case, when $\tsyba \geq 1/2$, \eqref{eqn:tsybakov-basic-initial} is at most
\begin{align*}
& \bar{c} \left( 16 k_{\eps} \tsybca^{1/\tsyba} \eps^{-1} + 128 \tsybca^{1/\tsyba} \eps^{-1}\right) \left( \vc \Log\left(\frac{1}{\eps} \right) + \Log\left(\frac{1}{\conf}\right)\right) \Log\left(\frac{\vc}{\eps\conf}\right) \Log\left(\frac{1}{\hat{\gamma}_{\eps}}\right)
\\ & \leq 
767 \bar{c} \frac{\tsybca^{1/\tsyba}}{\eps}\left( \vc \Log\left(\frac{1}{\eps} \right) + \Log\left(\frac{1}{\conf}\right)\right) \Log\left(\frac{\vc}{\eps\conf}\right) \Log^{2}\left(\frac{1}{\eps}\right),
\end{align*}
which is therefore an upper bound on $\LC_{\TN(\tsybca,\tsyba)}(\eps,\conf)$ in this case.

Otherwise, if $\tsyba \leq 1/2$, then $2^{k \frac{1-2\tsyba}{1-\tsyba}}$ is nondecreasing in $k$, so that
\begin{multline*}
\sum_{k=2}^{k_{\eps}} \min\left\{ \tsybca^{\prime} 2^{(3-k) \tsyba / (1-\tsyba)}, 1 \right\} \frac{2^{k}}{\eps}
\leq \sum_{k=2}^{k_{\eps}} 8 \tsybca^{\prime} 2^{(k-3) \frac{1-2\tsyba}{1-\tsyba}} \frac{1}{\eps}
\leq (k_{\eps}-1) 8 \tsybca^{\prime} 2^{(k_{\eps}-3) \frac{1-2\tsyba}{1-\tsyba}} \frac{1}{\eps}
\\ \leq (k_{\eps}-1) 8 \tsybca^{\prime} \left(\frac{2}{\hat{\gamma}_{\eps}}\right)^{\frac{1-2\tsyba}{1-\tsyba}} \frac{1}{\eps}
\leq (k_{\eps}-1) 8 \tsybca^{\prime} 2^{\frac{1-2\tsyba}{1-\tsyba}} \left( 2 \tsybca^{\prime} \right)^{1-2\tsyba} \left(\frac{1}{\eps}\right)^{2-2\tsyba}
\leq (k_{\eps}-1) 32 \left(\frac{\tsybca^{\prime}}{\eps}\right)^{2-2\tsyba}
\\ = (k_{\eps}-1) 32 (1-\tsyba)^{2-2\tsyba} (2\tsyba)^{2\tsyba} \tsybca^{2} \eps^{2\tsyba-2}
\leq (k_{\eps}-1) 32 \tsybca^{2} \eps^{2\tsyba-2}.
\end{multline*}
Therefore, \eqref{eqn:tsybakov-basic-initial} is at most
\begin{align*}
& \bar{c} \left( (k_{\eps}-1) 32 \tsybca^{2} \eps^{2\tsyba-2} + 128 \tsybca^{2} \eps^{2\tsyba-2}\right) \left( \vc \Log\left(\frac{1}{\eps} \right) + \Log\left(\frac{1}{\conf}\right)\right) \Log\left(\frac{\vc}{\eps\conf}\right) \Log\left(\frac{1}{\hat{\gamma}_{\eps}}\right)
\\ & \leq 
832 \bar{c} \tsybca^{2} \eps^{2\tsyba-2} \left( \vc \Log\left(\frac{1}{\eps} \right) + \Log\left(\frac{1}{\conf}\right)\right) \Log\left(\frac{\vc}{\eps\conf}\right) \Log^{2}\left(\frac{1}{\eps}\right).
\end{align*}
In particular, this implies $\LC_{\TN(\tsybca,\tsyba)}(\eps,\conf)$ is at most this large when $\tsyba \leq 1/2$.
Furthermore, this completes the proof of the upper bound for the cases where either $\tsyba \leq 1/2$, 
or $\tsyba \geq 1/2$ and $\frac{\s}{\vc} \geq \frac{1}{\tsybca^{1/\tsyba}\eps}$.

Next, consider the remaining case that $\tsyba \geq 1/2$ and $\frac{\s}{\vc} < \frac{1}{\tsybca^{1/\tsyba}\eps}$.
In particular, this requires that $\s < \infty$, and since $\s \geq \vc$, that $\eps < \tsybca^{-1/\tsyba}$.
In this case, let us take
\begin{equation*}
\bar{k} = 3+\left\lceil (1-\tsyba) \log_{2}\left( \frac{k_{\eps} \tsybca^{\prime}}{8 \eps} \frac{\vc\Log\left(\frac{1}{\eps}\right)+\Log\left(\frac{1}{\conf}\right)}{\s\Log\left(\frac{1}{\eps}\right)+\Log\left(\frac{1}{\conf}\right)} \right) \right\rceil.
\end{equation*}
Since $\s \geq \vc$, we have 
$\frac{\s \Log\left(\frac{1}{\eps}\right)+\Log\left(\frac{1}{\conf}\right)}{\vc\Log\left(\frac{1}{\eps}\right)+\Log\left(\frac{1}{\conf}\right)} \leq \frac{\s \Log\left(\frac{1}{\eps}\right)}{\vc \Log\left(\frac{1}{\eps}\right)} = \frac{\s}{\vc}$, 
so that, since $\frac{\s}{\vc} < \frac{1}{\tsybca^{1/\tsyba}\eps}$, 
we have $\frac{\s \Log\left(\frac{1}{\eps}\right)+\Log\left(\frac{1}{\conf}\right)}{\vc\Log\left(\frac{1}{\eps}\right)+\Log\left(\frac{1}{\conf}\right)}< \frac{1}{\tsybca^{1/\tsyba}\eps}$.
A bit of algebra reveals that, in this case, $\bar{k} \geq 2$.
Therefore, in this case, Lemma~\ref{lem:total-queries} implies that, with any budget $n$ of size at least
\begin{align}
& \bar{c} 2^{2\bar{k}} \left( \s \Log\left( \frac{1}{\eps} \right) + \Log\left(\frac{1}{\conf}\right) \right) \Log\left(\frac{\vc}{\eps\conf}\right) \Log\left( \frac{1}{\eps} \right) + \notag
\\ & \bar{c} \sum_{k=\bar{k}}^{k_{\eps}} \max\!\left\{\min\left\{ \tsybca^{\prime} 2^{(3-k) \tsyba / (1-\tsyba)}, 1 \right\}, \frac{\eps}{\hat{\gamma}_{\eps}}\right\} \frac{2^{k}}{\eps} \left( \vc \Log\!\left(\frac{1}{\eps} \right) + \Log\!\left(\frac{1}{\conf}\right)\right) \Log\!\left(\frac{\vc}{\eps\conf}\right) \Log\!\left(\frac{1}{\hat{\gamma}_{\eps}}\right), \label{eqn:tsybakov-star-initial}
\end{align}
\RQCAL~produces a classifier $\hat{h}_{n}$ with $\er_{\PXY}(\hat{h}_{n}) - \er_{\PXY}(\target_{\PXY}) \leq \eps$ with probability at least $1-\conf$,
and requests a number of labels at most $n$.  This implies $\LC_{\TN(\tsybca,\tsyba)}(\eps,\conf)$ is at most \eqref{eqn:tsybakov-star-initial}.

Now note that
\begin{align*}
& 2^{2\bar{k}} \left( \s \Log\left( \frac{1}{\eps} \right) + \Log\left(\frac{1}{\conf}\right) \right)
\\ & \leq 256 \left(\frac{k_{\eps} \tsybca^{\prime}}{8 \eps} \frac{\vc\Log\left(\frac{1}{\eps}\right)+\Log\left(\frac{1}{\conf}\right)}{\s\Log\left(\frac{1}{\eps}\right)+\Log\left(\frac{1}{\conf}\right)}\right)^{2-2\tsyba} \left( \s \Log\left( \frac{1}{\eps} \right) + \Log\left(\frac{1}{\conf}\right) \right)
\\ & \leq 1024 \tsybca^{2} \left(\frac{1}{\eps}\right)^{2-2\tsyba} \left(\frac{\s\Log\left(\frac{1}{\eps}\right)+\Log\left(\frac{1}{\conf}\right)}{\vc\Log\left(\frac{1}{\eps}\right)+\Log\left(\frac{1}{\conf}\right)}\right)^{2\tsyba-1} \left( \vc \Log\left( \frac{1}{\eps} \right) + \Log\left(\frac{1}{\conf}\right) \right) \Log^{2-2\tsyba}\left( \frac{1}{\eps} \right).
\end{align*}
Also, since $\tsyba \geq 1/2$, $2^{k \frac{1-2\tsyba}{1-\tsyba}}$ is nonincreasing in $k$, so that 
\begin{align*}
& \sum_{k=\bar{k}}^{k_{\eps}} \tsybca^{\prime} 2^{(3-k) \tsyba / (1-\tsyba)} \frac{2^{k}}{\eps} \left( \vc \Log\left(\frac{1}{\eps} \right) + \Log\left(\frac{1}{\conf}\right)\right)
\\ & \leq \frac{8 \tsybca^{\prime} k_{\eps}}{\eps} 2^{(\bar{k}-3) \frac{1-2\tsyba}{1-\tsyba}} \left( \vc \Log\left(\frac{1}{\eps} \right) + \Log\left(\frac{1}{\conf}\right)\right)
\\ & \leq \frac{8 \tsybca^{\prime} k_{\eps}}{\eps} \left( \frac{k_{\eps} \tsybca^{\prime}}{8 \eps} \frac{\vc\Log\left(\frac{1}{\eps}\right)+\Log\left(\frac{1}{\conf}\right)}{\s\Log\left(\frac{1}{\eps}\right)+\Log\left(\frac{1}{\conf}\right)} \right)^{1-2\tsyba} \left( \vc \Log\left(\frac{1}{\eps} \right) + \Log\left(\frac{1}{\conf}\right)\right)
\\ & \leq 256 \tsybca^{2}\left(\frac{1}{\eps}\right)^{2-2\tsyba} \left( \frac{\s\Log\left(\frac{1}{\eps}\right)+\Log\left(\frac{1}{\conf}\right)}{\vc\Log\left(\frac{1}{\eps}\right)+\Log\left(\frac{1}{\conf}\right)} \right)^{2\tsyba-1} \left( \vc \Log\left(\frac{1}{\eps} \right) + \Log\left(\frac{1}{\conf}\right)\right) \Log^{2-2\tsyba}\left(\frac{1}{\eps}\right).
\end{align*}
Furthermore, by \eqref{eqn:tsybakov-first-note},
\begin{align*}
& \bar{c} \sum_{k=\bar{k}}^{k_{\eps}} \frac{\eps}{\hat{\gamma}_{\eps}} \frac{2^{k}}{\eps} \left(\vc \Log\left(\frac{1}{\eps}\right) + \Log\left(\frac{1}{\conf}\right)\right)
\leq 128 \tsybca^{2} \left(\frac{1}{\eps}\right)^{2-2\tsyba} \left(\vc \Log\left(\frac{1}{\eps}\right) + \Log\left(\frac{1}{\conf}\right)\right)
\\ & \leq 128 \tsybca^{2} \left(\frac{1}{\eps}\right)^{2-2\tsyba} \left(\frac{\s\Log\left(\frac{1}{\eps}\right)+\Log\left(\frac{1}{\conf}\right)}{\vc\Log\left(\frac{1}{\eps}\right)+\Log\left(\frac{1}{\conf}\right)}\right)^{2\tsyba-1} \left(\vc\Log\left(\frac{1}{\eps}\right)+\Log\left(\frac{1}{\conf}\right)\right)\Log^{2-2\tsyba}\left(\frac{1}{\eps}\right).
\end{align*}
Therefore, since $\Log\left(\frac{1}{\hat{\gamma}_{\eps}}\right) \leq \Log\left(\frac{2}{\eps}\right) \leq 2\Log\left(\frac{1}{\eps}\right)$, 
\eqref{eqn:tsybakov-star-initial} is at most
\begin{equation}
\label{eqn:tsybakov-star-final-refined}
2^{11} \bar{c} \tsybca^{2} \left(\frac{1}{\eps}\right)^{2-2\tsyba}\!\!\left(\frac{\s\Log\!\left(\frac{1}{\eps}\right)+\Log\!\left(\frac{1}{\conf}\right)}{\vc\Log\!\left(\frac{1}{\eps}\right)+\Log\!\left(\frac{1}{\conf}\right)}\right)^{2\tsyba-1}\!\!\!\left(\!\vc\Log\!\left(\frac{1}{\eps}\right)\!+\!\Log\!\left(\frac{1}{\conf}\right)\!\right)\Log\!\left(\frac{\vc}{\eps\conf}\right)\Log^{3-2\tsyba}\!\left(\frac{1}{\eps}\right).
\end{equation}
The upper bound for the case $\tsyba \geq 1/2$ and $\frac{\s}{\vc} < \frac{1}{\tsybca^{1/\tsyba}\eps}$ then follows by further relaxing this (purely to simplify the theorem statement),
noting that $\Log^{3-2\tsyba}\left(\frac{1}{\eps}\right) \leq \Log^{2}\left(\frac{1}{\eps}\right)$, and
$\frac{\s\Log\left(\frac{1}{\eps}\right)+\Log\left(\frac{1}{\conf}\right)}{\vc\Log\left(\frac{1}{\eps}\right)+\Log\left(\frac{1}{\conf}\right)} \leq \frac{\s}{\vc}$.

Next, we turn to establishing the lower bound.
Fix any $\tsybca \in [4,\infty)$, $\tsyba \in (0,1)$,
$\conf \in (0,1/24]$,
and $\eps \in \left(0,1/(24 \tsybca^{1/\tsyba})\right)$.
For this range of values,
the recent article of \citet*{hanneke:survey} proves a lower bound of
\begin{equation*}
\LC_{\TN(\tsybca,\tsyba)}(\eps,\conf) \gtrsim \tsybca^{2} \left(\frac{1}{\eps}\right)^{2-2\tsyba} \left( \vc + \Log\left(\frac{1}{\conf}\right) \right),
\end{equation*}
based on techniques of \citet*{kaariainen:06,beygelzimer:09,hanneke:11a}.
It remains only to establish the remaining term in the lower bound for the case when $\tsyba > 1/2$, via Lemma~\ref{lem:rr11-star}.
In the cases that $\s \leq 2$, this term is implied by the above $\tsybca^{2} \eps^{2\tsyba-2} \Log\left(\frac{1}{\conf}\right)$ lower bound.
For the remainder of the proof, suppose $\s \geq 3$ and $\tsyba > 1/2$.
Let 
\begin{equation*}
k = \min\left\{\s-1,\left\lfloor \frac{\left(\tsybca^{\prime}\right)^{\frac{\tsyba-1}{\tsyba}}}{\eps} \right\rfloor, \left\lfloor \frac{\tsybca^{\prime}}{\eps} 4^{-\frac{1}{1-\tsyba}} \right\rfloor \right\},
\end{equation*}
$\bound = \frac{1}{2} - \left(\frac{k \eps}{\tsybca^{\prime}}\right)^{1-\tsyba}$,
and $\zeta = \frac{2\eps}{1-2\bound}$;
note that $\zeta \in (0,1]$, $\bound \in [0,1/2)$, and $2 \leq k \leq \min\!\left\{ \s\!-\!1, \left\lfloor 1/\zeta \right\rfloor \right\}$;
in particular, the fact that $k \leq \lfloor 1/\zeta \rfloor$ is established by concavity of the $x \mapsto \frac{\left(\tsybca^{\prime}\right)^{\tsyba-1}}{\eps^{\tsyba}} x^{1-\tsyba}$ function,
which equals $x$ at both $x=0$ and $x = x_{0} = \frac{\left(\tsybca^{\prime}\right)^{\frac{\tsyba-1}{\tsyba}}}{\eps}$; since this function is $1/\zeta$ at $x=k$,
and $0 < k \leq x_{0}$, concavity of the function implies $1/\zeta \geq k$, and integrality of $k$ implies $\lfloor 1/\zeta \rfloor \geq k$ as well.
Also note that any $\PXY \in \RR(k,\zeta,\bound)$ 
has a marginal distribution $\Px$ such that
\begin{align*}
\Px\left(x : \left|\eta(x;\PXY)-1/2\right| \leq 1/2-\bound\right) 
& = k \zeta 
= k\eps \frac{2}{1-2\bound}
\\ & = \tsybca^{\prime} \left(1/2 - \bound\right)^{\frac{1}{1-\tsyba}} \frac{2}{1-2\bound}
= \tsybca^{\prime} \left(1/2 - \bound\right)^{\frac{\tsyba}{1-\tsyba}}.
\end{align*}
Since every point $x$ in the support of $\Px_{k,\zeta}$ has either $|\eta(x;\PXY)-1/2| = 1/2-\bound$ or $|\eta(x;\PXY)-1/2|=1/2$,
this implies that any $\gamma \in [1/2-\bound,1/2)$ has
$\Px\left( x : \left| \eta(x;\PXY) - 1/2 \right| \leq \gamma \right) = \Px\left( x : \left| \eta(x;\PXY) - 1/2 \right| \leq 1/2-\bound \right) = \tsybca^{\prime} \left(1/2-\bound\right)^{\tsyba/(1-\tsyba)} \leq \tsybca^{\prime} \gamma^{\tsyba/(1-\tsyba)}$,
while any $\gamma \geq 1/2$ always has $\Px\left( x : \left| \eta(x;\PXY) - 1/2 \right| \leq \gamma \right) = 1 \leq \tsybca^{\prime} \gamma^{\tsyba/(1-\tsyba)}$. 
Furthermore, any $\gamma \in (0,1/2-\bound)$ has $\Px(x : |\eta(x;\PXY)-1/2| \leq \gamma) = 0 \leq \tsybca^{\prime} \gamma^{\tsyba/(1-\tsyba)}$.
Thus, $\PXY \in \TN(\tsybca,\tsyba)$ as well.  Since this holds for every $\PXY \in \RR(k,\zeta,\bound)$, this implies $\RR(k,\zeta,\bound) \subseteq \TN(\tsybca,\tsyba)$.
Therefore, Lemma~\ref{lem:rr11-star} implies 
\begin{align}
\LC_{\TN(\tsybca,\tsyba)}(\eps,\conf) 
\geq \LC_{\RR(k,\zeta,\bound)}(\eps,\conf) 
& = \LC_{\RR(k,\zeta,\bound)}((\zeta/2)(1-2\bound),\conf) \notag 
\\ & \geq \frac{\bound (k-1) \ln\left(\frac{1}{4\conf}\right)}{3(1-2\bound)^{2}}
\gtrsim \frac{\bound (k-1)}{(1-2\bound)^{2}} \Log\left(\frac{1}{\conf}\right). \label{eqn:tsybakov-pre-instantiation-lb}
\end{align}
Finally, note that
\begin{align}
& \frac{\bound (k-1)}{(1-2\bound)^{2}}
= \left( \frac{1}{2} - \left(\frac{k\eps}{\tsybca^{\prime}}\right)^{1-\tsyba}\right) \frac{1}{4} \left(\frac{\tsybca^{\prime}}{k \eps}\right)^{2-2\tsyba} (k-1)
\geq
\frac{1}{16} \left(\frac{\tsybca^{\prime}}{\eps}\right)^{2-2\tsyba} k^{2\tsyba-2} (k-1) \notag
\\ & \geq \frac{1}{32} \left(\frac{\tsybca^{\prime}}{\eps}\right)^{2-2\tsyba} (k-1)^{2\tsyba-1}
\geq \frac{\tsybca^{2}}{64} \left(\frac{1}{\eps}\right)^{2-2\tsyba} (k-1)^{2\tsyba-1}. \label{eqn:tsybakov-near-final}
\end{align}
Since $\tsybca \geq 4$,
\begin{align*}
\left(\tsybca^{\prime}\right)^{\frac{\tsyba-1}{\tsyba}} 
& = \tsybca^{\prime} \left(\tsybca^{\prime}\right)^{-1/\tsyba} 
= \tsybca^{\prime} (1-\tsyba)^{-1/\tsyba} (2\tsyba)^{-1/(1-\tsyba)} \tsybca^{- \frac{1}{\tsyba(1-\tsyba)}}
\\ & \leq \tsybca^{\prime} (1-\tsyba)^{-1/\tsyba} (2\tsyba)^{-1/(1-\tsyba)} 4^{-\frac{1}{\tsyba(1-\tsyba)}}
= \tsybca^{\prime} \left(4^{1/\tsyba} (1-\tsyba)^{(1-\tsyba)/\tsyba} (2\tsyba) \right)^{-1/(1-\tsyba)}.
\end{align*}
One can easily verify that 
$4^{1/\tsyba} (1-\tsyba)^{(1-\tsyba)/\tsyba} (2\tsyba) \geq 6$ for $\tsyba \in (1/2,1)$ (with minimum achieved at $\tsyba=3/4$),
so that 
$\tsybca^{\prime} \left(4^{1/\tsyba} (1-\tsyba)^{(1-\tsyba)/\tsyba} (2\tsyba) \right)^{-1/(1-\tsyba)}
\leq \tsybca^{\prime} 6^{-1/(1-\tsyba)} 
\leq \tsybca^{\prime} 4^{-1/(1-\tsyba)}$.
Thus, $\frac{\left(\tsybca^{\prime}\right)^{\frac{\tsyba-1}{\tsyba}}}{\eps} \leq \frac{\tsybca^{\prime}}{\eps} 4^{-\frac{1}{1-\tsyba}}$,
so that the third term in the definition of $k$ is redundant.
Therefore, \eqref{eqn:tsybakov-near-final} is at least
\begin{align*}
& \frac{\tsybca^{2}}{64} \left(\frac{1}{\eps}\right)^{2-2\tsyba} \min\left\{\s-2,\frac{\left(\tsybca^{\prime}\right)^{\frac{\tsyba-1}{\tsyba}}}{\eps} - 2 \right\}^{2\tsyba-1}
\geq \frac{\tsybca^{2}}{64} \left(\frac{1}{\eps}\right)^{2-2\tsyba} \min\left\{\s-2,\frac{1}{2 \tsybca^{1/\tsyba} \eps} - 2 \right\}^{2\tsyba-1}
\\ & \geq \frac{\tsybca^{2}}{64} \left(\frac{1}{\eps}\right)^{2-2\tsyba} \min\left\{\frac{\s}{3},\frac{1}{3 \tsybca^{1/\tsyba} \eps} \right\}^{2\tsyba-1}
\geq \frac{\tsybca^{2}}{192} \left(\frac{1}{\eps}\right)^{2-2\tsyba} \min\left\{\s,\frac{1}{\tsybca^{1/\tsyba} \eps} \right\}^{2\tsyba-1}.
\end{align*}
Plugging this into \eqref{eqn:tsybakov-pre-instantiation-lb} completes the proof.
\end{proof}

As an aside, we note that it is possible to improve the logarithmic factors in the upper bound in Theorem~\ref{thm:tsybakov}.
One clear refinement comes from using \eqref{eqn:tsybakov-star-final-refined} directly (rather than relaxing the factor depending on $\s$).
We can further reduce the bound by another logarithmic factor when $\tsyba$ is bounded away from $1/2$
by noting that the summations of terms $2^{(k-3)\frac{1-2\tsyba}{1-\tsyba}}$ in the above proof are geometric in that case.
We also note that, for very large values of $\tsybca$, the bounds (proven below) for 
$\LC_{\BE(1/2)}(\eps,\conf)$ may be more informative than those derived above.

\begin{proof}[of Theorem~\ref{thm:diameter-localization}]
The technique leading to Lemma~\ref{lem:total-queries} does not apply to $\DL(\tsybca,\tsyba)$,
since we are not guaranteed $\target_{\PXY} \in \C$ for $\PXY \in \DL(\tsybca,\tsyba)$.
We therefore base the upper bounds in Theorem~\ref{thm:diameter-localization} directly on 
existing results in the literature, in combination with Theorem~\ref{thm:dc-star}.  Thus,
the proof of this upper bound does not provide any new insights on improving the design of active learning algorithms
for distributions in $\DL(\tsybca,\tsyba)$.  Rather, it merely re-expresses the known results, 
in terms of the star number instead of a distribution-dependent complexity measure.
The lower bounds are directly inherited from Theorem~\ref{thm:tsybakov}.

Fix any $\tsybca \in [1,\infty)$, $\tsyba \in [0,1]$, and $\eps,\conf \in (0,1)$. 
Following the work of \citet*{hanneke:09a,hanneke:11a} and \citet*{koltchinskii:10}, 
the recent work of \citet*{hanneke:12b} studies an algorithm proposed by \citet*{hanneke:12a} 
(a modified variant of the $A^2$ algorithm of \citealp*{balcan:06,balcan:09}),
and shows that there exists a finite universal constant $\mathring{c} \geq 1$ such that,
for any $\PXY \in \DL(\tsybca,\tsyba)$, for any budget $n$ of size at least
\begin{equation}
\label{eqn:dl-dc-bound}
\mathring{c} \tsybca^{2} \left(\frac{1}{\eps}\right)^{2-2\tsyba} \dc_{\PXY}\left( \tsybca \eps^{\tsyba} \right) \left( \vc \Log\left( \dc_{\PXY}\left( \tsybca \eps^{\tsyba} \right) \right) + \Log\left(\frac{\Log(1/\eps)}{\conf}\right) \right) \Log\left(\frac{1}{\eps}\right),
\end{equation}
the algorithm produces a classifier $\hat{h}_{n}$ with $\er_{\PXY}(\hat{h}_{n}) - \inf_{h \in \C} \er_{\PXY}(h) \leq \eps$
with probability at least $1-\conf/4$, and requests a number of labels at most $n$
(see also \citealp*{hanneke:thesis,hanneke:09a,hanneke:11a,hanneke:12a,hanneke:survey,koltchinskii:10}, for similar results for related methods).
By Theorem~\ref{thm:dc-star}, when $\tsybca \eps^{\tsyba} \leq 1$, \eqref{eqn:dl-dc-bound} is at most
\begin{equation}
\label{eqn:dl-star-bound}
\mathring{c} \tsybca^{2} \left(\frac{1}{\eps}\right)^{2-2\tsyba} \min\left\{ \s, \frac{1}{\tsybca \eps^{\tsyba}} \right\} \left( \vc \Log\left( \min\left\{ \s, \frac{1}{\tsybca \eps^{\tsyba}} \right\} \right) + \Log\left( \frac{\Log(1/\eps)}{\conf}\right)\right) \Log\left(\frac{1}{\eps}\right),
\end{equation}
which is therefore an upper bound on $\LC_{\DL(\tsybca,\tsyba)}(\eps,\conf)$.
We can also extend this to the case $\tsybca \eps^{\tsyba} > 1$ as follows.
\citet*{vapnik:71,vapnik:82,vapnik:98} have proven that the sample complexity of passive learning satisfies
\begin{equation*}
\SC_{\AG(1)}(\eps,\conf) 
\lesssim \frac{1}{\eps^{2}} \left( \vc \Log\left(\frac{1}{\eps}\right) + \Log\left(\frac{1}{\conf}\right) \right).
\end{equation*}
In the case $\tsybca \eps^{\tsyba} > 1$, this is at most
\begin{align*}
& \tsybca \left(\frac{1}{\eps}\right)^{2-\tsyba} \left( \vc \Log\left(\frac{1}{\eps}\right) + \Log\left(\frac{1}{\conf}\right) \right)
\\ & = \tsybca^{2} \left( \frac{1}{\eps} \right)^{2-2\tsyba} \min\left\{ \s, \frac{1}{\tsybca \eps^{\tsyba}} \right\} \left( \vc \Log\left(\frac{1}{\eps}\right) + \Log\left(\frac{1}{\conf}\right) \right)
\\ & \leq \tsybca^{2} \left( \frac{1}{\eps} \right)^{2-2\tsyba} \min\left\{ \s, \frac{1}{\tsybca \eps^{\tsyba}} \right\} \left( \vc \Log\left(\min\left\{\s, \frac{1}{\tsybca\eps^{\tsyba}}\right\}\right) + \Log\left(\frac{\Log(1/\eps) }{\conf}\right) \right) \Log\left(\frac{1}{\eps}\right).
\end{align*}
Therefore, since $\LC_{\AG(1)}(\eps,\conf) \leq \SC_{\AG(1)}(\eps,\conf)$ and $\DL(\tsybca,\tsyba) \subseteq \AG(1)$, 
we may conclude that, regardless of whether $\tsybca \eps^{\tsyba}$ is greater than or less than $1$, 
we have that $\LC_{\DL(\tsybca,\tsyba)}(\eps,\conf)$ is bounded by a value proportional to \eqref{eqn:dl-star-bound}.
To match the form of the upper bound stated in Theorem~\ref{thm:diameter-localization}, we can simply relax this, noting that 
$\vc \Log\left(\min\left\{\s, \frac{1}{\tsybca \eps^{\tsyba}}\right\}\right) + \Log\left(\frac{\Log(1/\eps)}{\conf}\right)
\leq 2\vc \Log\left( \frac{1}{\eps} \right) + \Log\left(\frac{1}{\conf}\right)
\leq 2\vc\Log\left(\frac{1}{\eps\conf}\right)$.

Next, turning to the lower bound,
recall that $\TN(\tsybca,\tsyba) \subseteq \DL(\tsybca,\tsyba)$,
so that $\LC_{\TN(\tsybca,\tsyba)}(\eps,\conf) \leq \LC_{\DL(\tsybca,\tsyba)}(\eps,\conf)$ \citep*{mammen:99,tsybakov:04}.
Thus, the lower bound in Theorem~\ref{thm:tsybakov} (proven above) for $\LC_{\TN(\tsybca,\tsyba)}(\eps,\conf)$
also applies to $\LC_{\DL(\tsybca,\tsyba)}(\eps,\conf)$.
\end{proof}

\begin{proof}[of Theorem~\ref{thm:benign}]
Again, we begin with the upper bound.
Fix any $\nu \in [0,1/2]$, $\eps,\conf \in (0,1)$, and $\PXY \in \BE(\nu)$.
The case $\nu = 0$ is already addressed by the upper bound in Theorem~\ref{thm:realizable};
we therefore focus the remainder of the proof on the case of $\nu > 0$.
For $(X,Y) \sim \PXY$, any $x \in \X$ has $1 - 2\P( Y \neq \target_{\PXY}(X) | X = x ) = 2\gamma_{x}$.
Therefore, for any $\gamma \in [0,1/2)$, any $x \in \X$ with $\gamma_{x} \leq \gamma$ has $\P( Y \neq \target_{\PXY}(X) | X = x ) \geq 1/2 - \gamma$.
Thus, Markov's inequality implies 
\begin{equation}
\label{eqn:benign-margin-condition}
\Px(x : \gamma_{x} \leq \gamma) 
\leq \Px(x : \P( Y \neq \target_{\PXY}(X) | X = x ) \geq 1/2 - \gamma ) 
\leq \frac{2}{1-2\gamma} \er_{\PXY}(\target_{\PXY}) 
\leq \frac{2\nu}{1-2\gamma}.
\end{equation}
In particular, this implies that for $\gamma \leq \frac{\eps}{4\nu+2\eps}$,
$\gamma \Px(x : \gamma_{x} \leq \gamma) \leq \frac{2\nu\gamma}{1-2\gamma} \leq \frac{2\nu / (2\nu+\eps)}{1-\eps/(2\nu+\eps)} \frac{\eps}{2} = \frac{\eps}{2}$.
Thus, $\gamma_{\eps} \geq \frac{\eps}{4\nu+2\eps}$.
We can therefore take $\hat{\gamma}_{\eps} = \max\left\{ \frac{\eps}{4\nu+2\eps}, \frac{\eps}{2} \right\}$.

Also note that any $\gamma \geq 0$ has $\Px(x : \gamma_{x} \leq \gamma) \leq 1$, so that together with \eqref{eqn:benign-margin-condition}, we have
$\Px(x : \gamma_{x} \leq \gamma) \leq \frac{2\nu}{1-\min\{2\gamma,1-2\nu\}}$.
Now taking $\bar{k} = 2$, Lemma~\ref{lem:total-queries} implies that, with any budget $n$ of size at least
\begin{equation}
\label{eqn:benign-basic-initial}
\bar{c} \sum_{k=2}^{k_{\eps}} \max\left\{ \frac{2\nu}{1-\min\left\{2^{4-k},1-2\nu\right\}}, \frac{\eps}{\hat{\gamma}_{\eps}} \right\} \frac{2^{k}}{\eps} \left(\vc \Log\left(\frac{1}{\eps}\right) + \Log\left(\frac{1}{\conf}\right)\right)\Log\left(\frac{\vc}{\eps\conf}\right)\Log\left(\frac{1}{\hat{\gamma}_{\eps}}\right)\!,
\end{equation}
\RQCAL~produces a classifier $\hat{h}_{n}$ with $\er_{\PXY}(\hat{h}_{n}) - \er_{\PXY}(\target_{\PXY}) \leq \eps$ with probability at least $1-\conf$,
and requests a number of labels at most $n$.  This implies $\LC_{\BE(\nu)}(\eps,\conf)$ is at most \eqref{eqn:benign-basic-initial}.
Now note that
\begin{equation}
\label{eqn:benign-basic-eps-gamma}
\sum_{k=2}^{k_{\eps}} \frac{\eps}{\hat{\gamma}_{\eps}} \frac{2^{k}}{\eps}
\leq \frac{1}{\hat{\gamma}_{\eps}} 2^{1+k_{\eps}} 
\leq 512 \left(\frac{\nu+\eps}{\eps}\right)^{2}.
\end{equation}
Next, we have
\begin{align*}
& \sum_{k=2}^{k_{\eps}} \frac{2\nu}{1-\min\left\{2^{4-k},1-2\nu\right\}} \frac{2^{k}}{\eps}
\leq \frac{28}{\eps} + \sum_{k=5}^{k_{\eps}} \frac{2\nu}{1-2^{4-k}} \frac{2^{k}}{\eps}
\leq \frac{28}{\eps} + \sum_{k=5}^{k_{\eps}} \frac{4\nu}{\eps} 2^{k}
\\ & \leq \frac{28}{\eps} + \frac{4\nu}{\eps} 2^{1+k_{\eps}}
\leq \frac{28}{\eps} + \frac{128\nu}{\eps \hat{\gamma}_{\eps}}
\leq \frac{28}{\eps} + 512 \left(\frac{\nu + \eps}{\eps}\right)^{2}.
\end{align*}
Therefore, \eqref{eqn:benign-basic-initial} is at most
\begin{multline}
\label{eqn:benign-basic-final}
2^{10} \bar{c} \left( \left(\frac{\nu+\eps}{\eps}\right)^{2} + \frac{1}{\eps} \right) \left(\vc \Log\left(\frac{1}{\eps}\right) + \Log\left(\frac{1}{\conf}\right)\right)\Log\left(\frac{\vc}{\eps\conf}\right)\Log\left(\frac{1}{\hat{\gamma}_{\eps}}\right)
\\ \leq 2^{10} 3 \bar{c} \left( \left(\frac{\nu+\eps}{\eps}\right)^{2} + \frac{1}{\eps} \right) \left(\vc \Log\left(\frac{1}{\eps}\right) + \Log\left(\frac{1}{\conf}\right)\right)\Log\left(\frac{\vc}{\eps\conf}\right)\Log\left(\frac{\nu+\eps}{\eps}\right).
\end{multline}

Next, consider taking $\bar{k} = 5$. Lemma~\ref{lem:total-queries} implies that, with any budget $n$ of size at least
\begin{multline}
\label{eqn:benign-star-initial}
2^{10} \bar{c} \left( \s \Log\left(\frac{1}{\eps}\right) + \Log\left(\frac{1}{\conf}\right)\right) \Log\left(\frac{\vc}{\eps\conf}\right)\Log\left(\frac{1}{\eps}\right)
\\ + \bar{c} \sum_{k=5}^{k_{\eps}} \max\left\{ \frac{2\nu}{1-2^{4-k}}, \frac{\eps}{\hat{\gamma}_{\eps}} \right\} \frac{2^{k}}{\eps} \left(\vc \Log\left(\frac{1}{\eps}\right) + \Log\left(\frac{1}{\conf}\right)\right)\Log\left(\frac{\vc}{\eps\conf}\right)\Log\left(\frac{1}{\hat{\gamma}_{\eps}}\right),
\end{multline}
\RQCAL~produces a classifier $\hat{h}_{n}$ with $\er_{\PXY}(\hat{h}_{n}) - \er_{\PXY}(\target_{\PXY}) \leq \eps$ with probability at least $1-\conf$,
and requests a number of labels at most $n$.  This implies $\LC_{\BE(\nu)}(\eps,\conf)$ is at most \eqref{eqn:benign-star-initial}.
As above, we have
\begin{equation*}
\sum_{k=5}^{k_{\eps}} \frac{2\nu}{1-2^{4-k}} \frac{2^{k}}{\eps}
\leq 512 \left(\frac{\nu+\eps}{\eps}\right)^{2}.
\end{equation*}
Combined with \eqref{eqn:benign-basic-eps-gamma}, this implies \eqref{eqn:benign-star-initial} is at most
\begin{align}
& 2^{10} \bar{c} \left( \s \Log\left(\frac{1}{\eps}\right) + \Log\left(\frac{1}{\conf}\right)\right) \Log\left(\frac{\vc}{\eps\conf}\right)\Log\left(\frac{1}{\eps}\right) \notag
\\ & + 2^{10} \bar{c} \left(\frac{\nu+\eps}{\eps}\right)^{2} \left(\vc \Log\left(\frac{1}{\eps}\right) + \Log\left(\frac{1}{\conf}\right)\right)\Log\left(\frac{\vc}{\eps\conf}\right)\Log\left(\frac{1}{\hat{\gamma}_{\eps}}\right) \notag
\\ & \leq 2^{10} \bar{c} \left( \s \Log\left(\frac{1}{\eps}\right) + \Log\left(\frac{1}{\conf}\right)\right) \Log\left(\frac{\vc}{\eps\conf}\right)\Log\left(\frac{1}{\eps}\right) \notag
\\ & + 2^{10} 3 \bar{c} \left(\frac{\nu+\eps}{\eps}\right)^{2} \left(\vc \Log\left(\frac{1}{\eps}\right) + \Log\left(\frac{1}{\conf}\right)\right)\Log\left(\frac{\vc}{\eps\conf}\right)\Log\left(\frac{\nu+\eps}{\eps}\right). \label{eqn:benign-star-final}
\end{align}
In particular, when $\left(\s \Log\left(\frac{1}{\eps}\right)+\Log\left(\frac{1}{\conf}\right)\right)\Log\left(\frac{1}{\eps}\right) < \frac{3}{\eps} \left(\vc \Log\left(\frac{1}{\eps}\right)+\Log\left(\frac{1}{\conf}\right)\right)\Log\left(\frac{\nu+\eps}{\eps}\right)$,
this is smaller than \eqref{eqn:benign-basic-final}.  Thus, the minimum of these two expressions upper bounds $\LC_{\BE(\nu)}(\eps,\conf)$.

To simplify the expression of this bound into the form given in the statement of Theorem~\ref{thm:benign},
we note that $\vc \Log\left(\frac{1}{\eps}\right) + \Log\left(\frac{1}{\conf}\right) \leq \vc \Log\left(\frac{1}{\eps\conf}\right)$, $\s\Log\left(\frac{1}{\eps}\right)+\Log\left(\frac{1}{\conf}\right) \leq \s\Log\left(\frac{1}{\eps\conf}\right)$,
$\Log\left(\frac{\nu+\eps}{\eps}\right) \leq \Log\left(\frac{1}{\eps}\right)$,$\left(\frac{\nu+\eps}{\eps}\right)^{2} \leq 4 \frac{\max\{\nu,\eps\}^{2}}{\eps^{2}} \leq 4 \left( \frac{\nu^{2}}{\eps^{2}} + 1 \right)$,
and $\vc \leq \min\left\{ \s, \frac{\vc}{\eps} \right\}$,
so that the minimum of \eqref{eqn:benign-basic-final} and \eqref{eqn:benign-star-final} is at most
\begin{multline*}
2^{12} 3 \bar{c} \left(  \left(\frac{\nu^{2}}{\eps^{2}}+1\right) \vc + \min\left\{\s,\frac{\vc}{\eps}\right\} \right)\Log\left(\frac{\vc}{\eps\conf}\right)\Log\left(\frac{1}{\eps\conf}\right)\Log\left(\frac{1}{\eps}\right)
\\ \leq 2^{13} 3 \bar{c} \left(  \frac{\nu^{2}}{\eps^{2}} \vc + \min\left\{\s,\frac{\vc}{\eps}\right\} \right)\Log\left(\frac{\vc}{\eps\conf}\right)\Log\left(\frac{1}{\eps\conf}\right)\Log\left(\frac{1}{\eps}\right).
\end{multline*}
This completes the proof of the upper bound.

Next, we turn to establishing the lower bound. 
Fix $\nu \in [0,1/2)$, $\eps \in (0,(1-2\nu)/24)$, and $\conf \in (0,1/24]$.
Based on the works of \citet*{kaariainen:06,hanneke:07a,beygelzimer:09},
the recent article of \citet*{hanneke:survey} contains the following lower bound (in the proof of Theorem 4.3 there),
letting $\gamma = \frac{12\eps}{\nu+12\eps}$.
\begin{align}
\LC_{\BE(\nu)}(\eps,\conf) 
& \geq \max\left\{2 \left\lfloor \frac{1-\gamma^{2}}{2\gamma^{2}} \ln\left(\frac{1}{8\conf(1-2\conf)}\right) \right\rfloor, \frac{\vc-1}{6} \left\lfloor \frac{1-\gamma^{2}}{2\gamma^{2}} \ln\left(\frac{9}{8}\right) \right\rfloor \right\} \notag
\\ & \geq \max\left\{2 \left\lfloor \frac{1-\gamma^{2}}{2\gamma^{2}} \ln\left(\frac{1}{8\conf}\right) \right\rfloor, \frac{\vc-1}{6} \left\lfloor \frac{1-\gamma^{2}}{17 \gamma^{2}} \right\rfloor\right\}
\label{eqn:benign-lower-basic}
\end{align}
If $\frac{1-\gamma^{2}}{2\gamma^{2}} \ln\left(\frac{1}{8\conf}\right) \geq 1$, then 
$2 \left\lfloor \frac{1-\gamma^{2}}{2\gamma^{2}} \ln\left(\frac{1}{8\conf}\right) \right\rfloor \geq \frac{1-\gamma^{2}}{2\gamma^{2}} \ln\left(\frac{1}{8\conf}\right)$,
so that \eqref{eqn:benign-lower-basic} implies $\LC_{\BE(\nu)}(\eps,\conf) \gtrsim \frac{1-\gamma^{2}}{\gamma^{2}} \Log\left(\frac{1}{\conf}\right)$.
Otherwise, if $\frac{1-\gamma^{2}}{2\gamma^{2}} \ln\left(\frac{1}{8\conf}\right) < 1$, then since $\RE \subseteq \BE(\nu)$, 
and $|\C| \geq 2$ implies $\vc \geq 1 > \frac{1-\gamma^{2}}{2\gamma^{2}} \ln\left(\frac{1}{8\conf}\right)$,
Theorem~\ref{thm:realizable} (proven above) implies $\LC_{\BE(\nu)}(\eps,\conf) \geq \LC_{\RE}(\eps,\conf) \gtrsim \vc \gtrsim \frac{1-\gamma^{2}}{\gamma^{2}} \Log\left(\frac{1}{\conf}\right)$ in this case as well.
If $\vc = 1$, these observations further imply $\LC_{\BE(\nu)}(\eps,\conf) \gtrsim \vc \frac{1-\gamma^{2}}{\gamma^{2}}$.
On the other hand, if $\vc \geq 2$, and if $\frac{1-\gamma^{2}}{17 \gamma^{2}} \geq 1$, then 
$\frac{\vc-1}{6} \left\lfloor \frac{1-\gamma^{2}}{17 \gamma^{2}} \right\rfloor \geq \frac{\vc}{408} \frac{1-\gamma^{2}}{\gamma^{2}}$,
so that \eqref{eqn:benign-lower-basic} implies $\LC_{\BE(\nu)}(\eps,\conf) \gtrsim \vc \frac{1-\gamma^{2}}{\gamma^{2}}$.
Otherwise, if $\frac{1-\gamma^{2}}{17 \gamma^{2}} < 1$, then since $\RE \subseteq \BE(\nu)$, Theorem~\ref{thm:realizable} implies 
we still have $\LC_{\BE(\nu)}(\eps,\conf) \geq \LC_{\RE}(\eps,\conf) \gtrsim \vc \gtrsim \vc \frac{1-\gamma^{2}}{\gamma^{2}}$ in this case as well.
Altogether, we have that 
\begin{equation}
\label{eqn:benign-lower-simple}
\LC_{\BE(\nu)}(\eps,\conf) 
\gtrsim \frac{1-\gamma^{2}}{\gamma^{2}} \max\left\{ \vc, \Log\left(\frac{1}{\conf}\right) \right\} 
\gtrsim \frac{1-\gamma^{2}}{\gamma^{2}} \left( \vc + \Log\left(\frac{1}{\conf}\right) \right).
\end{equation}
When $\nu \geq 12\eps$, $\gamma \leq 1/2$, so that \eqref{eqn:benign-lower-simple} implies
\begin{equation*}
\LC_{\BE(\nu)}(\eps,\conf) 
\gtrsim \frac{1}{\gamma^{2}} \left( \vc + \Log\left( \frac{1}{\conf} \right) \right)
= \left( \frac{\nu+12\eps}{12\eps} \right)^{2} \left( \vc + \Log\left(\frac{1}{\conf}\right) \right)
\gtrsim \frac{\nu^{2}}{\eps^{2}} \left( \vc + \Log\left(\frac{1}{\conf}\right) \right).
\end{equation*}
Otherwise, if $\nu < 12\eps$, then 
\begin{equation}
\label{eqn:benign-small-nu-lower}
\frac{1-\gamma^{2}}{\gamma^{2}}
= \frac{(1-\gamma)(1+\gamma)}{\gamma^{2}}
= \left(\frac{\nu+12\eps}{12\eps}\right)^{2} \left(\frac{\nu}{\nu+12\eps}\right)\left(\frac{\nu+24\eps}{\nu+12\eps}\right)
\geq \frac{\nu}{\nu+12\eps}
\geq \frac{\nu}{12\eps}
\geq \frac{\nu^{2}}{144 \eps^{2}}.
\end{equation}
Therefore, if $\nu < 12\eps$, \eqref{eqn:benign-lower-simple} implies that
$\LC_{\BE(\nu)}(\eps,\conf) \gtrsim \frac{\nu^{2}}{\eps^{2}} \left( \vc + \Log\left( \frac{1}{\conf} \right) \right)$
in this case as well.
It remains only to establish the final term in the lower bound.
For this, we simply note that $\RE \subseteq \BE(\nu)$, so that Theorem~\ref{thm:realizable}
implies $\LC_{\BE(\nu)}(\eps,\conf) \geq \LC_{\RE}(\eps,\conf) \gtrsim \min\left\{\s,\frac{1}{\eps}\right\}$.
Combining these results implies
\begin{equation*}
\LC_{\BE(\nu)}(\eps,\conf) \gtrsim \max\!\left\{ \frac{\nu^{2}}{\eps^{2}} \left( \vc + \Log\!\left( \frac{1}{\conf} \right) \right), \min\!\left\{ \s, \frac{1}{\eps} \right\} \right\}
\gtrsim \frac{\nu^{2}}{\eps^{2}} \left( \vc + \Log\!\left( \frac{1}{\conf} \right) \right) + \min\!\left\{ \s, \frac{1}{\eps} \right\}.
\end{equation*}
\end{proof}

Examining the proof of the lower bound for $\LC_{\BE(\nu)}(\eps,\conf)$, we note that this argument also establishes a slightly stronger lower bound in the case $\eps > \nu$.
Specifically, if we use the expression just left of the right-most inequality in \eqref{eqn:benign-small-nu-lower}, rather than the right-most expression, 
we find that we can add a term $\frac{\nu}{\eps} \Log\left(\frac{1}{\conf}\right)$ to the stated lower bound.  This term can be larger than 
the stated term $\frac{\nu^{2}}{\eps^{2}} \Log\left(\frac{1}{\conf}\right)$ when $\eps > \nu$.  Additionally, since $\RE \subseteq \BE(\nu)$, 
we can of course also add a term $\vc$ to the stated lower bound, which again would increase the bound when $\eps > \nu$.

\begin{proof}[of Theorem~\ref{thm:agnostic}]
Again, we begin with the upper bounds.  As with the proof of Theorem~\ref{thm:diameter-localization},
we cannot use the technique leading to Lemma~\ref{lem:total-queries}; we turn instead to a simple
combination of an upper bound from the literature, combined with Theorem~\ref{thm:dc-star}.

Fix any $\nu \in [0,1]$ and $\eps,\conf \in (0,1)$.
Following the work of \citet*{hanneke:07b,dasgupta:07,koltchinskii:10},
the recent work of \citet*{hanneke:survey} studies a modified variant of the $A^2$ algorithm of \citet*{balcan:06,balcan:09},
showing that there exists a finite universal constant $\ddot{c} \geq 1$ such that,
for any $\PXY \in \AG(\nu)$, for any budget $n$ of size at least
\begin{equation}
\label{eqn:ag-dc-bound}
\ddot{c} \dc_{\PXY}\left( \nu + \eps \right) \left( \frac{\nu^{2}}{\eps^{2}} + \Log\left(\frac{1}{\eps}\right) \right) \left( \vc \Log\left( \dc_{\PXY}\left( \nu + \eps \right) \right) + \Log\left(\frac{\Log(1/\eps)}{\conf}\right)\right),
\end{equation}
the algorithm produces a classifier $\hat{h}_{n}$ with $\er_{\PXY}(\hat{h}_{n}) - \inf_{h \in \C} \er_{\PXY}(h) \leq \eps$
with probability at least $1-\conf$, and requests a number of labels at most $n$
(see also \citealp*{dasgupta:07,beygelzimer:09}, for similar results for related methods).
By Theorem~\ref{thm:dc-star}, 
\begin{equation*}
\dc_{\PXY}(\nu+\eps) = \dc_{\PXY}( (\nu+\eps) \land 1) \leq \min\!\left\{ \s, \frac{1}{(\nu+\eps) \!\land\! 1} \right\} \leq \min\!\left\{ \s, \frac{2}{\nu+\eps} \right\} \leq 2 \min\!\left\{ \s, \frac{1}{\nu+\eps} \right\}\!,
\end{equation*}
while $\Log\left( \dc_{\PXY}( \nu+\eps ) \right) \leq \Log\left( \min\left\{ \s, \frac{1}{\nu+\eps} \right\} \lor 1 \right) = \Log\left( \min\left\{ \s, \frac{1}{\nu+\eps} \right\} \right)$.
Therefore, \eqref{eqn:ag-dc-bound} is at most
\begin{equation*}
2 \ddot{c} \min\left\{ \s, \frac{1}{\nu+\eps} \right\} \left( \frac{\nu^{2}}{\eps^{2}} + \Log\left(\frac{1}{\eps}\right) \right) \left( \vc \Log\left( \min\left\{ \s, \frac{1}{\nu+\eps} \right\} \right) + \Log\left(\frac{\Log(1/\eps)}{\conf}\right)\right),
\end{equation*}
which is therefore an upper bound on $\LC_{\AG(\nu)}(\eps,\conf)$.
To match the form of the upper bound stated in Theorem~\ref{thm:agnostic}, we can relax this by noting that
$\vc \Log\left( \min\left\{ \s, \frac{1}{\nu+\eps} \right\} \right) + \Log\left(\frac{\Log(1/\eps)}{\conf}\right)
\leq 2 \vc \Log\left( \frac{1}{\eps} \right)$ $+ \Log\left(\frac{1}{\conf}\right)
\leq 2 \vc \Log\left(\frac{1}{\eps\conf}\right)$,
while $\frac{\nu^{2}}{\eps^{2}} + \Log\left(\frac{1}{\eps}\right) \leq \left(\frac{\nu^{2}}{\eps^{2}} + 1\right) \Log\left(\frac{1}{\eps}\right)$.

To prove the lower bound in Theorem~\ref{thm:agnostic}, we note that $\BE(\nu) \subseteq \AG(\nu)$ for $\nu \in [0,1/2)$,
so that $\LC_{\BE(\nu)}(\eps,\conf) \leq \LC_{\AG(\nu)}(\eps,\conf)$.  Thus, the lower bound on $\LC_{\BE(\nu)}(\eps,\conf)$ 
in Theorem~\ref{thm:benign} (proven above) also applies to $\LC_{\AG(\nu)}(\eps,\conf)$.
\end{proof}

\section{Proofs for Results in Section~\ref{sec:complexities}}
\label{app:complexities}

This section provides proofs of the equivalences between complexity measures stated in Section~\ref{sec:complexities}.

\subsection{The Disagreement Coefficient}
\label{app:dc}

Here we present the proof of Theorem~\ref{thm:dc-star}.
First, we have a helpful lemma, which allows us to restrict focus to \emph{finitely discrete} probability measures.
Let $\Pi$ denote the set of probability measures $\Px$ on $\X$ such that 
$\exists m \in \nats$ and a sequence $\{z_{i}\}_{i=1}^{m}$ in $\X$
for which $\Px(\{z_{i} : i \in \{1,\ldots,m\}\}) = 1$.

\begin{lemma}
\label{lem:dc-star-finitely-discrete}
If $\s < \infty$, then $\forall \eps \in (0,1]$, $\supdc(\eps) = \sup_{\Px \in \Pi} \sup_{h \in \C} \dc_{h,\Px}(\eps)$.
\end{lemma}
\begin{proof}
Suppose $\s < \infty$, and fix any $\eps \in (0,1]$.
Since $\PXY$ ranges over all probability measures over $\X \times \Y$ in the definition of $\supdc(\eps)$, including all those in $\RE$ with marginal $\Px$ over $\X$ contained in $\Pi$
(in which case, $\dc_{\PXY}(\eps) = \dc_{\target_{\PXY},\Px}(\eps)$),
we always have $\sup_{P \in \Pi} \sup_{h \in \C} \dc_{h,P}(\eps) \leq \supdc(\eps)$.  
Thus, it suffices to show that we also have $\sup_{\Px \in \Pi} \sup_{h \in \C} \dc_{h,\Px}(\eps) \geq \supdc(\eps)$.

The result trivially holds if $\supdc(\eps) = 1$, since \emph{every} $\Px$ and $h$ have $\dc_{h,\Px}(\eps) \geq 1$.
To address the nontrivial case, suppose $\supdc(\eps) > 1$.
Fix any $\gamma_{1}, \gamma_{2},\gamma_{3} \in (0,1)$.
Fix any $\PXY$ with $\dc_{\PXY}(\eps) > 1$, and as usual denote $\Px(\cdot) = \PXY(\cdot \times \Y)$.
Also let $h^{*}_{\PXY}$ be as in Definition~\ref{def:dc}, so that $\dc_{\PXY}(\eps) = \dc_{h^{*}_{\PXY},\Px}(\eps)$.
Let $r_{\eps} \in (\eps,1]$ be such that 
$\frac{\Px(\DIS(\Ball_{\Px}(h_{\PXY}^{*},r_{\eps})))}{r_{\eps}} \geq (1-\gamma_{1}) \dc_{\PXY}(\eps)$
(which exists, by the definition of the supremum, combined with the fact that $1 < \dc_{\PXY}(\eps) \leq 1/\eps < \infty$).
Also let $h \in \C$ have $\Px(x : h(x) \neq h^{*}_{\PXY}(x)) \leq \gamma_{3} r_{\eps}$, which exists by the definition of $h^{*}_{\PXY}$.

Let $m = \left\lceil \frac{8}{\gamma_{2}^{2} r_{\eps}^{2}} \left( 10 \vc \Log\left(\frac{8 e}{\gamma_{2}^{2} r_{\eps}^{2}}\right) + \Log(24) \right) \right\rceil$,
which is a finite natural number, since $\vc \leq \s < \infty$.
It follows from Lemma~\ref{lem:vc-of-pairs} and Lemma~\ref{lem:vc-relative} that, for $X_{1}^{\prime},\ldots,X_{m}^{\prime}$ independent $\Px$-distributed 
random variables, with probability at least $2/3$, every $g \in \C$ has $\frac{1}{m} \sum_{i=1}^{m} \ind_{\DIS(\{h,g\})}(X_{i}^{\prime}) \leq \Px(x : h(x) \neq g(x)) + \gamma_{2} r_{\eps} \leq \Px(x : h^{*}_{\PXY}(x) \neq g(x)) + (\gamma_{3} +\gamma_{2}) r_{\eps}$.
Furthermore, by Hoeffding's inequality, we also have that with probability at least $2/3$, 
$\frac{1}{m} \sum_{i=1}^{m} \ind_{\DIS(\Ball_{\Px}(h^{*}_{\PXY}, r_{\eps}))}(X_{i}^{\prime}) \geq \Px(\DIS(\Ball_{\Px}(h^{*}_{\PXY},r_{\eps}))) - \gamma_{2} r_{\eps}$.
By a union bound, both of these events happen with probability at least $1/3$.  In particular, this implies $\exists z_{1},\ldots,z_{m} \in \X$
such that, letting $\hat{\Px}$ be the probability measure with $\hat{\Px}(A) = \frac{1}{m} \sum_{i=1}^{m} \ind_{A}(z_{m})$ for all measurable $A \subseteq \X$,
we have, $\forall g \in \C$, $\hat{\Px}(\DIS(\{h,g\})) \leq \Px(\DIS(\{h^{*}_{\PXY},g\})) + (\gamma_{3}+\gamma_{2}) r_{\eps}$, 
and furthermore $\hat{\Px}(\DIS(\Ball_{\Px}(h^{*}_{\PXY},r_{\eps}))) \geq \Px(\DIS(\Ball_{\Px}(h^{*}_{\PXY},r_{\eps}))) - \gamma_{2} r_{\eps}$.
This further implies that $\Ball_{\Px}(h^{*}_{\PXY}, r_{\eps}) \subseteq \Ball_{\hat{\Px}}(h, (1+\gamma_{3}+\gamma_{2}) r_{\eps})$,
and thus 
\begin{align*}
\hat{\Px}(\DIS(\Ball_{\hat{\Px}}(h,(1+\gamma_{3}+\gamma_{2}) r_{\eps}))) 
& \geq \hat{\Px}(\DIS(\Ball_{\Px}(h^{*}_{\PXY},r_{\eps}))) 
\geq \Px(\DIS(\Ball_{\Px}(h^{*}_{\PXY},r_{\eps}))) - \gamma_{2} r_{\eps} 
\\ & \geq (1-\gamma_{1}) \dc_{\PXY}(\eps) r_{\eps} - \gamma_{2} r_{\eps}
\geq (1 - \gamma_{1} - \gamma_{2}) \dc_{\PXY}(\eps) r_{\eps}.
\end{align*}
Therefore,
\begin{equation*}
\dc_{h,\hat{\Px}}(\eps) 
\geq \frac{\hat{\Px}(\DIS(\Ball_{\hat{\Px}}(h,(1+\gamma_{3}+\gamma_{2}) r_{\eps})))}{ (1+\gamma_{3}+\gamma_{2}) r_{\eps} }
\geq \frac{1 - \gamma_{1} - \gamma_{2}}{1 + \gamma_{3} + \gamma_{2}} \dc_{\PXY}(\eps).
\end{equation*}
Noting that $\hat{\Px}(\{z_{1},\ldots,z_{m}\}) = 1$, so that $\hat{\Px} \in \Pi$, 
since $\PXY$ was arbitrary, we have established that $\forall \PXY$, $\exists P \in \Pi$ and $h \in \C$ such that
$\dc_{h,P}(\eps) \geq \frac{1-\gamma_{1}-\gamma_{2}}{1+\gamma_{3}+\gamma_{2}} \dc_{\PXY}(\eps)$.
Since this holds for any choices of $\gamma_{1},\gamma_{2},\gamma_{3} \in (0,1)$, taking the limits as $\gamma_{1} \to 0$, $\gamma_{3} \to 0$, and $\gamma_{2} \to 0$, 
we have $\sup_{P \in \Pi} \sup_{h \in \C} \dc_{h,P}(\eps) \geq \supdc(\eps)$.
\end{proof}

In fact, it is easy to show (based on the first part of the proof below) that the 
``$\s < \infty$'' constraint is unnecessary in Lemma~\ref{lem:dc-star-finitely-discrete},
though this is not important for our purposes.
We are now ready for the proof of Theorem~\ref{thm:dc-star}.

\begin{proof}[of Theorem~\ref{thm:dc-star}]
First, we prove $\supdc(\eps) \geq \s \land \frac{1}{\eps}$.
Toward this end, let $\{x_{i}\}_{i=1}^{\s}$ and $\{h_{i}\}_{i=0}^{\s}$ be as in Definition~\ref{def:star},
and let $m = \s \land \left\lceil \frac{1}{\eps} \right\rceil$.
Let $\Px$ be a probability measure on $\X$ with $\Px(\{x_{i}\}) = 1/m$ for each $i \in \{1,\ldots,m\}$.
In particular, this implies that every $i \in \{1,\ldots,m\}$ has $\Px(x : h_{i}(x) \neq h_{0}(x)) = 1/m$,
so that $h_{i} \in \Ball_{\Px}(h_{0},1/m)$.  Since clearly $h_{0} \in \Ball_{\Px}(h_{0},1/m)$ as well,
and every $i \in \{1,\ldots,m\}$ has $x_{i} \in \DIS(\{h_{i},h_{0}\})$,
every $r > 1/m$ has $\Px(\DIS(\Ball_{\Px}(h_{0},r))) = \Px(\{ x_{i} : i \in \{1,\ldots,m\}\}) = 1$.
Therefore, letting $\PXY$ be the distribution in $\RE$ with $\target_{\PXY} = h_{0}$ and marginal $\Px$ over $\X$,
\begin{align*}
\supdc(\eps) 
& \geq \dc_{\PXY}(\eps) = \dc_{h_{0},\Px}(\eps) 
\geq \frac{\Px(\DIS(\Ball_{\Px}(h_{0},\max\{1/m,\eps\})))}{\max\{1/m,\eps\}}
\\ & = \frac{1}{\max\{1/m,\eps\}}
= m \land \frac{1}{\eps}
= \s \land \frac{1}{\eps}.
\end{align*}

Next, we prove that $\supdc(\eps) \leq \s \land \frac{1}{\eps}$.
That $\supdc(\eps) \leq \frac{1}{\eps}$ follows directly from the definition, and the fact that probabilities are at most $1$:
that is, any $\Px$ and $h$ have $\sup_{r > \eps} \frac{\Px(\DIS(\Ball_{\Px}(h,r)))}{r} \leq \sup_{r > \eps} \frac{1}{r} = \frac{1}{\eps}$.
Therefore, it remains only to show that $\supdc(\eps) \leq \s$ when $\s < \frac{1}{\eps}$.
Furthermore, Lemma~\ref{lem:dc-star-finitely-discrete} implies that it suffices to show that $\sup_{\Px \in \Pi} \sup_{h \in \C} \dc_{h,\Px}(\eps) \leq \s$ in this case.
Toward this end, suppose $\s < \frac{1}{\eps}$.  We first stratify the set $\Pi$ based on the size of the support, defining, 
for each $m \in \nats$, $\Pi_{m} = \{ \Px \in \Pi : \exists z_{1},\ldots,z_{m} \in \X \text{ s.t. } \Px(\{z_{1},\ldots,z_{m}\})=1 \}$.
Thus, $\Pi_{m}$ is the set of probability measures on $\X$ for which the support of the probability mass function has cardinality at most $m$.

We now proceed by induction on $m$.  As a base case, 
fix any $m \leq \s$, any classifier $h$, and any $\Px \in \Pi_{m}$,
and let $z_{1},\ldots,z_{m} \in \X$ be such that $\Px(\{z_{1},\ldots,z_{m}\}) = 1$.
For any $r \in [1/\s,1]$, 
$\Px(\DIS(\Ball_{\Px}(h,r))) / r \leq 1 / r \leq \s$.
Furthermore (following an argument of \citealp*{hanneke:survey}), 
for any $r \in (\eps,1/\s)$, for any $g \in \C$ with $\Px(x : g(x) \neq h(x)) \leq r$,
every $z \in \X$ with $\Px(\{z\}) > r$ has $\Px(x : g(x) \neq h(x)) < \Px(\{z\})$, so that $g(z) = h(z)$;
thus, $z \notin \DIS(\Ball_{\Px}(h,r))$.  We therefore have that 
$\Px(\DIS(\Ball_{\Px}(h,r))) 
\leq \Px(x : \Px(\{x\}) \leq r) 
= \sum_{i = 1}^{m} \ind\left[ \Px(\{z_{i}\}) \leq r \right] \Px(\{z_{i}\})
\leq r | \{ i \in \{1,\ldots,m\} : \Px(\{z_{i}\}) \leq r \} |$.
Therefore, 
$\frac{\Px(\DIS(\Ball_{\Px}(h,r)))}{r} \leq |\{i \in \{1,\ldots,m\} : \Px(\{z_{i}\}) \leq r \}| \leq m \leq \s$,
so that (since $\s \geq 1$, due to the assumption that $|\C| \geq 2$), we have $\dc_{h,\Px}(\eps) \leq \s$.

Now take as an inductive hypothesis that, for some $m \in \nats$ with $m > \s$, we have 
\begin{equation*}
\sup_{\Px \in \Pi_{m-1}} \sup_{h \in \C} \dc_{h,\Px}(\eps) \leq \s.
\end{equation*}
Fix any $h \in \C$, $r > \eps$, and $\Px \in \Pi_{m}$, and let $z_{1},\ldots,z_{m} \in \X$ be such that $\Px(\{z_{1},\ldots,z_{m}\}) = 1$.
If $\exists i,j \in \{1,\ldots,m\}$ with $i \neq j$ and $z_{i} = z_{j}$, or if some $j \in \{1,\ldots,m\}$ has $\Px(\{z_{j}\}) = 0$, 
then since either of these has $\Px(\{z_{k} : k \in \{1,\ldots,m\} \setminus \{j\}\}) = 1$, 
we would also have $\Px \in \Pi_{m-1}$, so that $\dc_{h,\Px}(\eps) \leq \s$ by the inductive hypothesis.  To handle the remaining nontrivial cases,
suppose the $z_{1},\ldots,z_{m}$ are all distinct, and $\min_{i \in \{1,\ldots,m\}} \Px(\{z_{i}\}) > 0$.
Furthermore, note that, since $m > \s$, $\{z_{1},\ldots,z_{m}\}$ cannot be a star set for $\C$.

We now consider three cases.  
First, consider the case that $\exists k \in \{1,\ldots,m\}$ with $z_{k} \notin \DIS(\Ball_{\Px}(h,r))$.
In this case, define a probability measure $\Px^{\prime}$ over $\X$ such that, 
for any measurable $A \subseteq \X \setminus \{z_{k}\}$, $\Px^{\prime}(A) = \Px^{\prime}(A \cup \{z_{k}\}) = \Px(A) / (1-\Px(\{z_{k}\}))$.
Note that this is a well-defined probability measure, since $m \geq 2$ and $\min_{i \in \{1,\ldots,m\}} \Px(\{z_{i}\}) > 0$, 
so that $\Px(\X \setminus \{z_{k}\}) = 1 - \Px(\{z_{k}\}) > 0$.
Also note that (since $h \in \Ball_{\Px}(h,r)$) any $g \in \Ball_{\Px}(h,r)$ has $g(z_{k}) = h(z_{k})$, 
so that $\Px^{\prime}( x : g(x) \neq h(x) ) = \Px( x : g(x) \neq h(x) ) / (1-\Px(\{z_{k}\})) \leq r / (1 - \Px(\{z_{k}\}))$.
Therefore, $\Ball_{\Px^{\prime}}(h,r / (1-\Px(\{z_{k}\}))) \supseteq \Ball_{\Px}(h,r)$, and since $z_{k} \notin \DIS(\Ball_{\Px}(h,r))$,
$\Px^{\prime}(\DIS(\Ball_{\Px^{\prime}}(h, r / (1-\Px(\{z_{k}\}))))) \geq \Px^{\prime}(\DIS(\Ball_{\Px}(h,r))) = \Px(\DIS(\Ball_{\Px}(h,r)))/(1-\Px(\{z_{k}\}))$.
Thus, 
\begin{equation}
\label{eqn:dc-star-first-case-bound}
\Px(\DIS(\Ball_{\Px}(h,r))) \leq (1-\Px(\{z_{k}\}))\Px^{\prime}(\DIS(\Ball_{\Px^{\prime}}(h, r / (1-\Px(\{z_{k}\}))))).
\end{equation}
Noting that $\Px^{\prime}( \{z_{i} : i \in \{1,\ldots,m\} \setminus \{k\}\} ) = \Px( \{z_{1},\ldots,z_{m}\} \setminus \{z_{k}\}) / (1 - \Px(\{z_{k}\})) = 1$,
we have that $\Px^{\prime} \in \Pi_{m-1}$.  Therefore, by the inductive hypothesis and the fact that $r / (1-\Px(\{z_{k}\})) > r > \eps$, 
\begin{align*}
\Px^{\prime}\left(\DIS\left(\Ball_{\Px^{\prime}}\left(h, \frac{r}{1-\Px(\{z_{k}\})} \right)\right)\right) 
& \leq \dc_{h,\Px^{\prime}}(\eps) \frac{r}{1-\Px(\{z_{k}\})}
\\ & \leq \sup_{P \in \Pi_{m-1}} \sup_{h^{\prime} \in \C} \dc_{h^{\prime},P}(\eps) \frac{r}{1-\Px(\{z_{k}\})}
\leq \frac{\s r}{1-\Px(\{z_{k}\})}.
\end{align*}
Combined with \eqref{eqn:dc-star-first-case-bound},
this further implies that $\Px(\DIS(\Ball_{\Px}(h,r))) \leq (1-\Px(\{z_{k}\})) \s r / (1-\Px(\{z_{k}\}))$ $= \s r$.

Next, consider a second case, where $\{z_{1},\ldots,z_{m}\} \subseteq \DIS(\Ball_{\Px}(h,r))$,
and $\exists j,k \in \{1,\ldots,m\}$ with $j \neq k$ such that,
$\forall g \in \Ball_{\Px}(h,r)$, $g(z_{k}) \neq h(z_{k}) \Rightarrow g(z_{j}) \neq h(z_{j})$.
In this case, define a probability measure $\Px^{\prime}$ over $\X$ such that, for any measurable $A \subseteq \X \setminus \{z_{j},z_{k}\}$, 
$\Px^{\prime}(A) = \Px(A)$, $\Px^{\prime}(A \cup \{z_{j}\}) = \Px(A)$, and $\Px^{\prime}(A \cup \{z_{k}\}) = \Px^{\prime}(A \cup \{z_{j},z_{k}\}) = \Px(A \cup \{z_{j},z_{k}\})$:
in other words, $\Px^{\prime}$ has a probability mass function $x \mapsto \Px^{\prime}(\{x\})$ equal to $x \mapsto \Px(\{x\})$ everywhere, 
except that $\Px^{\prime}(\{z_{j}\}) = 0$ and $\Px^{\prime}(\{z_{k}\}) = \Px(\{z_{j}\}) + \Px(\{z_{k}\})$.
Note that, for any $g \in \Ball_{\Px}(h,r)$ with $g(z_{k}) = h(z_{k})$, $\Px^{\prime}(x : g(x) \neq h(x)) = \Px(x : g(x) \neq h(x)) - \ind[ g(z_{j}) \neq h(z_{j}) ] \Px(\{z_{j}\}) \leq \Px(x : g(x) \neq h(x)) \leq r$.
Furthermore, any $g \in \Ball_{\Px}(h,r)$ with $g(z_{k}) \neq h(z_{k})$ also has $g(z_{j}) \neq h(z_{j})$, so that $\Px^{\prime}( x : g(x) \neq h(x) ) = \Px( x : g(x) \neq h(x) ) \leq r$.
Therefore, $\Ball_{\Px^{\prime}}(h,r) \supseteq \Ball_{\Px}(h,r)$.  Since $z_{j},z_{k} \in \DIS(\Ball_{\Px}(h,r))$, this further implies that $z_{j},z_{k} \in \DIS(\Ball_{\Px^{\prime}}(h,r))$.
Therefore, by definition of $\Px^{\prime}$ and monotonicity of measures, 
$\Px^{\prime}(\DIS(\Ball_{\Px^{\prime}}(h,r))) = \Px( \DIS(\Ball_{\Px^{\prime}}(h,r)) ) \geq \Px(\DIS(\Ball_{\Px}(h,r)))$.
Noting that $\Px^{\prime}(\{z_{i} : i \in \{1,\ldots,m\} \setminus \{j\}) = \Px(\{z_{1},\ldots,z_{m}\}) = 1$, we have $\Px^{\prime} \in \Pi_{m-1}$,
and therefore (by the inductive hypothesis), $\Px^{\prime}(\DIS(\Ball_{\Px^{\prime}}(h,r))) \leq \dc_{h,\Px^{\prime}}(\eps) r \leq \sup_{P \in \Pi_{m-1}} \sup_{h^{\prime} \in \C} \dc_{h^{\prime},P}(\eps) r \leq \s r$.
Thus, since we established above that $\Px(\DIS(\Ball_{\Px}(h,r))) \leq \Px^{\prime}(\DIS(\Ball_{\Px^{\prime}}(h,r)))$,
we have that $\Px(\DIS(\Ball_{\Px}(h,r))) \leq \s r$. 

Finally, consider a third case (the complement of the first and second), in which 
$\{z_{1},\ldots,z_{m}\} \subseteq \DIS(\Ball_{\Px}(h,r))$, 
but 
$\nexists j,k \in \{1,\ldots,m\}$ with $j \neq k$ such that,
$\forall g \in \Ball_{\Px}(h,r)$, $g(z_{k}) \neq h(z_{k}) \Rightarrow g(z_{j}) \neq h(z_{j})$.
In particular, note that the first condition (which is, in fact, redundant, but included for clarity) implies $\Px(\DIS(\Ball_{\Px}(h,r))) = 1$.
In this case, since (as above) $\{z_{1},\ldots,z_{m}\}$ is not a star set for $\C$,
$\exists i \in \{1,\ldots,m\}$ such that $\forall g \in \C$ with $g(z_{i}) \neq h(z_{i})$, $\exists j \in \{1,\ldots,m\} \setminus \{i\}$ with $g(z_{j}) \neq h(z_{j})$ as well;
fix any such $i \in \{1,\ldots,m\}$.
Since $\{z_{1},\ldots,z_{m}\} \subseteq \DIS(\Ball_{\Px}(h,r))$, we have $z_{i} \in \DIS(\Ball_{\Px}(h,r))$.
Thus, we may let $g_{i} \in \Ball_{\Px}(h,r)$ be such that $g_{i}(z_{i}) \neq h(z_{i})$,
and let $j \in \{1,\ldots,m\} \setminus \{i\}$ be such that $g_{i}(z_{j}) \neq h(z_{j})$ (which exists, by our choice of $i$).
Let $\Px^{\prime}$ be a probability measure over $\X$ such that,
for all measurable $A \subseteq \X \setminus \{z_{i},z_{j}\}$,
$\Px^{\prime}(A) = \Px(A)$, $\Px^{\prime}(A \cup \{z_{i}\}) = \Px(A)$, and $\Px^{\prime}(A \cup \{z_{j}\}) = \Px^{\prime}(A \cup \{z_{i},z_{j}\}) = \Px(A \cup \{z_{i},z_{j}\})$:
in other words, $\Px^{\prime}$ has a probability mass function $x \mapsto \Px^{\prime}(\{x\})$ equal to $x \mapsto \Px(\{x\})$ everywhere, 
except that $\Px^{\prime}(\{z_{i}\}) = 0$ and $\Px^{\prime}(\{z_{j}\}) = \Px(\{z_{i}\}) + \Px(\{z_{j}\})$.
Note that, for any measurable set $A \subseteq \X$ with $\{z_{i},z_{j}\} \subseteq A$, $\Px^{\prime}(A) = \Px(A)$.
In particular, since $\{z_{i},z_{j}\} \subseteq \DIS(\{g_{i},h\})$, $\Px^{\prime}(\DIS(\{g_{i},h\})) = \Px(\DIS(\{g_{i},h\})) \leq r$,
so that $g_{i} \in \Ball_{\Px^{\prime}}(h,r)$, and therefore (since $h \in \Ball_{\Px^{\prime}}(h,r)$ as well) $\{z_{i},z_{j}\} \subseteq \DIS(\Ball_{\Px^{\prime}}(h,r))$.
Furthermore, for any $k \in \{1,\ldots,m\} \setminus \{i,j\}$, by the property characterizing this third case, 
and since $z_{k} \in \DIS(\Ball_{\Px}(h,r))$,
$\exists g \in \Ball_{\Px}(h,r)$ with $g(z_{k}) \neq h(z_{k})$ and $g(z_{j}) = h(z_{j})$,
so that $\Px^{\prime}(\DIS(\{g,h\})) = \Px(\DIS(\{g,h\}) \setminus \{z_{i}\}) \leq \Px(\DIS(\{g,h\})) \leq r$
(i.e., $g \in \Ball_{\Px^{\prime}}(h,r)$),
and therefore (since $h \in \Ball_{\Px^{\prime}}(h,r)$ as well) $z_{k} \in \DIS(\Ball_{\Px^{\prime}}(h,r))$ as well.
Altogether, we have that $\{z_{1},\ldots,z_{m}\} \subseteq \DIS(\Ball_{\Px^{\prime}}(h,r))$.
Therefore, since $\{z_{i},z_{j}\} \subseteq \DIS(\Ball_{\Px^{\prime}}(h,r))$, the definition of $\Px^{\prime}$ implies
$\Px^{\prime}(\DIS(\Ball_{\Px^{\prime}}(h,r))) = \Px(\DIS(\Ball_{\Px^{\prime}}(h,r))) \geq \Px(\{z_{1},\ldots,z_{m}\}) = 1 = \Px(\DIS(\Ball_{\Px}(h,r)))$.
Noting that $\Px^{\prime}(\{z_{k} : k \in \{1,\ldots,m\} \setminus \{i\}\}) = \Px(\{z_{1},\ldots,z_{m}\}) = 1$, 
we have that $\Px^{\prime} \in \Pi_{m-1}$, and therefore (by the inductive hypothesis), 
$\Px^{\prime}(\DIS(\Ball_{\Px^{\prime}}(h,r))) \leq \dc_{h,\Px^{\prime}}(\eps) r \leq \sup_{P \in \Pi_{m-1}} \sup_{h^{\prime} \in \C} \dc_{h^{\prime},P}(\eps) r \leq \s r$.
Since $\Px^{\prime}(\DIS(\Ball_{\Px^{\prime}}(h,r))) = 1 = \Px(\DIS(\Ball_{\Px}(h,r)))$, 
we have that $\Px(\DIS(\Ball_{\Px}(h,r))) \leq \s r$ as well.

Thus, in all three cases, we have that $\Px(\DIS(\Ball_{\Px}(h,r))) \leq \s r$.
Since this holds for every $r > \eps$, and $|\C| \geq 2$ implies $\s \geq 1$, 
we have that $\dc_{h,\Px}(\eps) \leq \s$.
Since this holds for every $h \in \C$ and $\Px \in \Pi_{m}$, we have established
that $\sup_{\Px \in \Pi_{m}} \sup_{h \in \C} \dc_{h,\Px}(\eps) \leq \s$, which completes
the inductive step.  It follows by the principle of induction that $\sup_{\Px \in \Pi_{m}} \sup_{h \in \C} \dc_{h,\Px}(\eps) \leq \s$
for every $m \in \nats$, and therefore, since $\Pi = \bigcup_{m} \Pi_{m}$, 
$\sup_{\Px \in \Pi} \sup_{h \in \C} \dc_{h,\Px}(\eps) \leq \s$.

The claim that $\supdc(0) = \s$ follows as a limiting case, due to continuity of the supremum from below.
Specifically, fix any sequence $\{A_{n}\}_{n=1}^{\infty}$ of nonempty subsets of $\reals$.
For each $m \in \nats$, $\bigcup_{n} A_{n} \supseteq A_{m}$, so 
$\sup \bigcup_{n} A_{n} \geq \sup A_{m}$ (allowing the supremum to take the value $\infty$ where appropriate),
and since this holds for every such $m$, we have $\sup \bigcup_{n} A_{n} \geq \sup_{n} \sup A_{n}$
Furthermore, $\forall a \in \bigcup_{n} A_{n}$, $\exists m \in \nats$ s.t. $a \in A_{m}$, 
so that $\sup_{n} \sup A_{n} \geq \sup A_{m} \geq a$, 
and therefore (since this holds for every such $a$) $\sup_{n} \sup A_{n} \geq \sup \bigcup_{n} A_{n}$.
Thus, $\sup \bigcup_{n} A_{n} = \sup_{n} \sup A_{n}$.
In particular, taking (for each $n \in \nats$)
\begin{equation*}
A_{n} = \left\{ \frac{\Px(\DIS(\Ball_{\Px}(h^{*}_{\PXY},r)))}{r} \lor 1 : r > 1/n, \PXY \in \AG(1) \right\},
\end{equation*}
(where, as usual, $\Px(\cdot) = \PXY(\cdot \times \Y)$ denotes the marginal of $\PXY$ over $\X$),
and noting that $\sup \bigcup_{n} A_{n} = \supdc(0)$ and $\forall n \in \nats$, $\sup A_{n} = \supdc(1/n)$,
we have that $\supdc(0) = \sup_{n} \supdc(1/n) = \sup_{n} \s \land n = \s$.
\end{proof}

\subsection{The Splitting Index}
\label{app:splitting}

Here we present the proof of Theorem~\ref{thm:splitting-star}.
First, we introduce a quantity related to $\infsplit(\eps)$, but slightly simpler.
For $\eps,\tau \in (0,1]$ and any probability measure $\Px$ over $\X$, define
\begin{equation*}
\basicsplit_{\Px}(\eps;\tau) = \sup\left\{ \rho \in [0,1] : \C \text{ is } (\rho,\eps,\tau)\text{-splittable under } \Px \right\},
\end{equation*}
and let
\begin{equation*}
\basicsplit(\eps) = \inf_{P} \lim_{\tau \to 0} \basicsplit_{P}(\eps;\tau).
\end{equation*}
In the arguments below, we will see that $\left\lfloor 1/\basicsplit(\eps) \right\rfloor = \left\lfloor 1/\infsplit(\eps) \right\rfloor$,
so that it suffices to work with this simpler quantity.
We begin with a lemma which allows us to restrict our focus (in part of the proof) to finitely discrete probability measures.
Recall the definition of $\Pi$ from Appendix~\ref{app:dc} above.

\begin{lemma}
\label{lem:splitting-star-finitely-discrete}
If $\vc < \infty$, then $\forall \eps \in (0,1]$, 
$\basicsplit(\eps) \geq \lim\limits_{\gamma \to 0} \inf\limits_{P \in \Pi} \lim\limits_{\tau \to 0} \basicsplit_{P}((1-\gamma)\eps;\tau)$.
\end{lemma}
\begin{proof}
Suppose $\vc < \infty$, and fix any $\eps \in (0,1]$.  
Fix arbitrary values $\gamma_{1},\gamma_{2} \in (0,1)$, and let
\begin{equation*}
m = \left\lceil \frac{8}{\gamma_{2}^{2} \eps^{2}} \left( 10 \vc \Log\left( \frac{8 e}{\gamma_{2}^{2} \eps^{2}} \right) + \Log(24) \right) \right\rceil,
\end{equation*}
which is a finite natural number.
Fix any probability measure $\Px$ over $\X$, and any $\tau \in (0,1/(3m))$,
and note that $\tau^{\prime} \mapsto \basicsplit_{\Px}(\eps;\tau^{\prime})$ is nonincreasing, so that
$\basicsplit_{\Px}(\eps;\tau) \leq \lim_{\tau^{\prime} \to 0} \basicsplit_{\Px}(\eps;\tau^{\prime})$.
For brevity, denote $\basicsplit = \basicsplit_{\Px}(\eps;\tau)$.
Since $\C$ is not $(\gamma_{1} + \basicsplit,\eps,\tau)$-splittable under $\Px$,
let $Q \subseteq \{\{f,g\} \subseteq \C : \Px(x : f(x) \neq g(x)) \geq \eps\}$ be a finite set such that
$\Px\!\left(x : \Split(Q,x) \geq \left(\gamma_{1} + \basicsplit\right) |Q|\right) < \tau$.

Let $X_{1}^{\prime},\ldots,X_{m}^{\prime}$ be independent $\Px$-distributed random variables.
Lemma~\ref{lem:vc-of-pairs} and Lemma~\ref{lem:vc-relative} imply that, with probability at least $2/3$, 
$\forall f,g \in \C$, 
\begin{equation*}
\left| \Px(x : f(x) \neq g(x)) - \frac{1}{m} \sum_{i=1}^{m} \ind\!\left[ f(X_{i}^{\prime}) \neq g(X_{i}^{\prime}) \right]\right| \leq \gamma_{2} \eps.
\end{equation*}
Furthermore, by a union bound, with probability at least $1 - m \Px\!\left(x : \Split(Q,x) \geq \left(\gamma_{1} + \basicsplit\right) |Q|\right) > 1 - m \tau > 1 - m (1/(3m)) = 2/3$, 
every $i \in \{1,\ldots,m\}$ has $\Split(Q,X_{i}^{\prime}) < (\gamma_{1} + \basicsplit) |Q|$.
By a union bound, both of the above events occur with probability at least $1/3$.  
In particular, this implies $\exists z_{1},\ldots,z_{m} \in \X$ such that, letting $\hat{\Px}$ be the probability measure with $\hat{\Px}(A) = \frac{1}{m} \sum_{i=1}^{m} \ind_{A}(z_{m})$ for all measurable $A \subseteq \X$,
we have, $\forall f,g \in \C$, $\left| \Px(x : f(x) \neq g(x)) - \hat{\Px}(x : f(x) \neq g(x)) \right| \leq \gamma_{2} \eps$,
and $\hat{\Px}( x : \Split(Q,x) \geq (\gamma_{1} + \basicsplit) |Q| ) = 0$.

For any $\{f,g\} \in Q$, we have 
$\hat{\Px}(x : f(x) \neq g(x)) \geq \Px(x : f(x) \neq g(x)) - \gamma_{2} \eps \geq (1-\gamma_{2}) \eps$.
Therefore, $\C$ is not $(\gamma_{1} + \basicsplit, (1-\gamma_{2}) \eps, \tau^{\prime})$-splittable under $\hat{\Px}$ for any $\tau^{\prime} > 0$,
which implies $\lim_{\tau^{\prime} \to 0} \basicsplit_{\hat{\Px}}( (1-\gamma_{2})\eps;\tau^{\prime}) \leq \gamma_{1} + \basicsplit_{\Px}(\eps;\tau)$.
Since $\hat{\Px} \in \Pi$, we have 
\begin{equation*}
\inf_{P \in \Pi} \lim_{\tau^{\prime} \to 0} \basicsplit_{P}( (1-\gamma_{2}) \eps; \tau^{\prime} ) 
\leq \gamma_{1} + \basicsplit_{\Px}(\eps;\tau) 
\leq \gamma_{1} + \lim_{\tau^{\prime} \to 0} \basicsplit_{\Px}(\eps;\tau^{\prime}).
\end{equation*}
Since this holds for any $\gamma_{1} \in (0,1)$, taking the limit as $\gamma_{1} \to 0$ implies
\begin{equation*}
\inf_{P \in \Pi} \lim_{\tau^{\prime} \to 0} \basicsplit_{P}( (1-\gamma_{2}) \eps; \tau^{\prime} ) \leq \lim_{\tau^{\prime} \to 0} \basicsplit_{\Px}(\eps;\tau^{\prime}).
\end{equation*}
Furthermore, since this holds for any $\gamma_{2} \in (0,1)$ and any $\Px$, 
we have 
\begin{equation*}
\lim_{\gamma_{2} \to 0} \inf_{P \in \Pi} \lim_{\tau^{\prime} \to 0} \basicsplit_{P}( (1-\gamma_{2}) \eps; \tau^{\prime} ) 
\leq \inf_{P} \lim_{\tau^{\prime} \to 0} \basicsplit_{P}(\eps;\tau^{\prime})
= \basicsplit(\eps).
\end{equation*}
\end{proof}

We are now ready for the proof of Theorem~\ref{thm:splitting-star}.

\begin{proof}[of Theorem~\ref{thm:splitting-star}]
We first establish that $\s \land \left\lfloor \frac{1}{\eps} \right\rfloor \leq \left\lfloor \frac{1}{\infsplit(\eps)} \right\rfloor$ for any $\eps \in (0,1]$.
The proof of this fact was implicitly established in the original work of \citet*[][Corollary 3]{dasgupta:05},
but we include the argument here for completeness.
Let $\{x_{i}\}_{i=1}^{\s}$ and $\{h_{i}\}_{i=0}^{\s}$ be as in Definition~\ref{def:star},
and let $m = \s \land \left\lfloor \frac{1}{\eps} \right\rfloor$.
Let $\Delta = 1/m$, and note that $\Delta \geq 1 / \left\lfloor \frac{1}{\eps} \right\rfloor \geq \eps$.
As in the proof of Theorem~\ref{thm:dc-star}, 
let $\Px$ be a probability measure on $\X$ with $\Px(\{x_{i}\}) = 1/m$ for each $i \in \{1,\ldots,m\}$.
Thus, every $i \in \{1,\ldots,m\}$ has $\Px(x : h_{i}(x) \neq h_{0}(x)) = \Delta$, so that
$h_{i} \in \Ball_{\Px}(h_{0},\Delta) \subseteq \Ball_{\Px}(h_{0},4\Delta)$, and the finite set $Q = \{ \{h_{0},h_{i}\} : i \in \{1,\ldots,m\} \}$ 
satisfies $Q \subseteq \{ \{f,g\} \subseteq \Ball_{\Px}(h_{0},4\Delta) : \Px(x : f(x) \neq g(x)) \geq \Delta \}$. 
In particular, since $\Px(\X \setminus \{x_{1},\ldots,x_{m}\}) = 0$, and every $i \in \{1,\ldots,m\}$ has $\Split(Q,x_{i}) = 1 = \frac{1}{m} |Q|$,
we have $\Px\left(x : \Split(Q,x) > \frac{1}{m} |Q| \right) = 0$.  Thus, for any $\rho > \frac{1}{m}$, and any $\tau > 0$, $\Ball_{\Px}(h_{0},4\Delta)$ is not $(\rho,\Delta,\tau)$-splittable.
Therefore, $\infsplit(\eps) \leq \lim_{\tau \to 0} \rho_{h_{0},\Px}(\eps;\tau) \leq \frac{1}{m}$, which implies
$\frac{1}{\infsplit(\eps)} \geq m$; since $m \in \nats$, it follows that $\left\lfloor \frac{1}{\infsplit(\eps)} \right\rfloor \geq m$.

Next, we prove that $\left\lfloor \frac{1}{\infsplit(\eps)} \right\rfloor \leq \s \land \left\lfloor \frac{1}{\eps} \right\rfloor$ for any $\eps \in (0,1]$.
Since, for every $h \in \C$, every probability measure $\Px$ over $\X$, and every $\Delta \geq \eps$, 
every finite $Q \subseteq \{ \{f,g\} \subseteq \Ball_{\Px}(h,4\Delta) : \Px(x : f(x) \neq g(x)) \geq \Delta \}$
also has $Q \subseteq \{ \{f,g\} \subseteq \C : \Px(x : f(x) \neq g(x)) \geq \eps \}$, we have $\basicsplit(\eps) \leq \infsplit(\eps)$.
Thus, it suffices to show $\left\lfloor \frac{1}{\basicsplit(\eps)} \right\rfloor \leq \s \land \left\lfloor \frac{1}{\eps} \right\rfloor$.

That $\basicsplit(\eps) \geq \eps$ was established by \citet*[][Lemma 1]{dasgupta:05};
we repeat the argument here for completeness.  Fix any probability measure $\Px$ over $\X$ and any $\eps,\tau \in (0,1]$ with $\tau < \eps$.
Fix any finite set $Q \subseteq \{ \{f,g\} \subseteq \C : \Px(x : f(x) \neq g(x)) \geq \eps \}$.  If $Q = \emptyset$, then trivially $\Px(x : \Split(Q,x) \geq \eps |Q|) = 1 \geq \tau$.
Otherwise, if $Q \neq \emptyset$, letting $X \sim \Px$, 
\begin{equation*}
\E[ \Split(Q,X) ] 
\geq \E\left[ \sum_{\{f,g\} \in Q} \ind[ f(Z) \neq g(Z) ] \right] 
= \sum_{\{f,g\} \in Q} \Px(x : f(x) \neq g(x)) 
\geq |Q| \eps.
\end{equation*}
Furthermore, since $\Split(Q,x) \leq |Q|$,
\begin{align*}
& \E[ \Split(Q,X) ] 
\\ & = \E\left[ \ind[ \Split(Q,X) \geq (\eps - \tau) |Q| ] \Split(Q,X) \right] + \E\left[ \ind[ \Split(Q,X) < (\eps - \tau) |Q| ] \Split(Q,X) \right]
\\ & < \Px\left( x : \Split(Q,x) \geq (\eps - \tau) |Q| \right) |Q| + (\eps - \tau) |Q|.
\end{align*}
Together, these inequalities imply 
\begin{equation*}
|Q| \eps < \Px\left( x : \Split(Q,x) \geq ( \eps - \tau ) |Q| \right) |Q| + ( \eps - \tau ) |Q|.
\end{equation*}
Subtracting $(\eps-\tau)|Q|$ from both sides and dividing by $|Q|$, we have 
\begin{equation*}
\tau < \Px\left( x : \Split(Q,x) \geq (\eps-\tau) |Q| \right).
\end{equation*}
Since this holds for any such $Q$, we have that $\C$ is $( (\eps-\tau), \eps, \tau )$-splittable under $\Px$, so that $\basicsplit_{\Px}(\eps;\tau) \geq \eps-\tau$.
Since this holds for every choice of $\Px$, we have that 
\begin{equation*}
\basicsplit(\eps) = \inf_{\Px} \lim_{\tau \to 0} \basicsplit_{\Px}(\eps;\tau) \geq \lim_{\tau \to 0} \eps-\tau = \eps,
\end{equation*}
from which it immediately follows that $\left\lfloor \frac{1}{\basicsplit(\eps)} \right\rfloor \leq \left\lfloor \frac{1}{\eps} \right\rfloor$.

It remains only to show that $\left\lfloor \frac{1}{\basicsplit(\eps)} \right\rfloor \leq \s$.
In particular, since this trivially holds when $\s = \infty$, for the remainder of the proof we suppose $\s < \infty$.
As argued in Section~\ref{sec:star}, we have $\vc \leq \s$, so that this also implies $\vc < \infty$.
Thus, Lemma~\ref{lem:splitting-star-finitely-discrete} implies that
$\basicsplit(\eps) \geq  \lim_{\gamma \to 0} \inf_{\Px \in \Pi} \lim_{\tau \to 0} \basicsplit_{\Px}((1-\gamma)\eps;\tau)$.
Therefore, if we can establish that, for every $\eps \in (0,1]$ and $\Px \in \Pi$, $\lim_{\tau \to 0} \basicsplit_{\Px}(\eps;\tau) \geq 1/\s$, 
then we would have that for every $\eps \in (0,1]$,
\begin{equation*}
\left\lfloor \frac{1}{\basicsplit(\eps)} \right\rfloor
\leq \frac{1}{\basicsplit(\eps)} 
\leq \lim_{\gamma \to 0} \sup_{\Px \in \Pi} \frac{1}{\lim_{\tau \to 0} \basicsplit_{\Px}((1-\gamma)\eps;\tau)}
\leq \s,
\end{equation*}
which would thereby complete the proof.

Toward this end, fix any $\eps \in (0,1]$, and for each $\Px \in \Pi$, denote
$\tau_{\Px} = \min\{ \Px(\{x\}) : x \in \X, \Px(\{x\}) > 0\}$;
in particular, note that (since $\Px \in \Pi$) $0 < \tau_{\Px} \leq 1$,
and therefore also that, $\forall \eps \in (0,1]$, $\lim_{\tau \to 0} \basicsplit_{\Px}(\eps;\tau) \geq \basicsplit_{\Px}(\eps;\tau_{\Px})$ (in fact, they are equal).
Furthermore, denoting $\supp(\Px) = \{ x \in \X : \Px(\{x\}) > 0 \}$, every $x \in \supp(\Px)$ has $\Px(\{x\}) \geq \tau_{\Px}$,
while $\Px( \X \setminus \supp(\Px) ) = 0$.  Thus, for any finite $Q \subseteq \{\{f,g\} \subseteq \C : \Px(x : f(x) \neq g(x)) \geq \eps\}$, and any $\rho \in [0,1]$,
$\Px(x : \Split(Q,x) \geq \rho |Q|) \geq \tau_{\Px}$ if and only if $\max_{x \in \supp(\Px)} \Split(Q,x) \geq \rho |Q|$.
Furthermore, since $\Px( \X \setminus \supp(\Px) ) = 0$, 
for any $\eps \in (0,1]$, every $\{f,g\} \subseteq \C$ with $\Px(x : f(x) \neq g(x)) \geq \eps$ must have $\DIS(\{f,g\}) \cap \supp(\Px) \neq \emptyset$.
Thus, defining
\begin{multline*}
\finitesplit_{\Px} = \sup\bigg\{ \rho \in [0,1] : \forall \text{ finite } Q \subseteq \{\{f,g\} \subseteq \C : \DIS(\{f,g\}) \cap \supp(\Px) \neq \emptyset \}, 
\\ \max_{x \in \supp(\Px)} \Split(Q,x) \geq \rho |Q| \bigg\},
\end{multline*}
we have $\finitesplit_{\Px} \leq \basicsplit_{\Px}(\eps;\tau_{\Px})$ for all $\eps \in (0,1]$ (in fact, they are equal for $\eps \leq \tau_{\Px}$).
Thus, it suffices to show that $\inf_{\Px \in \Pi} \finitesplit_{\Px} \geq 1/\s$.
Now partition the set $\Pi$ by the sizes of the supports, defining, for each $m \in \nats$,
$\Pi_{m} = \{ \Px \in \Pi : |\supp(\Px)| = m \}$ (this is slightly different from the definition used in the proof of Theorem~\ref{thm:dc-star}).
Note that, for any $\Px \in \Pi$, the value of $\finitesplit_{\Px}$ is entirely determined by $\supp(\Px)$.
Thus, defining, $\forall m \in \nats$ with $m \leq |\X|$,
\begin{multline*}
\finitesplit_{m} = \inf_{\X_{m} \subseteq \X : |\X_{m}|=m} \sup\bigg\{ \rho \in [0,1] : \forall \text{ finite } Q \subseteq \{\{f,g\} \subseteq \C : \DIS(\{f,g\}) \cap \X_{m} \neq \emptyset \}, 
\\ \max_{x \in \X_{m}} \Split(Q,x) \geq \rho |Q| \bigg\},
\end{multline*}
we have $\inf_{\Px \in \Pi_{m}} \finitesplit_{\Px} \geq \finitesplit_{m}$ (in fact, they are equal).
Thus, since $\Pi = \bigcup_{m \in \nats} \Pi_{m}$, we have $\inf_{\Px \in \Pi} \finitesplit_{\Px} = \inf_{m \in \nats : m \leq |\X|} \inf_{\Px \in \Pi_{m}} \finitesplit_{\Px} \geq \inf_{m \in \nats : m \leq |\X|} \finitesplit_{m}$.
Therefore, it suffices to show that $\finitesplit_{m} \geq 1/\s$ for all $m \in \nats$ with $m \leq |\X|$.

We proceed by induction on $m \in \nats$ with $m \leq |\X|$, 
combined with a nested inductive argument on $Q$.
As base cases (for induction on $m$), consider any $m \leq \s$.
Fix any $\X_{m} \subseteq \X$ with $|\X_{m}| = m$ (noting that $m \leq \s$ implies $m \leq |\X|$, since $\s \leq |\X|$ immediately follows from Definition~\ref{def:star}).
Also fix any finite set $Q \subseteq \{ \{f,g\} \subseteq \C : \DIS(\{f,g\}) \cap \X_{m} \neq \emptyset \}$.
Since $\forall \{f,g\} \in Q$, $\exists x \in \X_{m}$ such that $f(x) \neq g(x)$,
the pigeonhole principle implies $\exists x \in \X_{m}$ with $|\{ \{f,g\} \in Q : f(x) \neq g(x) \}| \geq |Q|/|\X_{m}| = |Q|/m$.
For this $x$, we have $\Split(Q,x) \geq |\{ \{f,g\} \in Q : f(x) \neq g(x) \}| \geq (1/m) |Q| \geq (1/\s) |Q|$.
Since this holds for any such choice of $Q$ and $\X_{m}$, we have that $\finitesplit_{m} \geq 1/\s$.

If $|\X| = \s$, this completes the proof.
Otherwise, take as an inductive hypothesis that, for some $m \in \nats$ with $\s < m \leq |\X|$, 
$\finitesplit_{m-1} \geq 1/\s$.  Fix any $\X_{m} \subseteq \X$ with $|\X_{m}| = m$.
We now introduce a nested inductive argument on $Q$ (based on the partial ordering induced by the subset relation).
As a base case, if $Q = \emptyset$, then trivially $\max_{x \in \X_{m}} \Split(Q,x) = 0 = (1/\s) |Q|$.
Now take as a nested inductive hypothesis that, for some nonempty finite set
$Q \subseteq \left\{ \{f,g\} \subseteq \C : \DIS(\{f,g\}) \cap \X_{m} \neq \emptyset \right\}$,
for every strict subset $R \subset Q$, $\max_{x \in \X_{m}} \Split(R,x) \geq (1/\s) |R|$.

First, consider the case in which $\exists x \in \X_{m}$ such that $x \notin \bigcup_{\{f,g\} \in Q} \DIS(\{f,g\})$.
In this case, every $\{f,g\} \in Q$ has $\DIS(\{f,g\}) \cap (\X_{m} \setminus \{x\}) = \DIS(\{f,g\}) \cap \X_{m} \neq \emptyset$,
so that $Q \subseteq \{ \{f,g\} \subseteq \C : \DIS(\{f,g\}) \cap (\X_{m} \setminus \{x\}) \neq \emptyset \}$.
Therefore, since $|\X_{m} \setminus \{x\}| = m-1$, by definition of $\finitesplit_{m-1}$ we have
$\max_{x^{\prime} \in \X_{m}} \Split(Q,x^{\prime}) \geq \max_{x^{\prime} \in \X_{m} \setminus \{x\}} \Split(Q,x^{\prime}) \geq \finitesplit_{m-1} |Q|$.
Combined with the inductive hypothesis (for $m$), this implies $\max_{x^{\prime} \in \X_{m}} \Split(Q,x^{\prime}) \geq (1/\s) |Q|$.

Now consider the remaining case, in which $\forall x \in \X_{m}$, $\exists \{f_{x},g_{x}\} \in Q$ with $x \in \DIS(\{f_{x},g_{x}\})$.
Since $\{f_{x},g_{x}\} \notin Q_{x}^{y}$ for every $y \in \Y$ and $x \in \X_{m}$, we have $\max_{x \in \X_{m}} \Split(Q,x) \geq 1$.
We proceed by a kind of set-covering argument, as follows.
For each $x \in \X_{m}$, denote $y_{x} = \argmax_{y \in \Y} |Q_{x}^{y}|$ (breaking ties arbitrarily),
and denote $S_{x} = \{ x^{\prime} \in \X_{m} : \{f_{x},g_{x}\} \notin Q_{x^{\prime}}^{y_{x^{\prime}}} \}$.
Let $z_{1}$ be any element of $\X_{m}$.
Then, for integers $i \geq 2$, inductively define 
$z_{i}$ as any element of $\X_{m} \setminus \bigcup_{j = 1}^{i-1} S_{z_{j}}$,
up until the smallest index $i \in \nats$ for which $\X_{m} \setminus \bigcup_{j=1}^{i} S_{z_{i}} = \emptyset$;
denote by $I$ this smallest $i$ with $\X_{m} \setminus \bigcup_{j=1}^{i} S_{z_{i}} = \emptyset$.
Note that, since $\{f_{x},g_{x}\} \notin Q_{x}^{y_{x}}$ (and hence $x \in S_{x}$) for each $x \in \X_{m}$, 
every $z_{i}$ is distinct, which further implies that $I \leq m$ (and in particular, that $I$ exists).
Furthermore, since any $i \in \{1,\ldots,I\}$ and $x \in \X_{m}$ with $\{f_{x},g_{x}\} = \{f_{z_{i}},g_{z_{i}}\}$
have $S_{x} = S_{z_{i}}$, and therefore $x \in S_{z_{i}}$, $\nexists j > i$ with $z_{j} = x$.
Thus, we also have that $\{f_{z_{i}},g_{z_{i}}\} \neq \{f_{z_{j}},g_{z_{j}}\}$ for every $i,j \in \{1,\ldots,I\}$ with $i \neq j$.

Now let $i_{1} = I$, and for integers $k \geq 2$, inductively define
\begin{equation*}
i_{k} = \max\left\{ i \in \{1,\ldots,i_{k-1}-1\} : \left( S_{z_{i}} \setminus \bigcup_{j = 1}^{i-1} S_{z_{j}} \right) \setminus \bigcup_{j=1}^{k-1} S_{z_{i_{j}}} \neq \emptyset \right\},
\end{equation*}
up to the smallest index $k \in \nats$ with
$\left\{ i \in \{1,\ldots,i_{k}-1\} : \left( S_{z_{i}} \setminus \bigcup_{j = 1}^{i-1} S_{z_{j}} \right) \setminus \bigcup_{j=1}^{k} S_{z_{i_{j}}} \neq \emptyset \right\} = \emptyset$;
denote by $K$ this final value of $k$ (which must exist, since $i_{k+1} \in \nats$ is defined and strictly smaller than $i_{k}$ for any $k$ for which this set is nonempty; in particular, $1 \leq K \leq I$).
Finally, let $x_{1} = z_{i_{1}}$, and for each $k \in \{1,\ldots,K\}$,
let $x_{k}$ denote any element of $\left( S_{z_{i_{k}}} \setminus \bigcup_{j = 1}^{i_{k}-1} S_{z_{j}} \right) \setminus \bigcup_{j=1}^{k-1} S_{z_{i_{j}}}$,
which is nonempty by definition of $i_{k}$.

We first establish, by induction, that $\bigcup_{k=1}^{K} S_{z_{i_{k}}} = \X_{m}$.
By construction, we have $\bigcup_{i=1}^{I} S_{z_{i}} = \X_{m}$.
Furthermore, for any $i \in \{1,\ldots,I\}$, if $\bigcup_{j \leq i} S_{z_{j}} \cup \bigcup_{1 \leq k \leq K : i_{k} \geq i+1} S_{z_{i_{k}}} = \X_{m}$, 
then either $i \in \{i_{1},\ldots,i_{K}\}$, in which case 
$\bigcup_{j < i} S_{z_{j}} \cup \bigcup_{1 \leq k \leq K : i_{k} \geq i} S_{z_{i_{k}}} = \bigcup_{j \leq i} S_{z_{j}} \cup \bigcup_{1 \leq k \leq K : i_{k} \geq i+1} S_{z_{i_{k}}} = \X_{m}$, 
or else $i \notin \{i_{1},\ldots,i_{K}\}$, which (by definition of the $i_{k}$ sequence) implies $S_{z_{i}} \subseteq \bigcup_{j=1}^{i-1} S_{z_{j}} \cup \bigcup_{1 \leq k \leq K : i_{k} \geq i+1} S_{z_{i_{k}}}$,
so that 
$\bigcup_{j < i} S_{z_{j}} \cup \bigcup_{1 \leq k \leq K : i_{k} \geq i} S_{z_{i_{k}}} = \bigcup_{j < i} S_{z_{j}} \cup \bigcup_{1 \leq k \leq K : i_{k} \geq i+1} S_{z_{i_{k}}} = 
\bigcup_{j \leq i} S_{z_{j}} \cup \bigcup_{1 \leq k \leq K : i_{k} \geq i+1} S_{z_{i_{k}}} = \X_{m}$.
By induction, we have that $\bigcup_{k=1}^{K} S_{z_{i_{k}}} = \bigcup_{j < 1} S_{z_{j}} \cup \bigcup_{1 \leq k \leq K : i_{k} \geq 1} S_{z_{i_{k}}} = \X_{m}$.
In other words, $\forall x \in \X_{m}$, $\exists k(x) \in \{1,\ldots,K\}$ with $\{f_{z_{i_{k(x)}}},g_{z_{i_{k(x)}}}\} \notin Q_{x}^{y_{x}}$.

In particular, letting $R = Q \setminus \{ \{f_{z_{i_{k}}},g_{z_{i_{k}}}\} : k \in \{1,\ldots,K\} \}$, we have that $\forall x \in \X_{m}$, 
$\{f_{z_{i_{k(x)}}},g_{z_{i_{k(x)}}}\} \in (Q \setminus R) \setminus (Q_{x}^{y_{x}} \setminus R)$ while $Q_{x}^{y_{x}} \setminus R \subseteq Q \setminus R$,
so that $|Q \setminus R| - |Q_{x}^{y_{x}} \setminus R| \geq 1$.
Therefore, $\forall x \in \X_{m}$,
\begin{align}
\Split(R,x) 
& = |R| - \max_{y \in \Y} |R_{x}^{y}|
\leq |R| - |R_{x}^{y_{x}}|
= |R| - |R \cap Q_{x}^{y_{x}}| \notag
\\ & = (|Q| - |Q \setminus R|) - (|Q_{x}^{y_{x}}| - |Q_{x}^{y_{x}} \setminus R|)
= (|Q| - |Q_{x}^{y_{x}}|) - (|Q \setminus R| - |Q_{x}^{y_{x}} \setminus R|) \notag
\\ & \leq |Q| - |Q_{x}^{y_{x}}| - 1
= |Q| - \max_{y \in \Y} |Q_{x}^{y}| - 1
= \Split(Q,x) - 1. \label{eqn:split-reduced-by-one}
\end{align}
Since $K \geq 1$, we may note that $R$ is a strict subset of $Q$, so that the (nested) inductive hypothesis implies 
that $\max_{x \in \X_{m}} \Split(R,x) \geq (1/\s) |R|$.  Combined with \eqref{eqn:split-reduced-by-one}, this implies
\begin{equation}
\label{eqn:max-split-apply-IH}
\max_{x \in \X_{m}} \Split(Q,x) \geq \max_{x \in \X_{m}} \Split(R,x) + 1 \geq (1/\s) |R| + 1.
\end{equation}

Next, we argue that $K \leq \s$, by proving that $\{x_{1},\ldots,x_{K}\}$ is a star set for $\C$.
By definition of $z_{I}$, 
we have $z_{I} \in \X_{m} \setminus \bigcup_{j = 1}^{I-1} S_{z_{j}} \subseteq \X_{m} \setminus \bigcup_{k=2}^{K} S_{z_{i_{k}}}$.
Furthermore, $z_{I} \in S_{z_{I}}$, so that $z_{I} \in S_{z_{I}} \setminus \bigcup_{k=2}^{K} S_{z_{i_{k}}}$.
Since $x_{1} = z_{i_{1}} = z_{I}$, we have $x_{1} \in S_{z_{i_{1}}} \setminus \bigcup_{k=2}^{K} S_{z_{i_{k}}}$.
Also, for each $k \in \{2,\ldots,K\}$, by definition,
$x_{k} \in \left( S_{z_{i_{k}}} \setminus \bigcup_{j=1}^{i_{k}-1} S_{z_{j}} \right) \setminus \bigcup_{j=1}^{k-1} S_{z_{i_{j}}} 
\subseteq \left( S_{z_{i_{k}}} \setminus \bigcup_{j = k+1}^{K} S_{z_{i_{j}}} \right) \setminus \bigcup_{j=1}^{k-1} S_{z_{i_{j}}}
= S_{z_{i_{k}}} \setminus \bigcup_{1 \leq j \leq K : j \neq k} S_{z_{i_{j}}}$.
Therefore, every $k \in \{1,\ldots,K\}$ has $x_{k} \in S_{z_{i_{k}}} \setminus \bigcup_{1 \leq j \leq K : j \neq k} S_{z_{i_{j}}}$.
In particular, for every $k \in \{1,\ldots,K\}$, since $x_{k} \in S_{z_{i_{k}}}$, we have $\{f_{z_{i_{k}}},g_{z_{i_{k}}}\} \notin Q_{x_{k}}^{y_{x_{k}}}$, so that $\exists h_{k} \in \{f_{z_{i_{k}}},g_{z_{i_{k}}}\}$ with $h_{k}(x_{k}) \neq y_{x_{k}}$.
Furthermore, for every $j \in \{1,\ldots,K\} \setminus \{k\}$, since $x_{j} \notin S_{z_{i_{k}}}$, we have $\{f_{z_{i_{k}}},g_{z_{i_{k}}}\} \in Q_{x_{j}}^{y_{x_{j}}}$, so that $f_{z_{i_{k}}}(x_{j}) = g_{z_{i_{k}}}(x_{j}) = y_{x_{j}}$,
and in particular, $h_{k}(x_{j}) = y_{x_{j}}$.
Also, since we have chosen $x_{1} = z_{i_{1}}$, so that $x_{1} \in \DIS(\{f_{z_{i_{1}}},g_{z_{i_{1}}}\})$, $\exists h_{0} \in \{f_{z_{i_{1}}},g_{z_{i_{1}}}\}$ with $h_{0}(x_{1}) \neq h_{1}(x_{1})$: that is, $h_{0}(x_{1}) = y_{x_{1}}$.
Thus, since $f_{z_{i_{1}}}(x_{j}) = g_{z_{i_{1}}}(x_{j}) = y_{x_{j}}$ for every $j \in \{2,\ldots,K\}$, we have that $h_{0}(x_{k}) = y_{x_{k}}$ for every $k \in \{1,\ldots,K\}$.
Altogether, we have that every $k \in \{1,\ldots,K\}$ has $h_{k}(x_{k}) \neq h_{0}(x_{k})$, while every $j \in \{1,\ldots,K\} \setminus \{k\}$ has $h_{k}(x_{j}) = h_{0}(x_{j})$.
In other words, $\forall k \in \{1,\ldots,K\}$, $\DIS(\{h_{0},h_{k}\}) \cap \{x_{1},\ldots,x_{K}\} = \{x_{k}\}$:
that is, $\{x_{1},\ldots,x_{K}\}$ is a star set for $\C$, witnessed by $\{h_{0},h_{1},\ldots,h_{K}\}$.
In particular, this implies $K \leq \s$.  

Therefore, since $|Q \setminus R| = K$ (by distinctness of the pairs $\{f_{z_{i}},g_{z_{i}}\}$ argued above), \eqref{eqn:max-split-apply-IH} implies
\begin{equation*}
\max_{x \in \X_{m}} \Split(Q,x) 
\geq (1/\s) |R| + \frac{K}{\s}
= (1/\s) ( |R| + |Q \setminus R| )
= (1/\s) |Q|.
\end{equation*}
By the principle of induction (on $Q$), we have $\max_{x \in \X_{m}} \Split(Q,x) \geq (1/\s)|Q|$ for every finite set $Q \subseteq \{ \{f,g\} \subseteq \C : \DIS(\{f,g\}) \cap \X_{m} \neq \emptyset \}$.
Since this holds for any choice of $\X_{m}$ with $|\X_{m}|=m$, we have $\finitesplit_{m} \geq 1/\s$.
By the principle of induction (on $m$), we have established that $\finitesplit_{m} \geq 1/\s$ for every $m \in \nats$ with $m \leq |\X|$,
which completes the proof of the theorem.
\end{proof}

\subsection{The Teaching Dimension}
\label{app:xtd}

Here we give the proofs of results from Section~\ref{sec:xtd}.
We first prove that every minimal specifying set is a star set (Lemma~\ref{lem:spec-star-set}).
In fact, we establish a slightly stronger claim here (which also applies to local minima), stated formally as follows.

\begin{lemma}
\label{lem:local-min-spec-star-set}
Fix any $h : \X \to \Y$, $m \in \nats$, $\U \in \X^{m}$,
and any specifying set $S$ for $h$ on $\U$ with respect to $\C[\U]$.
If $\forall x \in S$, $S \setminus \{x\}$ is not a specifying set for $h$ on $\U$ with respect to $\C[\U]$,
then $S$ is a star set for $\C \cup \{h\}$ centered at $h$.
\end{lemma}
\begin{proof}
Fix an arbitrary sequence $\U = \{x_{1},\ldots,x_{m}\} \in \X^{m}$ and any $h : \X \to \Y$.
Let $t \geq \TD(h,\C[\U],\U)$, and let $i_{1},\ldots,i_{t} \in \{1,\ldots,m\}$ be such that
$S = \{x_{i_{1}},\ldots,x_{i_{t}}\}$ is a specifying set for $h$ on $\U$ with respect to $\C[\U]$.
First note that, if $\exists j \in \{1,\ldots,t\}$ such that
every $g \in V_{S \setminus \{x_{i_{j}}\},h}$ has $g(x_{i_{j}}) = h(x_{i_{j}})$ (which includes the case $V_{S \setminus \{x_{i_{j}}\},h} = \emptyset$), 
then $V_{S \setminus \{x_{i_{j}}\},h} = V_{S,h}$, 
so that $|V_{S \setminus \{x_{i_{j}}\},h} \cap \C[\U]| = |V_{S,h} \cap \C[\U]| \leq 1$;
thus, $S \setminus \{x_{i_{j}}\}$ is also a specifying set for $h$ on $\U$ with respect to $\C[\U]$.

Therefore, if $S$ is such that $\forall j \leq t$, $S \setminus \{x_{i_{j}}\}$ is \emph{not} a specifying set for $h$ on $\U$ with respect to $\C[\U]$,
then $\forall j \in \{1,\ldots,t\}$, $\exists h_{j} \in V_{S \setminus \{x_{i_{j}}\},h}$ with $h_{j}(x_{i_{j}}) \neq h(x_{i_{j}})$;
noting that ``$h_{j} \in V_{S \setminus \{x_{i_{j}}\},h}$'' is equivalent to saying ``$h_{j}(x_{i_{k}}) = h(x_{i_{k}})$ for every $k \in \{1,\ldots,t\}\setminus\{j\}$,''
this precisely matches the definition of a star set in Section~\ref{sec:star}:
that is, we have proven that $\{x_{i_{1}},\ldots,x_{i_{t}}\}$ is a star set for $\C \cup \{h\}$, witnessed by $\{h,h_{1},\ldots,h_{t}\}$, 
and hence centered at $h$.
\end{proof}

\begin{proof}[of Lemma~\ref{lem:spec-star-set}]
Lemma~\ref{lem:spec-star-set} follows immediately from Lemma~\ref{lem:local-min-spec-star-set}
by noting that, for any \emph{minimal} specifying set $S$ for $h$ on $\U$ with respect to $\C[\U]$, 
$\forall x \in S$, $|S \setminus \{x\}| < \TD(h,\C[\U],\U)$, so that $S \setminus \{x\}$ cannot 
possibly be a specifying set for $h$ on $\U$ with respect to $\C[\U]$.
\end{proof}

We are now ready for the proof of Theorem~\ref{thm:xtd}.

\begin{proof}[of Theorem~\ref{thm:xtd}]
Fix any $m \in \nats$.
First, note that for $\{x_{i}\}_{i=1}^{\s}$ and $\{h_{i}\}_{i=0}^{\s}$ as in Definition~\ref{def:star}, 
letting $\U = \{x_{1},\ldots,x_{\min\{\s,m\}}\}$, for any positive integer $i \leq \min\{\s,m\}$, 
any subsequence $S \subseteq \U$ with $x_{i} \notin S$ has $\{h_{0},h_{i}\} \subseteq V_{S,h_{0}}$.
Thus, since $x_{i} \in \U$, and $h_{0}(x_{i}) \neq h_{i}(x_{i})$, we have $|V_{S,h_{0}} \cap \C[\U]| \geq 2$.
Since this is true for every such $i \leq \min\{\s,m\}$, every $S \subseteq \U$ without $\{x_{1},\ldots,x_{\min\{\s,m\}}\} \subseteq S$
has $|V_{S,h_{0}} \cap \C[\U]| \geq 2$.  Therefore, $\TD(h_{0},\C[\U],\U) \geq \min\{\s,m\}$.
Thus, by the definitions of $\XTD$ and $\TD$, 
monotonicity of maximization in the set maximized over,
and monotonicity of $t \mapsto \TD(\C,t)$,\footnote{$\forall S \in \X^{t}$, $\forall x \in S$, $\forall h$, $\TD(h,\C[S\cup\{x\}],S\cup\{x\}) = \TD(h,\C[S],S)$.
Thus, $\TD(\C,t+1) = \max_{h \in \C} \max_{S \in \X^{t}} \max_{x \in \X} \TD(h,\C[S\cup\{x\}],S\cup\{x\}) 
\geq \max_{h \in \C} \max_{S \in \X^{t}} \max_{x \in S} \TD(h, \C[S\cup\{x\}],S\cup\{x\}) = \max_{h \in \C} \max_{S \in \X^{t}} \TD(h,\C[S],S) = \TD(\C,t)$.}
we have
\begin{equation*}
\XTD(\C,m) \geq \TD(\C,m) \geq \TD(\C,\min\{\s,m\}) \geq \TD(h_{0},\C[\U],\U) \geq \min\{\s,m\}.
\end{equation*}
Furthermore, it follows immediately from the definition that $\XTD(\C,m) \leq m$.
Note that this completes the proof in the case that $\s \geq m$.  To address the remaining case,
for the remainder of the proof, we suppose $\s \leq m$, and focus on establishing $\XTD(\C,m) \leq \s$.

For this, we proceed by induction on $m$, taking as a base case the fact that $\XTD(\C,\s) \leq \s$, which trivially follows from the definition of $\XTD$.
Now take as an inductive hypothesis that for some $m > \s$, we have $\XTD(\C,m-1) \leq \s$.
Fix any sequence $\U_{m} = \{x_{1},\ldots,x_{m}\} \in \X^{m}$, and $h : \X \to \Y$, 
and denote $\U_{m-1} = \{x_{1},\ldots,x_{m-1}\}$.
Let $t \in \nats \cup \{0\}$ and $S \in \U_{m-1}^{t}$ be such that
$S$ is a minimal specifying set for $h$ on $\U_{m-1}$ with respect to $\C[\U_{m-1}]$.
If $|S| \geq \TD(h,\C[\U_{m}],\U_{m})$, then since $S$ is a \emph{minimal} specifying set for $h$ on $\U_{m-1}$ with respect to $\C[\U_{m-1}]$, 
we have $|S| = \TD(h,\C[\U_{m-1}],\U_{m-1}) \leq \XTD(\C,m-1) \leq \s$ by the inductive hypothesis; 
thus, in this case we have $\TD(h,\C[\U_{m}],\U_{m}) \leq |S| \leq \s$.
On the other hand, suppose $|S| < \TD(h,\C[\U_{m}],\U_{m})$.
In this case, since $S$ is a specifying set for $h$ on $\U_{m-1}$ with respect to $\C[\U_{m-1}]$, 
we have $\DIS(V_{S,h}) \cap \U_{m} \subseteq (\DIS(V_{S,h}) \cap \U_{m-1}) \cup \{x_{m}\} = \{x_{m}\}$.  
But since $|S| < \TD(h,\C[\U_{m}],\U_{m})$, $S$ cannot be a specifying set for $h$ on $\U_{m}$ with respect to $\C[\U_{m}]$, 
so that $\DIS(V_{S,h}) \cap \U_{m} \neq \emptyset$.  Therefore, $\DIS(V_{S,h}) \cap \U_{m} = \{x_{m}\}$.
In particular, this implies that $S \cup \{x_{m}\}$ is a specifying set for $h$ on $\U_{m}$ with respect to $\C[\U_{m}]$,
and in particular, must be a \emph{minimal} such specifying set, since $|S\cup\{x_{m}\}| = |S|+1 \leq \TD(h,\C[\U_{m}],\U_{m})$.
Therefore, Lemma~\ref{lem:spec-star-set} implies that $S\cup\{x_{m}\}$ is a star set for $\C \cup \{h\}$ centered at $h$.
If $h \in \C$, this already implies that $|S \cup \{x_{m}\}| \leq \s$; 
furthermore, we can argue that this remains the case even if $h \notin \C$, as follows.
Since $x_{m} \in \DIS(V_{S,h})$, we have $V_{S\cup\{x_{m}\},h} \neq \emptyset$, so that $\exists g_{0} \in \C$ 
such that $\forall x \in S \cup \{x_{m}\}$, $g_{0}(x) = h(x)$.
Therefore, $S \cup \{x_{m}\}$ is also a star set for $\C$ centered at $g_{0}$, so that $|S \cup \{x_{m}\}| \leq \s$.
In particular, since $S \cup \{x_{m}\}$ is a minimal specifying set for $h$ on $\U_{m}$ with respect to $\C[\U_{m}]$, 
we have $|S \cup \{x_{m}\}| = \TD(h,\C[\U_{m}],\U_{m})$, so that $\TD(h,\C[\U_{m}],\U_{m}) \leq \s$ in this case as well.
Thus, in either case, we have $\TD(h,\C[\U_{m}],\U_{m}) \leq \s$.
Maximizing over the choice of $h$ and $\{x_{1},\ldots,x_{m}\}$, we have $\XTD(\C,m) \leq \s$,
which completes the inductive step.
The result now follows by the principle of induction.
\end{proof}

Next, we prove Theorem~\ref{thm:xptd}.

\begin{proof}[of Theorem~\ref{thm:xptd}]
Fix any $m \in \nats$ and $\conf \in [0,1]$.
Let $\{x_{i}\}_{i=1}^{\s}$ and $\{h_{i}\}_{i=0}^{\s}$ be as in Definition~\ref{def:star}, 
and let $\U = \{x_{1},\ldots,x_{\min\{\s,m\}}\}$ and $\G = \{h_{i} : i \in \{0,\ldots,\min\{\s,m\}\}$.  
As in the proof of Theorem~\ref{thm:xtd}, for any positive integer $i \leq \min\{\s,m\}$, 
any subsequence $S \subseteq \U$ with $x_{i} \notin S$ has $\{h_{0},h_{i}\} \subseteq V_{S,h_{0}}$.
Thus, since $x_{i} \in \U$ for every $i \leq \min\{\s,m\}$, and every $h_{i}$ realizes a distinct classification of $\U$ ($i \leq \min\{\s,m\}$),
we have $|V_{S,h_{0}} \cap \G[\U]| \geq |\{i \in \{1,\ldots,\min\{\s,m\} \} : x_{i} \notin S\}| + 1 \geq \min\{\s,m\}-|S|+1$. 
In particular, to have $|V_{S,h_{0}} \cap \G[\U]| \leq \conf |\G[\U]|+1 = \conf (\min\{\s,m\}+1) + 1$, 
we must have $|S| \geq (1-\conf)\min\{\s,m\} - \conf$.
Therefore, $\XPTD(h_{0},\G[\U],\U,\conf) \geq (1-\conf) \min\{\s,m\}-\conf$.
By definition of $\XPTD(\H,m,\conf)$ and the fact that $\G \subseteq \C$, 
and since $t \mapsto \XPTD(\H,t,\conf)$ is nondecreasing
(since $\forall S \in \X^{t}$, $\forall x \in S$, $\forall h$, $\XPTD(h,\H[S\cup\{x\}],S\cup\{x\},\conf) = \XPTD(h,\H[S],S,\conf)$),
this further implies 
\begin{align*}
\max_{\H \subseteq \C} \XPTD(\H,m,\conf) 
& \geq \XPTD(\G,m,\conf) 
\geq \XPTD(\G,\min\{\s,m\},\conf) 
\\ & \geq \XPTD(h_{0},\G[\U],\U,\conf) 
\geq (1-\conf) \min\{\s,m\} - \conf 
\geq (1\!-\!2\conf)\min\{\s,m\},
\end{align*}
where this last inequality is due to the assumption that $|\C| \geq 3$ (Section~\ref{sec:definitions}), which implies $\s \geq 1$.
Since $\XPTD(\cdot,m,\conf) \in \nats \cup \{0\}$, this further implies $\max_{\H \subseteq \C} \XPTD(\H,m,\conf) \geq \lceil (1-2\conf) \min\{\s,m\} \rceil$
when $\conf \leq 1/2$.

To establish the right inequality, 
fix any $\H \subseteq \C$,
let $\U \in \X^{m}$ and $h : \X \to \Y$ be such that $\XPTD(h,\H[\U],\U,\conf) = \XPTD(\H,m,\conf)$,
and let $S \subseteq \U$ be a minimal specifying set for $h$ on $\U$ with respect to $\H[\U]$.
If $\conf = 0$ or $|S| < \frac{1+\conf}{\conf}$, then $|S|-1 < \left( 1 - \frac{\conf}{1+\conf} \right) |S| \leq |S|$,
so that $\XPTD(h,\H[\U],\U,\conf) \leq |S| = \left\lceil \left( 1 - \frac{\conf}{1+\conf} \right) |S| \right\rceil$.
Otherwise, suppose $\conf > 0$ and $|S| \geq \frac{1+\conf}{\conf}$, and 
let $k = \left\lfloor |S| / \left\lfloor \frac{\conf}{1+\conf} |S| \right\rfloor \right\rfloor$, and note that $k \geq 1$.
Let $R_{1},\ldots,R_{k}$ denote disjoint subsequences of $S$ with each $|R_{i}| = \left\lfloor \frac{\conf}{1+\conf} |S| \right\rfloor$,
which must exist since minimality of $S$ guarantees that its elements are distinct.
Note that, for each $i \in \{1,\ldots,k\}$, $(V_{S \setminus R_{i}, h} \setminus V_{S,h}) \cap \H[\U]$ is the set of 
classifiers $g$ in $\H[\U]$ with 
$\DIS(\{g,h\}) \cap (S \setminus R_{i}) = \emptyset$ but $\DIS(\{g,h\}) \cap R_{i} \neq \emptyset$;
in particular, for any $i,j \in \{1,\ldots,k\}$ with $i \neq j$, 
since $R_{j} \subseteq S \setminus R_{i}$ and $R_{i} \subseteq S \setminus R_{j}$,
$(V_{S \setminus R_{i}, h} \setminus V_{S,h}) \cap \H[\U]$ and $(V_{S \setminus R_{j},h} \setminus V_{S,h}) \cap \H[\U]$ are disjoint.
Thus, since $\H[\U] \supseteq (V_{S,h} \cap \H[\U]) \cup \bigcup_{i=1}^{k} (V_{S \setminus R_{i},h} \setminus V_{S,h}) \cap \H[\U]$,
we have
\begin{align*}
|\H[\U]| 
& \geq \left| (V_{S,h} \cap \H[\U]) \cup \bigcup_{i=1}^{k} (V_{S \setminus R_{i},h} \setminus V_{S,h}) \cap \H[\U] \right| &&
\\ & = |V_{S,h} \cap \H[\U]| + \sum_{i=1}^{k} \left| (V_{S \setminus R_{i},h} \setminus V_{S,h}) \cap \H[\U] \right|
&& \!\!\!\!\!\!\geq \sum_{i=1}^{k} \left| (V_{S \setminus R_{i},h} \setminus V_{S,h}) \cap \H[\U] \right|
\\ &&& \!\!\!\!\!\!\geq k \min_{i \in \{1,\ldots,k\}} \left| (V_{S \setminus R_{i},h} \setminus V_{S,h}) \cap \H[\U] \right|.
\end{align*}
Thus, letting $i^{*} = \argmin_{i \in \{1,\ldots,k\}} \left| (V_{S \setminus R_{i},h} \setminus V_{S,h}) \cap \H[\U] \right|$, 
we have $\left| (V_{S \setminus R_{i^{*}},h} \setminus V_{S,h}) \cap \H[\U] \right|$ $\leq \frac{1}{k} |\H[\U]|$.
Furthermore, since $S$ is a specifying set for $h$ on $\U$ with respect to $\H[\U]$, 
$|V_{S,h} \cap \H[\U]| \leq 1$, so that (since $V_{S,h} \subseteq V_{S \setminus R_{i^{*}},h}$)
\begin{align*}
\left| V_{S \setminus R_{i^{*}},h} \cap \H[\U] \right| 
& = \left| \left( \left(V_{S \setminus R_{i^{*}},h} \setminus V_{S,h}\right) \cap \H[\U]\right) \cup \left( V_{S,h} \cap \H[\U] \right) \right|
\\ & = \left| \left(V_{S \setminus R_{i^{*}},h} \setminus V_{S,h}\right) \cap \H[\U] \right| + \left| V_{S,h} \cap \H[\U] \right|
\leq \frac{1}{k} |\H[\U]| + 1.
\end{align*}
Also, since 
\begin{equation*}
\frac{1}{k} 
\leq \frac{1}{ \left\lfloor \frac{1+\conf}{\conf} \right\rfloor }
\leq \frac{1}{ \frac{1+\conf}{\conf} - 1 }
= \conf,
\end{equation*}
this implies $| V_{S \setminus R_{i^{*}},h} \cap \H[\U] | \leq \conf |\H[\U]| + 1$,
so that $\XPTD(h,\H[\U],\U,\conf) \leq |S \setminus R_{i^{*}}|$.
Furthermore, since $R_{i^{*}} \subseteq S$, 
$|S \setminus R_{i^{*}}| = |S| - |R_{i^{*}}| = |S| - \left\lfloor \frac{\conf}{1+\conf} |S| \right\rfloor = \left\lceil \left( 1 - \frac{\conf}{1+\conf} \right) |S| \right\rceil$.

Thus, for any $\conf \in [0,1]$ and regardless of the size of $|S|$, we have
$\XPTD(h,\H[\U],\U,\conf) \leq \left\lceil \left( 1 - \frac{\conf}{1+\conf} \right) |S| \right\rceil$.
Furthermore, since $S$ is a minimal specifying set for $h$ on $\U$ with respect to $\H[\U]$, 
we have $|S| \leq \XTD(\H,m) \leq \XTD(\C,m)$, and Theorem~\ref{thm:xtd} implies $\XTD(\C,m) = \min\{\s,m\}$.
Therefore, $\XPTD(h,\H[\U],\U,\conf) \leq \left\lceil \left( 1 - \frac{\conf}{1+\conf} \right) \min\{\s,m\} \right\rceil$.
Maximizing the left hand side over the choice of $h$, $\H$, and $\U$ completes the proof.
\end{proof}

\subsection{The Doubling Dimension}
\label{app:doubling}

We now present the proof of Theorem~\ref{thm:dd-star}

\begin{proof}[of Theorem~\ref{thm:dd-star}]
For the lower bound, fix any $\eps \in (0,1]$, and
take $\{x_{i}\}_{i=1}^{\s}$ and $\{h_{i}\}_{i=0}^{\s}$ as in Definition~\ref{def:star},
and let $m = \s \land \left\lfloor \frac{1}{\eps} \right\rfloor$.  Let $\Px$ be a probability measure on $\X$
with $\Px(\{x_{i}\}) = 1/m$ for each $i \in \{1,\ldots,m\}$.  Thus, $\{h_{0},h_{1},\ldots,h_{m}\} \subseteq \Ball_{\Px}(h_{0},1/m)$.
Furthermore, for any $i \in \{0,\ldots,m\}$ and any classifier $g$ with $\Px(x : g(x) \neq h_{i}(x)) \leq 1/(2m)$,
we must have $g(x_{j}) = h_{i}(x_{j})$ for every $j \in \{1,\ldots,m\}$.
Therefore, any $\frac{1}{2m}$-cover of $\Ball_{\Px}(h_{0},1/m)$ must contain classifiers 
$g_{0},\ldots,g_{m}$ with $\forall i \in \{0,\ldots,m\}$, $\forall j \in \{1,\ldots,m\}$, $g_{i}(x_{j}) = h_{i}(x_{j})$.
Thus, since each $h_{i}$ (with $i \leq m$) realizes a distinct classification of $\{x_{1},\ldots,x_{m}\}$, 
it follows that $\covering(1/(2m),\Ball_{\Px}(h_{0},1/m),\Px) \geq m+1$.
Noting that $1/m \geq \eps$, we have that
\begin{equation*}
\sup_{P} \sup_{h \in \C} \dd_{h,P}(\eps) \!\geq\! \dd_{h_{0},\Px}(\eps) \!\geq\! \log_{2}\!\left( \covering\!\left( \frac{1}{2m}, \Ball_{\Px}\!\left(h_{0},\frac{1}{m}\right)\!, \Px \right) \right)\! \geq\! \log_{2}(m+1) \!\geq\! \log_{2}\!\left( \s \land \frac{1}{\eps} \right)\!.
\end{equation*}

For the remaining term in the lower bound (i.e., $\vc$), we modify an argument of \citet*[][Proposition 3]{kulkarni:89}.
If $\vc < 5$, then $\vc \lesssim \Log\left( \s \land \frac{1}{\eps} \right)$, so that the lower bound follows from the above.
Otherwise, suppose $\vc \geq 5$.
We first let $\{x_{1}^{\prime},\ldots,x_{\vc}^{\prime}\}$ denote a set of $\vc$ points in $\X$ shattered by $\C$,
and we let $G$ denote the set of classifiers $g \in \C[ \{x_{1}^{\prime},\ldots,x_{\vc}^{\prime}\} ]$ with $g(x_{\vc}^{\prime}) = -1$ 
and $\sum_{i=1}^{\vc-1} \ind[ g(x_{i}^{\prime}) = +1 ] = \left\lfloor \frac{\vc-1}{4} \right\rfloor$.
For any $g \in G$, note that, if $H$ is a classifier sampled uniformly at random from $G$, a Chernoff bound (for sampling without replacement) implies 
\begin{equation*}
\P\left( \sum_{i=1}^{\vc-1} \ind[ H(x_{i}^{\prime}) = g(x_{i}^{\prime}) ] \geq \frac{\vc-1}{8} \right) 
\leq \exp\left\{ - \frac{\vc-1}{48} \right\}.
\end{equation*}
Thus, there are at most $|G| \exp\left\{ - \frac{\vc-1}{48} \right\}$
elements $h \in G$ with $\sum_{i=1}^{\vc-1} \ind[ h(x_{i}^{\prime}) = g(x_{i}^{\prime}) ] \geq \frac{\vc-1}{8}$.
Now take $\H_{0} = \{\}$,
and take as an inductive hypothesis that,
for some positive integer $k < 1+ \exp\left\{ \frac{\vc-1}{48} \right\}$,
there is a set $\H_{k-1} \subseteq G$ with $|\H_{k-1}| = k-1$ such that $\forall h,g \in \H_{k-1}$ with $h \neq g$, $\sum_{i=1}^{\vc-1} \ind[ h(x_{i}^{\prime}) = g(x_{i}^{\prime}) ] < \frac{\vc-1}{8}$.
Since $|\H_{k-1}| \cdot |G| \exp\left\{ - \frac{\vc-1}{48} \right\} < |G|$, $\exists g_{k} \in G$ such that $\forall h \in \H_{k-1}$, $\sum_{i=1}^{\vc-1} \ind[ h(x_{i}^{\prime}) = g_{k}(x_{i}^{\prime}) ] < \frac{\vc-1}{8}$.
Thus, defining $\H_{k} = \H_{k-1} \cup \{g_{k}\}$ extends the inductive hypothesis.
By induction, this establishes the existence of a set $\H \subseteq G$ with $|\H| \geq \exp\left\{ \frac{\vc-1}{48} \right\}$
such that $\forall h,g \in \H$ with $h \neq g$, $\sum_{i=1}^{\vc-1} \ind[ h(x_{i}^{\prime}) = g(x_{i}^{\prime}) ] < \frac{\vc-1}{8}$.
Fix any $\eps \in (0,1/4]$ and let $\Px$ denote a probability measure over $\X$ with $\Px(\{x_{i}^{\prime}\}) = \frac{4\eps}{\vc-1}$ for each $i \in \{1,\ldots,\vc-1\}$, and $\Px(\{x_{\vc}^{\prime}\}) = 1 - 4\eps$.
Note that any $h,g \in G$ with $\sum_{i=1}^{\vc-1} \ind[h(x_{i}^{\prime}) = g(x_{i}^{\prime})] < \frac{\vc-1}{8}$ have $\Px(x : h(x) \neq g(x)) > \frac{\vc-1}{4} \frac{4\eps}{\vc-1} = \eps$.
Thus, $\H$ is an \emph{$\eps$-packing} under the $L_{1}(\Px)$ pseudometric.  
Recall that this implies $|\H| \leq \covering(\eps/2, G, \Px)$ \citep*{kolmogorov:59,kolmogorov:61}.
Furthermore, note that any $g \in G$ has $\Px(x : g(x) = +1) = \left\lfloor \frac{\vc-1}{4} \right\rfloor \frac{4\eps}{\vc-1} \leq \eps$.
Thus, letting $h_{-} \in \C$ be such that $\forall i \in \{1,\ldots,\vc\}$, $h_{-}(x_{i}^{\prime}) = -1$ (which exists, by shatterability of $x_{1}^{\prime},\ldots,x_{\vc}^{\prime}$), 
we have $G \subseteq \Ball_{\Px}(h_{-},\eps)$.  Therefore, $\covering(\eps/2,G,\Px) \leq \covering(\eps/2, \Ball_{\Px}(h_{-},\eps), \Px)$.
Altogether, we have that 
\begin{equation*}
\vc \lesssim \frac{\vc-1}{48} \log_{2}(e) \leq \log_{2}( |\H| ) \leq \log_{2}\left( \covering(\eps/2, \Ball_{\Px}(h_{-},\eps), \Px) \right) \leq \dd_{h_{-},\Px}(\eps) \leq \sup_{P} \sup_{h \in \C} \dd_{h,P}(\eps).
\end{equation*}

For the upper bound, fix any $h \in \C$, any probability measure $\Px$ over $\X$, and any $\eps \in (0,1]$,
and fix any value $r \in [\eps,1]$.
Recall that any maximal subset $G_{r} \subseteq \Ball_{\Px}(h,r)$ of classifiers in $\Ball_{\Px}(h,r)$ with $\min_{f,g \in G_{r} : f \neq g} \Px(x : f(x) \neq g(x)) > r/2$
(called a maximal $(r/2)$-packing of $\Ball_{\Px}(h,r)$) is also an $(r/2)$-cover of $\Ball_{\Px}(h,r)$ \citep*[see e.g.,][]{kolmogorov:59,kolmogorov:61}.
Thus, 
we have that 
$\covering\left(\frac{r}{2},\Ball_{\Px}(h,r),\Px\right) \leq |G_{r}|$, for any such set $G_{r}$.
Let $m = \left\lceil \frac{4}{r} \ln( |G_{r}| ) \right\rceil$, and let $X_{1},X_{2},\ldots,X_{m}$ be independent $\Px$-distributed random variables.
Let $E_{1}$ denote the event that $\forall f,g \in G_{r}$ with $f \neq g$, $\exists i \in \{1,\ldots,m\}$ with $f(X_{i}) \neq g(X_{i})$.
For any $f,g \in G_{r}$ with $f \neq g$, 
$\P( \exists i \in \{1,\ldots,m\} : f(X_{i}) \neq g(X_{i}) ) 
= 1 - (1-\Px(x : f(x) \neq g(x)))^{m} 
> 1 - (1-r/2)^{m} > 1 - e^{-m r / 2} \geq 1 - 1/|G_{r}|^{2}$.
Therefore, by a union bound, $\P(E_{1}) > 1 - \binom{|G_{r}|}{2} \frac{1}{|G_{r}|^{2}} \geq \frac{1}{2}$.
In particular, note that on the event $E_{1}$, the elements of $G_{r}$ realize distinct classifications
of the sequence $(X_{1},\ldots,X_{m})$, so that (since $G_{r} \subseteq \Ball_{\Px}(h,r)$) 
$|G_{r}|$ is upper bounded by the number of distinct classifications of $(X_{1},\ldots,X_{m})$ 
realized by classifiers in $\Ball_{\Px}(h,r)$.
Furthermore, since all classifiers in $\Ball_{\Px}(h,r)$ agree on the classification of any points $X_{i} \notin \DIS(\Ball_{\Px}(h,r))$,
and $\Ball_{\Px}(h,r) \subseteq \C$,
we have that $|G_{r}|$ is upper bounded by the number of distinct classifications of $\{X_{1},\ldots,X_{m}\} \cap \DIS(\Ball_{\Px}(h,r))$
realized by classifiers in $\C$.

By a Chernoff bound, on an event $E_{2}$ of probability at least $1/2$, 
\begin{equation*}
|\{ X_{1},\ldots,X_{m} \} \cap \DIS(\Ball_{\Px}(h,r))| \leq 1 + 2 e \Px(\DIS(\Ball_{\Px}(h,r))) m.
\end{equation*}
By the definition of the disagreement coefficient, this is at most $1 + 2 e \dc_{h,\Px}(r) r m \leq 1 + 2 e+ 8 e \dc_{h,\Px}(r) \ln( |G_{r}| )$,
which, if $|G_{r}| \geq 3$, is at most $11 e \dc_{h,\Px}(r) \ln( |G_{r}| )$.
By a union bound, the
event $E_{1} \cap E_{2}$ has probability strictly greater than $0$. 
Thus, letting $m^{\prime} = \left\lceil 11 e \dc_{h,\Px}(r) \ln(|G_{r}|) \right\rceil$, there exists a sequence 
$x_{1},\ldots,x_{m^{\prime}} \in \X$ such that $|G_{r}|$ is at most the max of $2$ and the number of 
distinct classifications of $\{x_{1},\ldots,x_{m^{\prime}}\}$ realized by classifiers in $\C$.
In the case $|G_{r}| \geq 3$, this latter value 
is at most $\left( \frac{e m^{\prime}}{\vc} \right)^{\vc} \leq \left( \frac{22 e^{2} \dc_{h,\Px}(r) \ln(|G_{r}|)}{\vc} \right)^{\vc}$
by the VC-Sauer lemma \citep*{vapnik:71,sauer:72}.

Taking the logarithm, we have that
\begin{equation*}
\ln(|G_{r}|) \leq \max\left\{ \ln(2), \vc \ln\left( 22 e^{2} \dc_{h,\Px}(r) \right) + \vc \ln\left( \frac{\ln(|G_{r}|)}{\vc} \right) \right\},
\end{equation*}
which implies \citep*[see e.g.,][Corollary 4.1]{vidyasagar:03}
\begin{equation*}
\ln(|G_{r}|) 
< \max\left\{ 1, 2 \vc \ln\left(22 e^{2} \dc_{h,\Px}(r)\right)\right\} 
= 2 \vc \ln\left( 22 e^{2} \dc_{h,\Px}(r) \right).
\end{equation*}
Dividing both sides by $\ln(2)$, altogether we have that
\begin{align*}
\dd_{h,\Px}(\eps) & = \sup_{r \in [\eps,1]} \log_{2}\left( \covering\left( \frac{r}{2}, \Ball_{\Px}(h,r), \Px \right) \right)
\leq \sup_{r \in [\eps,1]} \log_{2}\left( |G_{r}| \right)
\\ & \leq \sup_{r \in [\eps,1]} 2 \vc \log_{2}\left( 22 e^{2} \dc_{h,\Px}(r) \right)
= 2 \vc \log_{2}\left( 22 e^{2} \dc_{h,\Px}(\eps) \right). 
\end{align*}
In particular, by Theorem~\ref{thm:dc-star}, this is at most $2 \vc \log_{2}\left( 22 e^{2} \left( \s \land \frac{1}{\eps} \right) \right)$,
so that maximizing the left hand side over the choice of $h \in \C$ and $\Px$ completes the proof.
\end{proof}

\section{Examples Spanning the Gaps} 
\label{app:gaps}

In this section, taking $\vc$ and $\s$ as fixed values in $\nats$ (with $\vc \geq 3$ and $\s \geq 4\vc$),
and taking $\X = \nats$,
we establish that the upper bounds in Theorems~\ref{thm:realizable}, \ref{thm:bounded}, \ref{thm:tsybakov}, and \ref{thm:benign}
are all tight (up to universal constant and logarithmic factors) when we take $\C = \{ x \mapsto 2\ind_{S}(x)-1 : S \subseteq \{1,\ldots,\s\}, |S| \leq \vc \}$,
and that the lower bounds in these theorems are all tight (up to logarithmic factors) when we take 
$\C = \{ x \mapsto 2\ind_{S}(x) - 1 : S \in 2^{\{1,\ldots,\vc\}} \cup \{ \{i\} : \vc+1 \leq i \leq \s \} \}$.
One can easily verify that, in both cases, the VC dimension is indeed $\vc$, and the star number is indeed $\s$.

\subsection{The Upper Bounds are Sometimes Tight}
\label{app:ub-tight}

We begin with the upper bounds.
In this case, 
take 
\begin{equation}
\label{eqn:C-spec-ub-tight}
\C = \{ x \mapsto 2\ind_{S}(x)-1 : S \subseteq \{1,\ldots,\s\}, |S| \leq \vc \}.
\end{equation}
For this hypothesis class, we argue that the lower bounds can be increased to match the upper bounds (up to logarithmic factors).
We begin with a general lemma.

For each $i \in \{1,\ldots,\vc\}$, let $\X_{i} = \left\{\lfloor \s / \vc \rfloor (i-1) + 1, \ldots, \lfloor \s / \vc \rfloor i \right\}$,
$\C_{i} = \{ x \mapsto 2 \ind_{\{t\}}(x) - 1 : t \in \X_{i} \} \cup \{ x \mapsto -1 \}$,
and let $\Dset_{i}$ be a finite nonempty set of probability measures $P_{i}$ on $\X \times \Y$ such that $P_{i}(\X_{i} \times \Y) = 1$ (i.e., with marginal over $\X$ supported only on $\X_{i}$).  
Let $\Dset = \left\{ \frac{1}{\vc} \sum_{i=1}^{\vc} P_{i} : \forall i \in \{1,\ldots,\vc\}, P_{i} \in \Dset_{i} \right\}$.
Note that for any choices of $P_{i} \in \Dset_{i}$ for each $i \in \{1,\ldots,\vc\}$, 
letting $P = \frac{1}{\vc} \sum_{i=1}^{\vc} P_{i}$, 
we have that $\forall i \in \{1,\ldots,\vc\}$, $\forall x \in \X_{i}$ with $P_{i}(\{x\} \times \Y) > 0$, 
\begin{align*}
P(\{(x,+1)\} | \{x\} \times \Y) 
& = \frac{P( \{ (x,+1) \} )}{P(\{x\} \times \Y)} 
= \frac{ \frac{1}{\vc} \sum_{j=1}^{\vc} P_{j}( \{ (x,+1) \} ) }{ \frac{1}{\vc} \sum_{j=1}^{\vc} P_{j}( \{x\} \times \Y ) } 
\\ & = \frac{P_{i}( \{ (x,+1) \} )}{P_{i}(\{x\} \times \Y)} 
= P_{i}(\{(x,+1)\} | \{x\} \times \Y),
\end{align*}
so that the conditional distribution of $Y$ given $X=x$ (for $(X,Y) \sim P$) is specified by the conditional of $Y^{\prime}$ given $X^{\prime} = x$ for $(X^{\prime},Y^{\prime}) \sim P_{i}$, for the value $i$ with $x \in \X_{i}$.
Furthermore, since any $x \in \X_{i}$ has $P(\{x\} \times \Y) = 0$ if and only if $P_{i}(\{x\} \times \Y) = 0$, without loss we may define $P(\{(x,+1)\}| \{x\} \times \Y) = P_{i}(\{(x,+1)\} | \{x\} \times \Y)$ for any such $x$.
For each $i \in \{1,\ldots,\vc\}$ and $\eps,\conf \in (0,1)$, let $\LC_{i}(\eps,\conf)$ denote the minimax label complexity under $\Dset_{i}$ with respect to $\C_{i}$
(i.e., the value of $\LC_{\Dset_{i}}(\eps,\conf)$ when $\C = \C_{i}$).
The value $\LC_{\Dset}(\eps,\conf)$ remains defined as usual (i.e., with respect to the set $\C$ specified in \eqref{eqn:C-spec-ub-tight}).

\begin{lemma}
\label{lem:general-gaps-reduction}
Fix any $\gamma \in (2/\vc,1)$, $\eps \in (0,\gamma/4)$, and $\conf \in \left(0,\frac{\gamma}{4-\gamma}\right)$.
If $\min\limits_{i \in \{1,\ldots,\vc\}} \LC_{i}((4/\gamma)\eps,\gamma) \geq 2$, then 
\begin{equation*}
\LC_{\Dset}(\eps,\conf) \geq (\gamma/4) \vc \min_{i \in \{1,\ldots,\vc\}} \LC_{i}((4/\gamma)\eps,\gamma).
\end{equation*}
\end{lemma}
\begin{proof}
Fix any $n \in \nats$ with $n < (\gamma/4) \vc \min_{i \in \{1,\ldots,\vc\}} \LC_{i}((4/\gamma)\eps,\gamma)$.
Denote $n^{\prime} = \left\lceil \frac{n}{(\gamma/2) \vc} \right\rceil$, and note that $n^{\prime} \leq n$ and $n^{\prime} < \min_{i \in \{1,\ldots,\vc\}} \LC_{i}((4/\gamma)\eps,\gamma)$.
For each $i \in \{1,\ldots,\vc\}$, let $P_{i} \in \Dset_{i}$, and denote $g_{i}^{*} = \argmin_{g \in \C_{i}} \er_{P_{i}}(g)$ (breaking ties arbitrarily).
We will later optimize over the choice of these $P_{i}$.
Also let $g^{*} = \sum_{i=1}^{\vc} g_{i}^{*} \ind_{\X_{i}}$, the classifier that predicts with $g_{i}^{*}$ on each respective $\X_{i}$ set;
note that, since each $g_{i}^{*}$ classifies at most one point as $+1$, we have $g^{*} \in \C$.
Denote $P = \frac{1}{\vc} \sum_{i=1}^{\vc} P_{i}$.
Let $\hat{h}_{P}$ denote the (random) classifier produced by $\alg(n)$ when $\PXY = P$.
Note that if $\sum_{i=1}^{\vc} \ind\left[ \er_{P_{i}}\left(\hat{h}_{P}\right) - \er_{P_{i}}\left(g_{i}^{*}\right) > (4/\gamma) \eps \right] > (\gamma/4) \vc$, 
then 
\begin{align*}
\er_{P}\left(\hat{h}_{P}\right) - \inf_{h \in \C} \er_{P}(h)
& = \frac{1}{\vc} \sum_{i=1}^{\vc} \er_{P_{i}}\left(\hat{h}_{P}\right) - \inf_{h \in \C} \frac{1}{\vc} \sum_{i=1}^{\vc} \er_{P_{i}}(h)
\\ & \geq \frac{1}{\vc} \sum_{i=1}^{\vc} \er_{P_{i}}\left(\hat{h}_{P}\right) - \frac{1}{\vc} \sum_{i=1}^{\vc} \er_{P_{i}}\left( g^{*} \right)
= \frac{1}{\vc} \sum_{i=1}^{\vc} \left( \er_{P_{i}}\left(\hat{h}_{P}\right) - \er_{P_{i}}\left( g_{i}^{*} \right) \right)
\\ & \geq \frac{1}{\vc} \sum_{i=1}^{\vc} \ind\left[ \er_{P_{i}}\left(\hat{h}_{P}\right) - \er_{P_{i}}\left( g_{i}^{*} \right) > (4/\gamma) \eps \right] (4/\gamma) \eps 
> \eps.
\end{align*}
Therefore, 
\begin{align}
& \P\left( \er_{P}\left(\hat{h}_{P}\right) - \inf_{h \in \C} \er_{P}(h) > \eps \right) 
\geq \P\left( \sum_{i=1}^{\vc} \ind\left[ \er_{P_{i}}\left( \hat{h}_{P} \right) - \er_{P_{i}}\left( g_{i}^{*} \right) > (4/\gamma) \eps \right] > (\gamma/4) \vc \right) \notag
\\ & = 1 - \P\left( \sum_{i=1}^{\vc} \ind\left[ \er_{P_{i}}\left( \hat{h}_{P} \right) - \er_{P_{i}}\left( g_{i}^{*} \right) > (4/\gamma) \eps \right] \leq (\gamma/4) \vc \right) \notag
\\ & = 1 - \P\left( \sum_{i=1}^{\vc} \left(1-\ind\left[ \er_{P_{i}}\left( \hat{h}_{P} \right) - \er_{P_{i}}\left( g_{i}^{*} \right) > (4/\gamma) \eps \right]\right) \geq  (1-\gamma/4) \vc \right) \notag
\\ & \geq 1 - \frac{1}{(1-\gamma/4) \vc} \sum_{i=1}^{\vc} \left(1 - \P\left( \er_{P_{i}}\left( \hat{h}_{P} \right) - \er_{P_{i}}\left( g_{i}^{*} \right) > (4/\gamma) \eps \right) \right) \notag
\\ & = - \frac{\gamma}{4-\gamma} + \frac{4}{4-\gamma} \frac{1}{\vc} \sum_{i=1}^{\vc} \P\left( \er_{P_{i}}\left( \hat{h}_{P} \right) - \er_{P_{i}}\left( g_{i}^{*} \right) > (4/\gamma) \eps \right), \label{eqn:ub-tightness-avg-step}
\end{align}
where the second inequality is due to Markov's inequality and linearity of expectations.

Now note that there is a simple reduction from the problem of learning with $\C_{i}$ under $P_{i}$ to the problem of learning with $\C$ under $P$.
Specifically, for a given i.i.d. $P_{i}$-distributed sequence $(X_{i1},Y_{i1}),(X_{i2},Y_{i2}),\ldots$, 
we can construct an i.i.d. $P$-distributed sequence $(X_{1}^{\prime},Y_{1}^{\prime}),(X_{2}^{\prime},Y_{2}^{\prime}),\ldots$ as follows.
For each $j \in \{1,\ldots,\vc\} \setminus \{i\}$, let $(X_{j1},Y_{j1}),(X_{j2},Y_{j2}),\ldots \sim P_{j}$ be i.i.d., and independent over $j$, and all independent from the $(X_{it},Y_{it})$ sequence.
Let $j_{1},j_{2},\ldots$ be independent ${\rm Uniform}(\{1,\ldots,\vc\})$ random variables (also independent from the above sequences).
Then for each $t \in \nats$, let $r_{t} = \sum_{s = 1}^{t} \ind[j_{s} = j_{t}]$, and define $(X_{t}^{\prime},Y_{t}^{\prime}) = (X_{j_{t} r_{t}}, Y_{j_{t} r_{t}})$.
One can easily verify that this these $(X_{t}^{\prime},Y_{t}^{\prime})$ are independent and $P$-distributed.
Now we can construct an active learning algorithm for the problem of learning with $\C_{i}$ under $P_{i}$, given the budget $n^{\prime} \leq n$, as follows.
We execute the algorithm $\alg(n)$.  If at any time it requests the label $Y_{t}^{\prime}$ of some $X_{t}^{\prime}$ in the sequence such that $j_{t} \neq i$, 
then we simply use the value $Y_{t}^{\prime} = Y_{j_{t} r_{t}}$ (which, for the purpose of this reduction, is considered an accessible quantity).
Otherwise, if $\alg(n)$ requests the label $Y_{t}^{\prime}$ of some $X_{t}^{\prime}$ in the sequence such that $j_{t} = i$, then 
our algorithm will request the label $Y_{i r_{t}}$ and provide that as the value of $Y_{t}^{\prime}$ to be used in the execution of $\alg(n)$.
If at any time $\alg(n)$ has already requested $n^{\prime}$ labels $Y_{t}^{\prime}$ such that $j_{t} = i$, and attempts to request another label $Y_{t}^{\prime}$ with $j_{t} = i$, 
our algorithm simply returns an arbitrary classifier, and this is considered a ``failure'' event.  Otherwise, upon termination of $\alg(n)$, our algorithm halts and 
returns the classifier $\alg(n)$ produces.
Note that this is a valid active learning algorithm for the problem of learning $\C_{i}$ under $P_{i}$ with budget $n^{\prime}$,
since the algorithm requests at most $n^{\prime}$ labels from the $P_{i}$-distributed sequence.
In particular, in this reduction, we are thinking of the samples $(X_{t}^{\prime},Y_{t}^{\prime})$ with $j_{t} \neq i$ as simply part of the 
internal randomness of the learning algorithm.

Let $\hat{h}_{P,i}^{\prime}$ denote the classifier returned by the algorithm constructed via this reduction.
Furthermore, if we consider also the classifier $\hat{h}_{P,i}$ returned by $\alg(n)$ when run (unmodified) on the 
$P$-distributed sequence $(X_{1}^{\prime},Y_{1}^{\prime}),(X_{2}^{\prime},Y_{2}^{\prime}),\ldots$, and denote by $n_{P,i}^{\prime}$ the number of labels $Y_{t}^{\prime}$ 
with $j_{t} = i$ that this unmodified $\alg(n)$ requests, then on the event that $n_{P,i}^{\prime} \leq n^{\prime}$, we have
$\hat{h}_{P,i}^{\prime} = \hat{h}_{P,i}$.
Additionally, let $n_{P,i}$ denote the number of labels $Y_{t}$ requested by $\alg(n)$ with $X_{t} \in \X_{i}$ (when $\alg(n)$ is run with the sequence $\{(X_{t},Y_{t})\}_{t=1}^{\infty}$),
and note that the sequences $\{(X_{t}^{\prime},Y_{t}^{\prime})\}_{t=1}^{\infty}$ and $\{(X_{t},Y_{t})\}_{t=1}^{\infty}$ are distributionally equivalent,
so that $(\hat{h}_{P,i},n_{P,i}^{\prime})$ and $(\hat{h}_{P},n_{P,i})$ are distributionally equivalent as well.
Therefore,
\begin{align*}
& \P\left( \er_{P_{i}}\left( \hat{h}_{P} \right) - \er_{P_{i}}\left( g_{i}^{*} \right) > (4/\gamma) \eps \right)
\geq \P\left( \er_{P_{i}}\left( \hat{h}_{P} \right) - \er_{P_{i}}\left( g_{i}^{*} \right) > (4/\gamma) \eps \text{ and } n_{P,i} \leq n^{\prime} \right)
\\ & = \P\left( \er_{P_{i}}\left( \hat{h}_{P,i} \right) - \er_{P_{i}}\left( g_{i}^{*} \right) > (4/\gamma) \eps \text{ and } n_{P,i}^{\prime} \leq n^{\prime} \right)
\\ & = \P\left( \er_{P_{i}}\left( \hat{h}_{P,i}^{\prime} \right) - \er_{P_{i}}\left( g_{i}^{*} \right) > (4/\gamma) \eps \text{ and } n_{P,i}^{\prime} \leq n^{\prime} \right)
\\ & = \P\left( \er_{P_{i}}\left( \hat{h}_{P,i}^{\prime} \right) - \er_{P_{i}}\left( g_{i}^{*} \right) > (4/\gamma) \eps \right) - \P\left( \er_{P_{i}}\left( \hat{h}_{P,i}^{\prime} \right) - \er_{P_{i}}\left( g_{i}^{*} \right) > (4/\gamma) \eps \text{ and } n_{P,i}^{\prime} > n^{\prime} \right)
\\ & \geq \P\left( \er_{P_{i}}\left( \hat{h}_{P,i}^{\prime} \right) - \er_{P_{i}}\left( g_{i}^{*} \right) > (4/\gamma) \eps \right) - \P\left( n_{P,i}^{\prime} > n^{\prime} \right)
\\ & = \P\left( \er_{P_{i}}\left( \hat{h}_{P,i}^{\prime} \right) - \er_{P_{i}}\left( g_{i}^{*} \right) > (4/\gamma) \eps \right) - \P\left( n_{P,i} > n^{\prime} \right)
\\ & \geq \P\left( \er_{P_{i}}\left( \hat{h}_{P,i}^{\prime} \right) - \er_{P_{i}}\left( g_{i}^{*} \right) > (4/\gamma) \eps \right) - \frac{\E[ n_{P,i} ]}{n^{\prime}},
\end{align*}
where this last inequality is due to Markov's inequality.

Applying this to every $i \in \{1,\ldots,\vc\}$, this implies
\begin{multline*}
\frac{1}{\vc} \sum_{i=1}^{\vc} \P\left( \er_{P_{i}}\left( \hat{h}_{P} \right) - \er_{P_{i}}\left( g_{i}^{*} \right) > (4/\gamma) \eps \right)
\\ \geq - \frac{1}{\vc n^{\prime}} \sum_{i=1}^{\vc} \E[ n_{P,i} ] + \frac{1}{\vc} \sum_{i=1}^{\vc} \P\left( \er_{P_{i}}\left( \hat{h}_{P,i}^{\prime} \right) - \er_{P_{i}}\left( g_{i}^{*} \right) > (4/\gamma) \eps \right).
\end{multline*}
By linearity of the expectation, $\frac{1}{\vc n^{\prime}} \sum_{i=1}^{\vc} \E[ n_{P,i} ] = \frac{1}{\vc n^{\prime}} \E\left[ \sum_{i=1}^{\vc} n_{P,i} \right] \leq \frac{n}{\vc n^{\prime}} \leq \frac{\gamma}{2}$,
so that the above is at least
\begin{equation*}
- \frac{\gamma}{2} + \frac{1}{\vc} \sum_{i=1}^{\vc} \P\left( \er_{P_{i}}\left( \hat{h}_{P,i}^{\prime} \right) - \er_{P_{i}}\left( g_{i}^{*} \right) > (4/\gamma) \eps \right).
\end{equation*}
Plugging this into \eqref{eqn:ub-tightness-avg-step}, we have that
\begin{equation*}
\P\!\left( \er_{P}\!\left(\hat{h}_{P}\right) - \inf_{h \in \C} \er_{P}(h) > \eps \right) 
\geq - \frac{3\gamma}{4-\gamma} + \frac{4}{4-\gamma} \frac{1}{\vc} \sum_{i=1}^{\vc} \P\!\left( \er_{P_{i}}\!\left( \hat{h}_{P,i}^{\prime} \right) - \er_{P_{i}}\!\left( g_{i}^{*} \right) > (4/\gamma) \eps \right)\!.
\end{equation*}

The above strategy, producing $\hat{h}_{P,i}^{\prime}$, is a valid active learning algorithm (with budget $n^{\prime}$) for any choices of the probability measures $P_{j}$, $j \in \{1,\ldots,\vc\} \setminus \{i\}$.
We may therefore consider its behavior if we choose these at random.  Specifically, for any probability measure $\Pi^{\backslash i}$ over $\times_{j \neq i} \Dset_{j}$,
let $\{\tilde{P}_{j, \Pi^{\backslash i}}\}_{j \neq i} \sim \Pi^{\backslash i}$, and for any $P_{i} \in \Dset_{i}$, let $\tilde{P}_{\Pi^{\backslash i},P_{i}} = \frac{1}{\vc} P_{i} + \frac{1}{\vc} \sum_{j \neq i} \tilde{P}_{j,\Pi^{\backslash i}}$.
Then $\hat{h}_{\tilde{P}_{\Pi^{\backslash i},P_{i}},i}^{\prime}$ is the output of a valid active learning algorithm (with budget $n^{\prime}$); in particular, here we are considering
the $\tilde{P}_{j, \Pi^{\backslash i}}$ as internal random variables to the algorithm (along with their corresponding $(X_{jt},Y_{jt})$ samples used in the algorithm, which are now
considered conditionally independent given $\{\tilde{P}_{j,\Pi^{\backslash i}}\}_{j \neq i}$, where each $(X_{jt},Y_{jt})$ has conditional distribution $\tilde{P}_{j, \Pi^{\backslash i}}$):
that is, random variables that are independent from the data sequence $(X_{i1},Y_{i1}),(X_{i2},Y_{i2}),\ldots$.
Now note that,
since $n^{\prime} < \LC_{i}((4/\gamma)\eps,\gamma)$, 
\begin{equation}
\label{eqn:equilibrium-value-lb}
\max_{P_{i} \in \Dset_{i}} \P\left( \er_{P_{i}}\left( \hat{h}_{\tilde{P}_{\Pi^{\backslash i},P_{i}},i}^{\prime} \right) - \inf_{g \in \C_{i}} \er_{P_{i}}( g ) > (4/\gamma) \eps \right) > \gamma.
\end{equation}

For any given sequence $P_{1},\ldots,P_{\vc}$, with $P_{j} \in \Dset_{i}$ for each $j \in \{1,\ldots,\vc\}$,
for every $i \in \{1,\ldots,\vc\}$, denote 
$\psi_{i}(P_{i},\{P_{j}\}_{j \neq i}) = \P\left( \er_{P_{i}}\left( \hat{h}_{P,i}^{\prime} \right) - \inf_{g \in \C_{i}} \er_{P_{i}}( g ) > (4/\gamma) \eps \right)$,
where $P = \frac{1}{\vc} \sum_{j=1}^{\vc} P_{j}$ as above.
Then, by the law of total probability, \eqref{eqn:equilibrium-value-lb} may be restated as 
\begin{equation*}
\max_{P_{i} \in \Dset_{i}} \E\left[ \psi_{i}\left( P_{i}, \left\{ \tilde{P}_{j,\Pi^{\backslash i}} \right\}_{j \neq i} \right) \right] > \gamma.
\end{equation*}
Since this holds for every choice of $\Pi^{\backslash i}$, we have that 
\begin{equation*}
\inf_{\Pi^{\backslash i}} \max_{P_{i} \in \Dset_{i}} \E\left[ \psi_{i}\left( P_{i}, \left\{ \tilde{P}_{j,\Pi^{\backslash i}} \right\}_{j \neq i} \right) \right] \geq \gamma.
\end{equation*}
Since each $\Dset_{j}$ is finite, 
by the minimax theorem \citep*{von-neumann:28,von-neumann:44},
for each $i \in \{1,\ldots,\vc\}$, 
there exists a probability measure $\Pi_{i}$ over $\Dset_{i}$ such that, if $\tilde{P}_{i} \sim \Pi_{i}$ (independent from every $\{ \tilde{P}_{j,\Pi^{\backslash i}} \}_{j \neq i}$),
then 
\begin{equation*}
\inf_{\Pi^{\backslash i}} \E\left[  \psi_{i}\left( \tilde{P}_{i}, \left\{ \tilde{P}_{j,\Pi^{\backslash i}} \right\}_{j \neq i} \right) \right]
= \inf_{\Pi^{\backslash i}} \max_{P_{i} \in \Dset_{i}} \E\left[ \psi_{i}\left( P_{i}, \left\{ \tilde{P}_{j,\Pi^{\backslash i}} \right\}_{j \neq i} \right) \right].
\end{equation*}
In particular, taking these $\{\tilde{P}_{i}\}_{i=1}^{\vc}$ to be independent, we have that $\forall i \in \{1,\ldots,\vc\}$, 
\begin{equation*}
\E\!\left[ \psi_{i}\!\left( \tilde{P}_{i}, \left\{ \tilde{P}_{j} \right\}_{j \neq i} \right) \right]
\!\geq\! \inf_{\Pi^{\backslash i}} \E\!\left[  \psi_{i}\!\left( \tilde{P}_{i}, \left\{ \tilde{P}_{j,\Pi^{\backslash i}} \right\}_{j \neq i} \right) \right]
\!=\! \inf_{\Pi^{\backslash i}} \max_{P_{i} \in \Dset_{i}} \E\!\left[ \psi_{i}\!\left( P_{i}, \left\{ \tilde{P}_{j,\Pi^{\backslash i}} \right\}_{j \neq i} \right) \right] 
\!\geq\! \gamma.
\end{equation*}
Thus,
\begin{equation*}
\sup_{\substack{P_{i} \in \Dset_{i} : \\i \in \{1,\ldots,\vc\}}} \sum_{i=1}^{\vc} \psi_{i}\!\left( P_{i}, \{P_{j}\}_{j \neq i} \right)
\geq \E\!\left[ \sum_{i=1}^{\vc} \psi_{i}\!\left( \tilde{P}_{i}, \left\{ \tilde{P}_{j} \right\}_{j \neq i} \right) \right]
= \sum_{i=1}^{\vc} \E\!\left[ \psi_{i}\!\left( \tilde{P}_{i}, \left\{ \tilde{P}_{j} \right\}_{j \neq i} \right) \right]
\geq \gamma \vc.
\end{equation*}
Altogether, we have that
\begin{align*}
\sup_{\substack{P_{i} \in \Dset_{i} : \\ i \in \{1,\ldots,\vc\}}} \P\left( \er_{P}\left(\hat{h}_{P}\right) - \inf_{h \in \C} \er_{P}(h) > \eps \right) 
& \geq - \frac{3\gamma}{4-\gamma} + \frac{4}{4-\gamma} \frac{1}{\vc} \sup_{\substack{P_{i} \in \Dset_{i} : \\i \in \{1,\ldots,\vc\}}} \sum_{i=1}^{\vc} \psi_{i}\left( P_{i}, \{P_{j}\}_{j \neq i} \right)
\\ & \geq - \frac{3\gamma}{4-\gamma} + \frac{4 \gamma}{4-\gamma}
= \frac{\gamma}{4-\gamma} > \conf.
\end{align*}
Since this holds for any active learning algorithm $\alg$ and $n < (\gamma/4) \vc \min_{i \in \{1,\ldots,\vc\}} \LC_{i}((4/\gamma)\eps,\gamma)$,
the lemma follows.
\end{proof}

With this lemma in hand, we can now plug in various sets $\Dset_{i}$ to obtain lower bounds for learning with this set $\C$ under various noise models.
In particular, we can make use of the constructions of lower bounds on $\LC_{i}(\eps,\conf)$ given in the proofs of the theorems in Section~\ref{sec:main},
noting that the VC dimension of $\C_{i}$ is $1$, and the star number of $\C_{i}$ is $\lfloor \s / \vc \rfloor$.
Note that, in the case $\vc \lesssim 1$, the lower bounds in each of these theorems already match their respective upper bounds up to constant and logarithmic factors
(using the lower bound from Theorem~\ref{thm:realizable} as a lower bound on $\LC_{\BN(\bound)}(\eps,\conf)$ for $\bound$ near $0$).
We may therefore suppose $\vc \geq 32$ for the remainder of this subsection.

\paragraph{The realizable case:}
For the realizable case, 
for each $i \in \{1,\ldots,\vc\}$ and $t \in \{1,\ldots,\lfloor \s/\vc \rfloor \}$,
let $\Px_{it}$ be a uniform distribution on $\{ \lfloor \s/\vc \rfloor (i-1)+1,\ldots,\lfloor \s/\vc \rfloor (i-1) + t \} \subseteq \X_{i}$, 
and let $\Dset_{i}$ denote the set of probability measures $P_{i}$ in $\RE$ having marginal over $\X$ among $\{\Px_{it} : 1 \leq t \leq \lfloor \s/\vc \rfloor \}$ and having $\target_{P_{i}} \in \C_{i}$.
Noting that the star number of $\C_{i}$ is $\lfloor \s / \vc \rfloor$ and that $\X_{i}$ is a (maximal) star set for $\C_{i}$, 
and recalling that the first term in the ``$\max$'' in the lower bound of Theorem~\ref{thm:realizable}
was proven in Appendix~\ref{app:realizable} under the uniform marginal distribution on the first $t$ elements of a maximal star set 
(for an appropriate value of $t$, of size at least $1$ and at most the star number),
we have that for $\eps \in \left(0,\frac{1}{9 \cdot 16}\right)$,
\begin{equation*}
\LC_{i}(16 \eps, 1/4) \gtrsim \min\left\{ \frac{\s}{\vc}, \frac{1}{\eps} \right\}.
\end{equation*}
Therefore, Lemma~\ref{lem:general-gaps-reduction} (with $\gamma = 1/4$) implies that for 
$\Dset = \left\{ \frac{1}{\vc} \sum_{i=1}^{\vc} P_{i} : \forall i \!\in \!\{1,\ldots,\vc\}, P_{i} \!\in\! \Dset_{i} \right\}$,
$\forall \conf \in \left( 0, \frac{1}{15} \right)$,
\begin{equation*}
\LC_{\Dset}(\eps,\conf) \gtrsim \min\left\{ \s, \frac{\vc}{\eps} \right\}.
\end{equation*}
Furthermore, for each choice of $P_{1},\ldots,P_{\vc}$ (with each $P_{i} \in \Dset_{i}$),
by construction, every $i \in \{1,\ldots,\vc\}$ has at most one $x \in \X_{i}$ with $P_{i}( \{(x,+1)\} | \{x\}\times\Y ) = 1$,
and every other $x^{\prime}$ in  $\X_{i}$ has $P_{i}( \{(x^{\prime},+1)\} | \{x^{\prime}\}\times\Y ) = 0$.
Therefore, since $P( \{(x,+1)\} | \{x\}\times\Y ) = P_{i}( \{(x,+1)\} | \{x\}\times\Y )$ for every $x \in \X_{i}$, for $P = \frac{1}{\vc} \sum_{j=1}^{\vc} P_{j}$, 
we have that there are at most $\vc$ points $x$ in $\bigcup_{i=1}^{\vc} \X_{i}$ with $P( \{(x,+1)\} | \{x\}\times\Y ) = 1$,
and all other points $x$ in $\bigcup_{i=1}^{\vc} \X_{i}$ have $P(\{(x,+1)\} | \{x\}\times\Y) = 0$.
In particular, this implies that for $(X,Y) \sim P$, $\P( \target_{P}(X) \neq Y | X \in \bigcup_{i=1}^{\vc} \X_{i} ) = 0$.
Since we also have that $\forall t \in \nats \setminus \bigcup_{i=1}^{\vc} \X_{i}$, $P(\{t\} \times \Y) = 0$,
we can take $\target_{P}(x) = -1$ for every $x \in \X \setminus \bigcup_{i=1}^{\vc} \X_{i}$ while guaranteeing $\er_{P}(\target_{P}) = 0$.
Since $\bigcup_{i=1}^{\vc} \X_{i} \subseteq \{1,\ldots,\s\}$, we also have that $\target_{P} \in \C$.
Together, these facts imply $P \in \RE$.
Thus, $\Dset \subseteq \RE$, which implies $\LC_{\RE}(\eps,\conf) \geq \LC_{\Dset}(\eps,\conf)$, so that
\begin{equation*}
\LC_{\RE}(\eps,\conf) \gtrsim \min\left\{ \s, \frac{\vc}{\eps} \right\}
\end{equation*}
as well.
Since the upper bound in Theorem~\ref{thm:realizable} is within a factor proportional to $\Log(1/\eps)$ of this,\footnote{
Note that, although $\frac{\s \vc}{\Log(\s)}$ can sometimes be much smaller than $\s \land \frac{\vc}{\eps}$,
we always have $\s \land \frac{\vc}{\eps} \lesssim \frac{\s \vc}{\Log(\s)} \Log\left(\frac{1}{\eps}\right)$, so that this $\s \land \frac{\vc}{\eps}$
lower bound does not contradict the $\frac{\s \vc}{\Log(\s)} \Log\left(\frac{1}{\eps}\right)$ upper bound.}
this establishes that the upper bound is sometimes tight to within a factor proportional to $\Log(1/\eps)$.

\paragraph{Bounded noise:}
In the case of bounded noise,
fix any $\bound \in (0,1/2)$ and $\eps \in (0 , (1-2\bound)/(256 e))$.
Take $\zeta = \frac{32e\eps}{1-2\bound}$ and $k = \min\left\{ \lfloor \s/\vc \rfloor - 1, \lfloor 1/\zeta \rfloor \right\}$,
and for each $i \in \{1,\ldots,\vc\}$, let $\Dset_{i}$ be defined as the set $\RR(k,\zeta,\bound)$ in Lemma~\ref{lem:rr11-star},
as applied to the hypothesis class $\C_{i}$ with $\{x_{1},\ldots,x_{k+1}\} = \left\{ \lfloor \s/\vc \rfloor (i-1)+1,\ldots,  \lfloor \s/\vc \rfloor (i-1) + k+1 \right\}$,
$h_{0} = -1$, and $h_{j} = 2 \ind_{\{ \lfloor \s/\vc \rfloor (i-1) + j\}}-1$ for each $j \in \{1,\ldots,k\}$.
Then Lemma~\ref{lem:rr11-star} implies
\begin{equation*}
\LC_{i}(16e \eps, 1/(4e)) \geq \frac{\bound (k-1)}{3 (1-2\bound)^{2}} \gtrsim \frac{\bound}{(1-2\bound)^{2}} \min\left\{ \frac{\s}{\vc}, \frac{1-2\bound}{\eps} \right\}.
\end{equation*}
Furthermore, recall from the definition of $\RR(k,\zeta,\bound)$ in Section~\ref{sec:rr-lemma} that $\Dset_{i}$ 
is a finite set of probability measures, and every $P_{i} \in \Dset_{i}$ has $P_{i}((\X \setminus \{x_{1},\ldots,x_{k+1}\}) \times \Y) = 0$.
In particular, note that $\{x_{1},\ldots,x_{k+1}\} \subseteq \X_{i}$ in this case.
Furthermore, every $P_{i} \in \Dset_{i}$ has $\forall x \in \{x_{1},\ldots,x_{k}\}$, $P_{i}( \{(x,+1)\} | \{x\}\times\Y ) \in \{\bound, 1-\bound\}$,
and at most one $x \in \{x_{1},\ldots,x_{k}\}$ has $P_{i}( \{(x,+1)\} | \{x\}\times\Y ) = 1-\bound$,
while $P_{i}( \{(x_{k+1},+1)\} | \{x_{k+1}\} \times \Y) = 0$.
Thus, for any choices of $P_{i} \in \Dset_{i}$ for each $i \in \{1,\ldots,\vc\}$, 
the probability measure $P = \frac{1}{\vc} \sum_{i=1}^{\vc} P_{i}$ satisfies the property that,
$\forall x \in \X$ with $P(\{x\} \times \Y) > 0$, $P(\{(x,+1)\}|\{x\}\times\Y) \in \{0,\bound,1-\bound\}$, 
and there are at most $\vc$ values $x \in \X$ with $P(\{x\}\times\Y) > 0$ and $P(\{(x,+1)\}|\{x\}\times\Y) = 1-\bound$.
In particular, this implies that without loss, we can take $\target_{P} \in \C$, and furthermore that $P \in \BN(\bound)$.
Thus, for the set $\Dset = \left\{ \frac{1}{\vc} \sum_{i=1}^{\vc} P_{i} : \forall i \in \{1,\ldots,\vc\}, P_{i} \in \Dset_{i} \right\}$,
we have $\Dset \subseteq \BN(\bound)$.
Lemma~\ref{lem:general-gaps-reduction} (with $\gamma = 1/(4e)$) then implies that $\forall \conf \in \left( 0, \frac{1}{16 e - 1} \right)$, 
\begin{equation*}
\LC_{\BN(\bound)}(\eps,\conf) \geq \LC_{\Dset}(\eps,\conf) \gtrsim \vc \min_{i \in \{1,\ldots,\vc\}} \LC_{i}( 16 e \eps, 1/(4e) ) \gtrsim \frac{\bound}{(1-2\bound)^{2}} \min\left\{ \s, \frac{(1-2\bound) \vc}{\eps} \right\}. 
\end{equation*}
For $\bound$ bounded away from $0$, the upper bound in Theorem~\ref{thm:bounded} is within a $\polylog\left(\frac{\vc}{\eps\conf}\right)$ factor of this,
so that this establishes that the upper bound is sometimes tight to within logarithmic factors when $\bound$ is bounded away from $0$.
Furthermore, since $\RE \subseteq \BN(\bound)$, the above result for sometimes-tightness of the upper bound in the realizable case implies 
that the upper bound in Theorem~\ref{thm:bounded} is also sometimes tight to within logarithmic factors for any $\bound$ near $0$.

\paragraph{Tsybakov noise:}
For the case of Tsybakov noise, the tightness (up to logarithmic factors) of the upper bound for $\tsyba \leq 1/2$ is already 
established by the lower bound for that case in Theorem~\ref{thm:tsybakov}.  Thus, it remains only to consider $\tsyba \in (1/2,1)$.
Fix any values $\tsybca \in [4,\infty)$, $\tsyba \in (1/2,1)$, and $\eps \in \left(0,1/(2^{11} \tsybca^{1/\tsyba})\right)$,
let $\tsybca^{\prime}$ be as in the definition of $\TN(\tsybca,\tsyba)$,
and let 
\begin{equation*}
k = \min\left\{ \left\lfloor \frac{\s}{\vc} \right\rfloor - 1, \left\lfloor \frac{(\tsybca^{\prime})^{\frac{\tsyba-1}{\tsyba}}}{64\eps} \right\rfloor, \left\lfloor \frac{\tsybca^{\prime}}{64 \eps} 4^{-\frac{1}{1-\tsyba}} \right\rfloor \right\},
\end{equation*}
$\bound = \frac{1}{2} - \left(\frac{k 64 \eps}{\tsybca^{\prime}}\right)^{1-\tsyba}$, and $\zeta = \frac{128 \eps}{1-2\bound}$.
Note that $\zeta \in (0,1)$, $\bound \in [1/4,1/2)$, and $2 \leq k \leq \min\left\{ \lfloor \s/\vc \rfloor - 1, \lfloor 1/\zeta \rfloor \right\}$
(following the arguments from the proof of Theorem~\ref{thm:tsybakov}, with $\eps$ replaced by $64\eps$).
Furthermore, $\forall i \in \{1,\ldots,\vc\}$, let $\Dset_{i}$ be the set $\RR(k,\zeta,\bound)$ in Lemma~\ref{lem:rr11-star},
as applied to the class $\C_{i}$, with 
$\{x_{1},\ldots,x_{k+1}\} = \left\{ \lfloor \s/\vc \rfloor (i-1) + 1, \ldots,  \lfloor \s/\vc \rfloor (i-1) + k+1 \right\}$,
$h_{0} = -1$, and $h_{j} = 2 \ind_{\{ \lfloor \s/\vc \rfloor (i-1) + j\}}-1$ for each $j \in \{1,\ldots,k\}$.
Thus, by Lemma~\ref{lem:rr11-star}, 
\begin{align*}
& \LC_{i}( 64 \eps, 1/16 ) 
\geq \frac{\bound (k-1) \ln(4)}{3(1-2\bound)^{2}}
\gtrsim \left(\frac{\eps}{\tsybca^{\prime}}\right)^{2\tsyba-2} k^{2\tsyba-1}
\\ & \gtrsim \tsybca^{2} \left(\frac{1}{\eps}\right)^{2-2\tsyba} 
\min\left\{ \frac{\s}{\vc}, \frac{(\tsybca^{\prime})^{\frac{\tsyba-1}{\tsyba}}}{\eps}, \frac{\tsybca^{\prime}}{\eps} 4^{- \frac{1}{1-\tsyba}} \right\}^{2\tsyba-1}
\gtrsim \tsybca^{2} \left(\frac{1}{\eps}\right)^{2-2\tsyba} \min\left\{ \frac{\s}{\vc}, \frac{1}{\tsybca^{1/\tsyba} \eps} \right\}^{2\tsyba-1},
\end{align*}
where this last inequality relies on the fact (established in the proof of Theorem~\ref{thm:tsybakov}) that $(\tsybca^{\prime})^{\frac{\tsyba-1}{\tsyba}} \leq \tsybca^{\prime} 4^{-\frac{1}{1-\tsyba}}$.

We note that any $P_{i} \in \Dset_{i}$ has $P_{i}( ( \X \setminus \{ \lfloor \s/\vc \rfloor (i-1) + 1, \ldots,  \lfloor \s/\vc \rfloor (i-1) + k+1\} ) \times \Y ) = 0$.
Without loss of generality, suppose each $P_{i} \in \Dset_{i}$ has $\eta(x;P_{i}) = 0$ for every $x \in \X \setminus \left\{ \lfloor \s/\vc \rfloor (i-1) + 1, \ldots,  \lfloor \s/\vc \rfloor (i-1) + k+1 \right\}$.
As in the proof of the lower bound in Theorem~\ref{thm:tsybakov}, we note that any $P_{i} \in \Dset_{i}$ 
has $P_{i}( (x,y) : | \eta(x; P_{i}) - 1/2 | \leq t ) \leq \tsybca^{\prime} t^{\tsyba/(1-\tsyba)}$ for every $t > 0$,
and furthermore that $\target_{P_{i}}(\cdot) = \sign(2\eta(\cdot;P_{i})-1)$, which has at most one $x$ with $\target_{P_{i}}(x_{i}) = +1$ (by definition of $\RR(k,\zeta,\bound)$ in Section~\ref{sec:rr-lemma}).
This further implies that, for any choices of $P_{i} \in \Dset_{i}$ for each $i \in \{1,\ldots,\vc\}$, 
the probability measure $P = \frac{1}{\vc} \sum_{i=1}^{\vc} P_{i}$ has support for its marginal over $\X$ only in
$\bigcup_{i=1}^{\vc} \left\{ \lfloor \s/\vc \rfloor (i-1) + 1, \ldots,  \lfloor \s/\vc \rfloor (i-1) + k+1 \right\}$,
and for each $i \in \{1,\ldots,\vc\}$, $\forall x \in \left\{ \lfloor \s/\vc \rfloor (i-1) + 1, \ldots,  \lfloor \s/\vc \rfloor (i-1) + k+1 \right\}$, 
$\eta(x;P) = \eta(x;P_{i})$, while we may take $\eta(x;P) = 0$ for every $x \notin \bigcup_{i=1}^{\vc} \left\{ \lfloor \s/\vc \rfloor (i-1) + 1, \ldots,  \lfloor \s/\vc \rfloor (i-1) + k+1 \right\}$.
Therefore, $\target_{P}$ has at most $\vc$ points $x \in \bigcup_{i=1}^{\vc} \X_{i}$ with $\target_{P}(x) = +1$, and $\target_{P}(x) = -1$ for all other $x \in \X$: that is, $\target_{P} \in \C$.
Additionally, since the supports of the marginals of the $P_{i}$ distributions over $\X$ are disjoint, we have that $\forall t > 0$, 
\begin{align*}
& P\left( (x,y) : |\eta(x;P) - 1/2| \leq t\right)
= \frac{1}{\vc} \sum_{i=1}^{\vc} P_{i}\left( (x,y) : |\eta(x;P) - 1/2| \leq t \right)
\\ & = \frac{1}{\vc} \sum_{i=1}^{\vc} P_{i}\left( (x,y) : |\eta(x;P_{i}) - 1/2| \leq t \right)
\leq \frac{1}{\vc} \sum_{i=1}^{\vc} \tsybca^{\prime} t^{\tsyba/(1-\tsyba)}
= \tsybca^{\prime} t^{\tsyba/(1-\tsyba)}.
\end{align*}
Thus, the set $\Dset = \left\{ \frac{1}{\vc} \sum_{i=1}^{\vc} P_{i} : \forall i \in \{1,\ldots,\vc\}, P_{i} \in \Dset_{i} \right\}$
satisfies $\Dset \subseteq \TN(\tsybca,\tsyba)$.
Combined with the fact that each set $\Dset_{i}$ is finite (by the definition of $\RR(k,\zeta,\bound)$ in Section~\ref{sec:rr-lemma}), 
Lemma~\ref{lem:general-gaps-reduction} (with $\gamma=1/16$) implies that $\forall \conf \in \left(0, \frac{1}{63}\right)$,
\begin{equation*}
\LC_{\TN(\tsybca,\tsyba)}(\eps,\conf) \geq \LC_{\Dset}(\eps,\conf) \gtrsim \vc \min_{i \in \{1,\ldots,\vc\}} \LC_{i}( 64 \eps, 1/16 ) 
\gtrsim \tsybca^{2} \left(\frac{1}{\eps}\right)^{2-2\tsyba} \min\left\{ \frac{\s}{\vc}, \frac{1}{\tsybca^{1/\tsyba} \eps} \right\}^{2\tsyba-1} \vc.
\end{equation*}
Since this is within logarithmic factors of the upper bound of Theorem~\ref{thm:tsybakov}, this establishes that the upper bound
is sometimes tight to within logarithmic factors (for sufficiently small values of $\eps$).

\paragraph{Benign noise:}
We can establish that the upper bound in Theorem~\ref{thm:benign} is sometimes tight by reduction from the above problems.
Specifically, since $\RE \subseteq \BE(\nu)$ for every $\nu \in [0,1/2)$,
for the above choice of $\C$ we have that $\forall \nu \in [0,1/2]$, 
$\forall \eps \in \left(0,\frac{1}{9 \cdot 16}\right)$, $\forall \conf \in \left( 0, \frac{1}{15} \right)$,
\begin{equation*}
\LC_{\BE(\nu)}(\eps,\conf) \geq \LC_{\RE}(\eps,\conf) \gtrsim \min\left\{ \s, \frac{\vc}{\eps} \right\}.
\end{equation*}
Furthermore, the lower bound in Theorem~\ref{thm:benign} already implies that
$\forall \eps \in \left( 0, \frac{1-2\nu}{24} \right)$, $\forall \conf \in \left(0, \frac{1}{24} \right]$, 
\begin{equation*}
\LC_{\BE(\nu)}(\eps,\conf) \gtrsim \frac{\nu^{2}}{\eps^{2}} \vc.
\end{equation*}
Together, we have that $\forall \nu \in [0,1/2)$, $\forall \eps \in \left(0, \frac{1-2\nu}{9 \cdot 16} \right)$, $\forall \conf \in \left( 0, \frac{1}{24} \right]$,
\begin{equation*}
\LC_{\BE(\nu)}(\eps,\conf) \gtrsim \max\left\{ \frac{\nu^{2}}{\eps^{2}}\vc, \min\left\{ \s, \frac{\vc}{\eps}\right\} \right\}
\gtrsim \frac{\nu^{2}}{\eps^{2}} \vc + \min\left\{ \s, \frac{\vc}{\eps} \right\}.
\end{equation*}
Thus, the upper bound in Theorem~\ref{thm:benign} is sometimes tight to within logarithmic factors.

\subsection{The Lower Bounds are Sometimes Tight}
\label{app:lb-tight}

We now argue that the lower bounds in Theorems~\ref{thm:realizable}, \ref{thm:bounded}, \ref{thm:tsybakov}, and \ref{thm:benign} are sometimes tight (up to logarithmic factors).
First we have a general lemma.
Let $\X_{1} \subset \X$ and $\X_{2} = \X \setminus \X_{1}$,
and let $\C_{1},\C_{2}$ be hypothesis classes such that $\forall i \in \{1,2\}$,
$\forall h \in \C_{i}$, $\forall x \in \X \setminus \X_{i}$, $h(x) = -1$.
Further suppose that $\forall i \in \{1,2\}$, the all-negative classifier $x \mapsto h_{-}(x) = -1$ is in $\C_{i}$.
For each $i \in \{1,2\}$ and $\gamma \in [0,1]$, let $\Dset_{i}(\gamma)$ be a nonempty set of probability measures on $\X \times \Y$
such that $\forall P_{i} \in \Dset_{i}(\gamma)$, $P_{i}(\X_{i} \times \Y) = 1$; further suppose $\forall \gamma,\gamma^{\prime} \in [0,1]$ with $\gamma \leq \gamma^{\prime}$, $\Dset_{i}(\gamma) \supseteq \Dset_{i}(\gamma^{\prime})$.
Also, for each $i \in \{1,2\}$, $\gamma,\conf \in [0,1]$, and $\eps > 0$, let $\LC_{i,\gamma}(\eps,\conf)$ denote 
the minimax label complexity under $\Dset_{i}(\gamma)$ with respect to $\C_{i}$
(i.e., the value of $\LC_{\Dset_{i}(\gamma)}(\eps,\conf)$ when $\C = \C_{i}$).
Let $\Dset = \left\{ \gamma P_{1} + (1-\gamma) P_{2} : P_{1} \in \Dset_{1}(\gamma), P_{2} \in \Dset_{2}(1-\gamma),  \gamma \in [0,1] \right\}$.

\begin{lemma}
\label{lem:general-gaps-lb-tight}
For $\C = \C_{1} \cup \C_{2}$, $\forall \eps,\conf \in (0,1)$,
\begin{equation*}
\LC_{\Dset}(\eps,\conf) \leq 2 \!\sup_{\gamma \in [0,1]} \max\!\left\{ \LC_{1,(\gamma-\eps/8) \lor 0}\!\left(\frac{\eps}{2(\gamma+\eps/8)},\frac{\conf}{3}\right)\!, \LC_{2, (1-\gamma-\eps/8) \lor 0}\!\left(\frac{\eps}{2(1-\gamma+\eps/8)},\frac{\conf}{3}\right) \right\}\!.
\end{equation*}
\end{lemma}
\begin{proof}
For each $i \in \{1,2\}$ and $\gamma \in [0,1]$, let $\alg_{\gamma,i}$ be an active learning algorithm
such that, for any integer $n \geq \LC_{i,\gamma}\left(\frac{\eps}{2(\gamma+\eps/8)},\frac{\conf}{3}\right)$, if $\PXY \in \Dset_{i}(\gamma)$,
then with probability at least $1-\conf/3$, the classifier $\hat{h}$ produced by $\alg_{\gamma,i}(n)$
satisfies $\er_{\PXY}(\hat{h}) - \inf_{h \in \C_{i}} \er_{\PXY}(h) \leq \frac{\eps}{2(\gamma+\eps/8)}$;
such an algorithm is guaranteed to exist by the definition of $\LC_{i,\gamma}(\cdot,\cdot)$.

Now suppose $\PXY \in \Dset$, so that $\PXY = \gamma P_{1} + (1-\gamma) P_{2}$ 
for some $\gamma \in [0,1]$, $P_{1} \in \Dset_{1}(\gamma)$, and $P_{2} \in \Dset_{2}(1-\gamma)$.
Let $(X_{1},Y_{1}),(X_{2},Y_{2}),\ldots$ be the data sequence, as usual (i.i.d. $\PXY$).
Consider an active learning algorithm $\alg$ defined as follows.
We first split the sequence of indices into three subsequences: $i_{0,k} = 2k-1$ for $k \in \nats$,
$i_{1,1},i_{1,2},\ldots$ is the increasing subsequence of indices $i$ such that $i/2 \in \nats$ and $X_{i} \in \X_{1}$,
and $i_{2,1},i_{2,2},\ldots$ is the remaining increasing subsequence (i.e., indices $i$ such that $i/2 \in \nats$ and $X_{i} \in \X_{2}$).
Given a budget $n \in \nats$, $\alg(n)$ proceeds as follows.
First, we let $m = \left\lceil \frac{128}{\eps^{2}}\ln\left(\frac{12}{\conf}\right) \right\rceil$,
$\gamma_{1} = \max\left\{\frac{1}{m} \sum_{k=1}^{m} \ind_{\X_{1}}( X_{i_{0,k}} ) - \frac{\eps}{16}, 0 \right\}$,
and $\gamma_{2} = \max\left\{ \frac{1}{m} \sum_{k=1}^{m} \ind_{\X_{2}}( X_{i_{0,k}} ) - \frac{\eps}{16}, 0 \right\}$.
By Hoeffding's inequality and a union bound, with probability at least $1-\conf/3$, $\forall i \in \{1,2\}$,
\begin{equation}
\label{eqn:gaps-lb-hoeffding}
\PXY(\X_{i} \times \Y) - \frac{\eps}{8} \leq \gamma_{i} \leq \PXY(\X_{i} \times \Y).
\end{equation}
Denote by $H$ this event.

Next, for each $j \in \{1,2\}$, if the subsequence $i_{j,1},i_{j,2},\ldots$ is infinite, 
then we run $\alg_{\gamma_{j},j}(\lfloor n/2 \rfloor)$ with the data subsequence $\{X_{k}^{(j)}\}_{k=1}^{\infty} = \{X_{i_{j,k}}\}_{k=1}^{\infty}$;
if the algorithm $\alg_{\gamma_{j},j}$ requests the label for an index $k$ (i.e., corresponding to $X_{k}^{(j)}$), then $\alg(n)$
requests the corresponding label $Y_{i_{j,k}}$ and provides this value to $\alg_{\gamma_{j},j}$ as the label of $X_{k}^{(j)}$.
Let $\hat{h}_{j}$ denote the classifier returned by this execution of $\alg_{\gamma_{j},j}(\lfloor n/2 \rfloor)$.
On the other hand, if the subsequence $i_{j,1},i_{j,2},\ldots$ is finite (or empty), then we let $\hat{h}_{j}$ denote an arbitrary classifier.
Finally, let $\alg(n)$ return the classifier $\hat{h} = \hat{h}_{1} \ind_{\X_{1}} + \hat{h}_{2} \ind_{\X_{2}}$.
In particular, note that this method requests at most $n$ labels, since all labels are requested by one of the $\alg_{\gamma_{j},j}$ algorithms,
each of which requests at most $\lfloor n/2 \rfloor$ labels.

For this method, we have that
\begin{multline*}
\er_{\PXY}(\hat{h}) - \inf_{h \in \C} \er_{\PXY}(h)
= \gamma \er_{P_{1}}( \hat{h}_{1} )+ (1-\gamma) \er_{P_{2}}( \hat{h}_{2} ) - \inf_{h \in \C} \left( \gamma \er_{P_{1}}(h)+ (1-\gamma) \er_{P_{2}}(h) \right)
\\ \leq \gamma \left( \er_{P_{1}}( \hat{h}_{1} ) - \inf_{h \in \C} \er_{P_{1}}(h)\right)
+ (1-\gamma) \left( \er_{P_{2}}( \hat{h}_{2} ) - \inf_{h \in \C} \er_{P_{2}}(h) \right).
\end{multline*}
For each $j \in \{1,2\}$, since every $h \in \C \setminus \C_{j}$ has $h(x) = h_{-}(x)$ for every $x \in \X_{j}$, and $h_{-} \in \C_{j}$, 
we have that $\inf_{h \in \C} \er_{P_{j}}(h) = \inf_{h \in \C_{j}} \er_{P_{j}}(h)$.  Thus, the above implies
\begin{equation}
\label{eqn:gaps-lb-split-bound}
\er_{\PXY}(\hat{h}) - \inf_{h \in \C} \er_{\PXY}(h) \leq 
\gamma \left( \er_{P_{1}}( \hat{h}_{1} ) - \inf_{h \in \C_{1}} \er_{P_{1}}(h)\right)
+ (1-\gamma) \left( \er_{P_{2}}( \hat{h}_{2} ) - \inf_{h \in \C_{2}} \er_{P_{2}}(h) \right)\!.
\end{equation}

If $\gamma = 0$, then with probability one, every $X_{i} \in \X_{2}$, 
and $\{(X_{i_{2,k}},Y_{i_{2,k}})\}_{k=1}^{\infty}$ is an infinite i.i.d. $P_{2}$-distributed sequence.
Furthermore, $1-\eps/8 < \gamma_{2} = 1-\eps/16 < 1$, so that $\PXY \in \Dset_{2}(\gamma_{2})$.
Thus, if $n \geq 2 \LC_{2,1-\eps/8}\left( \frac{\eps}{2(1+\eps/8)}, \frac{\conf}{3} \right)$,
then we also have $n \geq \LC_{2,\gamma_{2}}\left(\frac{\eps}{2(\gamma_{2}+\eps/8)}, \frac{\conf}{3} \right)$
(by monotonicity of $\Dset_{2}(\cdot)$ and the label complexity),
so that with probability at least $1-\conf/3$,
$\er_{P_{2}}(\hat{h}_{2}) - \inf_{h \in \C_{2}} \er_{P_{2}}(h) \leq \frac{\eps}{2(\gamma_{2}+\eps/8)} = \frac{\eps}{2(1+\eps/16)} < \frac{\eps}{2}$
(here we are evaluating the label complexity guarantee of $\alg_{\gamma_{2},2}$ under the conditional distribution given $\gamma_{2}$, and then invoking the law of total probability and intersecting with the above probability-one event).
Combined with \eqref{eqn:gaps-lb-split-bound}, this implies $\er_{\PXY}(\hat{h}) - \inf_{h \in \C} \er_{\PXY}(h) < \frac{\eps}{2}$.
If $\gamma = 1$, then a symmetric argument implies that if $n \geq 2 \LC_{1,1-\eps/8}\left( \frac{\eps}{2(1+\eps/8)},\frac{\conf}{3} \right)$,
then with probability at least $1-\conf/3$, 
$\er_{\PXY}(\hat{h}) - \inf_{h \in \C} \er_{\PXY}(h) < \frac{\eps}{2}$.

Otherwise, suppose $0 < \gamma < 1$.  
Note that, on the event $H$, 
$\gamma-\eps/8 \leq \gamma_{1} \leq \gamma$ and
$1-\gamma-\eps/8 \leq \gamma_{2} \leq 1-\gamma$,
so that
$\Dset_{1}(\gamma_{1}) \subseteq \Dset_{1}((\gamma-\eps/8) \lor 0)$
and $\Dset_{2}(\gamma_{2}) \subseteq \Dset_{2}((1-\gamma-\eps/8) \lor 0)$,
and hence that 
\begin{equation*}
\LC_{1,\gamma_{1}}\left( \frac{\eps}{2(\gamma_{1}+\eps/8)}, \frac{\conf}{3} \right) \leq \LC_{1,(\gamma-\eps/8) \lor 0}\left( \frac{\eps}{2(\gamma+\eps/8)}, \frac{\conf}{3} \right)
\end{equation*}
and 
\begin{equation*}
\LC_{2,\gamma_{2}}\left( \frac{\eps}{2(\gamma_{2}+\eps/8)}, \frac{\conf}{3} \right) \leq \LC_{2,(1-\gamma-\eps/8)\lor 0}\left( \frac{\eps}{2(1-\gamma+\eps/8)}, \frac{\conf}{3} \right).
\end{equation*}
In this case, by the strong law of large numbers, with probability one, 
$\forall j \in \{1,2\}$, the sequence $i_{j,1},i_{j,2},\ldots$ exists and is infinite.
Since the support of the marginal of $P_{j}$ over $\X$ is contained within $\X_{j}$, and $\X_{1}$ and $\X_{2}$ are disjoint,
we may note that $(X_{i_{j,1}},Y_{i_{j,1}}),(X_{i_{j,2}},Y_{i_{j,2}}),\ldots$ are independent $P_{j}$-distributed random variables.
In particular, if
\begin{equation*}
n \geq 2 \max\left\{  \LC_{1,(\gamma-\eps/8) \lor 0}\left( \frac{\eps}{2(\gamma+\eps/8)}, \frac{\conf}{3} \right), \LC_{2,(1-\gamma-\eps/8)\lor 0}\left( \frac{\eps}{2(1-\gamma+\eps/8)}, \frac{\conf}{3} \right) \right\},
\end{equation*}
then (by the label complexity guarantee of $\alg_{\gamma_{j},j}$ applied under the conditional distribution given $\gamma_{j}$, combined with the law of total probability, and intersecting with the above probability-one event)
there are events $H_{1}$ and $H_{2}$, each of probability at least $1-\conf/3$, such that on the event $H \cap H_{1}$,
$\er_{P_{1}}(\hat{h}_{1}) - \inf_{h \in \C_{1}} \er_{P_{1}}(h) \leq \frac{\eps}{2(\gamma_{1}+\eps/8)} \leq \frac{\eps}{2 \gamma}$,
and on the event $H \cap H_{2}$, 
$\er_{P_{2}}(\hat{h}_{2}) - \inf_{h \in \C_{2}} \er_{P_{2}}(h) \leq \frac{\eps}{2(\gamma_{2}+\eps/8)} \leq \frac{\eps}{2 (1-\gamma)}$.
Therefore, on the event $H \cap H_{1} \cap H_{2}$, 
the right hand side of \eqref{eqn:gaps-lb-split-bound} is at most 
$\gamma \frac{\eps}{2\gamma} + (1-\gamma) \frac{\eps}{2 (1-\gamma)} = \eps$,
so that $\er_{\PXY}(\hat{h}) - \inf_{h \in \C} \er_{\PXY}(h) \leq \eps$.
By a union bound, the probability of $H \cap H_{1} \cap H_{2}$ is at least $1-\conf$.
Since this holds for any $\PXY \in \Dset$, the result follows.
 \end{proof}

We can now apply this result with various choices of the sets $\Dset_{1}(\gamma)$ and $\Dset_{2}(\gamma)$ to obtain 
upper bounds for the above space $\C$, matching the lower bounds proven above for various noise 
models.
Specifically,
consider $\X = \nats$, $\X_{1} = \{1,\ldots,\vc\}$, $\X_{2} = \{\vc+1,\vc+2,\ldots\}$,
$\C_{1} = \left\{ x \mapsto 2\ind_{S}(x)-1 : S \subseteq \{1,\ldots,\vc\} \right\}$,
and $\C_{2} = \left\{ x \mapsto 2\ind_{\{t\}}(x) \!-\! 1 : t \in \{\vc+1,\vc+2,\ldots,\s\} \right\} \cup \{x \mapsto -1\}$.
Note that $\C_{1}$ and $\C_{2}$ satisfy the requirements specified above,
and also that the VC dimension of $\C_{1}$ is $\vc$ and the star number of $\C_{1}$ is $\vc$,
while the VC dimension of $\C_{2}$ is $1$ and the star number of $\C_{2}$ is $\s-\vc$.
Furthermore, take 
$\C = \{ x \mapsto 2\ind_{S}(x) - 1 : S \in 2^{\{1,\ldots,\vc\}} \cup \{ \{i\} : \vc+1 \leq i \leq \s \} \}$,
and note that this satisfies $\C = \C_{1} \cup \C_{2}$, and $\C$ has VC dimension $\vc$ and star number $\s$.

\paragraph{The realizable case:}
For the realizable case, we can in fact show that that lower bound in Theorem~\ref{thm:realizable} is sometimes tight 
up to \emph{universal constant} factors.
Specifically, let $\Dset_{i}$ denote the set of all $P_{i} \in \RE$ with $P_{i}(\X_{i} \times \Y) = 1$, for each $i \in \{1,2\}$.
For every $\gamma \in [0,1]$ and $i \in \{1,2\}$, define $\Dset_{i}(\gamma) = \Dset_{i}$.
In particular, note that for any $P \in \RE$, for any measurable $A \subseteq \X \times \Y$, 
$P(A) = P(\X_{1} \times \Y) P(A| \X_{1} \times \Y) + P(\X_{2} \times \Y) P(A | \X_{2} \times \Y)$.
Furthermore, note that any $i \in \{1,2\}$ with $P(\X_{i} \times \Y) > 0$ has $P(\cdot\times\Y | \X_{i} \times \Y)$
supported only in $\X_{i}$, and has $P(\cdot | \X_{i} \times \Y) \in \RE$, so that $P(\cdot | \X_{i} \times \Y) \in \Dset_{i}$.
Thus, $P \in \Dset = \{ \gamma P_{1} + (1-\gamma) P_{2} : P_{1} \in \Dset_{1}, P_{2} \in \Dset_{2}, \gamma \in [0,1] \}$.
Therefore, $\RE \subseteq \Dset$.
Together with Lemma~\ref{lem:general-gaps-lb-tight}, this implies $\forall \eps,\conf \in (0,1)$,
\begin{align*}
\LC_{\RE}(\eps,\conf) \leq \LC_{\Dset}(\eps,\conf) & \leq 2 \max\left\{ \LC_{1,0}\left(\frac{\eps}{2(1+\eps/8)},\frac{\conf}{3}\right), \LC_{2,0}\left(\frac{\eps}{2(1+\eps/8)},\frac{\conf}{2}\right) \right\}
\\ & \leq 2 \max\left\{ \LC_{1,0}\left(\frac{\eps}{3},\frac{\conf}{3}\right), \LC_{2,0}\left(\frac{\eps}{3},\frac{\conf}{2}\right) \right\},
\end{align*}
for $\LC_{i,0}(\cdot,\cdot)$ defined as above.

Now note that, since every $P_{1} \in \Dset_{1}$ has $P_{1}(\cdot \times \Y)$ supported only in $\X_{1}$, 
and $P_{1} \in \RE$, and since $\C_{1}$ contains classifiers realizing all $2^{\vc}$ distinct classifications of $\X_{1}$,
$\exists h_{P_{1}} \in \C_{1}$ with $\er_{P_{1}}(h_{P_{1}}) = 0$; thus, without loss, we can
take $\target_{P_{1}} = h_{P_{1}}$, so that $P_{1}$ is in the realizable case with respect to $\C_{1}$. 
In particular, since there are only $\vc$ points in $\X_{1}$,
if we consider the active learning algorithm that (given a budget $n \geq \vc$) simply requests $Y_{i}$ for exactly one $i$ s.t. $X_{i} = x$,
for each $x \in \X_{1}$ for which $\exists X_{i} = x$, and then returns any classifier $\hat{h}$ consistent
with these labels, if $\PXY \in \Dset_{1}$, with probability one every $x \in \X_{1}$ with $\PXY(\{x\}\times\Y) > 0$
has some $X_{i} = x$, so that $\er_{\PXY}(\hat{h}) = 0$.  Noting that this algorithm requests at most $\vc$ labels,
we have that $\forall \eps,\conf \in (0,1)$,
\begin{equation*}
\LC_{1,0}\left(\frac{\eps}{3},\frac{\conf}{3}\right) \leq \vc.
\end{equation*}

Similarly, since every $P_{2} \in \Dset_{2}$ has $P_{2}(\cdot \times \Y)$ supported only in $\X_{2}$,
and $P_{2} \in \RE$, $\target_{P_{2}}$ is either equal $-1$ with $P_{2}$-probability one,
or else $\exists x \in \{\vc+1,\ldots,\s\}$ with $\target_{P_{2}}(x) = +1$; in either case, $\exists h_{P_{2}} \in \C_{2}$
with $\er_{P_{2}}(h_{P_{2}}) = 0$; thus, without loss, we can take $\target_{P_{2}} = h_{P_{2}}$, so that $P_{2}$ is in the realizable case with respect to $\C_{2}$.
Now consider an active learning algorithm that first calculates the empirical frequency $\hat{\Px}(\{x\}) = \frac{1}{m} \sum_{i=1}^{m} \ind[ X_{i} = x ]$
for each $x \in \{\vc+1,\ldots,\s\}$
among the first $m = \left\lceil \frac{3^{4}}{2 \eps^{4}} \ln\left( \frac{3(\s-\vc)}{\conf} \right) \right\rceil$
unlabeled data points.  Then, for each $x \in \{\vc+1,\ldots,\s\}$, if $\hat{\Px}(\{x\}) > (1-\eps/3) \eps/3$, 
the algorithm requests $Y_{i}$ for the first $i \in \nats$ with $X_{i} = x$ (supposing the budget $n$ has not yet been reached).
If any requested value $Y_{i}$ equals $+1$, then for the $x \in \{\vc+1,\ldots,\s\}$ with $X_{i} = x$, 
the algorithm returns the classifier $x^{\prime} \mapsto 2 \ind_{\{x\}}(x^{\prime}) - 1$.
Otherwise, the algorithm returns the all-negative classifier: $x^{\prime} \mapsto -1$.
Denote by $\hat{h}$ the classifier returned by the algorithm.
By Hoeffding's inequality and a union bound, with probability at least $1-\conf/3$, every $x \in \{\vc+1,\ldots,\s\}$
has $\hat{\Px}(\{x\}) \geq \PXY(\{x\} \times \Y) - (\eps/3)^{2}$.  Also, if $\PXY \in \RE$, then with probability one, every $Y_{i} = \target_{\PXY}(X_{i})$.
Therefore, if $\PXY \in \Dset_{2}$, on these events, every $x \in \{\vc+1,\ldots,\s\}$ with $\PXY(\{x\} \times \Y) > \eps/3$
will have a label $Y_{i}$ with $X_{i} = x$ requested by the algorithm (supposing sufficiently large $n$), which implies $\hat{h}(x) = \target_{\PXY}(x)$.
Since $\target_{\PXY}$ has at most one $x \in \X_{2}$ with $\target_{\PXY}(x) = +1$, and if such an $x$ exists it must be in $\{\vc+1,\ldots,\s\}$,
if any requested $Y_{i} = +1$, we have $\er_{\PXY}(\hat{h}) = 0$, and otherwise either no $x \in \X_{2}$ has $\target_{\PXY}(x) = +1$ or else
the one such $x$ has $\PXY(\{x\} \times \Y) \leq \eps/3$; in either case, we have 
$\er_{\PXY}(\hat{h}) = \PXY(\{x : \target_{\PXY}(x) = +1\} \times \Y) \leq \eps/3$.
Thus, regardless of whether the algorithm requests a $Y_{i}$ with value $+1$, we have $\er_{\PXY}(\hat{h}) \leq \eps/3$.
By a union bound for the two events, we have that $\P( \er_{\PXY}(\hat{h}) > \eps/3 ) \leq \conf/3$ (given a sufficiently large $n$).
Furthermore, there are at most $\min\left\{ \s-\vc, \frac{1}{(1-\eps/3)\eps/3} \right\}$ points $x \in \{\vc+1,\ldots,\s\}$
with $\hat{\Px}(\{x\}) > (1-\eps/3) \eps/3$, and therefore at most this many labels $Y_{i}$ are requested by the algorithm.
Thus, a budget $n$ of at least this size suffices for this guarantee.
Since this holds for every $\PXY \in \Dset_{2}$, we have that
\begin{equation*}
\LC_{2,0}\left(\frac{\eps}{3},\frac{\conf}{3}\right) \leq \min\left\{ \s-\vc, \frac{1}{(1-\eps/3) \eps/3} \right\}
\lesssim \min\left\{ \s, \frac{1}{\eps} \right\}.
\end{equation*}

Altogether, we have that $\forall \eps,\conf \in (0,1)$,
\begin{equation*}
\LC_{\RE}(\eps,\conf) \lesssim \max\left\{ \min\left\{ \s, \frac{1}{\eps} \right\}, \vc \right\}.
\end{equation*}
Thus, the lower bound in Theorem~\ref{thm:realizable} is tight up to universal constant factors in this case.\footnote{The
term $\Log\left(\min\left\{\frac{1}{\eps},|\C|\right\}\right)$ in the lower bound is dominated by the other terms in this
example, so that this upper bound is still consistent with the existence of this term in the lower bound.}

\paragraph{Bounded noise:}
To prove that the lower bound in Theorem~\ref{thm:bounded} is sometimes tight, 
fix any $\bound \in (0,1/2)$, and 
let $\Dset_{i}$ denote the set of all $P_{i} \in \BN(\bound)$ with $P_{i}(\X_{i} \times \Y) = 1$, for each $i \in \{1,2\}$.
For all $\gamma \in [0,1]$ and $i \in \{1,2\}$, define $\Dset_{i}(\gamma) = \Dset_{i}$.
As above, note that for any $P \in \BN(\bound)$, for any measurable $A \subseteq \X \times \Y$, 
$P(A) = P(\X_{1} \times \Y) P(A| \X_{1} \times \Y) + P(\X_{2} \times \Y) P(A | \X_{2} \times \Y)$.
Furthermore, any $i \in \{1,2\}$ with $P(\X_{i} \times \Y) > 0$ has $P(\cdot\times\Y | \X_{i} \times \Y)$
supported only on $\X_{i}$, and since $\eta(x; P(\cdot | \X_{i} \times \Y)) = \eta(x; P)$ for every $x \in \X_{i}$,
we have $P(\cdot | \X_{i} \times \Y) \in \BN(\bound)$, so that $P(\cdot | \X_{i} \times \Y) \in \Dset_{i}$.
Thus, $P \in \Dset = \{ \gamma P_{1} + (1-\gamma) P_{2} : P_{1} \in \Dset_{1}, P_{2} \in \Dset_{2}, \gamma \in [0,1] \}$.
Therefore, $\BN(\bound) \subseteq \Dset$.
Together with Lemma~\ref{lem:general-gaps-lb-tight}, this implies $\forall \eps,\conf \in (0,1)$,
\begin{equation*}
\LC_{\BN(\bound)}(\eps,\conf) \leq \LC_{\Dset}(\eps,\conf) \leq 2 \max\left\{ \LC_{1,0}\left(\frac{\eps}{3},\frac{\conf}{3}\right), \LC_{2,0}\left(\frac{\eps}{3},\frac{\conf}{3}\right) \right\},
\end{equation*}
for $\LC_{i,0}(\cdot,\cdot)$ defined as above.

Now note that, for each $i \in \{1,2\}$, since every $P_{i} \in \Dset_{i}$ has $P_{i} \in \BN(\bound)$, 
we have $\target_{P_{i}} \in \C$.  Furthermore, since every $h \in \C \setminus \C_{i}$ has $h(x) = -1$ for every $x \in \X_{i}$, 
and the all-negative function $x \mapsto -1$ is contained in $\C_{i}$, and since $P_{i}(\X_{i} \times \Y) = 1$, 
without loss we can take $\target_{P_{i}} \in \C_{i}$ (i.e., there is a version of $\target_{P_{i}}$ contained in $\C_{i}$).
Together with the condition on $\eta(\cdot;P_{i})$ from the definition of $\BN(\bound)$, this implies each $P_{i}$
satisfies the bounded noise condition (with parameter $\bound$) with respect to $\C_{i}$.

Since this is true of every $P_{1} \in \Dset_{1}$, and the star number and VC dimension of $\C_{1}$ are both equal $\vc$,
the upper bound in Theorem~\ref{thm:bounded} implies
$\forall \eps \in (0,(1-2\bound)/8)$, $\conf \in (0,1/8]$, 
\begin{equation*}
\LC_{1,0}\left(\frac{\eps}{3},\frac{\conf}{3}\right) \lesssim \frac{1}{(1-2\bound)^{2}} \vc \cdot \polylog\left(\frac{\vc}{\eps\conf}\right).
\end{equation*}
Similarly, 
since every $P_{2} \in \Dset_{2}$ satisfies the bounded noise condition (with parameter $\bound$) with respect to $\C_{2}$, 
and the star number of $\C_{2}$ is $\s-\vc \leq \s$ while the VC dimension of $\C_{2}$ is $1$,
the upper bound in Theorem~\ref{thm:bounded} implies
$\forall \eps \in (0,(1-2\bound)/8)$, $\conf \in (0,1/8]$, 
\begin{equation*}
\LC_{2,0}\left(\frac{\eps}{3},\frac{\conf}{3}\right) \lesssim \frac{1}{(1-2\bound)^{2}} \min\left\{ \s, \frac{1-2\bound}{\eps} \right\} \polylog\left(\frac{1}{\eps\conf}\right).
\end{equation*}
Altogether, we have that
\begin{equation*}
\LC_{\BN(\bound)}(\eps,\conf) \lesssim \frac{1}{(1-2\bound)^{2}} \max\left\{\min\left\{ \s, \frac{1-2\bound}{\eps} \right\}, \vc \right\} \polylog\left(\frac{\vc}{\eps\conf}\right).
\end{equation*}
For $\bound$ bounded away from $0$, this is within logarithmic factors of the lower bound in Theorem~\ref{thm:bounded},
so that we may conclude that the lower bound is sometimes tight to within logarithmic factors in this case.
Furthermore, when $\bound$ is near $0$, it is within logarithmic factors of the lower bound in Theorem~\ref{thm:realizable}, 
which is also a lower bound on $\LC_{\BN(\bound)}(\eps,\conf)$ since $\RE \subseteq \BN(\bound)$; thus, this 
inherited lower bound on $\LC_{\BN(\bound)}(\eps,\conf)$ is also sometimes tight to within logarithmic factors
when $\bound$ is near $0$.

\paragraph{Tsybakov noise:}
The case of Tsybakov noise is slightly more involved than the above.  In this case, fix any $\tsybca \in [1,\infty)$, $\tsyba \in (0,1)$.
Since the upper bound in Theorem~\ref{thm:tsybakov} already matches the lower bound up to logarithmic factors
when $\tsyba \in (0,1/2]$, it suffices to focus on the case $\tsyba \in (1/2,1)$.
In this case, for $\gamma \in (0,1]$, 
let $\Dset_{i}(\gamma)$ denote the set of all $P_{i} \in \TN(\tsybca/\gamma^{1-\tsyba},\tsyba)$ with $P_{i}(\X_{i} \times \Y) = 1$, for each $i \in \{1,2\}$.
Also let $\Dset_{i}(0)$ denote the set of all probability measures $P_{i}$ with $P_{i}(\X_{i} \times \Y) = 1$, for each $i \in \{1,2\}$.
Again, for any $P \in \TN(\tsybca,\tsyba)$, $P(\cdot) = P(\X_{1} \times \Y) P(\cdot | \X_{1} \times \Y) + P(\X_{2} \times \Y) P(\cdot | \X_{2} \times \Y)$,
and for any $i \in \{1,2\}$ with $P(\X_{i} \times \Y) > 0$, $P(\cdot \times \Y | \X_{i} \times \Y)$ is supported 
only in $\X_{i}$, and $\eta(\cdot ; P(\cdot | \X_{i} \times \Y)) = \eta(\cdot; P)$ on $\X_{i}$,
so that for any $t > 0$, 
\begin{align*}
& P\left( \left\{x : \left|\eta(x; P(\cdot | \X_{i} \times \Y)) - 1/2\right| \leq t \right\} \times \Y \Big| \X_{i} \times \Y \right)
\\ & = \frac{1}{P(\X_{i} \times \Y)} P\left( \left\{ x \in \X_{i} : \left| \eta(x;P) - 1/2 \right| \leq t \right\} \times \Y \right)
\\ & \leq \frac{1}{P(\X_{i} \times \Y)} \tsybca^{\prime} t^{\tsyba/(1-\tsyba)}
= (1-\tsyba) (2\tsyba)^{\tsyba/(1-\tsyba)} \left(\frac{\tsybca}{P(\X_{i} \times \Y)^{1-\tsyba}}\right)^{1/(1-\tsyba)} t^{\tsyba/(1-\tsyba)}.
\end{align*}
Also, since $\target_{P} \in \C$, and $\eta(\cdot ; P(\cdot | \X_{i} \times \Y)) = \eta(\cdot;P)$ on $\X_{i}$, 
we can take $\target_{P(\cdot|\X_{i} \times \Y)}(x) = \target_{P}(x)$ for every $x \in \X_{i}$, 
so that there exists a version of $\target_{P(\cdot | \X_{i} \times \Y)}$ contained in $\C$.
Together, these imply that $P(\cdot | \X_{i} \times \Y) \in \Dset_{i}( P(\X_{i} \times \Y) )$.
We therefore have that $\forall P \in \TN(\tsybca,\tsyba)$, $P = \gamma P_{1} + (1-\gamma) P_{2}$ for some $\gamma \in [0,1]$, $P_{1} \in \Dset_{1}(\gamma)$, and $P_{2} \in \Dset_{2}(1-\gamma)$:
that is, $\TN(\tsybca,\tsyba) \subseteq \Dset$, for $\Dset$ as in Lemma~\ref{lem:general-gaps-lb-tight} (with respect to these definitions of $\Dset_{i}(\cdot)$).
Therefore, Lemma~\ref{lem:general-gaps-lb-tight} implies that $\forall \eps,\conf \in (0,1)$, 
\begin{multline}
\label{eqn:tn-lb-tight}
\LC_{\TN(\tsybca,\tsyba)}(\eps,\conf) \leq \LC_{\Dset}(\eps,\conf)
\\ \lesssim \sup_{\gamma \in [0,1]} \max\left\{ \LC_{1,(\gamma-\eps/8) \lor 0}\left( \frac{\eps}{2(\gamma+\eps/8)}, \frac{\conf}{3} \right), \LC_{2,(1-\gamma-\eps/8) \lor 0}\left( \frac{\eps}{2(1-\gamma+\eps/8)}, \frac{\conf}{3} \right) \right\}.
\end{multline}

First note that, for the case $\gamma \leq \eps/4$, we trivially have 
\begin{equation*}
\LC_{1,(\gamma-\eps/8) \lor 0}\left( \frac{\eps}{2(\gamma+\eps/8)},\frac{\conf}{3}\right) \leq \LC_{1,0}\left( \frac{\eps}{2(\gamma+\eps/4)}, \frac{\conf}{3} \right) \leq \LC_{1,0}\left( 1, \frac{\conf}{3} \right) = 0,
\end{equation*}
and similarly for the case $\gamma \geq 1-\eps/4$, we have $\LC_{2,(1-\gamma-\eps/8) \lor 0}\left(\frac{\eps}{2(1-\gamma+\eps/8)}, \frac{\conf}{3}\right) = 0$.

For the remaining cases, for any $\gamma \in (0,1]$, since every $P_{i} \in \Dset_{i}(\gamma)$
has $\target_{P_{i}} \in \C$, and every $h \in \C \setminus \C_{i}$ has $h(x) = -1$ for every $x \in \X_{i}$, 
and the all-negative function $x \mapsto -1$ is contained in $\C_{i}$, and $P_{i}(\X_{i} \times \Y) = 1$,
without loss we can take $\target_{P_{i}} \in \C_{i}$.  Together with the definition of $\Dset_{i}(\gamma)$,
we have that $\Dset_{i}(\gamma)$ is contained in the set of probability measures $P_{i}$ satisfying the 
Tsybakov noise condition with respect to the hypothesis class $\C_{i}$, with parameters
$\frac{\tsybca}{\gamma^{1-\tsyba}}$ and $\tsyba$.  Therefore, since the star number and VC dimension of $\C_{1}$ are both $\vc$, 
Theorem~\ref{thm:tsybakov} implies that for any $\gamma \in (\eps/4,1]$,\footnote{Recall that, as mentioned in Section~\ref{sec:main}, the 
upper bounds on the label complexities stated in Section~\ref{sec:main} hold without the stated restrictions
on the values $\eps,\conf \in (0,1)$ and $\tsybca$.}
\begin{align*}
& \LC_{1,\gamma-\eps/8}\left( \frac{\eps}{2(\gamma+\eps/8)}, \frac{\conf}{3} \right)
\leq \LC_{1,\gamma/2}\left( \frac{\eps}{3\gamma}, \frac{\conf}{3} \right)
\\ & \lesssim \left( \frac{\tsybca}{ \gamma^{1-\tsyba} } \right)^{2} \left(\frac{\gamma}{\eps}\right)^{2-2\tsyba} \vc \cdot \polylog\left(\frac{\vc}{\eps\conf}\right)
= \tsybca^{2} \left(\frac{1}{\eps}\right)^{2-2\tsyba} \vc \cdot \polylog\left(\frac{\vc}{\eps\conf}\right).
\end{align*}
Similarly, since the star number of $\C_{2}$ is $\s-\vc$ and the VC dimension of $\C_{2}$ is $1$, 
Theorem~\ref{thm:tsybakov} implies that for any $\gamma \in [0,1-\eps/4)$, 
\begin{align*}
& \LC_{2,1-\gamma-\eps/8}\left( \frac{\eps}{2(1-\gamma+\eps/8)}, \frac{\conf}{3} \right) 
\leq \LC_{2,(1-\gamma)/2}\left( \frac{\eps}{3(1-\gamma)}, \frac{\conf}{3} \right) 
\\ & \lesssim \left( \frac{\tsybca}{ (1-\gamma)^{1-\tsyba} } \right)^{2} \left(\frac{1-\gamma}{\eps}\right)^{2-2\tsyba} \min\left\{\s-\vc,\frac{(1-\gamma)^{1/\tsyba} (1-\gamma)}{\tsybca^{1/\tsyba} \eps}\right\}^{2\tsyba-1} \polylog\left(\frac{1}{\eps\conf}\right)
\\ & \leq \tsybca^{2} \left(\frac{1}{\eps}\right)^{2-2\tsyba} \min\left\{\s,\frac{1}{\tsybca^{1/\tsyba} \eps}\right\}^{2\tsyba-1} \polylog\left(\frac{1}{\eps\conf}\right).
\end{align*}
Plugging this into \eqref{eqn:tn-lb-tight}, we have that
\begin{equation*}
\LC_{\TN(\tsybca,\tsyba)}(\eps,\conf) \lesssim 
\tsybca^{2} \left(\frac{1}{\eps}\right)^{2-2\tsyba} \max\left\{\min\left\{\s,\frac{1}{\tsybca^{1/\tsyba} \eps}\right\}^{2\tsyba-1}, \vc \right\} \polylog\left(\frac{\vc}{\eps\conf}\right).
\end{equation*}
As claimed, this is within logarithmic factors of the lower bound in Theorem~\ref{thm:tsybakov} (for $1/2 < \tsyba < 1$, $\tsybca \geq 4$, $\eps \in (0,1/(24\tsybca^{1/\tsyba}))$, and $\conf \in (0,1/24]$),
so that, combined with the tightness (always) for the case $0 < \tsyba \leq 1/2$, we may conclude that the lower bounds in Theorem~\ref{thm:tsybakov} are sometimes tight to within logarithmic factors.

\paragraph{Benign noise:}
The case of benign noise proceeds analogously to the above.
Since $\BE(0) = \RE$, tightness of the lower bound for the case $\nu=0$ (up to constant factors) 
has already been addressed above (supposing we include the lower bound from Theorem~\ref{thm:realizable} as 
a lower bound on $\LC_{\BE(\nu)}(\eps,\conf)$ to strengthen the lower bound in Theorem~\ref{thm:benign}).
For the remainder, we suppose $\nu \in (0,1/2)$.
For $\gamma \in [0,1]$, 
let $\Dset_{i}(\gamma)$ denote the set of all $P_{i} \in \BE(\nu/(\gamma \lor 2\nu))$ with $P_{i}(\X_{i} \times \Y) = 1$, for each $i \in \{1,2\}$.
Again, for any $P \in \BE(\nu)$, $P(\cdot) = P(\X_{1} \times \Y) P(\cdot | \X_{1} \times \Y) + P(\X_{2} \times \Y) P(\cdot | \X_{2} \times \Y)$,
and for any $i \in \{1,2\}$ with $P(\X_{i} \times \Y) > 0$, $P(\cdot \times \Y | \X_{i} \times \Y)$ is supported 
only in $\X_{i}$, and $\eta(\cdot ; P(\cdot | \X_{i} \times \Y)) = \eta(\cdot; P)$ on $\X_{i}$,
so that we can take $\target_{P(\cdot | \X_{i} \times \Y)}(x) = \target_{P}(x)$ for every $x \in \X_{i}$;
thus, there is a version of $\target_{P(\cdot | \X_{i} \times \Y)}$ contained in $\C$.
Furthermore,
\begin{align*}
\er_{P(\cdot | \X_{i} \times \Y)}(\target_{P(\cdot | \X_{i} \times \Y)})
& = \frac{1}{P(\X_{i} \times \Y)} P\left( (x,y) : \target_{P}(x) \neq y \text{ and } x \in \X_{i} \right)
\\ & \leq \frac{1}{P(\X_{i} \times \Y)} P\left( (x,y) : \target_{P}(x) \neq y \right)
\leq \frac{\nu}{P(\X_{i} \times \Y)}.
\end{align*}
Also, since every $x \in \X_{i}$ has $\target_{P(\cdot | \X_{i} \times \Y)}(x) = \target_{P}(x) = \sign( 2 \eta(x;P) - 1 ) = \sign( 2 \eta(x;P(\cdot | \X_{i} \times \Y)) - 1 )$,
we have $P( (x,y) : \target_{P(\cdot | \X_{i} \times \Y)}(x) = y | x \in \X_{i} ) \geq 1/2$, so that $\er_{P(\cdot | \X_{i} \times \Y)}(\target_{P(\cdot | \X_{i} \times \Y)}) \leq 1/2$.
Together, these imply that $P(\cdot | \X_{i} \times \Y) \in \Dset_{i}( P(\X_{i} \times \Y) )$.
We therefore have that $\forall P \in \BE(\nu)$, $P = \gamma P_{1} + (1-\gamma) P_{2}$ for some $\gamma \in [0,1]$, $P_{1} \in \Dset_{1}(\gamma)$, and $P_{2} \in \Dset_{2}(1-\gamma)$:
that is, $\BE(\nu) \subseteq \Dset$, for $\Dset$ as in Lemma~\ref{lem:general-gaps-lb-tight} (with respect to these definitions of $\Dset_{i}(\cdot)$).
Therefore, Lemma~\ref{lem:general-gaps-lb-tight} implies that $\forall \eps,\conf \in (0,1)$, 
\begin{multline}
\label{eqn:be-lb-tight}
\LC_{\BE(\nu)}(\eps,\conf) \leq \LC_{\Dset}(\eps,\conf)
\\ \lesssim \sup_{\gamma \in [0,1]} \max\left\{ \LC_{1,(\gamma-\eps/8) \lor 0}\left( \frac{\eps}{2(\gamma+\eps/8)}, \frac{\conf}{3} \right), \LC_{2,(1-\gamma-\eps/8) \lor 0}\left( \frac{\eps}{2(1-\gamma+\eps/8)}, \frac{\conf}{3} \right) \right\}.
\end{multline}

First note that, as above, for the case $\gamma \leq \eps/4$, we trivially have 
\begin{equation*}
\LC_{1,(\gamma-\eps/8) \lor 0}\left( \frac{\eps}{2(\gamma+\eps/8)},\frac{\conf}{3}\right) \leq \LC_{1,0}\left( \frac{\eps}{2(\gamma+\eps/4)}, \frac{\conf}{3} \right) \leq \LC_{1,0}\left( 1, \frac{\conf}{3} \right) = 0,
\end{equation*}
and similarly for the case $\gamma \geq 1-\eps/4$, we have $\LC_{2,(1-\gamma-\eps/8) \lor 0}\left(\frac{\eps}{2(1-\gamma+\eps/8)}, \frac{\conf}{3}\right) = 0$.

For the remaining cases, for any $\gamma \in (0,1]$, since every $P_{i} \in \Dset_{i}(\gamma)$
has $\target_{P_{i}} \in \C$, and every $h \in \C \setminus \C_{i}$ has $h(x) = -1$ for every $x \in \X_{i}$, 
and the all-negative function $x \mapsto -1$ is contained in $\C_{i}$, and $P_{i}(\X_{i} \times \Y) = 1$,
without loss we can take $\target_{P_{i}} \in \C_{i}$.  Together with the definition of $\Dset_{i}(\gamma)$,
we have that $\Dset_{i}(\gamma)$ is contained in the set of probability measures $P_{i}$ satisfying the 
benign noise condition with respect to the hypothesis class $\C_{i}$, with parameter
$\frac{\nu}{\gamma} \land \frac{1}{2}$.
Therefore, since the star number and VC dimension of $\C_{1}$ are both $\vc$, 
Theorem~\ref{thm:benign} implies that for any $\gamma \in (\eps/4,1]$,\footnote{Again, as mentioned in Section~\ref{sec:main}, 
the restrictions on $\eps,\conf$ stated in Theorem~\ref{thm:benign} are only required for the lower bounds.}
\begin{align*}
\LC_{1,\gamma-\eps/8}\left( \frac{\eps}{2(\gamma+\eps/8)}, \frac{\conf}{3} \right)
\leq \LC_{1,\gamma/2}\left( \frac{\eps}{3\gamma}, \frac{\conf}{3} \right)
& \lesssim \left( \frac{(\nu/\gamma)^{2}}{(\eps/\gamma)^{2}} \vc + \vc \right) \polylog\left(\frac{\vc}{\eps\conf}\right)
\\ & \lesssim \left( \frac{\nu^{2}}{\eps^{2}} \lor 1 \right) \vc \cdot \polylog\left(\frac{\vc}{\eps\conf}\right).
\end{align*}
Similarly, since the star number of $\C_{2}$ is $\s-\vc$ and the VC dimension of $\C_{2}$ is $1$, 
Theorem~\ref{thm:benign} implies that for any $\gamma \in [0,1-\eps/4)$, 
\begin{align*}
& \LC_{2,1-\gamma-\eps/8}\left( \frac{\eps}{2(1-\gamma+\eps/8)}, \frac{\conf}{3} \right) 
\leq \LC_{2,(1-\gamma)/2}\left( \frac{\eps}{3(1-\gamma)}, \frac{\conf}{3} \right) 
\\ & \lesssim \left( \frac{(\nu/(1-\gamma))^{2}}{(\eps/(1-\gamma))^{2}} + \min\left\{ \s-\vc, \frac{1}{\eps} \right\} \right) \polylog\left(\frac{1}{\eps\conf}\right)
\lesssim \left( \frac{\nu^{2}}{\eps^{2}} \lor \min\left\{ \s, \frac{1}{\eps} \right\} \right) \polylog\left(\frac{1}{\eps\conf}\right).
\end{align*}
Plugging these into \eqref{eqn:be-lb-tight}, we have that for $\eps \in (0,\nu)$, 
\begin{equation*}
\LC_{\BE(\nu)}(\eps,\conf)
\lesssim \left( \frac{\nu^{2}}{\eps^{2}} \vc + \min\left\{ \s, \frac{1}{\eps} \right\} \right) \polylog\left(\frac{\vc}{\eps\conf}\right).
\end{equation*}
Again, this is within logarithmic factors of the lower bound in Theorem~\ref{thm:benign} (for $\eps \in (0,(1-2\nu)/24)$ and $\conf \in (0,1/24]$),
so that we may conclude that this lower bound is sometimes tight to within logarithmic factors when $\nu$ is not near $0$ (specifically, when $\eps < \nu$).
For $\nu \leq \eps$, the above implies 
\begin{equation*}
\LC_{\BE(\nu)}(\eps,\conf) \lesssim \max\left\{ \vc, \min\left\{ \s, \frac{1}{\eps} \right\} \right\} \polylog\left(\frac{\vc}{\eps\conf}\right),
\end{equation*}
which is within logarithmic factors of the lower bound in Theorem~\ref{thm:realizable} (for $\eps \in (0,1/9)$ and $\conf \in (0,1/3)$).  
Since $\RE \subseteq \BE(\nu)$, this is also a lower bound on $\LC_{\BE(\nu)}(\eps,\conf)$.  Thus, in this case, we may conclude that this 
inherited lower bound on $\LC_{\BE(\nu)}(\eps,\conf)$ is sometimes tight to within logarithmic factors, for $\nu$ near $0$ (specifically, when $\eps \geq \nu$).

\bibliography{learning}

\end{document}

%% file: macros.tex
\newlength{\dhatheight}
\newcommand{\hhatDS}[1]{%
    \settoheight{\dhatheight}{\ensuremath{\hat{#1}}}%
    \addtolength{\dhatheight}{-0.35ex}%
    \hat{\vphantom{\rule{1pt}{\dhatheight}}%
    \smash{\hat{#1}}}}
\newcommand{\hhatTS}[1]{%
    \settoheight{\dhatheight}{\ensuremath{\hat{#1}}}%
    \addtolength{\dhatheight}{-0.35ex}%
    \hat{\vphantom{\rule{1pt}{\dhatheight}}%
    \smash{\hat{#1}}}}    
\newcommand{\hhatS}[1]{%
    \settoheight{\dhatheight}{\ensuremath{\scriptstyle{\hat{#1}}}}%
    \addtolength{\dhatheight}{-0.175ex}%
    \hat{\vphantom{\rule{1pt}{\dhatheight}}%
    \smash{\hat{#1}}}}
\newcommand{\hhatSS}[1]{%
    \settoheight{\dhatheight}{\ensuremath{\scriptscriptstyle{\hat{#1}}}}%
    \addtolength{\dhatheight}{-0.07ex}%
    \hat{\vphantom{\rule{1pt}{\dhatheight}}%
    \smash{\hat{#1}}}}
\newcommand{\hhat}[1]{\mathchoice{\hhatDS{#1}}{\hhatTS{#1}}{\hhatS{#1}}{\hhatSS{#1}}}

\newcommand{\Px}{\mathcal{P}}
\newcommand{\PXY}{\mathcal{P}_{XY}}

\newcommand{\tsyba}{\alpha}
\newcommand{\tsybca}{a}

\newcommand{\TN}{{\rm TN}}
\newcommand{\BN}{{\rm BN}}
\newcommand{\RE}{{\rm RE}}
\newcommand{\AG}{{\rm AG}}
\newcommand{\BE}{{\rm BE}}
\newcommand{\DL}{{\rm BC}}
\newcommand{\RR}{{\rm RR}}
\newcommand{\s}{\mathfrak{s}}
\newcommand{\XTD}{{\rm XTD}}
\newcommand{\TD}{{\rm TD}}
\newcommand{\XPTD}{{\rm XPTD}}

\newcommand{\dc}{\theta}
\newcommand{\supdc}{\hhat{\dc}}
\newcommand{\infsplit}{\hhat{\rho}}
\newcommand{\basicsplit}{\bar{\rho}}
\newcommand{\finitesplit}{\mathring{\rho}}

\newcommand{\dd}{D}

\newcommand{\Epsilon}{\mathscr{E}}

\newcommand{\bound}{\beta}

\newcommand{\LC}{\Lambda}
\newcommand{\SC}{\mathcal{M}}

\newcommand{\vc}{d}
\newcommand{\covering}{{\cal{N}}}
\newcommand{\G}{{\cal{G}}}

\newcommand{\Log}{{\rm Log}}
\newcommand{\polylog}{{\rm polylog}}

\newcommand{\BorelX}{{\cal B}_{\X}}

\newcommand{\conf}{\delta}
\newcommand{\eps}{\varepsilon}

\newcommand{\X}{\mathcal X}

\newcommand{\Y}{\mathcal Y}
\newcommand{\alg}{\mathcal A}
\renewcommand{\H}{\mathcal H}

\newcommand{\U}{\mathcal U}
\newcommand{\target}{f^{\star}}

\newcommand{\er}{{\rm er}}

\newcommand{\Ball}{{\rm B}}
\newcommand{\ind}{\mathbbm{1}}

\newcommand{\C}{\mathbb{C}}

\newcommand{\sign}{{\rm sign}}
\newcommand{\T}{\mathcal T}

\renewcommand{\P}{\mathbb P}
\newcommand{\nats}{\mathbb{N}}
\newcommand{\reals}{\mathbb{R}}

\newcommand{\E}{\mathbb{E}}

\newcommand{\Data}{\mathcal{Z}}
\newcommand{\DataX}{\mathbb{X}}
\newcommand{\DataY}{\mathbb{Y}}

\newcommand{\DIS}{{\rm DIS}}

\newcommand{\argmax}{\mathop{\rm argmax}}
\newcommand{\argmin}{\mathop{\rm argmin}}
\newcommand{\Dset}{\mathbb{D}}

\newcommand{\supp}{{\rm supp}}
\newcommand{\RQCAL}{Algorithm 1}
\newcommand{\Filter}{Subroutine 1}
\newcommand{\RealizableAlg}{Algorithm 0}
\newcommand{\MembHalvingII}{{\sc Memb-Halving-2}}
\newcommand{\truV}{V^{\star}}
\newcommand{\Split}{{\rm Split}}

\newsavebox{\savepar}
\newenvironment{bigboxit}{\begin{center}\begin{lrbox}{\savepar}
\begin{minipage}[h]{5.8in}
\normalfont
\begin{flushleft}}
{\end{flushleft}\end{minipage}\end{lrbox}\fbox{\usebox{\savepar}}
\end{center}}

\makeatletter
\newcommand{\vast}{\bBigg@{3}}
\newcommand{\Vast}{\bBigg@{4}}
\makeatother

%% file: minimax.bbl
\begin{thebibliography}{88}
\providecommand{\natexlab}[1]{#1}
\providecommand{\url}[1]{\texttt{#1}}
\expandafter\ifx\csname urlstyle\endcsname\relax
  \providecommand{\doi}[1]{doi: #1}\else
  \providecommand{\doi}{doi: \begingroup \urlstyle{rm}\Url}\fi

\bibitem[Adams and Nobel(2010)]{adams:10}
T.~M. Adams and A.~B. Nobel.
\newblock Uniform convergence of {V}apnik-{C}hervonenkis classes under ergodic
  sampling.
\newblock \emph{Annals of Probability}, 38\penalty0 (4):\penalty0 1345--1367,
  2010.

\bibitem[Adams and Nobel(2012)]{adams:12}
T.~M. Adams and A.~B. Nobel.
\newblock Uniform approximation and bracketing properties of {VC} classes.
\newblock \emph{Bernoulli}, 18:\penalty0 1310--1319, 2012.

\bibitem[Ailon et~al.(2012)Ailon, Begleiter, and Ezra]{ailon:12}
N.~Ailon, R.~Begleiter, and E.~Ezra.
\newblock Active learning using smooth relative regret approximations with
  applications.
\newblock In \emph{Proceedings of the $25^{{\rm th}}$ Conference on Learning
  Theory}, 2012.

\bibitem[Anthony and Bartlett(1999)]{anthony:99}
M.~Anthony and P.~L. Bartlett.
\newblock \emph{Neural Network Learning: Theoretical Foundations}.
\newblock Cambridge University Press, 1999.

\bibitem[Awasthi et~al.(2014)Awasthi, Balcan, and Long]{awasthi:14}
P.~Awasthi, M.-F. Balcan, and P.~M. Long.
\newblock The power of localization for efficiently learning linear separators
  with noise.
\newblock In \emph{Proceedings of the $46^{{\rm th}}$ {ACM} Symposium on the
  Theory of Computing}, 2014.

\bibitem[Balcan and Hanneke(2012)]{hanneke:12c}
M.-F. Balcan and S.~Hanneke.
\newblock Robust interactive learning.
\newblock In \emph{Proceedings of the $25^{{\rm th}}$ Conference on Learning
  Theory}, 2012.

\bibitem[Balcan and Long(2013)]{balcan:13}
M.-F. Balcan and P.~M. Long.
\newblock Active and passive learning of linear separators under log-concave
  distributions.
\newblock In \emph{Proceedings of the $26^{{\rm th}}$ Conference on Learning
  Theory}, 2013.

\bibitem[Balcan et~al.(2006)Balcan, Beygelzimer, and Langford]{balcan:06}
M.-F. Balcan, A.~Beygelzimer, and J.~Langford.
\newblock Agnostic active learning.
\newblock In \emph{Proceedings of the 23$^{{\rm rd}}$ International Conference
  on Machine Learning}, 2006.

\bibitem[Balcan et~al.(2007)Balcan, Broder, and Zhang]{balcan:07}
M.-F. Balcan, A.~Broder, and T.~Zhang.
\newblock Margin based active learning.
\newblock In \emph{Proceedings of the $20^{{\rm th}}$ Conference on Learning
  Theory}, 2007.

\bibitem[Balcan et~al.(2009)Balcan, Beygelzimer, and Langford]{balcan:09}
M.-F. Balcan, A.~Beygelzimer, and J.~Langford.
\newblock Agnostic active learning.
\newblock \emph{Journal of Computer and System Sciences}, 75\penalty0
  (1):\penalty0 78--89, 2009.

\bibitem[Balcan et~al.(2010)Balcan, Hanneke, and Vaughan]{hanneke:10a}
M.-F. Balcan, S.~Hanneke, and J.~Wortman Vaughan.
\newblock The true sample complexity of active learning.
\newblock \emph{Machine Learning}, 80\penalty0 (2--3):\penalty0 111--139, 2010.

\bibitem[Bartlett et~al.(2006)Bartlett, Jordan, and Mc{A}uliffe]{bartlett:06}
P.~Bartlett, M.~I. Jordan, and J.~Mc{A}uliffe.
\newblock Convexity, classification, and risk bounds.
\newblock \emph{Journal of the American Statistical Association}, 101:\penalty0
  138--156, 2006.

\bibitem[Bartlett et~al.(2004)Bartlett, Mendelson, and Philips]{bartlett:04}
P.~L. Bartlett, S.~Mendelson, and P.~Philips.
\newblock Local complexities for empirical risk minimization.
\newblock In \emph{Proceedings of the $17^{{\rm th}}$ Conference on Learning
  Theory}, 2004.

\bibitem[Beygelzimer et~al.(2009)Beygelzimer, Dasgupta, and
  Langford]{beygelzimer:09}
A.~Beygelzimer, S.~Dasgupta, and J.~Langford.
\newblock Importance weighted active learning.
\newblock In \emph{Proceedings of the $26^{{\rm th}}$ International Conference
  on Machine Learning}, 2009.

\bibitem[Beygelzimer et~al.(2010)Beygelzimer, Hsu, Langford, and
  Zhang]{beygelzimer:10}
A.~Beygelzimer, D.~Hsu, J.~Langford, and T.~Zhang.
\newblock Agnostic active learning without constraints.
\newblock In \emph{Advances in Neural Information Processing Systems 23}, 2010.

\bibitem[Blumer et~al.(1989)Blumer, Ehrenfeucht, Haussler, and
  Warmuth]{blumer:89}
A.~Blumer, A.~Ehrenfeucht, D.~Haussler, and M.~Warmuth.
\newblock Learnability and the {Vapnik-Chervonenkis} dimension.
\newblock \emph{Journal of the Association for Computing Machinery},
  36\penalty0 (4):\penalty0 929--965, 1989.

\bibitem[Boucheron et~al.(2005)Boucheron, Bousquet, and Lugosi]{boucheron:05}
S.~Boucheron, O.~Bousquet, and G.~Lugosi.
\newblock Theory of classification: A survey of some recent advances.
\newblock \emph{{ESAIM}: Probability and Statistics}, 9:\penalty0 323--375,
  November 2005.

\bibitem[Bousquet et~al.(2004)Bousquet, Boucheron, and Lugosi]{bousquet:04}
O.~Bousquet, S.~Boucheron, and G.~Lugosi.
\newblock Introduction to statistical learning theory.
\newblock \emph{Lecture Notes in Artificial Intelligence}, 3176:\penalty0
  169--207, 2004.

\bibitem[Bshouty et~al.(2009)Bshouty, Li, and Long]{long:09}
N.~H. Bshouty, Y.~Li, and P.~M. Long.
\newblock Using the doubling dimension to analyze the generalization of
  learning algorithms.
\newblock \emph{Journal of Computer and System Sciences}, 75\penalty0
  (6):\penalty0 323--335, 2009.

\bibitem[Castro and Nowak(2008)]{castro:08}
R.~Castro and R.~Nowak.
\newblock Minimax bounds for active learning.
\newblock \emph{{IEEE} Transactions on Information Theory}, 54\penalty0
  (5):\penalty0 2339--2353, July 2008.

\bibitem[Castro and Nowak(2006)]{castro:06}
R.M. Castro and R.D. Nowak.
\newblock Upper and lower error bounds for active learning.
\newblock In \emph{The 44$^{{\rm th}}$ Annual Allerton Conference on
  Communication, Control and Computing}, 2006.

\bibitem[Cohn et~al.(1994)Cohn, Atlas, and Ladner]{cohn:94}
D.~Cohn, L.~Atlas, and R.~Ladner.
\newblock Improving generalization with active learning.
\newblock \emph{Machine Learning}, 15\penalty0 (2):\penalty0 201--221, 1994.

\bibitem[Dasgupta(2004)]{dasgupta:04}
S.~Dasgupta.
\newblock Analysis of a greedy active learning strategy.
\newblock In \emph{Advances in Neural Information Processing Systems 17}, 2004.

\bibitem[Dasgupta(2005)]{dasgupta:05}
S.~Dasgupta.
\newblock Coarse sample complexity bounds for active learning.
\newblock In \emph{Advances in Neural Information Processing Systems 18}, 2005.

\bibitem[Dasgupta et~al.(2005)Dasgupta, Kalai, and Monteleoni]{dasgupta:05b}
S.~Dasgupta, A.~Kalai, and C.~Monteleoni.
\newblock Analysis of perceptron-based active learning.
\newblock In \emph{Proceedings of the $18^{{\rm th}}$ Conference on Learning
  Theory}, 2005.

\bibitem[Dasgupta et~al.(2007)Dasgupta, Hsu, and Monteleoni]{dasgupta:07}
S.~Dasgupta, D.~Hsu, and C.~Monteleoni.
\newblock A general agnostic active learning algorithm.
\newblock In \emph{Advances in Neural Information Processing Systems 20}, 2007.

\bibitem[Ehrenfeucht et~al.(1989)Ehrenfeucht, Haussler, Kearns, and
  Valiant]{ehrenfeucht:89}
A.~Ehrenfeucht, D.~Haussler, M.~Kearns, and L.~Valiant.
\newblock A general lower bound on the number of examples needed for learning.
\newblock \emph{Information and Computation}, 82:\penalty0 247--261, 1989.

\bibitem[El-Yaniv and Wiener(2010)]{el-yaniv:10}
R.~El-Yaniv and Y.~Wiener.
\newblock On the foundations of noise-free selective classification.
\newblock \emph{Journal of Machine Learning Research}, 11:\penalty0 1605--1641,
  2010.

\bibitem[El-Yaniv and Wiener(2012)]{el-yaniv:12}
R.~El-Yaniv and Y.~Wiener.
\newblock Active learning via perfect selective classification.
\newblock \emph{Journal of Machine Learning Research}, 13:\penalty0 255--279,
  2012.

\bibitem[Fan(2012)]{fan:12}
G.~Fan.
\newblock A graph-theoretic view of teaching.
\newblock Master's thesis, Department of Computer Science, University of
  Regina, 2012.

\bibitem[Floyd and Warmuth(1995)]{floyd:95}
S.~Floyd and M.~Warmuth.
\newblock Sample compression, learnability, and the {V}apnik-{C}hervonenkis
  dimension.
\newblock \emph{Machine Learning}, 21:\penalty0 269--304, 1995.

\bibitem[Freund et~al.(1997)Freund, Seung, Shamir, and Tishby]{freund:97}
Y.~Freund, H.~S. Seung, E.~Shamir, and N.~Tishby.
\newblock Selective sampling using the query by committee algorithm.
\newblock \emph{Machine Learning}, 28:\penalty0 133--168, 1997.

\bibitem[Friedman(2009)]{friedman:09}
E.~Friedman.
\newblock Active learning for smooth problems.
\newblock In \emph{Proceedings of the $22^{{\rm nd}}$ Conference on Learning
  Theory}, 2009.

\bibitem[Gin\'{e} and Koltchinskii(2006)]{gine:06}
E.~Gin\'{e} and V.~Koltchinskii.
\newblock Concentration inequalities and asymptotic results for ratio type
  empirical processes.
\newblock \emph{The Annals of Probability}, 34\penalty0 (3):\penalty0
  1143--1216, 2006.

\bibitem[Goldman and Kearns(1995)]{goldman:95}
S.~A. Goldman and M.~J. Kearns.
\newblock On the complexity of teaching.
\newblock \emph{Journal of Computer and System Sciences}, 50:\penalty0 20--31,
  1995.

\bibitem[Gupta et~al.(2003)Gupta, Krauthgamer, and Lee]{gupta:03}
A.~Gupta, R.~Krauthgamer, and J.~R. Lee.
\newblock Bounded geometries, fractals, and low-distortion embeddings.
\newblock In \emph{Proceedings of the $44^{{\rm th}}$ Annual IEEE Symposium on
  Foundations of Computer Science}, 2003.

\bibitem[Hanneke(2006)]{hanneke:06}
S.~Hanneke.
\newblock The cost complexity of interactive learning, 2006.

\bibitem[Hanneke(2007{\natexlab{a}})]{hanneke:07a}
S.~Hanneke.
\newblock Teaching dimension and the complexity of active learning.
\newblock In \emph{Proceedings of the $20^{{\rm th}}$ Conference on Learning
  Theory}, 2007{\natexlab{a}}.

\bibitem[Hanneke(2007{\natexlab{b}})]{hanneke:07b}
S.~Hanneke.
\newblock A bound on the label complexity of agnostic active learning.
\newblock In \emph{Proceedings of the $24^{{\rm th}}$ International Conference
  on Machine Learning}, 2007{\natexlab{b}}.

\bibitem[Hanneke(2009{\natexlab{a}})]{hanneke:09a}
S.~Hanneke.
\newblock Adaptive rates of convergence in active learning.
\newblock In \emph{Proceedings of the $22^{{\rm nd}}$ Conference on Learning
  Theory}, 2009{\natexlab{a}}.

\bibitem[Hanneke(2009{\natexlab{b}})]{hanneke:thesis}
S.~Hanneke.
\newblock \emph{Theoretical Foundations of Active Learning}.
\newblock PhD thesis, Machine Learning Department, School of Computer Science,
  Carnegie Mellon University, 2009{\natexlab{b}}.

\bibitem[Hanneke(2011)]{hanneke:11a}
S.~Hanneke.
\newblock Rates of convergence in active learning.
\newblock \emph{The Annals of Statistics}, 39\penalty0 (1):\penalty0 333--361,
  2011.

\bibitem[Hanneke(2012)]{hanneke:12a}
S.~Hanneke.
\newblock Activized learning: Transforming passive to active with improved
  label complexity.
\newblock \emph{Journal of Machine Learning Research}, 13\penalty0
  (5):\penalty0 1469--1587, 2012.

\bibitem[Hanneke(2014)]{hanneke:survey}
S.~Hanneke.
\newblock Theory of {A}ctive {L}earning. {V}ersion 1.1.
\newblock \url{http://www.stevehanneke.com}, 2014.

\bibitem[Hanneke and Yang(2012)]{hanneke:12b}
S.~Hanneke and L.~Yang.
\newblock Surrogate losses in passive and active learning.
\newblock \emph{{arXiv}:1207.3772}, 2012.

\bibitem[Haussler and Welzl(1987)]{haussler:87}
D.~Haussler and E.~Welzl.
\newblock $\varepsilon$-nets and simplex range queries.
\newblock \emph{Discrete Computational Geometry}, 2:\penalty0 127--151, 1987.

\bibitem[Haussler et~al.(1994)Haussler, Littlestone, and Warmuth]{haussler:94}
D.~Haussler, N.~Littlestone, and M.~Warmuth.
\newblock Predicting $\{0,1\}$-functions on randomly drawn points.
\newblock \emph{Information and Computation}, 115:\penalty0 248--292, 1994.

\bibitem[Heged\"{u}s(1995)]{hegedus:95}
T.~Heged\"{u}s.
\newblock Generalized teaching dimension and the query complexity of learning.
\newblock In \emph{Proceedings of the 8$^{{\rm th}}$ Conference on
  Computational Learning Theory}, 1995.

\bibitem[Hellerstein et~al.(1996)Hellerstein, Pillaipakkamnatt, Raghavan, and
  Wilkins]{hellerstein:96}
L.~Hellerstein, K.~Pillaipakkamnatt, V.~Raghavan, and D.~Wilkins.
\newblock How many queries are needed to learn?
\newblock \emph{Journal of the Association for Computing Machinery},
  43\penalty0 (5):\penalty0 840--862, 1996.

\bibitem[Hsu(2010)]{hsu:thesis}
D.~Hsu.
\newblock \emph{Algorithms for Active Learning}.
\newblock PhD thesis, Department of Computer Science and Engineering, School of
  Engineering, University of California, San Diego, 2010.

\bibitem[K\"{a}\"{a}ri\"{a}inen(2006)]{kaariainen:06}
M.~K\"{a}\"{a}ri\"{a}inen.
\newblock Active learning in the non-realizable case.
\newblock In \emph{Proceedings of the 17$^{{\rm th}}$ International Conference
  on Algorithmic Learning Theory}, 2006.

\bibitem[Kalai et~al.(2005)Kalai, Klivans, Mansour, and Servedio]{kalai:05}
A.~T. Kalai, A.~R. Klivans, Y.~Mansour, and R.~A. Servedio.
\newblock Agnostically learning halfspaces.
\newblock In \emph{Proceedings of the $46^{{\rm th}}$ Annual IEEE Symposium on
  Foundations of Computer Science}, 2005.

\bibitem[Kallenberg(2002)]{kallenberg:02}
O.~Kallenberg.
\newblock \emph{Foundations of Modern Probability, 2$^{{\rm nd}}$ Edition}.
\newblock Springer Verlag, New York, 2002.

\bibitem[Kearns et~al.(1994)Kearns, Schapire, and Sellie]{kearns:94a}
M.~J. Kearns, R.~E. Schapire, and L.~M. Sellie.
\newblock Toward efficient agnostic learning.
\newblock \emph{Machine Learning}, 17:\penalty0 115--141, 1994.

\bibitem[Kolmogorov and Tikhomirov(1959)]{kolmogorov:59}
A.~N. Kolmogorov and V.~M. Tikhomirov.
\newblock $\varepsilon$-entropy and $\varepsilon$-capacity of sets in
  functional spaces.
\newblock \emph{Uspekhi Matematicheskikh Nauk}, 14\penalty0 (2):\penalty0
  3--86, 1959.

\bibitem[Kolmogorov and Tikhomirov(1961)]{kolmogorov:61}
A.~N. Kolmogorov and V.~M. Tikhomirov.
\newblock $\varepsilon$-entropy and $\varepsilon$-capacity of sets in
  functional spaces.
\newblock \emph{American Mathematical Society Translations, Series 2},
  17:\penalty0 277--364, 1961.

\bibitem[Koltchinskii(2006)]{koltchinskii:06}
V.~Koltchinskii.
\newblock Local {R}ademacher complexities and oracle inequalities in risk
  minimization.
\newblock \emph{The Annals of Statistics}, 34\penalty0 (6):\penalty0
  2593--2656, 2006.

\bibitem[Koltchinskii(2010)]{koltchinskii:10}
V.~Koltchinskii.
\newblock {R}ademacher complexities and bounding the excess risk in active
  learning.
\newblock \emph{Journal of Machine Learning Research}, 11:\penalty0 2457--2485,
  2010.

\bibitem[Kulkarni(1989)]{kulkarni:89}
S.~R. Kulkarni.
\newblock On metric entropy, {Vapnik-Chervonenkis} dimension, and learnability
  for a class of distributions.
\newblock Technical Report CICS-P-160, Center for Intelligent Control Systems,
  1989.

\bibitem[Kulkarni et~al.(1993)Kulkarni, Mitter, and Tsitsiklis]{kulkarni:93}
S.~R. Kulkarni, S.~K. Mitter, and J.~N. Tsitsiklis.
\newblock Active learning using arbitrary binary valued queries.
\newblock \emph{Machine Learning}, 11:\penalty0 23--35, 1993.

\bibitem[Le{C}am(1973)]{lecam:73}
L.~Le{C}am.
\newblock Convergence of estimates under dimensionality restrictions.
\newblock \emph{The Annals of Statistics}, 1\penalty0 (1):\penalty0 38--53,
  1973.

\bibitem[Li and Long(2007)]{long:07}
Y.~Li and P.~M. Long.
\newblock Learnability and the doubling dimension.
\newblock In \emph{Advances in Neural Information Processing 20}, 2007.

\bibitem[Littlestone and Warmuth(1986)]{littlestone:86}
N.~Littlestone and M.~Warmuth.
\newblock Relating data compression and learnability, 1986.

\bibitem[Mammen and Tsybakov(1999)]{mammen:99}
E.~Mammen and A.B. Tsybakov.
\newblock Smooth discrimination analysis.
\newblock \emph{The Annals of Statistics}, 27:\penalty0 1808--1829, 1999.

\bibitem[Massart and N\'{e}d\'{e}lec(2006)]{massart:06}
P.~Massart and E.~N\'{e}d\'{e}lec.
\newblock Risk bounds for statistical learning.
\newblock \emph{The Annals of Statistics}, 34\penalty0 (5):\penalty0
  2326--2366, 2006.

\bibitem[Mc{D}iarmid(1998)]{mcdiarmid:98}
C.~Mc{D}iarmid.
\newblock Concentration.
\newblock In \emph{Probabilistic Methods for Algorithmic Discrete Mathematics},
  pages 195--248. Springer-Verlag, 1998.

\bibitem[Minsker(2012)]{minsker:12}
S.~Minsker.
\newblock Plug-in approach to active learning.
\newblock \emph{Journal of Machine Learning Research}, 13\penalty0
  (1):\penalty0 67--90, 2012.

\bibitem[Mitchell(1977)]{mitchell:77}
T.~Mitchell.
\newblock Version spaces: a candidate elimination approach to rule learning.
\newblock In \emph{{IJCAI'77:} Proceedings of the 5$^{{\rm th}}$ international
  joint conference on Artificial intelligence}, pages 305--310, 1977.

\bibitem[Raginsky and Rakhlin(2011)]{raginsky:11}
M.~Raginsky and A.~Rakhlin.
\newblock Lower bounds for passive and active learning.
\newblock In \emph{Advances in Neural Information Processing Systems 24}, 2011.

\bibitem[Sauer(1972)]{sauer:72}
N.~Sauer.
\newblock On the density of families of sets.
\newblock \emph{Journal of Combinatorial Theory (A)}, 13:\penalty0 145--147,
  1972.

\bibitem[Settles(2012)]{settles:12}
B.~Settles.
\newblock \emph{Active Learning}.
\newblock Synthesis Lectures on Artificial Intelligence and Machine Learning,
  Morgan \& Claypool Publishers, 2012.

\bibitem[Tsybakov(2004)]{tsybakov:04}
A.~B. Tsybakov.
\newblock Optimal aggregation of classifiers in statistical learning.
\newblock \emph{The Annals of Statistics}, 32\penalty0 (1):\penalty0 135--166,
  2004.

\bibitem[Valiant(1984)]{valiant:84}
L.~G. Valiant.
\newblock A theory of the learnable.
\newblock \emph{Communications of the ACM}, 27\penalty0 (11):\penalty0
  1134--1142, November 1984.

\bibitem[van~der Vaart and Wellner(1996)]{van-der-Vaart:96}
A.~W. van~der Vaart and J.~A. Wellner.
\newblock \emph{Weak Convergence and Empirical Processes}.
\newblock Springer, 1996.

\bibitem[van~der Vaart and Wellner(2011)]{van-der-Vaart:11}
A.~W. van~der Vaart and J.~A. Wellner.
\newblock A local maximal inequality under uniform entropy.
\newblock \emph{Electronic Journal of Statistics}, 5:\penalty0 192--203, 2011.

\bibitem[van Handel(2013)]{van-handel:13}
R.~van Handel.
\newblock The universal {G}livenko-{C}antelli property.
\newblock \emph{Probability and Related Fields}, 155:\penalty0 911--934, 2013.

\bibitem[Vapnik(1982)]{vapnik:82}
V.~Vapnik.
\newblock \emph{Estimation of Dependencies Based on Empirical Data}.
\newblock Springer-Verlag, New York, 1982.

\bibitem[Vapnik(1998)]{vapnik:98}
V.~Vapnik.
\newblock \emph{Statistical Learning Theory}.
\newblock John Wiley $\&$ Sons, Inc., New York, 1998.

\bibitem[Vapnik and Chervonenkis(1971)]{vapnik:71}
V.~Vapnik and A.~Chervonenkis.
\newblock On the uniform convergence of relative frequencies of events to their
  probabilities.
\newblock \emph{Theory of Probability and its Applications}, 16:\penalty0
  264--280, 1971.

\bibitem[Vapnik and Chervonenkis(1974)]{vapnik:74}
V.~Vapnik and A.~Chervonenkis.
\newblock \emph{Theory of Pattern Recognition}.
\newblock Nauka, Moscow, 1974.

\bibitem[Vidyasagar(2003)]{vidyasagar:03}
M.~Vidyasagar.
\newblock \emph{Learning and Generalization with Applications to Neural
  Networks, 2$^{{\rm nd}}$ Edition}.
\newblock Springer-Verlag, 2003.

\bibitem[von Neumann(1928)]{von-neumann:28}
J.~von Neumann.
\newblock Zur theorie der gesellschaftsspiele.
\newblock \emph{Mathematische Annalen}, 100\penalty0 (1):\penalty0 295--320,
  1928.

\bibitem[von Neumann and Morgenstern(1944)]{von-neumann:44}
J.~von Neumann and O.~Morgenstern.
\newblock \emph{Theory of Games and Economic Behavior}.
\newblock Princeton University Press, 1944.

\bibitem[Wald(1945)]{wald:45}
A.~Wald.
\newblock Sequential tests of statistical hypotheses.
\newblock \emph{The Annals of Mathematical Statistics}, 16\penalty0
  (2):\penalty0 117--186, 1945.

\bibitem[Wald(1947)]{wald:47}
A.~Wald.
\newblock \emph{Sequential Analysis}.
\newblock John Wiley $\&$ Sons, Inc., New York, 1947.

\bibitem[Wang(2011)]{wang:11}
L.~Wang.
\newblock Smoothness, disagreement coefficient, and the label complexity of
  agnostic active learning.
\newblock \emph{Journal of Machine Learning Research}, 12:\penalty0 2269--2292,
  2011.

\bibitem[Wiener et~al.(2014)Wiener, Hanneke, and {El-Yaniv}]{hanneke:14a}
Y.~Wiener, S.~Hanneke, and R.~{El-Yaniv}.
\newblock A compression technique for analyzing disagreement-based active
  learning.
\newblock \emph{arXiv:1404.1504}, 2014.

\bibitem[Yang and Barron(1999)]{yang:99b}
Y.~Yang and A.~Barron.
\newblock Information-theoretic determination of minimax rates of convergence.
\newblock \emph{The Annals of Statistics}, 27\penalty0 (5):\penalty0
  1564--1599, 1999.

\end{thebibliography}
